\documentclass[12pt,fleqn]{ucithesis}

\usepackage{xcolor}
\usepackage{amssymb}
\usepackage{amsmath}
\usepackage{amsthm}
\usepackage{array}
\usepackage{graphicx}
\usepackage{natbib}
\usepackage{relsize}
\usepackage[titletoc]{appendix}
\usepackage{setspace}
\usepackage{enumitem}
\usepackage{booktabs}       %

\usepackage{bbm}
\usepackage{float}
\usepackage[makeroom]{cancel}
\usepackage{pifont}%
\usepackage{algorithm}
\usepackage{algpseudocode}
\usepackage{textcomp}

\usepackage{caption}
\usepackage{subcaption}  %
\usepackage{multirow}

\usepackage{tabularx}

\usepackage[plainpages=false,pdfborder={0 0 0}]{hyperref}
\usepackage[capitalise]{cleveref}

\newtheorem{theorem}{\textsc{Theorem}}[chapter]
\newtheorem{lemma}{\textsc{Lemma}}[chapter]

\DeclareMathAlphabet\mathbfcal{OMS}{cmsy}{b}{n}
\DeclareMathOperator*{\argmax}{arg\,max}

\newcommand{\cmark}{\ding{51}}%
\newcommand{\xmark}{\ding{55}}%

\newcommand{\R}{\mathbb{R}}

\newcommand{\N}{\mathbb{N}}

\newcommand{\hist}{\mathcal{H}}
\newcommand{\F}{\mathcal{F}}
\newcommand{\filter}{\F}
\newcommand{\E}{\mathbb{E}}
\newcommand{\Prob}{\mathbb{P}}
\newcommand{\q}{\mathbb{Q}}
\newcommand{\prob}{\Prob}
\newcommand{\var}{\text{Var}}
\newcommand{\cov}{\text{Cov}}
\newcommand{\ind}{\mathbbm{1}}
\newcommand{\sep}{\!\;|\;\!}
\newcommand{\hit}{\text{hit}}

\newcommand{\cadlag}{càdlàg }
\newcommand{\eff}{\text{eff}}
\newcommand{\bigO}{\mathcal{O}}
\newcommand{\query}[1]{$\text{Q}{#1}$}

\newcommand{\V}{\mathcal{X}}
\newcommand{\vocab}{\mathcal{M}}
\newcommand{\obs}{\mathcal{O}}
\newcommand{\cen}{\mathcal{C}}
\newcommand{\mc}[1]{\multicolumn{1}{c}{#1}}

\thesistitle{On the Efficient Marginalization of Probabilistic Sequence Models}

\documenttitle{Dissertation}

\degreename{Doctor of Philosophy}

\degreefield{Statistics}

\authorname{Alex Boyd}

\committeechair{Padhraic Smyth, Distinguished Professor}
\othercommitteemembers
{
 Stephan Mandt, Associate Professor \\
 Babak Shahbaba, Professor
}

\degreeyear{2024}

\copyrightdeclaration
{
  {\copyright} {\Degreeyear} \Authorname
}

\dedications
{
  To my fianc\'ee and parents. 
}

\acknowledgments
{
I would first like to take the chance to sincerely thank my advisors, Padhraic Smyth and Stephan Mandt. Both have been incredible mentors to me throughout my time here at UCI and I am extremely grateful to have had the chance to learn and grow under them the past several years. Their patience, availability, passion, and outlooks on research all facilitated a wonderful environment for me to develop my own confidence and skill set as a researcher. While having two advisors brings its own set of challenges (twice as many meetings, more frequent tangent conversations in meetings, etc.), if I were to start my Ph.D. again I would not change a thing. Being exposed to differing perspectives from my advisors throughout all of my research endeavors encouraged me to actively think and come into my own. In the future, I will look back fondly on all the times I walked into their offices to brainstorm and provide progress updates. After all is said and done, I will consider myself fortunate should I ever get the chance to collaborate with either of my advisors again after graduating. And to Padhraic in particular, thank you for taking a chance and recruiting a fresh undergraduate student without any research experience. Your endorsement of my potential at that time meant more to me than I can express in words.

I would also like to thank Babak Shahbaba for serving on my defense committee. I had the fortune of taking Babak's Bayesian statistics class early on while at UCI, which was foundational in the development of my statistical and probabilistic thinking. His continued support throughout my Ph.D. and feedback for my dissertation have been invaluable.  

Graduate school in general can be greatly enhanced by the company that is kept, and I was tremendously lucky to be surrounded by incredibly bright, passionate, and personable labmates and friends. To Catarina Bel\'em, Yuxin Chang, Chis Galbraith, Disi Ji, Markelle Kelly, Gavin Kerrigan, Aodong Li, Robert Logan, Rachel Longjohn, Yang Meng, Giosue Migliorini, Kushagra Pandey, Jihyun Park, Tuan Pham, Sam Showalter, Prakhar Srivastava, Justus Will, Eliot Wong-Toi, Shang Wu, Ruihan Yang, and Yibo Yang, thank you all for providing countless stimulating conversations and making working out of the office an enjoyable experience every day. In particular to Sam and Yuxin, I thank you both for your energy and excitement during our frequent collaborations. Working on our projects together was never difficult due to how natural it felt to bounce ideas off of one another, trouble shoot error messages together, and buckle down when deadlines approached. I am thankful for the opportunities we had to research together as I sincerely believe that our experiences led me to growing stronger as a collaborator, and I hope that I was able to have a fraction of an positive impact on both of you as the two of you had on me. 

Those who knew me prior to graduate school would know that attending it in general, let alone acquiring a Ph.D., was not a goal of mine until late into my undergraduate education; at least not until Dennis Sun, my undergraduate advisor, encouraged me to consider this pursuit. Were it not for him, I would not have even applied to any Ph.D. programs in the first place. I sincerely thank Dennis for our long meetings (that often bled into his office hours) where he convinced me that I was capable and could indeed succeed at this. I fully attribute the start of my long journey towards understanding the foundations of ``why machine learning works'' to Dennis igniting my curiosity. Alongside Dennis were other professors and mentors that encouraged me pursue this path and supported my endeavors. To Alex Dekhtyar, Monte LaBute, and Teng Qu, thank you all for believing in me and helping me recognize my own potential. You all never shied away from the myriad of questions I had throughout my learning process; your patience is much appreciated. I am also grateful to Kyle Vigil and especially Nick Russo for their continued interest in my success and reliable assistance during both the highs and lows of graduate school.

I was lucky to have had internships at NVIDIA, Microsoft Research, and Apple during time as a Ph.D. student, and I would like to thank my hosts and mentors Bryan Catanzaro, Mohammad Shoeybi, Mostofa Patwary, Chris Meek, Alex Polozov, Yang Song, and Shiwen Shen for taking a chance on me and showing me a completely different side to machine learning research that can only be found in industry. While the internships were only ever for three months at a time, I consistently felt a significant acceleration in my growth as a researcher after each one.

Finally, on a more personal note I would like to thank my friends and family. To Sierra, Ian, Alex, Minnal, and Andrew, thank you for your patience whenever I would disappear for months at a time to dive deep into research. I could always count on you being there when I would come back up for air. To my parents, thank you for always being there for me and simply encouraging me to try my best regardless of the situation. Consistently striving for my best inevitably got me to where I am now, and I could not be more thankful for your support and your love. Last but not least, to Shannon, thank you for your understanding and unwavering encouragement. Were it not for you, I would have not been able to finish out strong. This milestone is as much my achievement as it is yours. I am, and always will be, thankful to have you in my life. 

\vspace*{\fill}
The material presented in this dissertation was funded in part by the National Science Foundation (NSF) through a Graduate Research Fellowship Program (GRFP) grant DGE-1839285, the National Institute of Health (NIH) under award 1OT2OD032581-01, and the Patient-Centered Outcomes Research Institute\textregistered{} (PCORI\textregistered) award ME-1602-34167. The views expressed are those of the author and do not reflect the official policy or position of the funding agencies.

}

\curriculumvitae
{

\textbf{EDUCATION}
  
  \begin{tabular*}{1\textwidth}{@{\extracolsep{\fill}}lr}
    \textbf{Doctor of Philosophy in Statistics} & \textbf{2024} \\
    \vspace{6pt}
    University of California, Irvine & \emph{Irvine, CA} \\
    \textbf{Bachelor of Science in Software Engineering} & \textbf{2018} \\
    \vspace{6pt}
    California Polytechnic State University & \emph{San Luis Obispo, CA}
  \end{tabular*}

\vspace{12pt}
\textbf{RESEARCH EXPERIENCE}

  \begin{tabular*}{1\textwidth}{@{\extracolsep{\fill}}lr}
    \textbf{Graduate Research Assistant} & \textbf{2018--2024} \\
    \vspace{6pt}
    University of California, Irvine & \emph{Irvine, CA}  \\
    \textbf{Applied Machine Learning Researcher, Intern} & \textbf{2021} \\
    \vspace{6pt}
    Apple & \emph{Cupertino, CA}  \\
    \textbf{Deep Learning Researcher, Intern} & \textbf{2020} \\
    \vspace{6pt}
    Microsoft Research & \emph{Redmond, WA}  \\
    \textbf{Applied Deep Learning Researcher, Intern} & \textbf{2019} \\
    \vspace{6pt}
    NVIDIA & \emph{Santa Clara, CA}  
  \end{tabular*}

\vspace{12pt}
\textbf{TEACHING EXPERIENCE}

  \begin{tabular*}{1\textwidth}{@{\extracolsep{\fill}}lr}
    \textbf{Teaching Assistant} & \textbf{2018, 2023} \\
    \vspace{6pt}
    University of California, Irvine & \emph{Irvine, CA}
  \end{tabular*}

\pagebreak

\textbf{REFEREED CONFERENCE PUBLICATIONS}

{ \leftskip=15pt \parindent=-\leftskip

Sam Showalter$^*$, \underline{Alex Boyd}$^*$, Padhraic Smyth, and Mark Steyvers. Bayesian online learning for consensus prediction. To appear in \textit{Proceedings of The 27th International Conference on Artificial Intelligence and Statistics}, 2024

Yuxin Chang, \underline{Alex Boyd}, and Padhraic Smyth. Probabilistic modeling for sequences of sets in continuous-time. To appear in \textit{Proceedings of The 27th International Conference on Artificial Intelligence and Statistics}, 2024

\underline{Alex Boyd}, Yuxin Chang, Stephan Mandt, and Padhraic Smyth. Inference for mark-censored temporal point processes. In \textit{Proceedings of the Thirty-Ninth Conference on Uncertainty in Artificial Intelligence}, volume 216 of \textit{Proceedings of Machine Learning Research}, pages 226–236. PMLR, 2023

\underline{Alex Boyd}, Yuxin Chang, Stephan Mandt, and Padhraic Smyth. Probabilistic querying of continuous-time event sequences. In \textit{International Conference on Artificial Intelligence and Statistics}, pages 10235–10251. PMLR, 2023

\underline{Alex Boyd}$^*$, Sam Showalter$^*$, Stephan Mandt, and Padhraic Smyth. Predictive querying for autoregressive neural sequence models. In \textit{Advances in Neural Information Processing Systems}, volume 35. Curran Associates, Inc., 2022

Antonios Alexos$^*$, \underline{Alex Boyd}$^*$, and Stephan Mandt. Structured stochastic gradient MCMC. In \textit{Proceedings of the 39th International Conference on Machine Learning}, volume 162, pages 414–434. PMLR, 2022

Aodong Li, \underline{Alex Boyd}, Padhraic Smyth, and Stephan Mandt. Detecting and adapting to irregular distribution shifts in Bayesian online learning. In \textit{Advances in Neural Information Processing Systems}, volume 34, pages 6816–6828. Curran Associates, Inc., 2021

Preston Putzel, Hyungrok Do, \underline{Alex Boyd}, Hua Zhong, and Padhraic Smyth. Dynamic survival analysis for EHR data with personalized parametric distributions. In \textit{Proceedings of the 6th Machine Learning for Healthcare Conference}, volume 149, pages 648–673. PMLR, 2021

\underline{Alex Boyd}$^*$, Raul Puri$^*$, Mohammad Shoeybi, Mostofa Patwary, and Bryan Catanzaro. Large scale multi-actor generative dialog modeling. In \textit{Proceedings of the 58th Annual Meeting of the Association for Computational Linguistics}, pages 66–84. Association for Computational Linguistics, 2020

\underline{Alex Boyd}, Robert Bamler, Stephan Mandt, and Padhraic Smyth. User-dependent neural sequence models for continuous-time event data. In \textit{Advances in Neural Information Processing Systems}, volume 33, pages 21488–21499. Curran Associates, Inc., 2020

{ \small \hspace{1em} $*$ \textit{denotes shared first-authorship}}
}

}

\thesisabstract
{
Real-world data often exhibits sequential dependence, across diverse domains such as human behavior, medicine, finance, and climate modeling. Probabilistic methods capture the inherent uncertainty associated with prediction in these contexts, with autoregressive models being especially prominent. This dissertation focuses on using autoregressive models to answer complex probabilistic queries that go beyond single-step prediction, such as the timing of future events or the likelihood of a specific event occurring before another. In particular, we develop a broad class of novel and efficient approximation techniques for marginalization in sequential models that are model-agnostic. These techniques rely solely on access to and sampling from next-step conditional distributions of a pre-trained autoregressive model, including both traditional parametric models as well as more recent neural autoregressive models. Specific approaches are presented for discrete sequential models, for marked temporal point processes, and for stochastic jump processes, each tailored to a well-defined class of informative, long-range probabilistic queries.

}

\hypersetup{
	pdftitle={\Thesistitle},
	pdfauthor={\Authorname},
	pdfsubject={\Degreefield},
}

\allowdisplaybreaks

\begin{document}

\preliminarypages

\chapter{Introduction}

\noindent

In real-world settings, data rarely arises instantaneously and independently. 
Instead, systems and individuals typically evolve over time, generating data that is ordered sequentially. 
While this order may be disregarded in specific contexts, understanding its significance often leads to richer insights. 
Ignoring the temporal evolution of a population's characteristics, for example, might simplify analysis but risks overlooking valuable information.

More significantly, sequential data allows for future prediction, a capability far more impactful than static analysis. 
This applies to diverse domains, ranging from individual health outcomes \citep{clark2003survival,van2011dynamic,haider2020effective}, shopping behavior \citep{rendle2010factorizing,guidotti2018personalized,shou2023concurrent}, and technology interactions \citep{quadrana2018sequence,app_dataset} to large-scale phenomena like climate patterns \citep{besse2000autoregressive,ise2019forecasting} and financial market dynamics \citep{kijima2002stochastic,rolski2009stochastic,sezer2020financial}. 
In all of these cases, modeling outcomes from a sequential perspective can unlock significant predictive power.

Probabilistic approaches address this by capturing the joint distribution of the entire sequence, e.g., $\prob(X_1, X_2, \dots, X_N)$ for the random sequence $X_1, X_2, \dots, X_N$. 
This framework encompasses diverse modeling techniques, regardless of whether time is treated continuously \citep{capasso2021introduction} or discretely \citep{rajarshi2014statistical}, or whether measurements are sparse \citep{hamilton2020time} or dense \citep{daley2003introduction}.\footnote{By sparse, we are referencing scenarios with individual events that occur at varying points in time. An example is the history of earthquakes, both when and where they occur, in a given geographical area. In contrast to this, dense information is often in reference to a system with values of interest defined at any given time and are regularly measured. Temperature measurements over time are a classic example of this.}
By far, the most common way of modeling the joint distribution of sequences, especially when focusing on future prediction, is to use an \textit{autoregressive} approach. This involves decomposing the joint distribution and directly model the distributions of the next values in a sequence given the prior values.\footnote{Autoregressive models originally were defined in the time-series literature as a method of \textit{regressing} future values of a sequence based on past values (hence the ``auto'' in the name); however, in recent years this term has generalized past time-series into a general framework as described in this work \citep{hamilton2020time,murphy2022probabilistic,pmlr-v32-gregor14}.}
For example, an autoregressive model factorizes the joint distribution $\prob(X_1, \dots, X_N)$ as $\prod_{i=1}^N \prob(X_i \sep X_1, \dots, X_{i-1})$, where each conditional probability $\prob(X_i \sep X_1, \dots, X_{i-1})$ for all $i \in \mathbb{N}$ becomes the target for modeling. 

Autoregressive models, by capturing the joint distribution of random sequences, hold insights into various future possibilities beyond just the next single value.
Instead of merely predicting the immediate successor for a sequence, we may be interested in the timing of a specific value's next appearance or the likelihood of one value occurring before another.
This work focuses on answering such \textit{probabilistic queries} using autoregressive models. 
Quantifying these probabilities often involves marginalizing over intermediate elements in a sequence, the specific ones depending on the query itself.
For instance, understanding the distribution of $X_{10} \sep X_1$ requires marginalizing out all potential realizations and combinations of $X_2, X_3, \dots, X_9$ to follow $X_1$ prior to $X_{10}$.

Efficient computation of probabilistic queries and general marginalization have been active research areas in machine learning and AI, dating back to exact inference methods in multivariate graphical models \citep{pearl1988probabilistic,koller2009probabilistic}.
Traditionally, two types of queries dominated the research: \textit{conditional probability queries} and \textit{assignment queries}. This dissertation focuses on the former.

Conditional probability queries aim to estimate the distribution of a specific subset of variables ($X$) conditioned on observed values for another subset ($Y=y$). 
Marginalization over the remaining variables ($Z$) is required to achieve this. 
Graphical models facilitate these computations by leveraging assumed conditional independence relationships between $X, Y,$ and $Z$ (e.g., see \citet{koster2002marginalizing}).

For models with sparse Markov dependence structures, efficient inference algorithms exploit this structure, especially for sequential models where recursive computation is advantageous \citep{koller2009probabilistic,bilmes2010dynamic}.
However, neural network-based autoregressive models inherently violate the Markov assumption due to their hidden states encoding the entire sequence of previous events.
Consequently, assignment and conditional probability queries become computationally intractable (NP-hard), rendering techniques like dynamic programming (which are effective in Markov models) inapplicable in this more general setting \citep{chen2018recurrent}.

Despite analytical solutions existing for specific model settings and parameterizations, the recent success of neural network-based autoregressive models across diverse domains necessitates efficient and tractable methods for extracting probabilistic information from them.
This dissertation proposes novel techniques that achieve this goal with minimal assumptions about the underlying prediction model, ensuring broad applicability without requiring additional training or learning.

\section{Motivating Use Cases}
Before detailing the contributions of this dissertation, we will first examine contexts where autoregressive models have proven successful and highlight their potential benefit from extracting longer-range distributional information about future trajectories.

\paragraph{Contextual Queries}
Understanding customer behavior through population-level statistics, such as consumption patterns, is crucial for business decisions.  
Large user behavior databases allow companies to estimate general purchase tendencies, like the likelihood of users buying item A before item B or the frequency of users purchasing item C in a specific time frame.
However, predicting such tendencies for individual customers based on their unique past actions becomes challenging. 
With limited data points for each user, direct statistical computations become unreliable.

This motivates the use of autoregressive models trained on the entire customer dataset. 
These models learn to predict a user's next action based on their historical sequence of events. 
Once trained, they can be used to answer contextual queries specific to each user, conditioned on their unique history. 
This is achieved by marginalizing over all possible behavior sequences leading to the query instance (e.g., all possible sequences where a user purchases item A before B).

Chapters 3, 4, and 6 explore various techniques for tractably achieving this in various modeling settings.

\paragraph{Forecasting}
Long range predictions, or \textit{forecasting}, is a crucial task in various domains, including web server load management \citep{santana2011power,hu2019applying}, anticipating economic trends in business and finance \citep{franses1998time,elliott2008economic}, predicting long-range climate trends \citep{barnett1995monte,shukla1998predictability,troccoli2010seasonal}, and more \citep{abraham2009statistical,petropoulos2022forecasting}.
In web server load management for example, accurate forecasts are essential to anticipate usage spikes and proactively scale resources, preventing downtime and ensuring optimal system performance.
Traditionally, autoregressive time series models are employed for such purposes.
In load management, these models utilize past load data, potentially combined with aggregate activity information, to predict the immediate next load value.
While some approaches offer limited multi-step ahead forecasts, they are often inflexible and do not directly address the critical concern of server overload.

An alternative approach leverages the existing autoregressive model with a predefined server usage threshold.
By marginalizing over potential future trajectories exceeding this threshold, we can extract the distribution of when such overload is predicted to occur.
This distribution then can inform proactive resource allocation decisions, determining if, when, and by how much resources should be scaled to accommodate anticipated usage.
Techniques for addressing this use case are developed in Chapters 3, 4, and especially 6 for a variety of different settings.

\paragraph{Missing Data}
Marginalization of sequences need not encompass only future trajectories; it can also be applied to internal segments within existing sequences.
This capability proves particularly valuable when dealing with unobserved or censored portions of data.

Consider the previously discussed load forecasting scenario. 
Real-world systems are prone to intermittent outages and equipment failures, which may affect data collection services that record pertinent information for server load forecasting. 
However, halting production during such events is often impractical. 
In such situations, marginalizing over the missing past information becomes crucial to accurately account for its potential impact on future predictions.

Another relevant application arises in medical settings.
Imagine training a model to capture the dynamics between events of interest and various medical measurements. 
This model might then be deployed in hospitals lacking the full set of equipment used for data collection during training, resulting in missing information at inference time. 
Rather than retraining the model entirely, marginalization offers a more practical and cost-effective solution to address the missing information.
By marginalizing, the model can leverage the available data for accurate inference, maintaining its predictive and analytical capabilities.
Methods to account for missing information are the sole focus of Chapter 5.

\section{Contributions and Dissertation Outline}
This dissertation presents novel and efficient approximation techniques for marginalization in sequential models, applicable to various classes of models without requiring particular parameterizations or model-specific modifications.
These techniques rely solely on access to and sampling from the next-step conditional distributions of a pre-trained autoregressive model.

Chapter 2 provides essential background information and establishes the notation used throughout the dissertation.

Chapter 3 focuses on probabilistic query estimation for discrete-time sequential models with a finite vocabulary. Key contributions include:
\begin{itemize}
\item Formalization of all possible queries in this setting and an accompanying notation system for their representation.
\item Analytical solutions for query estimation in $n^\text{th}$-order Markov models.
\item A sequential importance sampling approximation technique using a novel query-specific proposal distribution.
\item Computable lower-bounds for query probabilities using beam search combined with query-derived heuristics.
\item A hybrid method combining importance sampling and beam search for improved resource efficiency.
\item Systematic experiments on four real-world datasets, investigating the efficiency of the proposed algorithm compared to simpler baselines.
\end{itemize}

Chapter 4 extends the methods to sparse continuous-time sequential models where events occur at random times with varying time gaps. The contributions include:
\begin{itemize}
\item Extension of the class of discrete-time queries to the continuous-time setting, encompassing queries like hitting times of specific events.
\item Adaptation of the importance sampling approach with a novel proposal marked temporal point process.
\item Theoretical guarantee of improved sample efficiency compared to naive counting-based methods.
\item Empirical demonstration of significant efficiency gains, often exceeding three orders of magnitude in variance reduction for certain queries, across three different real-world datasets.
\end{itemize}

Chapter 5 addresses marginalization of missing or censored information in the same modeling setting as Chapter 4. The contributions include:
\begin{itemize}
\item A novel principled approach for marginalizing over missing information in arbitrary temporal point processes by characterizing the intensity of observed events.
\item Tractable estimation procedure for this observed intensity through using importance sampling in conjunction with the proposal process from Chapter 4 and justified via the point process superposition property.
\item Analysis of bias and variance of derived estimators, exploration of resource-saving variants, and empirical validation in real-world settings.
\end{itemize}

Chapter 6 further generalizes the results to dense continuous-time models and expands the considered class of queries. The contributions include:
\begin{itemize}
\item Characterization of a query class for stochastic jump processes, including a new kind of random times.
\item Extension of importance sampling estimation techniques to this more general class of processes.
\item Development of explicit estimators for joint distributions of the proposed random times.
\item Empirical demonstrations of improved computation efficiency of proposed techniques for a variety of different queries on stochastic jump processes. 
\end{itemize}

Chapter 7 concludes the dissertation with potential future research directions and final thoughts.

\chapter{Background}

\noindent

This chapter lays the foundation for the subsequent chapters by providing essential background information in three key areas: probability theory, methods for approximating expected values, and probabilistic autoregressive modeling of sequential data (both classical and modern approaches). 
Additionally, we introduce the notation employed throughout the dissertation. 
We assume the reader possesses an undergraduate-level understanding of linear algebra and machine learning, as well as familiarity with graduate-level statistics, particularly measure-theoretic probability. 
For readers seeking a more comprehensive treatment of these topics, we recommend \citet{macdonald2010linear} for linear algebra, \citet{murphy2022probabilistic} and \citet{murphy2023probabilistic} for machine learning, \citet{casella2021statistical} statistical theory, and  \citet{klenke2013probability} for probability theory.

\section{Notation}\label{sec:notation}

This dissertation employs a consistent notation system to enhance clarity and readability. 
For random variables, capital letters denote the variable itself (e.g., $X$), while corresponding lowercase letters represent specific values (e.g., $x$). 
Matrices, vectors, and scalars lack a predefined formatting style; their dimensions and properties will be explicitly declared upon introduction (e.g., vector $x\in\R^d$ or matrix-valued random variable $X\in\R^{n\times m}$). 
Greek letters serve as constants or functions, with their specific meaning clarified upon first use. 
Notably, $\theta$ and $\Theta$ consistently represent parameter values and the parameter space of a model, respectively. 
Sets are denoted by uppercase script letters (e.g., $\mathcal{S}$).

Central to this work are sequences of random variables and their realized values. 
We differentiate them from individual random variables by using bold uppercase letters (e.g., $\mathbf{X}$).
These sequences represent ordered collections of random variables, defined as $\mathbf{X} := (X_i)_{i\in\mathcal{I}} \equiv \{X_i \sep i \in \mathcal{I}\}$ where $\mathcal{I} \subseteq \R$ is the \emph{indexing set}.
Subsets of the indexing set define specific portions of the sequence (e.g., $\mathbf{X}_{<t} \equiv \{X_i \sep i \in \mathcal{I} \text{ and } i < t\}$, $\mathbf{X}_{[a,b)} \equiv (X_i)_{i \in \mathcal{I} \cap [a, b)}$ for continuous $\mathcal{I}$, and $\mathbf{X}_{n:m} \equiv \{X_i \sep i =n, n+1, \dots, m$\} for discrete $\mathcal{I}$).
Similarly, realized sequences are denoted by bold lowercase letters (e.g., $\mathbf{x}$).

It is important to note that these random sequences are also known as \textit{stochastic processes}.
This term and \textit{sequential model} may be used interchangeably depending on the context.

\section{Relevant Probability Theory}

\subsubsection{Probability and Random Variables}  %

Within this work, we operate within the framework of a probability space $(\Omega, \mathcal{F}, \prob)$. This space consists of three key elements:
\begin{itemize}
\item the \textit{sample space} $\Sigma$, a set encompassing all possible outcomes of a random experiment or system,
\item the \textit{event space} $\mathcal{F}$, a $\sigma$-algebra comprising all collections of outcomes or events considered, and
\item the \textit{probability measure} $\prob:\mathcal{F}\rightarrow[0,1]$, a mapping of events in $\mathcal{F}$ to their corresponding probabilities ranging from 0 to 1.
\end{itemize}

While a probability space fully captures the random behavior of a system, the sample space itself can be too abstract for practical inference.
To bridge this gap, we utilize \textit{random variables}.
A random variable, denoted by $X:\Omega \rightarrow \mathcal{X}$, is a measurable mapping from abstract sample outcomes, $\omega \in \Omega$, to more concrete and interpretable values within the measurable space $\mathcal{X}$.
For instance, consider the outcome of rolling a six-sided die. Here, $\mathcal{X}=\{1, 2, 3, 4, 5, 6\}$ represents the possible values a roll can produce, while $X(\omega)$ assigns the specific number rolled for a particular sample outcome $\omega \in \Omega$.

This work employs probability statements of the following form:
\begin{align}
    \prob(X \in A) := \prob(X^{-1}(A)) := \prob(\{\omega \in \Omega \sep X(\omega) \in A\}) \label{eq:prob_gen}
\end{align}
where $X^{-1}(A) \in \mathcal{F}$.
For a random variable $X$, we denote the \textit{cumulative distribution function} (CDF) as $F_X(x) := \prob(X \leq x)$ and $p_X$ as either the \textit{probability density function} (PDF), $p_X(x) := \prob(X \in [x, x+dx))$, if $X$ is continuous or the \textit{probability mass function} (PMF), $p_X(x) := \prob(X=x)$, if $X$ is discrete.\footnote{We refer to a random variable as \textit{continuous} when its CDF $F_X(x)$ is continuous with respect to $x$. Likewise, it is discrete when the CDF is a step function (with either a finite or countably infinite amount of steps). It is possible for a random variable to be neither, e.g., a mixture distribution between a continuous distribution and a point mass.}
While the CDF exists for any random variable, the PDF/PMF may not.

To facilitate calculation, we often relate general probability statements to known measures through \textit{expectation}.
Recall the \textit{expected value} of $h(X)$ is:
\begin{align}
    \E^\prob\left[h(X)\right] & := \int_{\Omega} h(X(\omega))d\prob(\omega) \equiv \int_\Omega h(X)d\prob \label{eq:ev_gen} \\
    & = \int_\mathcal{X} h(x)dF_X(x) \\
    & = \begin{cases} \int_{\mathcal{X}} h(x)p_X(x)dx & \text{ if } X \text{ is continuous} \\
    \sum_{x\in\mathcal{X}} h(x)p_X(x) & \text{ if } X \text{ is discrete}
    \end{cases} \label{eq:ev_cases}
\end{align}
where $h$ is some measurable function of $X$. 
Equations \ref{eq:ev_cases} simplify the general integral in \cref{eq:ev_gen} using the Lebesgue measure for continuous $X$ or the counting measure for discrete $X$ as the reference measure.
We often need to find the expected value of an expression involving multiple random variables while ``marginalizing out'' one of them. This is defined as:
\begin{align}
    \E^\prob_{X}\left[h(X, Y)\right] & := \int_{\mathcal{X}} h(x, Y)dF_X(x)
\end{align}
where $Y$ is another random variable under the same probability space as $X$ and 
$h$ is a measurable function of both $X$ and $Y$.\footnote{More often, we marginalize out a variable while using its conditional distribution with respect to the other variables, i.e., $X \sep Y$, which will be more precisely defined in \cref{eq:cond_prob}.} 
This operation effectively transforms $Y$ into a new random variable based on its relationship with $X$.

The expectation definition of a probability statement provides a valuable interpretation:
\begin{align}
    \prob(X \in A) := \E^\prob_X\left[\ind(X \in A)\right] \label{eq:ind_prob}
\end{align}
where $\ind$ denotes the indicator function, which in this case returns $1$ if $X \in A$ and 0 otherwise. 
Intuitively, this relationship reveals one perspective of the underlying meaning of $\prob(X \in A)$. 
It expresses the probability as the expected value of the indicator function, representing the proportion of times $X$ falls within set $A$ over an infinitely large number of random draws.\footnote{It is worth noting that this interpretation of probability is specifically a \textit{frequentist} one. This is the perspective adopted for this dissertation as it complies with all the methods derived; however, there do exist alternative interpretations as well, e.g., Bayesian probability theory. For a more philosophical discussion on various interpretations, please refer to \citet{sep-probability-interpret}.}

\subsubsection{Conditioning and Organization of Information}

This work frequently encounters the need to condition on existing information before assessing the probability of future events. 
In the context of sequential modeling, this often involves conditioning on a portion of a sequence to improve predictions for future occurrences.
Formally, a \textit{conditional expectation} is itself a random variable, where the randomness stems from the information used for conditioning.

$\sigma$-algebras are commonly interpreted as representations of potential information, based on how finely they distinguish individual events within them.
Consequently, the dynamics of this random variable are governed by a sub-$\sigma$-algebra, denoted $\mathcal{F}'\subseteq\mathcal{F}$, which captures the relevant events via:
\begin{align}
\int_A \E^\prob\left[h(X) \sep \mathcal{F}'\right] d\prob = \int_A h(X) d\prob \text{ for all } A \in \mathcal{F}'.
\end{align}
While the equation might seem complex, it essentially highlights how $\mathcal{F}'$ determines the level of aggregation for $h(X)$.
To better understand this concept, consider two extreme edge-cases:
\begin{itemize}
\item Conditioning on no information, where $\mathcal{F}':=\{\emptyset, \Omega\}$, enforces that $\E^\prob\left[h(X) \sep \mathcal{F}'\right] = \E^\prob\left[h(X)\right]$, or rather the mean value of $h(X)$.\footnote{Technically, $\E^\prob\left[h(X) \sep \{\emptyset, \Omega\}\right]$ \textit{is} a random variable, but it is a degenerate one with 100\% of its mass on a single value: $\E^\prob\left[h(X)\right]$.} 
\item Conditioning on all relevant information, where $\mathcal{F}':=\sigma(X)$ which encompasses all events generated by $X$, yields the function itself with $\E^\prob\left[h(X)\sep \mathcal{F}'\right]=h(X)$. This makes sense, as knowing the exact value of $X$ reveals the value of $h(X)$.
\end{itemize}
These examples illustrate how the choice of the sub-$\sigma$-algebra determines the level of detail we consider when calculating the conditional expectation. It's akin to focusing on specific aspects of the available information to make more targeted predictions.

While mathematically rigorous, conditioning on arbitrary $\sigma$-algebras can be conceptually challenging. 
Our primary interest lies in conditioning on random variables, as it offers a more intuitive understanding. 
However, it's essential to remember that, ultimately, conditioning always occurs on a specific $\sigma$-algebra.
Consider the conditional probability $\prob(X \in A \sep Y)$, often written for convenience.
This concise notation actually translates to the more precise statement:
\begin{align}
\prob(X \in A \sep Y) & \equiv \prob(X \in A \sep \sigma(Y)) \\
& = \E^\prob_X\left[\ind(X \in A) \sep \sigma(Y)\right] \text{ by \cref{eq:ind_prob}} \label{eq:cond_prob}.
\end{align}
This highlights that the conditional probability is equivalent to the conditional expectation of the indicator function for event $X\in A$, given the $\sigma$-algebra generated by $Y$.

A common application of conditional expectations is through the \textit{law of iterated expectations}, also known as the \textit{tower rule} in stochastic process literature \citep{casella2021statistical}. 
This law helps us compute marginal expectations using readily available conditional representations.
For any two sub-$\sigma$-algebras $\mathcal{F}'' \subseteq \mathcal{F}' \subseteq \mathcal{F}$, the law is defined as:
\begin{align}
\E^\prob\left[h(X) \sep \mathcal{F}'' \right] \equiv \E^\prob\left[\E^\prob\left[h(X) \sep \mathcal{F}'\right] \sep \mathcal{F}''\right] \text{ almost surely}.
\end{align}
The most common application involves $\mathcal{F}'':=\{\emptyset, \Sigma\}$, resulting in the expansion $\E^\prob\left[h(X)\right] = \E^\prob\left[\E^\prob\left[h(X) \sep \mathcal{F}'\right]\right]$.

While the provided definition of conditional expectations is valuable, it's not necessarily conducive to direct computation.
For practical applications, we typically require knowledge of the conditional distribution of $X$ given $Y$, represented by either the conditional CDF $F_{X \sep Y}(x \sep y)$ or the conditional PDF/PMF $p_{X \sep Y}(x \sep y)$.
These allow for the following tractable formulas:
\begin{align}
\prob(X \in A \sep Y) & = \int_A dF_{X\sep Y}(x \sep Y) \\
& = \begin{cases}
    \int_A p_{X\sep Y}(x \sep Y) dx & \text{ if } X \text{ is continuous} \\
    \sum_{x \in A} p_{X \sep Y}(x \sep Y) & \text{ if } X \text{ is discrete} 
    \end{cases}
\end{align}
Due to this relation to conditional distributions, we will often denote these as expectations via $\E^\prob_{X \sep Y}\left[\ind(X \in A)\right]$.

When working with sequences of random variables, we often deal with varying information levels depending on the subsequence length.
For example, $(X_1, X_2, X_3)$ clearly offers more information than $(X_1, X_2)$.
This progression is formally captured by \textit{filtrations}, which are a collection of nested sub-$\sigma$-algebras: $(\filter_t)_{t\in\mathcal{I}}$ where $\filter_t \subseteq \filter_{t'} \subseteq \filter$ for $t < t'$.
Intuitively, $\filter_t$ represents the current information available at time $t$, and in our context is commonly equivalent to $\sigma(\textbf{X}_{\leq t})$.
If $A \in \filter_t$, then the event $A$ can be determined to have occurred (or not) at or before time $t$ within the system.

While filtrations are essential for conditioning on past information, their notation can be cumbersome.
Depending on the context, we might opt for clearer notation by directly writing the conditional with the relevant subsequence, e.g., $\prob( \cdot \sep X_1, X_2, X_3)$ instead of $\prob(\cdot \sep \filter_3):= \prob(\cdot \sep \sigma(X_1, X_2, X_3))$.
This will be further discussed in \cref{sec:seq_models} for each specific type of sequential model considered.

\subsubsection{Limiting Theorems}

This work heavily relies on several fundamental theorems from statistics and probability theory. 
To enhance clarity, we restate them below and provide interpretations of their meanings and applications.

\begin{theorem} \citep{casella2021statistical}
Let $X, X_1, X_2, \dots \overset{iid}{\sim} F_X$ where $\mu=\E^\prob\left[X\right]$ exists and is finite, and denote the sample mean as $\overline{X}_n := \frac{1}{n}\sum_{i=1}^n X_i$. By the \textbf{strong law of large numbers} it follows that $\prob\left(\lim_{n\rightarrow\infty} \overline{X}_n = \mu\right) = 1$, or rather that $\overline{X}_n \rightarrow \mu$ almost surely when $n\rightarrow \infty$.
\end{theorem}

The strong law of large numbers establishes the consistency of the sample mean as an estimator for the population mean $\mu$.
This concept becomes instrumental when discussing techniques for approximating expectations.

When developing estimators and bounds, it's crucial to determine when we can interchange the order of operations involving limits and integrals. 
The following two theorems justify swapping an expectation with a limit and swapping two nested expectations, respectively.  

\begin{theorem} \citep{klenke2013probability}
    Let $X_1, X_2, \dots$ be random variables that converge in probability to the random variable $X$ as $n\rightarrow\infty$. The \textbf{dominated convergence theorem} states that should there exist some random variable $Y$ where $\E^\prob\left[|Y|\right] < \infty$ and for all $n \in \mathbb{N}$ it holds that $|X_n| \leq Y$, then $\lim_{n\rightarrow\infty}\E^\prob\left[X_n\right] = \E^\prob\left[\lim_{n\rightarrow\infty} X_n\right] = \E^\prob\left[X\right]$.
\end{theorem}

\begin{theorem} \citep{klenke2013probability}
Let $(\Omega_1, \mathcal{F}_1, \prob_1)$ and $(\Omega_2, \mathcal{F}_2, \prob_2)$ be probability spaces with $(\Omega, \mathcal{F}, \prob):=(\Omega_1\times\Omega_2, \mathcal{F}_1\times\mathcal{F}_2, \prob_1 \times \prob_2)$ be the product probability space between the two. Let $X:\Omega_1\times\Omega_2 \rightarrow [0,\infty)$ be a non-negative random variable defined on the product probability space. If $\int_{\Omega_1} X(\omega_1, \omega_2)d\prob_1(\omega)$ is measurable with respect to $\omega_2$ and $\int_{\Omega_2} X(\omega_1, \omega_2) d\prob_2(\omega)$, then by \textbf{Fubini's theorem} it follows that $\int_{\Omega_1}\int_{\Omega_2} X d\prob_2(\omega_2) d\prob_1(\omega_1) = \int_{\Omega_2}\int_{\Omega_1} X d\prob_1(\omega_1) d\prob_2(\omega_2) = \int_{\Omega_1\times\Omega_2} X d\prob(\omega_1, \omega_2)$.
\end{theorem}

\section{Expectation Approximation Techniques}\label{sec:imp_sampling_notes}

This work focuses on estimating CDF values for various distributions. 
To achieve this, we leverage \textit{Monte Carlo methods}, a class of sampling-based techniques.

A fundamental principle of the Monte Carlo method involves estimating the expected value
$\mu:=\E^\prob_X\left[h(X)\right]$ 
using a computable sample mean.
This requires drawing $n$ i.i.d. samples $X_1, \dots, X_n \sim F_X$ from the target distribution and calculating the same mean of the function $h(X)$ at each sample: $\hat{\mu} := \frac{1}{n}\sum_{i=1}^n h(X_i)$.
Due to the linearity of the expectation operator, the estimator possesses the attractive property of being unbiased:
\begin{align}
    \E^\prob\left[\hat{\mu}\right] & = \frac{1}{n} \sum_{i=1}^n \E^\prob_{X_i}\left[h(X_i)\right] \\
    & = \frac{1}{n} \sum_{i=1}^n \E^\prob_X\left[h(X)\right] \\
    & = \mu.
\end{align}
Furthermore, by virtue of the strong law of large numbers, this estimator is also consistent, meaning it converges to the true value $\mu$ as the sample size $n$ tends to infinity.

The variance of the estimator depends on the variance of the transformed random variable $h(X)$ and the number of samples used:
\begin{align}
    \var^\prob\left[\hat{\mu}\right] & = \frac{1}{n^2}\sum_{i=1}^n \var^\prob_{X_i}\left[h(X_i)\right] \\
    & = \frac{1}{n^2}\sum_{i=1}^n \var^\prob_{X}\left[h(X)\right] \\
    & = \frac{1}{n} \var^\prob_X\left[h(X)\right].
\end{align}
The direct variance of $h(X)$, $\sigma^2 := \var^\prob_X\left[h(X)\right]$, can be approximated as a Monte Carlo estimate itself, reusing the same samples for $\hat{\mu}$: $\hat{\sigma}^2 := \frac{1}{n-1}\sum_{i=1}^n (X_i - \hat{\mu})^2$. 
The normalization by $n-1$ instead of $n$ is crucial to mitigate bias introduced by using the same samples for both estimating $\hat{\mu}$ and its variance. 
However, for large sample sizes, this distinction has negligible practical impact.

When selecting between two unbiased Monte Carlo estimators for the same target value, a critical factor is the computational cost required to achieve low estimator variance. 
Typically, this translates to comparing their variances when using the same number of samples. However, computational complexity per sample also plays a role.
Therefore, a comprehensive evaluation involves assessing both \textit{sample efficiency} and \textit{runtime performance}.

Sample efficiency is quantified by the ratio of variances between two unbiased estimators: $\eff\left[\hat{\mu}_1, \hat{\mu}_2\right] := \var^\prob\left[\hat{\mu}_2\right]/\var^\prob\left[\hat{\mu}_1\right]$ where the number of samples $n$ is the same for both estimators \cite{casella2021statistical}.
Values greater than 1 indicate $\hat{\mu}_1$ is more efficient, meaning fewer samples are needed to achieve the same variance as $\hat{\mu}_2$.
For instance, $\text{eff}=100$ signifies that using one sample for $\hat{\mu}_1$ is as informative as using 100 samples for $\hat{\mu}_2$.
In such cases, significant efficiency gains may outweigh any discrepancies in runtime.

This text focuses on improving the variance per sample, as more samples always improve estimates. 
One technique to achieve this is \textit{importance sampling} \citep{robert1999monte}. 
Its aim is to change the probability measure governing the expectation $\E\left[h(X)\right]$ from $\prob$ to proposal probability measure $\q$ to reduce estimator variance.

Importance sampling first identifies a random variable $L:\Omega \rightarrow [0,\infty]$ with specific properties: $\E^\prob\left[L\right]=1$ and that $L < \infty$ almost surely with respect to $\prob(\cdot \sep h(X) \neq 0)$. 
Informally, the second condition ensures $L$ is finite where $\prob$ either assigns non-zero probability or results in $h(X)=0$. 
Should $h(X)=0$, then $L$ can be any value in the range $[0,\infty]$. 
When these conditions hold, $L$ can be interpreted as the \textit{likelihood ratio} between $\prob$ and $\q$: $L(\omega) = \frac{d\prob}{d\q}\big|_\omega$.\footnote{We assign the value of $\infty$ to $L$ over regions of $\Omega$ with no $\q$-support.} 
This allows changing the measure in the expectation:
\begin{align}
\E^\prob\left[h(X)\right] \equiv \E^\q\left[L \cdot h(X)\right].
\end{align}

Rather than determining $\q$ directly, a change of measure can be induced by altering the distribution of a specific variable, e.g., swapping PDF/PMF functions $p_X$ with \textit{proposal distribution} $q_X$. 
This results in $L$ being a transformation of $X$ yielding:
\begin{align}
\E^\prob_X\left[h(X)\right] & \equiv \E^\q_X\left[L(X)h(X)\right] \\
\text{for } L(x) &:= \frac{\mathcal{L}^\prob(X=x)}{\mathcal{L}^\q(X=x)} \\
& = \frac{p_X(x)}{q_X(x)} \text{ when } p_X \text{ and } q_X \text{ exist}
\end{align}
where $\mathcal{L}^\prob$ and $\mathcal{L}^\q$ are likelihood functions under $\prob$ and $\q$ respectively.\footnote{Likelihood functions in this context can be treated as generalized probability density/mass functions that account for some random variables belonging to mixture distributions. When appropriate, they reduce to the variable's corresponding PDF or PMF.} Under importance sampling, we refer to the integrand $L(X)h(X)$ as the \textit{estimator} taken with respect to $\q$.

Monte Carlo estimates based on importance sampling remain unbiased, but their variance differs due to the changed distribution. 
The natural question arises: what is the optimal choice of $\q$?
It turns out, choosing a $\q$ such that $\mathcal{L}^\q(X)$ is proportional to $|h(X)|\mathcal{L}^\prob(X)$ minimizes the estimator variance.
However, this distribution requires knowledge of the target value $\E^\prob\left[h(X)\right]$, which we aim to estimate in the first place.
Nonetheless, knowing the form of this optimal distribution guides us in designing other effective proposal distributions based on the specific task.
These details will be elaborated on later in the dissertation.

\section{Sequential Models}\label{sec:seq_models}

Sequential models seek to learn and represent the probability distribution of sequential data.
As introduced in \cref{sec:notation}, we denote such random sequences as $\mathbf{X} := (X_i)_{i\in\mathcal{I}} \in \mathcal{X}^\mathcal{I}$, where $\mathcal{I}$ is the index set and each element $X_i$ lies in a common state space $\mathcal{X}$.
This work focuses on a specific class of sequential models known as \textit{autoregressive models}.
These models parameterize the joint distribution of entire sequences $\mathbf{X}$ as a product of individual conditional distributions for each element $X_i$ given its predecessors $\mathbf{X}_{<i}$. 
The specific characteristics of the sequences and the model parameterizations are discussed in detail below.

\subsection{Autoregressive Modeling of Categorical Sequences}

Chapter 3 focuses on the simplest setting of sequential modeling: discrete time. In this setting, the index set $\mathcal{I}$ is the set of natural numbers $\mathbb{N}$.
We specifically handle categorical data, meaning each element $X_i$ in the sequence takes on of $v$ possible values within the state space $\mathcal{X}:=\{1, \dots, v\}$.
This setting is relevant for various tasks, including natural language processing and user behavior analysis.

While filtrations are typically used to represent conditional history in sequential models, they become cumbersome in the discrete time setting.
Instead, we directly condition on past segments of the sequence using notation like $\prob(\cdot \sep \mathbf{X}_{1:n})$ and $\prob(\cdot \sep \mathbf{X}_{<i})$.
Our focus lies on autoregressive sequential models that factorize the joint distribution of sequences as:
\begin{align}
\mathcal{L}^\prob(\mathbf{x}_{1:n}; \theta) := \prob_\theta(\mathbf{X_{1:n}}=\mathbf{x}_{1:n}) := p_\theta(\mathbf{x}_{1:n}) := \prod_{i=1}^{n} p_\theta(x_i \sep \mathbf{x}_{<i}) \label{eq:disc_likelihood}
\end{align}
where $n$ is the length of a given sequence and $p_\theta$ is a function parameterized by $\theta$ which represents the PMF of $X_i$ given its preceding context $\textbf{X}_{<i}$ (i.e., $p_{X_i \sep \textbf{X}_{<i}}$) for all $i \in \mathbb{N}$. 
Specific definitions for $p_\theta$ will be provided later.

Model training typically involves maximizing the likelihood of a dataset $\mathcal{D}:=\{\mathbf{x}_j\}_{j=1}^N$ containing $N$ sequences (possibly of varying lengths). 
This is often achieved by using (stochastic) gradient optimization methods to maximize the log-likelihood:
\begin{align}
    \ell(\theta; \mathcal{D}) := \sum_{\mathbf{x}_{1:n}\in\mathcal{D}} \log \mathcal{L}^\prob(\mathbf{x}_{1:n}; \theta) = \sum_{\mathbf{x}_{1:n}\in\mathcal{D}} \log p_\theta(\mathbf{x}_{1:n}).
\end{align}
Note that maximizing the log-likelihood is equivalent to minimizing the Kullback–Leibler (KL) divergence between the data probability measure that $\mathcal{D}$ was sampled with respect to and the parameterized measure $\prob_\theta$ \citep{murphy2023probabilistic}.

\subsubsection{Example Models}

We list below brief descriptions of two different parameterizations of discrete time, autoregressive models used for categorical sequences.

\paragraph{Markov Models}

A classic model employed for categorical sequences is the $n^\text{th}$ order Markov model \citep{howard2012dynamic}. These models are characterized by exhibiting the \textit{Markov property}. A process that respects this property implies that the future trajectory of the process is solely determined by its current state (or last $n$ states for $n>1$). Formally, this means
\begin{align}
    \prob_\theta(\cdot \sep \textbf{X}_{1:N}) = \prob_\theta(\cdot \sep \textbf{X}_{N-n+1:N})
\end{align}
for all values of $N>0$. The dynamics of the model can be further restricted by assuming the transition probabilities do not depend on the current timestep. Models that respect this are referred to as \textit{time-homogeneous Markov models} \citep{howard2012dynamic}. These are simply parameterized by a transition matrix $P \in \mathbb{R}^{v^n \times v}$ wherein each entry $P_{ij}$ describes the probability of transitioning from some unique state $\textbf{x}_{1:n}^{(i)} \in \mathcal{X}^n$ to $x_{n+1}=j$. 

\paragraph{Recurrent Neural Networks}
The main downsides to Markov models are their fixed context window and inability to share information across different states. This is remedied by more modern neural network approaches, such as the \textit{recurrent neural network}. While there are many variations that exist, the typical model usually complies with the following setup:
\begin{align}
    z_t & := f_\theta(x_{t-1}, z_{t-1}) \\
    p_\theta(x_t \sep \textbf{x}_{<t}) & := \text{Cat}(x_t; \text{softmax}(g_\theta(z_t))) 
\end{align}
for $t=1,2,\dots,$ where the functions $f_\theta:\mathcal{X}\times\R^d\rightarrow\R^d$ and $g_\theta:\R^d\rightarrow\R^{v}$ are parameterized by $\theta$ and $d$ describes the dimensionality of the hidden state $z_t\in\R^d$. Typically, the initial hidden state $z_0$ is set to some fixed value, e.g., $z_0=0$. Here, the hidden state $z_t$ is entrusted with retaining information from $\textbf{x}_{<t}$ that is pertinent not only to the immediate next step of $x_t$, but also for states further in the future. Specific forms of $f_\theta$ and $g_\theta$, collectively referred to as \textit{recurrent cells} in the literature, have been proposed to encourage and more easily learn this behavior during model training. Details about specific variants can be found in \citet{yu2019review}. 

\subsection{Marked Temporal Point Processes}\label{sec:mtpp_notation}

Chapters 4 and 5 focus on continuous-time event sequences with potentially varying intervals between events.
Similar to the discrete time setting, sequences in this space retain the index set $\mathcal{I}:=\mathbb{N}$.
However, the state space expands to represent both the \textit{event time} and additional information associated with each event, often referred to as a \textit{mark}. 
We denote this extended space as $\mathcal{X}:=[0,\infty)\times\mathcal{M}$, where $[0,\infty)$ contains possible timings of events and $\mathcal{M}$ is the \textit{mark space} for additional information. 
Similar to the discrete time case, we often focus on a categorical mark space, $\mathcal{M}:=\{1, \dots, v\}$. This chapter will assume $\mathcal{M}$ is discrete. For more general treatments of the mark space, please consult \citet{daley2003introduction}.

A random sequence $\mathbf{X}:=(X_i)_{i\in\mathcal{I}}$ in this setting consists of pairs $(T_i, M_i)$, where $T_i$ is the occurrence time of the $i^\text{th}$ event and $M_i$ is its associated mark. 
We enforce $T_i < T_j$ for $i < j$ to ensure temporal order and no simultaneous occurrence of events.
Additionally, we assume a finite number of events occur within any finite time interval.

Since events occur in continuous time, history is indexed by time t rather than event count.
The history filtration is defined as $(\hist_t)_{t \geq 0}$, where $\hist_t := \sigma(\{X_i \sep i \in \mathbb{N} \text{ and } T_i \leq t\})$. 
Similarly, $\hist_{t-}$ denotes the history over $[0, t)$ and more generally $\hist_{[a,b]}$ for the $\sigma$-algebra generated by events occurring over the time range $[a,b]$.

The literature uses the same $\hist_t$ notation for both the filtration and the random subset representing a time-based subsequence, e.g., $\{X_i \sep i \in \mathbb{N} \text{ and } T_i \leq t\}$ (which is \textit{not} a $\sigma$-algebra).
While not equivalent, we adopt this practice for consistency, but differentiate them when ambiguity arises. 
Often, random subsets are paired with realizations, $\textbf{h}_{t}:=\{(t_i,m_i) \sep i \in \mathbb{N} \text{ and } t_i \leq t\}$, whereas filtrations are exclusively used for conditioning, such as $\prob(\cdot \sep \hist_t)$.
In either case, $|\hist_t|$ (or $|\textbf{h}_t|$ for realized sequences) denotes the sequence length at time $t$.

Models capturing the dynamics of such sequences are called \textit{marked temporal point processes} (MTPPs).
MTPPs are typically specified by the \textit{marked intensity function} $\lambda_k(t \sep \hist_{t-})$ which describes the expected instantaneous rate of occurrence for events with mark $k$ at a given point in time $t$ given the past history $\hist_{t-}$ \citep{daley2003introduction}. 
This is formally defined by
\begin{align}
\lambda_k&(t \sep \hist_{t-}) := \lim_{\Delta \downarrow 0} \frac{1}{\Delta} \E^\prob\left[\ind(T_i \in [t, t+\Delta), M_i=k) \sep \hist_{t-}\right] \text{ where } i-1=|\hist_{t-}| \label{eq:bg_intensity} \\
& = \lim_{\Delta \downarrow 0} \frac{1}{\Delta} \prob(T_i \in [t, t+\Delta), M_i=k \sep \hist_{t-}) \\
& = \lim_{\Delta \downarrow 0} \frac{1}{\Delta} \prob(T_i \in [t, t+\Delta) \sep \hist_{t-})\prob(M_i=k \sep T_i \in [t, t+\Delta), \hist_{T_{i-1}}) \\
& = \lim_{\Delta \downarrow 0} \frac{1}{\Delta} \prob(T_i \in [t, t+\Delta) \sep T_i \notin (T_{i-1}, t), \hist_{T_{i-1}})\prob(M_i=k \sep T_i \in [t, t+\Delta), \hist_{T_{i-1}}) \\
& = \lim_{\Delta \downarrow 0} \frac{1}{\Delta} \frac{\prob(T_i \in [t, t+\Delta), T_i \notin (T_{i-1}, t) \sep \hist_{T_{i-1}})}{\prob(T_i \notin (T_{i-1}, t) \sep \hist_{T_{i-1}})}\prob(M_i=k \sep T_i \in [t, t+\Delta), \hist_{T_{i-1}}) \\
& = \lim_{\Delta \downarrow 0} \frac{1}{\Delta} \frac{\prob(T_i \in [t, t+\Delta) \sep \hist_{T_{i-1}})}{1 - \prob(T_i \leq t \sep \hist_{T_{i-1}})}\prob(M_i=k \sep T_i \in [t, t+\Delta), \hist_{T_{i-1}}).
\end{align}
Since the intensity is always conditional, we denote it as $\lambda^*_k(t):=\lambda_k(t\sep\hist_{t-})$ for brevity.
When evaluated on a partial sequence $\mathbf{x}_{<i}$, the marked intensity $\lambda^*_k(t)$ can be decomposed into individual distribution functions for $T_i$ and $M_i$:
\begin{align}
\lambda_k^*(t) := \frac{p_{T_i \sep \hist_{T_{i-1}}}(t \sep \mathbf{x}_{<i})}{1 - F_{T_i \sep \hist_{T_{i-1}}}(t \sep \mathbf{x}_{<i})}p_{M_i \sep T_i, \hist_{T_{i-1}}}(k \sep t, \mathbf{x}_{<i}). \label{eq:marked_intensity}
\end{align}
The \textit{total intensity function} $\lambda^*(t)$, also known as the ground intensity, describes the overall expected event occurrence rate regardless of mark.
It is equivalent to the sum of all marked intensities: $\lambda^*(t) := \sum_{k\in\mathcal{M}} \lambda^*_k(t)$. 
The total intensity is further linked to the general event occurrence by noting that the expected value of the \textit{compensator} for a point process at time $t$, or rather the integrated total intensity $\Lambda^*(t) := \int_0^t \lambda^*(s)ds$, is equivalent to the expected number of events to have occurred by time $t$: $\E^\prob\left[\Lambda^*(t)\right] \equiv \E^\prob\left[|\hist_t|\right]$ \citep{daley2003introduction}.

From \cref{eq:marked_intensity}, it follows that $p_{M_i \sep T_i, \hist_{T_{i-1}}}(k \sep t, \mathbf{x}_{<i}) = \lambda_k^*(t) / \lambda^*(t)$.
This is further justified by the \textit{superposition property} of MTPPs \citep{daley2003introduction}. This property states that the superposition of any MTPPs results in another MTPP, with the resulting intensity being the sum of the individual intensities. 
Furthermore, the ratio of individual intensities to total intensity defines the probability an event was produced by a given individual MTPP, assuming it belongs to the superposition. 
Therefore, each marked intensity $\lambda^*_k$ can be seen as describing an individual point process with the total intensity $\lambda^*$ describing their superposition.

MTPPs often directly parameterize and represent the marked intensity function $\lambda^*_k$ for all marks as the primary output and inference target. 
This is convenient from a modeling perspective as the only restriction that must be made for the parameterization is that $\lambda_k^*(t)$ must be non-negative for all values of $t$ and $k$. 
Furthermore, using \cref{eq:marked_intensity}, we can derive:
\begin{align}
    p_{T_i}^*(t) & = \lambda^*(t)\exp\left(-\int_{T_{i-1}}^t \lambda^*(s)ds\right) \\
    \text{ and } F_{T_i}^*(t) & = 1 - \exp\left(-\int_{T_{i-1}}^t \lambda^*(s)ds\right),
\end{align}
for $i \in \mathbb{N}$.\footnote{Similar to the notation for the conditional intensity functions, we use $p_{T_i}^*(t)$, $p_{M_i \sep T_i}^*(k \sep t)$, and $F_{T_i}^*(t)$ as shorthand to refer to $p_{T_i\sep\hist_{T_{i-1}}}(t \sep \mathbf{x}_{<i})$, $p_{M_i\sep T_i, \hist_{T_{i-1}}}(k \sep t, \mathbf{x}_{<i})$, and $F_{T_i\sep\hist_{T_{i-1}}}(t \sep \mathbf{x}_{<i})$ respectively.} 

Let $\textbf{h}_\tau:=(t_i,m_i)_{i=1}^n$ be a sequence of length $n$ observed over the time window $[0, \tau]$. 
The likelihood of this sequence is defined as:
\begin{align}
\mathcal{L}^\prob&(\hist_\tau=\textbf{h}_\tau)  := (1 - F_{T_{n+1}}^*(\tau))\prod_{i=1}^n p_{T_i}^*(t_i)p_{M_i \sep T_i}^*(m_i \sep t_i) \\
& = \exp\left(-\int_{t_n}^\tau \lambda^*(s) ds\right) \prod_{i=1}^n \lambda^*(t_i)\exp\left(-\int_{t_{i-1}}^{t_i} \lambda^*(s) ds\right)\frac{\lambda^*_{k_i}(t_i)}{\lambda^*(t_i)} \\
& = \left[\prod_{i=1}^n \lambda^*_{k_i}(t_i)\right] \exp\left(-\int_0^\tau \lambda^*(s)ds \right)
\end{align}
where $t_0=0$.
Note that we only write the likelihood $\mathcal{L}^\prob$ and not also as $\prob(\hist=\mathbf{h})$, unlike in \cref{eq:disc_likelihood}, because the probability mass of any specific sequence in this space is 0 due to the continuous components.\footnote{Technically speaking, there actually is one sequence that can have a non-zero probability mass assigned to it: the sequence where no events occur over the observable time window. The probability for that sequence is equivalent to $\prob(T_1 > \tau)$.} 
Training MTPPs involves maximizing the likelihood over a dataset $\mathcal{D}$ of sequences with (potentially) varying lengths and observation windows.
Stochastic gradient optimization is often used to maximize the the log-likelihood:
\begin{align}
    \ell(\theta; \mathcal{D}) := \sum_{\textbf{h}_\tau\in\mathcal{D}} \log \mathcal{L}^\prob&(\hist_\tau = \textbf{h}_\tau),
\end{align}
where $\theta$ represents the model parameters (often used to define the marked intensity functions).

While direct sampling of MTPPs is generally not feasible (except for a few parametric forms, such as constant $\lambda^*(t)$), rejection sampling remains a viable option for sampling from any arbitrary intensity-based MTPPs. 
The employed technique, often called the \textit{thinning procedure}, leverages the superposition property of MTPPs in conjunction with forms of MTPPs that can be directly sampled \citep{ogata1981lewis}. 

Let the MTPP with intensity $\lambda$ to sample be $A$. The core idea of the thinning procedure is there exists some MTPP $B$, such that the superposition MTPP $C$ yields a constant total intensity, say of value $c$.
Proposal times are drawn directly from MTPP $C$ until the desired time window $[0,\tau]$ is covered.
Each proposed time is then randomly accepted sequentially with probability $\lambda^*(t) / c$. If accepted, the time of occurrence originated from MTPP $A$ and is added to the sampled history.
Furthermore, the mark is then sampled from $p_{M \sep T}^*$ using the accepted time to condition on.
This procedure requires $c$ to dominate $\lambda^*$ for all times to ensure the intensity for MTPP $B$ is non-negative.
A formal description of this sampling procedure can be found in \cref{alg:mtpp_sampling}.

\begin{algorithm}
\caption{Thinning Procedure for Sampling from a MTPP \citep{ogata1981lewis}}\label{alg:mtpp_sampling}
\begin{algorithmic}
\Require Start Time $t \geq 0$, End Time $\tau > t$, Partial History $\textbf{h}_t$
\Ensure Dominating Rate $c > \lambda^*$
\While{$t < \tau$}
\State $t' \sim t + \text{Exponential}(c)$ \Comment{Sample candidate time}
\State $u \sim \text{Uniform}[0,1]$ \Comment{Sample acceptance rate}
\If{$u < \frac{\lambda^*(t')}{c}$ and $t' < \tau$} \Comment{New event occurs at time $t'$}
    \State $m \sim \text{Cat}\left(\left[\frac{\lambda^*_1(t')}{\lambda^*(t')}, \dots, \frac{\lambda^*_K(t')}{\lambda^*(t')}\right]\right)$ \Comment{Draw associated mark}
    \State $\textbf{h}_{t'} \gets \textbf{h}_{t} \cup \{(t', m)\}$ \Comment{Append to accumulated realized history}
\Else \Comment{Candidate event is rejected}
    \State $\textbf{h}_{t'} \gets \textbf{h}_t$
\EndIf
\State $t \gets t'$ \Comment{Advance forward in time}
\EndWhile
\end{algorithmic}
\end{algorithm}

\subsubsection{Example Models}

We list four different examples of (marked) temporal point processes, commonly used across a variety of different real-world settings.

\paragraph{Poisson Process}
The simplest of temporal point processes is the widely known \textit{Poisson process} \citep{daley2003introduction}. A homogeneous Poisson process is characterized by a counting process $(N_t)_{t\geq 0}$ with constant intensity $\lambda\in[0,\infty)$ which satisfies the following properties:
\begin{itemize}
    \item $N_0=0$ almost surely.
    \item $(N_{t_2} - N_{t_1})$ is independent of $(N_{t_4} - N_{t_3})$ for all $0 \leq t_1 < t_2 \leq t_3 < t_4$.
    \item $N_{t+\Delta} - N_{t} \sim \text{Poisson}(\lambda\Delta)$ for $\Delta > 0$.
\end{itemize}
Likewise, an inhomogeneous Poisson process allows the intensity $\lambda(t)$ to vary with respect to time, \textit{but not the history of events}. It shares the same properties in general as the homogeneous Poisson process, with the exception that $N_{t+\Delta}-N_t \sim \text{Poisson}(\int_{t}^{t+\Delta} \lambda(s)ds)$. Note that one can interpret an inhomogeneous Poisson process, or really \textit{any} temporal point process with conditional intensity $\lambda^*(t)$, as resembling a homogeneous Poisson process with intensity $\lambda^*(t)$ over the small interval of time $[t, t+dt)$. Put differently, every temporal point process resembles a Poisson process locally in time.

\paragraph{Hawkes Process}
Self-exciting temporal point processes, or as they are more commonly referred to as \textit{Hawkes processes}, are designed to model clusters of events where one occurrence encourages additional events to happen \citep{hawkes1971spectra}. While several variations exist, we will present one specific form that allows for cross-excitation between events of differing marks:
\begin{align}
\lambda^*_k(t) := \mu_k(t) + \sum_{(T,M)\in\hist_{t-}} \phi_{M,k}(t-T),
\end{align}
for $k\in\mathcal{M}$ where $\phi_{i,j}$ is a kernel describing the effect that events of type $i$ have on events of type $j$ and $\mu_k$ describes the background expected rate of occurrence for events of type $k$. One common form of this is exponential kernel: $\phi_{i,j}(z)=\alpha_{i,j}\exp(-\beta_{i,j}z)$ with parameters $\alpha_{i,j},\beta_{i,j} > 0$ for $i,j\in\mathcal{M}$.

\paragraph{Self-Correcting Process}
Standing in opposition to Hawkes processes are \textit{self-correcting processes} \citep{isham1979self}. Instead of encouraging clustering of events, this process inhibits it by decrementing the intensity whenever an event occurs, only to have it slowly build back up over time. Typically described for the non-marked setting, we have adapted the typical intensity parameterization here to account for cross-inhibition between events of differing marks:
\begin{align}
\lambda^*_k(t) := \exp\left(\eta_k t - \sum_{(T,M) \in \hist_{t-}} \delta_{M,k}\right)
\end{align}
for $k\in\mathcal{M}$ where parameters $\eta_k > 0$ determine the rate of compounding growth for the intensity and $\delta_{i,j}>0$ describes the inhibition that prior events of type $i$ have on future events of type $j$.

\paragraph{Neural MTPPs}
Recently, many different neural-network based approaches to modeling MTPPs have been proposed. While some choose to model the next event densities \citep{shchur2019intensity}, or the compensator $\Lambda^*(t):=\int_0^t \lambda^*(s)ds$ \citep{NEURIPS2019_39e4973b}, or eschew distributions entirely for an implicit approach \citep{xiao2018learning,lin2022exploring}, the majority parameterize the intensity function directly \citep{du2016recurrent,mei2017neural,jia2019neural,xiao2017modeling,zuo2020transformer,zhang2022temporal}. 

Neural approaches to representing the intensity function can be roughly summarized in the following representation:
\begin{align}
z_i & := f_\theta(z_{i-1}, t_i, m_i) \\
z(t) & := g_\theta(z_i,t) \text{ for } t_i < t \leq t_{i+1} \\
\lambda^*_k(t) & := h_\theta(z(t))
\end{align}
for $t \geq 0$ and $k\in\mathcal{M}$ where the functions $f_\theta:\R^d\times[0,\infty)\times\mathcal{M}\rightarrow\R^d$, $g_\theta:\R^d\times[0,\infty)\rightarrow\R^d$, and $h_\theta:\R^d\rightarrow[0,\infty)$
are parameterized by $\theta$ and $d$ describes the dimensionality of the hidden state $z\in\R^d$. For details on the specifics of these functions for various proposed approaches, refer to \citet{shchur2021neural}.

\subsection{Continuous-Time Stochastic Jump Processes}

Chapter 6 delves into the most general class of sequential models considered in this work: \textit{continuous-time stochastic jump processes}.
These processes, denoted by $X:=(X_t)_{t\in\mathcal{I}}$, take values $X_t$ for each continuous point in time $t \in \mathcal{I} := [0,\infty]$.\footnote{In earlier settings, we commonly represented sequences or portions thereof with bold characters (e.g., $\mathbf{X}_{1:N}:=(X_1, \dots, X_N)$); however, in the stochastic process literature a more functional perspective tends to be preferred where a process $X$ is really a random \textit{function} draw where time $t$ is an argument to the function. As such, when in this setting, we will not bold the process as a whole to comply with those standards.}
They exhibit both continuous variations over time segments and discontinuous jumps at finitely many points within a finite time range.
As such, it is helpful to conceptualize this not as a single process but rather as a pair of processes: one component governs the continuous segments, often described by a stochastic differential equation, while the other controls the jumps, typically modeled by a temporal point process. 

This topic, like the broader literature surrounding it, can be dense and require specific prior knowledge.
While essential terms and concepts will be formally defined, complete understanding might require additional background.
For an exhaustive explanation, we highly recommend consulting \citet{woyczynski2022diffusion}.

\paragraph{Setting}
Let $(\Omega, \filter, \prob)$ be a probability space equipped with a complete filtration $(\filter_t)_{t\geq 0}$ such that $\filter_\infty := \filter$. A stochastic process is defined as a function $X:\Omega \times [0,\infty] \rightarrow \mathcal{X}$ where $\mathcal{X}$ is the domain of the process and the process will be referred to as either $X_t(\omega)$ or $X_t$ for $t\in[0,\infty]$ and $\omega \in \Omega$. 
For our purposes, the process domain $\mathcal{X}$ is assumed to be $\mathbb{R}^d$.
Unless stated %
otherwise, we %
assume that a process $X$ is predictable and adapted to the filtration $(\filter_t)_{t\geq 0}$, meaning $X_t$ is $\filter_{t-}$-measurable. This implies that the natural filtration of a process, $\filter_t^X := \sigma(X_s \sep s \leq t)$, is a subset of $\filter_t$.

\paragraph{Jump Stochastic Differential Equations}
We focus on \cadlag processes, meaning they are right-continuous with well-defined left-limits. 
Formally, $X_t=X_{t+}=\lim_{s\downarrow t} X_s$ for all $t \geq 0$, but not necessarily $X_t=X_{t-}$ with $X_{t-}=\lim_{s\uparrow t} X_{s}$.
\textit{Jumps}, denoted by $\Delta$, are defined as:
\begin{align}
(\Delta X)_t \equiv \Delta X_t := X_t - X_{t-}.
\label{eq:del}
\end{align}
Jump times ($\{t \sep \Delta X_t \neq 0\}$) are assumed to occur randomly and be finite in number over any finite time window.

A general class of stochastic processes satisfying these constraints are \textit{jump stochastic differential equations} (jump SDEs). These processes can be decomposed into two components: continuous segments and discontinuous jumps. 

\paragraph{Brownian Motion}
Before characterizing the continuous segments, we introduce the quintessential continuous-time stochastic process driving them: \textit{Brownian motion}, also known as the \textit{Wiener process} \citep{woyczynski2022diffusion}. Denoted by $W:=(W_t)_{t \geq 0}$,  Brownian motion satisfies four conditions:
\begin{enumerate}
    \item $W_0=0$ almost surely.
    \item $W_t(\omega)$ is almost surely continuous in $t$ for fixed $\omega\in\Omega$.
    \item $W_t$ is independent of $W_{t'}$ for all $t,t'\in[0,\infty)$.
    \item $W_t - W_{t'} \sim \mathcal{N}(0, t-t')$ for all $0 \leq t < t'$.
\end{enumerate}
It should be noted that this describes scalar-valued Brownian motion. This generalizes to $d$-dimensional Brownian motion by having each coordinate be driven by an independent scalar-valued Brownian motion.

While Brownian motion is suitable for describing various systems, including natural phenomena like particle movement in fluids, it can be too limited in expressivity for our purposes.
Similar to how Poisson processes generalize to more complex temporal point processes, Brownian motion can be generalized to more expressive continuous stochastic processes.

\paragraph{Continuous Segments} 
Assuming no jump occurs at time $t$, the continuous change in the process $X$ at time $t$ follows:
\begin{align}
dX_{t} & = \mu(t, X_{t-})dt + \sigma(t, X_{t-})dW_t  \label{eq:cont_def}
\end{align}
for drift $\mu:[0,\infty)\times\mathbb{R}^d\rightarrow\mathbb{R}^d$ and scaling $\sigma:[0,\infty)\times\mathbb{R}^d\rightarrow\mathbb{R}^{d\times d}$ functions. $W$ is a Brownian motion of the same dimensionality as $X$ and is adapted to $(\filter_t)_{t\geq0}$. 
Using left-limits for arguments to $\mu$ and $\sigma$ ensures predictability of the entire process.
Note that this representation can be thought of as the continuous-time equivalent of the reparameterization of the Normal distribution, i.e., if $X \sim \mathcal{N}(\mu, \sigma)$ then $X=\mu+W\sigma$ for $W\sim\mathcal{N}(0,1)$.

If no jumps occur almost surely, then $X$ described by \cref{eq:cont_def} becomes a \textit{diffusion process}, known to possess the Markov property.

\paragraph{Instantaneous Jumps} 
Consider a marked temporal point process over the time interval $[0,\infty)$, with a mark-space $\mathcal{M}$ and marked counting process $N$.\footnote{To be clear, the mark-space $\mathcal{M}$ in this setting could be discrete, continuous, or a combination thereof.} 
The process is characterized by the conditional intensity function:
\begin{align}
\lambda_t(m)dtdm = \prob(N(dt\times dm)=1 \sep \filter_{t-}),
\end{align}
for $m \in \mathcal{M}$.
This definition is similar to \cref{eq:bg_intensity} with two distinct nuances:
\begin{itemize}
\item In the stochastic process literature, $\lambda$ is often considered a stochastic process itself, hence the time $t$ is the subscript rather than the mark $m$.\footnote{Technically speaking, jump SDEs also require $\lambda$ to be described via a differential equation; however, this dissertation lifts this restriction and considers more general, but still predictable and adapted to $(\filter_t)_{t\geq 0}$, intensity functions.} Conditioning on past events is therefore implicit.
\item Here, $\mathcal{M}$ is allowed to be discrete, continuous, or a mixture thereof. Because of this, $dm$ corresponds to an appropriate reference measure over this space.  
\end{itemize}
Should the point process dynamics be independent of the continuous segments, then it is sufficient for the intensity to condition on solely generated MTPP events $\hist_{t-}$ instead of the more complete $\filter_{t-}$.

The MTPP drives both the timing and value of the instantaneous jumps in $X$:
\begin{align}
d\Delta X_t & := \nu(t, X_{t-}, M_{N_t})dN_t
\label{eq:disc_def}
\end{align}
with jump strength $\nu:[0,\infty]\times \mathbb{R}^d \times \mathcal{M} \rightarrow \mathbb{R}^d$ and where $M_{i}$ is the MTPP's $i^\text{th}$ sampled mark. 
Both the continuous diffusion process and the MTPP are assumed predictable and adapted to the same filtration.
A common choice of jump mapping is $\nu(t, X_{t-}, M_{N_t})=M_{N_t}$ with $\mathcal{M}\equiv\mathcal{X}$, i.e., the marks are the jumps.
Under this setting, if $N_t$ is a Poisson process and $M_i$ are i.i.d. for all $i>0$, then $\Delta X$ becomes a compound Poisson process.

\paragraph{Full Model Definition}
From \cref{eq:del}, we have $X_t = \Delta X_t + X_{t-}$ and likewise $dX_t = d\Delta X_t + dX_{t-}$. 
Combining \cref{eq:cont_def} and \cref{eq:disc_def}, the full jump SDE governing our process of interest becomes:
\begin{align}
dX_t & = \mu(t, X_{t-})dt + \sigma(t, X_{t-})dW_t + \nu(t, X_{t-}, M_{N_t})dN_t.
\label{eq:full_model}
\end{align}
Equivalently, we can express the process as:
\begin{align}
X_t & = X_0 + \int_0^t \mu(t, X_{t-})dt + \int_0^t \sigma(t, X_{t-})dW_t + \sum_{i=1}^{N_t} \nu(T_i, X_{t_i-}, M_i).
\label{eq:full_model_abs}
\end{align}
where $T_i$ are the random jump times and $X_0$ is the initial value of the process (often assumed to be a constant value in $\mathcal{X}$).

\chapter{General Probabilistic Query Estimation for Discrete Time Models}

\noindent

One of the major successes in machine learning in recent years has been the development of neural sequence models for categorical sequences, particularly in natural language applications but also in other areas such as automatic code generation and program synthesis \citep{shin2019program,chen2021evaluating}, computer security \citep{brown2018recurrent},  recommender systems \citep{wu2017recurrent}, genomics \citep{shin2021protein,amin2021generative}, and survival analysis \citep{lee2019dynamic}. Many of the  models (although not all) 
rely on autoregressive training and prediction, allowing for the sequential generation of sequence completions in a recursive manner conditioned on sequence history.

A natural question in this context is how to compute answers to probabilistic queries that go beyond   traditional one-step-ahead predictions. Examples of such  queries are ``how likely is event $A$ to occur before event $B$?'' and ``how likely is event $C$ to occur (once or more) within the next $K$ steps of the sequence?'' These types of queries are very natural across a wide variety of application contexts, for example, the probability that an individual will 
finish speaking or writing a sentence within the next $K$ words, or that a user will use one app before another. See \cref{fig:3_flashy_example} for an example.

\begin{figure}
    \centering
    \includegraphics[width=0.95\textwidth]{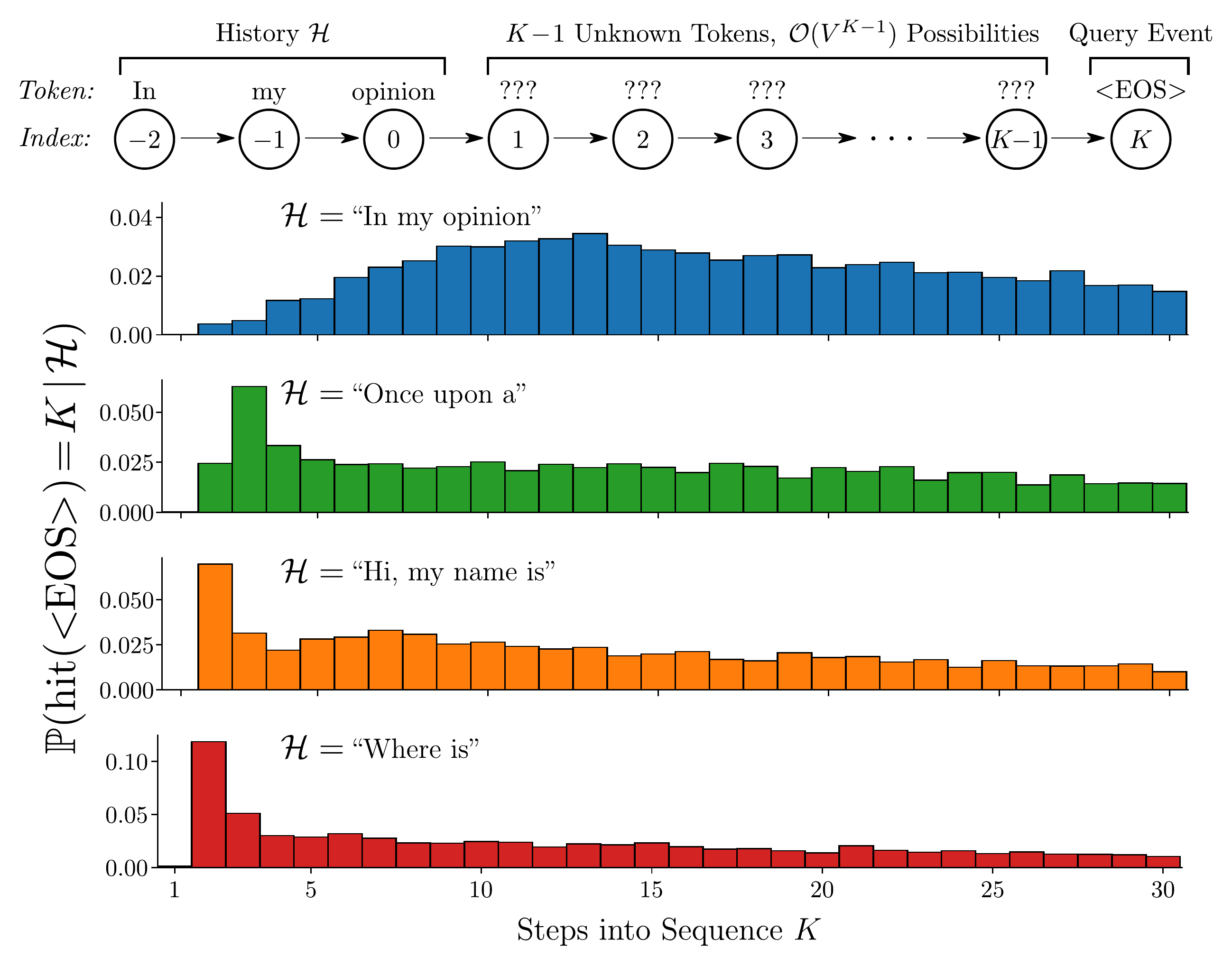}
    \caption{(top) Illustration of a query for the probability of a given sentence ``In my opinion...'' ending in $K$ steps. (bottom) GPT-2 \citep{gpt2-radford} hitting time estimates for sentence ending across 4 prefixes with $V=50,257, K\leq30$. Importance sampling query estimates maintain a 6x reduction in variance relative to naive model sampling for the same computation budget. Open-ended prefixes (top-left) generally possess longer-tailed distributions relative to simple prefixes. Almost no probability mass is found for $K=1$ due to the extremely high likelihood that at least one more token succeeds the prompts prior to ending in order to ensure proper grammar. Additional details provided in \cref{sec:3_methods,sec:3_experiments} and \cref{sec:3_gpt_exp}.}
    
    \label{fig:3_flashy_example}
\end{figure}

In this chapter we develop a general framework for answering such predictive queries in the context of autoregressive (AR) neural sequence models. 
This amounts to computing conditional probabilities of propositional statements about future events, conditioned on the history of the sequence as summarized by the current hidden state representation. 
We focus in particular on how to perform near real-time computation of such queries, motivated by use-cases such as answering human-generated queries and utilizing query estimates within the optimization loop of training a model. 
Somewhat surprisingly, although there has been extensive prior work on multivariate probabilistic querying in areas such as graphical models and database querying, as well as for restricted types of queries for traditional sequence models such as Markov models, querying for neural sequence models appears to be unexplored. 
One possible reason is that the problem is computationally intractable in the general case (as we discuss later in Section \ref{sec:3_queries}), typically scaling as $\bigO \! \left(V^{K-1}\right)$ or worse for predictions $K$-steps ahead, given a sequence model with a $V$-ary alphabet (e.g. compared to $\bigO \! \left(KV^2\right)$ for Markov chains).

\section{Related Work}
\label{sec:3_related}

Research on efficient computation of probabilistic queries has a long history in machine learning and AI, going back to   work on exact inference in multivariate  graphical models  \citep{pearl1988probabilistic,koller2009probabilistic}.
Queries in this context are typically of two types. The first are {\it conditional probability queries}, which are the focus of our attention here: computing probabilities defined for a subset $X$ of variables of interest, conditioned on a second subset $Y=y$ of observed variable values, and marginalizing over the set $Z$ of all other variables. The second type of queries can broadly be referred to as {\it assignment queries}, seeking the most likely (highest conditional probability) assignment of values $x$ for $X$, again conditioned on $Y=y$ and marginalizing over the set $Z$. Assignment queries are also referred to as most probable explanation (MPE) queries, or as maximum a posteriori (MAP) queries when $Z$ is the empty set \citep{koller2007graphical}.
 
For models that can be characterized with sparse  Markov dependence structure,   there is a significant body of work on efficient inference algorithms that can leverage such structure  \citep{koller2009probabilistic},  in particular for sequential models where   recursive computation can be effectively leveraged \citep{bilmes2010dynamic}. However, autoregressive neural sequence models are inherently non-Markov since the real-valued current hidden state is a function of the entire history of the sequence. Each hidden state vector induces a tree containing $V^K$ unique future trajectories with state-dependent probabilities for each path of length $K$. Techniques such as dynamic programming (used effectively in Markov-structured sequence models) are not applicable in this context, and both assignment queries and conditional probability queries are NP-hard in general  \cite{chen2018recurrent}.

For assignment-type queries  there has been considerable work in natural language processing with neural sequence models, particularly for the MAP problem of generating high-quality/high-probability  sequences conditioned on sequence history or other conditioning information. A variety of heuristic decoding methods have been developed and found to be useful in practice, including   beam search \citep{sutskever2014sequence},  best-first search \citep{xu2021massive}, sampling methods \citep{holtzman2019curious}, and hybrid variants \citep{shaham2021you}.  However, for conditional probability queries with neural sequence models (the focus of this chapter and to a large degree this dissertation as a whole), there has been no prior work in general on this problem to our knowledge. While  decoding techniques such as beam search can also be useful in the context of conditional probability queries, as we will see later in Section \ref{sec:3_methods}, such techniques have significant limitations in this context, since by definition they produce lower-bounds on the probabilities of interest and, hence, are biased estimators.

\section{Probabilistic Queries}
\label{sec:3_queries}

\paragraph{Notation}
Let $\textbf{X}_{1:N} := [X_1, X_2, \dots, X_N]$ be a sequence of random variables with arbitrary %
length $N$.\footnote{We will also use the notation $\textbf{X}_{<N}:=[X_1, \dots, X_{N-1}$] depending on which is more convenient.} This can be considered a truncation of the more general stochastic process $\textbf{X}:=(X_i)_{i\in\mathbb{N}}$. Additionally, let $\textbf{x}_{1:N} := [x_1, x_2, \dots, x_N]$ be the respective observed values of the random variables where each $x_i$ takes on values from a fixed vocabulary $\V := \{1, \dots, V\}$ of size $V$. 
Examples of these sequences include sentences where each letter or word is a single value, or streams of discrete events generated by some process or user. 
We will refer to individual variable-value pairs $(X_i, x_i)$ in the sequence as events. Recall that the filtration $(\hist_i)_{i\in\mathbb{N}}$ where $\hist_i:=\sigma(\{X_j \sep j = 1,\dots,i\})$ captures all measurable events up to step $i$ generated by the sequence $\mathbf{X}_{1:i}$ and is the formal way of conditioning on portions of a sequence; however, for clarity we will often opt for conditioning on the sequence directly, e.g., $(-\sep\mathbf{X}_{1:i})$.%

We consider an autoregressive model $p_\theta(x_N \sep \textbf{x}_{<N})$ that is able to condition on realized sequences $\textbf{x}_{<N}$ of arbitrary length $N-1$ and produce the conditional PMF values for the next event in the sequence over all possible vocabulary values, e.g., $p_{X_N \sep \textbf{X}_{1:N-1}}(x_N \sep \textbf{x}_{1:N-1})$ for each $x_N \in \V$. Note that this factorization fully defines the probability measure $\prob_\theta$ over all events generated by the stochastic process $\mathbf{X}$, i.e., $\hist_\infty := \cup_{i\in\mathbb{N}}\hist_i$. This implied probability measure and event space contain measurable events of interest to us; however, the probability values assigned to them are not always readily accessible should they not comply with the autoregressive nature of the model $p_\theta$. We consider the only readily accessible probability statements from the model include the originally mentioned conditional PMF values, $p_{X_N \sep \textbf{X}_{1:N-1}}(x_N \sep \textbf{x}_{1:N-1}):=p_\theta(x_N \sep \textbf{x}_{1:N-1})$, and the joint PMF over an entire sequence $p_{\textbf{X}_{1:N}}(\textbf{x}_{1:N}):=\prod_{i=1}^Np_\theta(x_i \sep \textbf{x}_{<i})$. When appropriate, we will opt for using the model $p_\theta$ instead of the implied measure $\prob_\theta$ when the statement is in a readily accessible form. 

For brevity, the methods presented in the remainder of the chapter consider sequences starting from time step 1 and do not condition on any additional information. That being said, it is often the case that a practitioner will want to contextualize a query by conditioning on the history up to a given point and then ask about potential future trajectories. All of the formulas that will be derived can be adapted to this scenario easily by simply remapping the indexing values such that $\hist_0$ contains all realized conditional information and $X_i$ for $i=1,2,\dots$ will all take place in the future. Then, all probability statements can condition on $\hist_0$ without any other changes to make the methods compliant with this contextual information.\footnote{For example, say we know $\textbf{X}_{1:3}=\textbf{x}_{1:3}$ and are concerned with the event 3 more steps in the future, $X_6$. This results in the query $p_{X_6 \sep \hist_3}(x_6 \sep \textbf{x}_{1:3})$; however, we can alternatively represent this via $p_{X_3' \sep \hist_0'}(x_3' \sep \textbf{x}_{<1}')$ where $X_i'=X_{i+3}$ and similarly for $x'$ and $\hist'$.} We will use $\hist_0$ to reference to this history of fixed events, with the knowledge that it could contain no events in the use case of wanting unconditional queries.

\begin{figure}
    \centering
    \includegraphics[width=0.8\textwidth]{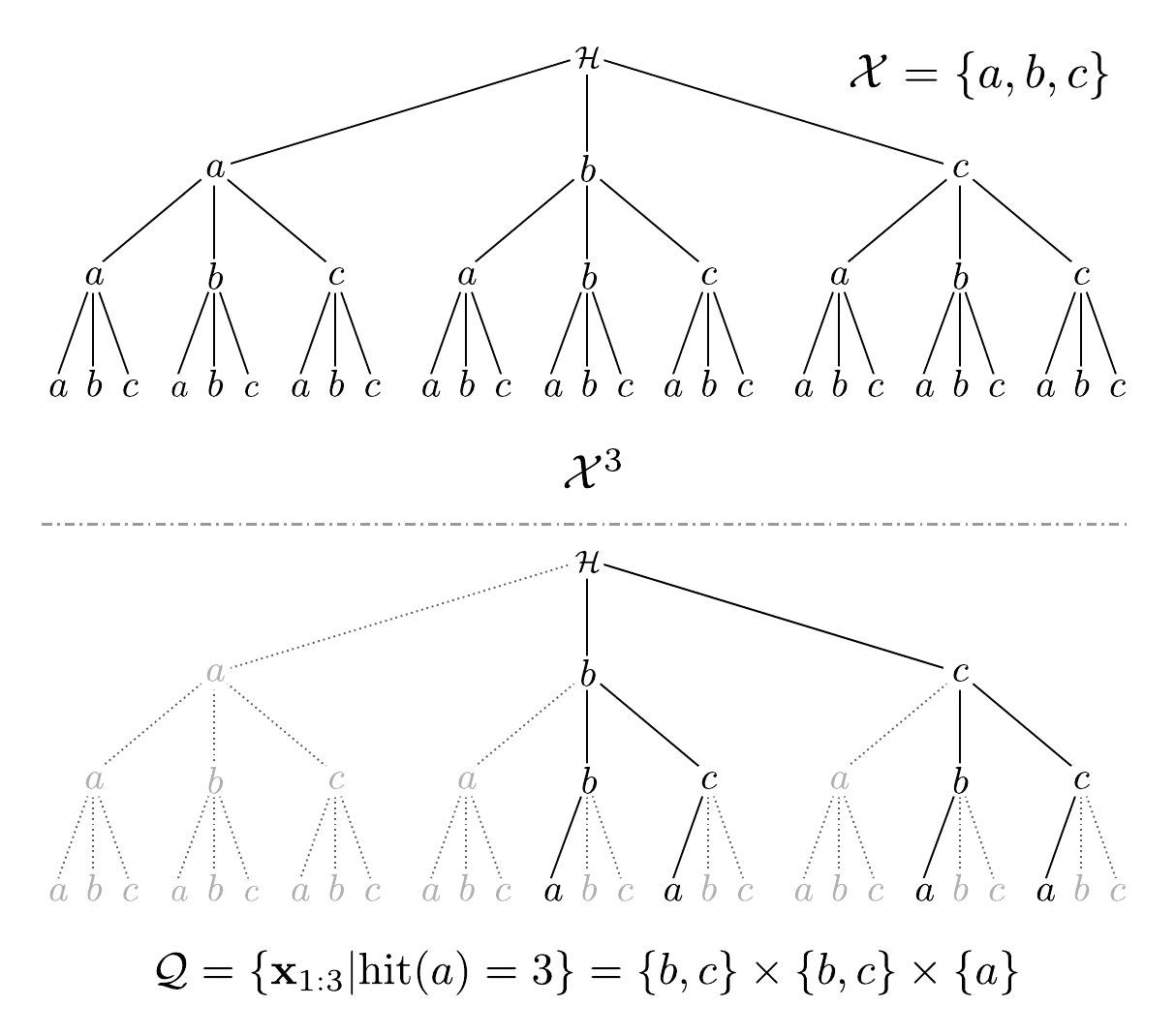}
    \caption{(left) Tree diagram of the complete sequence space for a vocabulary $\V=\{a,b,c\}$ and the corresponding query space $\mathcal{Q}$ (right) for when \emph{the first appearance of $a$} occurs on the third step (i.e., $\hit(a)=3$), defined as the set product of restricted domains listed below the figure.} 
    \label{fig:3_query_example}
\end{figure}

\paragraph{Defining Probabilistic Queries}
Given a specific history of events $\hist_0$, there are a variety of different questions one could ask about the future beyond where the history ends: ($\text{Q}1$) What event is likely to happen next? ($\text{Q}2$) Which events are likely to occur $K>1$ steps from now? ($\text{Q}3$) What is the distribution of when the next instance of $a \in \V$ occurs? ($\text{Q}4$) How likely is it that we will see event $a\in\V$ occur before $b\in\V$? ($\text{Q}5$) 
How likely is it for $a\in\V$ to occur $n$ times in the next $K$ steps?

We define a common framework for such queries  by defining probabilistic queries to be of the form
$\prob_\theta(\textbf{X}_{1:K} \in \mathcal{Q})$
with $\mathcal{Q}\subset \V^K$. This %
can be extended to the infinite setting (i.e., $\prob_\theta(\textbf{X} \in \mathcal{Q})$ where $\mathcal{Q}\subset \V^\infty$). Exact computation of an arbitrary query is straightforward to represent:
\begin{align}
\prob_\theta(\textbf{X}_{1:K} \in \mathcal{Q}) & = \sum_{\textbf{x}_{1:K}\in\mathcal{Q}} p_\theta(\textbf{x}_{1:K})  = \sum_{\textbf{x}_{1:K}\in\mathcal{Q}} \prod_{k=1}^K p_\theta(x_{k} \sep \textbf{x}_{<k}). \label{eq:3_exact_computations_cond}
\end{align}
Depending on number of different sequences in the query space, $|\mathcal{Q}|$, performing this calculation can quickly become intractable, %
motivating lower bounds or approximations (developed in more detail in \cref{sec:3_methods}). In this context it is helpful to impose structure on the query $\mathcal{Q}$ to make subsequent estimation  easier, in particular by
breaking $\mathcal{Q}$ into the following structured partition:
\begin{align}
\mathcal{Q} & = \cup_i \mathcal{Q}^{(i)} \text{ where } \mathcal{Q}^{(i)} \cap \mathcal{Q}^{(j)} = \emptyset \text{ for } i\neq j \label{eq:3_structure_1} \\
\text{and } \mathcal{Q}^{(i)} & = \mathcal{V}^{(i)}_1 \times \mathcal{V}^{(i)}_2 \times \dots \times \mathcal{V}^{(i)}_K \text{ where } \mathcal{V}^{(i)}_k \subseteq \V \text{ for } k=1,\dots,K. \label{eq:3_structure_2}
\end{align}
In words, this means a given query $\mathcal{Q}$ can be broken into a partition of simpler queries $\mathcal{Q}^{(i)}$ which take the form of a set cross product between restricted domains $\mathcal{V}^{(i)}_k$, one domain for each token $X_k$.\footnote{Ideally, the partitioning is chosen to have the smallest number of $\mathcal{Q}^{(i)}$'s needed.} An illustration of an example query set can be seen in \cref{fig:3_query_example}. A natural consequence of this is that:
\begin{align}
\prob_\theta(\textbf{X}_{1:K}\in\mathcal{Q}) & = \sum_i \prob_\theta\left(\textbf{X}_{1:K}\in\mathcal{Q}^{(i)}\right) = \sum_i \prob_\theta\left(\cap_{k=1}^K X_k \in\mathcal{V}^{(i)}_k \right),
\end{align}
which lends itself to more easily estimating each term in the sum. This will be discussed in  \cref{sec:3_methods}.

\begin{table}[]
    \centering
    \begin{tabular}{clll}
    \toprule
    \# & Question & Probabilistic Query & Cost $\left(K\cdot|\mathcal{Q}|\right)$ \\
    \midrule
    \vspace{2pt} 
    \query{1} & Next event? & $\prob_\theta(X_1=x_1)$ & $\bigO (1)$\\
    \vspace{2pt} 
    \query{2} & Event $K$ steps from now? & $\prob_\theta(X_K=x_K)$ & $\bigO \! \left(V^{K-1}\right)$ \\
    \query{3} & Next instance of $a$? & $\prob_\theta(\hit(a)=K)$ & $\bigO \! \left((V-1)^{K-1}\right)$ \\
    \vspace{2pt} 
    \query{4} & Will $a$ happen before $b$? & $\prob_\theta(\hit(a) < \hit(b))$ &  $\bigO \! \left((V-2)^{K}\right)^\dagger$ \\
    \query{5} & How many instances of $a$ in $K$ steps? & $\prob_\theta(N_a(K)=n)$ &  $\bigO \! \left(\binom{K}{n}(V-1)^{K-n}\right)$ \\
    \bottomrule
    \end{tabular}
    \caption{List of example questions, corresponding probabilistic queries, and associated costs of exact computation computation with an autoregressive model. 
    The cost of incorporating a history $\hist$ is assumed to be an additive constant for all queries. %
    While \query{4} extends to infinite time, the cost reported is for computing a lower bound up to $K$ steps. }
    \label{table:3_prob_query_examples}
\end{table}

\paragraph{Queries of Interest}
All of the queries posed  earlier in this section can be represented under the framework detailed in \cref{eq:3_structure_1} and \cref{eq:3_structure_2}, as illustrated in \cref{table:3_prob_query_examples}.

$\textbf{Q}\mathbf{1}$ \quad The query $\prob_\theta(X_1=x_1)$ for $x_1 \in \V$ can be represented with $\mathcal{Q}=\{x_1\}$ and is already naturally in a form that our model can directly estimate due to the autoregressive factorization imposed by the model: $p_{X_1}(x_1):=p_\theta(x_1)$. Furthermore, if we are interested in any potential continuing sequence $\mathbf{X}_{1:K}=\mathbf{x}_{1:K}$ this can easily be computed via $\prob_\theta(\mathbf{X}_{1:K}=\mathbf{x}_{1:K})=\prod_{k=1}^K p_\theta(x_k\sep \mathbf{x}_{<k})$ as discussed earlier.

$\textbf{Q}\mathbf{2}$ \quad The query $\prob_\theta(X_K=x_K)$ for some $x_K \in \V$ and $K>1$ can be represented with $\mathcal{Q}=\V^{K-1}\times \{a\}$. In order to evaluate $\prob_\theta(X_K=x_K)$, the previous terms $\textbf{X}_{<K}$ need to be marginalized out. This is naturally represented as
\begin{align}
\prob_\theta(X_K=x_K) & = \E^\prob_{\mathbf{X}_{<K}}\left[p_\theta(x_K\sep \mathbf{x}_{<K})\right] = \sum_{\mathbf{x}_{<K}\in \V^{K-1}} p_\theta(\mathbf{x}_{1:K}). 
\end{align}
It is helpful to show both the exact summation form as well as the expected value representation as both will be useful in \cref{sec:3_methods}.

$\textbf{Q}\mathbf{3}$ \quad The probability of the next instance of $a \in \V$ occurring at some point in time $K\geq 1$, $\prob_\theta(\hit(a)=K)$ where $\hit(\cdot)$ is the \emph{hitting time}, can be represented as $\mathcal{Q}=(\V\setminus\{a\})^{K-1}\times\{a\}$. Evaluating the distribution of the hitting time with our model can be computed as
\begin{align}
\prob_\theta(\hit(a)=K) & = \prob_\theta(X_K=a, \textbf{X}_{<K} \neq a) \\
& = \sum_{\textbf{x}_{<K}\in (\V\setminus\{a\})^{K-1}} p_\theta(X_K=a, X_{<K}=x_{<K}). \label{eq:3_hit_full}
\end{align}
The value $a\in\V$ can be easily replaced with a set of values $A\subset \V$ in these representations:
\begin{align}
\prob_\theta(\hit(A)=k) & = \prob_\theta(X_K\in A, \textbf{X}_{<K} \in (\V\setminus A)^{K-1}) \\
& = \sum_{\textbf{x}_{<K}\in (\V\setminus A)^{K-1}} \sum_{a\in A} p_\theta(X_K=a, \textbf{X}_{<K}=\textbf{x}_{<K}),
\end{align}
with $\mathcal{Q}=(\V\setminus A)^{K-1}\times A$.

$\textbf{Q}\mathbf{4}$ \quad The probability of $a\in\V$ occurring before $b\in\V$, $\prob_\theta(\hit(a)<\hit(b))$, is represented as $\mathcal{Q}=\cup_{i=1}^\infty\mathcal{Q}^{(i)}$ where $\mathcal{Q}^{(i)}=(\V\setminus\{a,b\})^{i-1}\times\{a\}$. 
To evaluate $\prob_\theta(\hit(a) < \hit(b))$, we must consider the possible instances where this condition is fulfilled. In doing so, we end up evaluating multiple queries similar in form to $\textbf{Q3}$, in the following manner:
\begin{align}
\prob_\theta(\hit(a) < \hit(b) ) & = \sum_{k=1}^\infty \prob_\theta(\hit(a)=k, \hit(b)>k)  \\
& = \sum_{k=1}^\infty \prob_\theta(X_k=a, \textbf{X}_{<k}\in (\V\setminus\{a,b\})^{k-1}) \label{eq:3_hit_comp_sum}
\end{align}
In practice, computing this expression exactly is intractable due to it being an infinite sum. There are two potential approaches one could take to address this. The first of which is to ask a slightly different query:
\begin{align}
    \prob_\theta(\hit(a) < \hit(b)\sep \hit(a) \leq K) & = \frac{\prob_\theta(\hit(a) < \hit(b), \hit(a) \leq K)}{\prob_\theta(\hit(a) \leq K)} \\
    & = \frac{\sum_{k=1}^{K} \prob_\theta(X_k=a, \textbf{X}_{<k}\in (\V\setminus \{a,b\})^{k-1})}{\sum_{k=1}^{K} \prob_\theta(\hit(a)=k)} \\
    & = \frac{\sum_{k=1}^{K} \sum_{\textbf{x}_{<k}\in (\V\setminus \{a,b\})^{k-1}} p_\theta(\textbf{x}_{1:k})}{\sum_{k=1}^{K} \prob_\theta(\hit(a)=k)},
\end{align}
where $x_k=a$ and the denominator can be decomposed as shown in \cref{eq:3_hit_full}.

The other option is to produce a lower bound on this expression by evaluating the sum in \cref{eq:3_hit_comp_sum} for the first $K$ terms. We can achieve error bounds on this estimate by noting that $\prob_\theta(\hit(a) < \hit(b) ) + \prob_\theta(\hit(a) > \hit(b)) = 1$. As such, if we evaluate \cref{eq:3_hit_comp_sum} up to $K$ terms for both $\prob_\theta(\hit(a) < \hit(b))$ and $\prob_\theta(\hit(b) < \hit(a) )$, the difference between the sums will be the maximum error either lower bound can have.

Similar to $\textbf{Q3}$, we can also ask this query with sets $A, B\subset \V$ instead of values $a,b$, so long as $A \cap B = \emptyset$: 
\begin{align}
\prob_\theta(\hit(A) < \hit(B) ) & = \sum_{k=1}^\infty \prob_\theta(\hit(A)=k, \hit(B)>k) \\
& = \sum_{k=1}^\infty \sum_{a\in A} \prob_\theta(X_k=a, \textbf{X}_{<k}\in (\V\setminus(A\cup B))^{k-1}).
\end{align}

$\textbf{Q}\mathbf{5}$ \quad The probability of $a\in\V$ occurring $n$ times in the next $K$ steps, $\prob_\theta(N_a(K)=n)$, is represented as $\mathcal{Q}=\cup_{i=1}^{C(K,n)} \mathcal{Q}^{(i)}$, where  $N_a(K)$ is a random variable for the number of occurrences of events of type $a$ from steps $1$ to $K$ and 
$\mathcal{Q}^{(i)}$'s are defined to cover all unique permutations of orders of products composed of: $\{a\}^n$ and $(\V\setminus\{a\})^{K-n}$.

Evaluating $\prob_\theta(N_a(K)=n)$ also involves decomposing this value of interest into statements involving hitting times. Let $\hit^{(l)}(a)$ be the $l^\text{th}$ hitting time of a specific event of interest $a$. In other words, the time of the $l^\text{th}$ occurrence of $a$. Assuming $n \leq K$:
\begin{align}
&\prob_\theta(N_a(K)=n)\\
& = \sum_{i_1<\dots<i_n\leq K} \!\!\prob_\theta(\hit^{(1)}(a)=i_1, \dots, \hit^{(n)}(a)=i_n, \hit^{(n+1)}(a)>K) \\
& = \sum_{i_1<\dots<i_n\leq K}\!\!\prob_\theta(\{X_j=a \sep j = i_1, \dots, i_n\}, \{X_j \neq a \sep j \in \{1, \dots, K\} \setminus  \{i_1, \dots, i_n\})
\end{align}
For brevity, we will not decompose this further into tractable model calls $p_\theta(x_i \sep \textbf{x}_{<i})$; however, it should be clear that this falls into the typical query decomposition due to term-wise vocabulary restrictions. Additionally, similar to earlier query types this can easily be extended for $A\subset\V$.

\paragraph{Query Complexity}
From \cref{eq:3_exact_computations_cond}, exact computation of a query involves computing $K\cdot|\mathcal{Q}|$ conditional distributions  (e.g., $p_\theta(x_k|\textbf{x}_{<k})$) in an autoregressive manner. 
Under the structured representation, the number of conditional distributions needed is equivalently $\sum_i \prod_{k=1}^K |\mathcal{V}_k^{(i)}|$.
Non-attention based neural sequence models often define $p_\theta(x_k|\textbf{x}_{<k}) := f_\theta(h_k)$ where $h_k = \text{RNN}_\theta(h_{k-1}, x_{k-1})$ is the result of a recurrent neural network (RNN) cell parameterized by $\theta$ that takes as input the previous hidden state $h_{k-1}$ and event $x_{k-1}$. As such, the computation complexity for any individual conditional distribution remains constant with respect to sequence length. We will refer to the complexity of this atomic action as being $\bigO(1)$. Naturally, the actual complexity depends on the model architecture and has a multiplicative scaling on the cost of computing a query. The number of atomic operations needed to exactly compute $Q1$-$Q5$ for this class of models can be found in \cref{table:3_prob_query_examples}. Should $p_\theta$ be an attention-based model (e.g., a transformer \citep{vaswani2017attention}) then the time complexity of computing a single one-step-ahead distribution becomes $\bigO(K)$, further exacerbating the \textbf{exponential growth} of many queries.
Note that some particular parametric forms of $p_\theta$ admit more efficient query computation. The next section will explore one specific setting of this with Markov models.

\section{Querying Markov Models}

Certain models $p_\theta$ allow for computing various queries directly as a function of the parameters and specifics of the query of interest. One particular class that allows for this are \textit{Markov models}. Markov models of the $m^\text{th}$ order are characterized by only depending on the $m$ previous terms in a sequence to determine the distribution for the immediate next step. We will first derive more efficient representations for some of the specifically introduced query classes for a $1^\text{st}$ order homogeneous Markov model, and then will generalize to generic queries with higher order models.

\subsection{Queries for Homogeneous First Order Markov Models}
Let $p_\theta$ be a first order ergodic homogeneous Markov model characterized by a transition matrix $(\pi_{i,j})_{i,j\in\V}$ where $p_\theta(x_k \sep \textbf{x}_{<k}) := \pi_{x_{k-1}, x_{k}}$ for all $k\in\mathbb{N}$ with steady state probabilities $(\pi_i)_{i\in\V}$.\footnote{There are various ways that Markov models handle beginnings of sequences where there are fewer events to condition on than the model's order. A common way of accommodating this is to append the vocabulary $\V$ with a special beginning of sequence value and enforce that $x_i$ takes on this value for $i < 1$. Our derivations in this section will always assume that we are contextualizing our query with enough information to not have to worry about this edge case.} Using basic results from the theory of finite Markov chains, e.g., see \citep{kemeny1983finite}, we analyze below the complexity of three of the previously introduced queries conditioned on the current observation $X_0=x_0$ for $x_0\in\V$.

$\textbf{Q}\mathbf{2}$ \quad The $k^\text{th}$ marginal query, $\prob_\theta(X_k = x_k \sep X_0 = x_0)$, can be computed recursively by computing the conditional distribution $p_{X_1\sep X_0}(x_1 \sep x_0) = \pi_{x_0, x_1}$, then $p_{X_2 \sep X_0}(x_2\sep x_0) = \sum_{x_1\in\V} \pi_{x_1,x_2}\pi_{x_0,x_1}$, and so on, resulting in $k$ matrix multiplications, with complexity $\bigO(kV^2)$.

$\textbf{Q}\mathbf{3}$ \quad The hitting time of $a \in \V$ at time $k$ query, $\prob_\theta(\hit(a)=k \sep X_0=x_0)$, can be computed for all values $k' \leq k$ using $k-1$ matrix multiplications where, at each step $k'=2,\dots, k-1$, all events except $a$ are marginalized over. This results in a complexity of $\bigO(kV^2)$.

In general it is straightforward to show that 
\begin{align}   
\prob_\theta(\hit(a) = k \sep X_0=x_0) = \begin{cases}
    \pi_{x_0,a} & \text{ for } k=1 \\
    \pi_{x_0,\cancel{a}}\pi_{\cancel{a}, a}\pi_{\cancel{a}, \cancel{a}}^{k-2} & \text{ for } k > 1
\end{cases}
\end{align}
where $\pi_{x_0, \cancel{a}}=\sum_{i\neq a} \pi_{x_0,i}$ is the probability of transitioning to a state other than $a$ from $x_0$, $\pi_{\cancel{a}, \cancel{a}}=\sum_{i\neq a} \pi_{\cancel{a}, i}$ is the probability of remaining not at state $a$, and lastly $\pi_{\cancel{a}, a}$ is the probability of transitioning to state $a$ from not $a$. Computation of $\pi_{\cancel{a}, a}$ requires knowledge of the steady-state probabilities of the chain $\pi_a$ and $\pi_{\cancel{a}} = \sum_{i \ne a} \pi_{i}$. Computing the steady-state probabilities requires inversion of a $V \times V$ matrix, with a complexity of $O(V^3)$. This general solution will be faster to compute than the version using matrix multiplication (up to horizon $k$) whenever $k > V$.

$\textbf{Q}\mathbf{4}$ \quad The probability of $a\in\V$ occurring before $b\in\V$, $\prob_\theta(\hit(a) < \hit(b) \sep X_0=x_0)$, can be computed exactly for a $1^\text{st}$ order Markov model. Let $\cancel{ab}$ represent the state of not being $a$ or $b$, i.e., some element of $\V\setminus\{a,b\}$.

Let $c$ be the set of all states (events) in $\V$ except for $a$ and $b$. It is straightforward to show that
\begin{align}
\prob_\theta(\hit(a) < \hit(b) \sep X_0=x_0) & = \pi_{x_0,a} + \pi_{x_0,\cancel{ab}} \frac{ \pi_{\cancel{ab}, a} }{\pi_{\cancel{ab}, a} + \pi_{\cancel{ab}, a}}\\
\text{and } \prob_\theta(\hit(b) < \hit(a) \sep X_0=x_0) & = \pi_{x_0,b} + \pi_{x_0,\cancel{ab}} \frac{ \pi_{\cancel{ab}, b} }{\pi_{\cancel{ab}, a} + \pi_{\cancel{ab}, a}}
\end{align}
where $\pi_{\cancel{ab}, a}$ and $\pi_{\cancel{ab}, b}$ are the probabilities of transitioning to $a$ and $b$, respectively, given that the system is currently in a state that is neither $a$ or $b$. Computing these probabilities again requires knowledge of the steady-state probabilities $\pi_a, \pi_b, \pi_{\cancel{ab}}$, resulting in $O(V^3)$ time complexity.

\subsection{General Queries and Higher Order Markov Models}
For simplicity, we will assume that a query takes the form $\mathcal{Q}=\mathcal{V}_1\times\dots\times\mathcal{V}_K$ and we are interested in $\prob_\theta(X_{1:K}\in\mathcal{Q} \sep \hist_0)$ for a $m^\text{th}$-order Markov model $p_\theta$. This model can be defined with an $(m+1)$-dimensional tensor $\Pi\in\mathbb{R}^{V\times\dots\times V}$ with elements $\pi_{i_1,\dots,i_m,i_{m+1}}$ such that $\sum_{j=1}^V \pi_{i_1,\dots,i_m,j} = 1$ for all $i_1,\dots,i_m\in \V$. Alternatively, $\pi_{i_1,\dots,i_m,i_{m+1}}=\prob_\theta(X_{j+m+1}=i_{m+1} \sep \textbf{X}_{j+1:j+m}=\textbf{i}_{1:m})$ for $j\geq 0$.

To marginalize out $X_{m+1}$ and compute the conditional distribution $\prob_\theta(X_{m+2}=x_{m+2} \sep \textbf{X}_{1:m}=\textbf{x}_{1:m})$  requires the following computation:
\begin{align}
& \prob_\theta(X_{m+2}=x_{m+2} \sep \textbf{X}_{1:m}=\textbf{x}_{1:m}) = \sum_{v\in\V} \prob_\theta(X_{m+2}=a,X_{m+1}=v\sep\textbf{X}_{1:m}=\textbf{x}_{1:m}) \\
& = \sum_{v\in\V} p_\theta(X_{m+2}=x_{m+2}\sep X_{m+1}=v,\textbf{X}_{2:m}=\textbf{x}_{2:m})p_\theta(X_{m+1}=v\sep \textbf{X}_{1:m}=\textbf{x}_{1:m}) \\
& = \sum_{v\in\V} \pi_{x_2,\dots,x_m,v,a}\pi_{x_1,\dots,x_{m-1},x_m,v} 
\end{align}
If we perform this operation over all values of $x_{1:m}$, we can construct a new transition tensor representing $\prob_\theta(X_{m+2}\sep X_{1:m})$ in general. We will denote this new tensor as being equal to $\Pi \otimes \Pi$ where $\left(\Pi \otimes \Pi\right)_{i_1,\dots,i_m,v}=p(X_{j+m+2}=v\sep \textbf{X}_{j+1:j+m}=\textbf{i}_{1:m})$. This operation has a computation complexity of $\mathcal{O}(V^{m+1})$.

This special product can be carried out repeatedly to marginalize out further into the future. For instance, performing this operation on $\Pi$ $(n-1)$-times results in $(\Pi\otimes\dots\otimes\Pi)_{i_1,\dots,i_m,v}=p(X_{j+m+n}=v\sep \textbf{X}_{j+1:j+m}=\textbf{i}_{1:m})$, thus marginalizing out $n-1$ terms: $\textbf{X}_{j+m+1:j+m+n-1}$.

Transitioning into a restricted space $\mathcal{V}$ can be done easily by defining a restricted transition tensor $\Pi_{\mathcal{V}}$ such that $\left(\Pi_{\mathcal{V}}\right)_{i_1,\dots,i_m,v} \propto \pi_{i_1,\dots,i_m,v}$ if $v \in \V$, otherwise $\left(\Pi_{\mathcal{V}}\right)_{i_1,\dots,i_m,v}=0$ for all $i_1,\dots,i_m,v \in \V$. 

With this, we have everything we need to compute $p(X_{1:K}\in\mathcal{Q}|\hist)$. If the last $m$-values of the history $\hist$ are equal to $i_1,\dots,i_m$, then:
\begin{align}
\prob_\theta(\textbf{X}_{1:K}\in\mathcal{Q}\sep \hist_0) &= \prod_{k=1}^K \prob_\theta(X_k\in\mathcal{V}_k\sep \textbf{X}_{<k}\in\mathcal{V}_1\times\dots\times\mathcal{V}_{k-1})\\
&= \prod_{k=1}^K\sum_{v\in\mathcal{V}_k} \prob_\theta(X_k=v|\textbf{X}_{<k}\in\mathcal{V}_1\times\dots\times\mathcal{V}_{k-1})\\
& = \prod_{k=1}^K \sum_{v\in\mathcal{V}_k} \left(\Pi_{\mathcal{V}_1}\otimes\dots\otimes\Pi_{\mathcal{V}_k}\right)_{i_1,\dots,i_m,v}. \label{eq:3_markov_gen}
\end{align}
While this result is derived for homogeneous Markov models, a similar form exists for non-homogeneous variants as well. Assuming that the time-dependent transition matrices are known ahead of time, and denoted as $\pi^{(i)}$ for describing the transition to state $X_i$, it follows then that $\prob_\theta(\textbf{X}_{1:K}\in\mathcal{Q}\sep \hist_0) = \prod_{k=1}^K \sum_{v\in\mathcal{V}_k} \left(\Pi_{\mathcal{V}_1}^{(1)}\otimes\dots\otimes\Pi_{\mathcal{V}_k}^{(k)}\right)_{i_1,\dots,i_m,v}$.

The dominant factor in the computational complexity of \cref{eq:3_markov_gen} is computing all of the special products of $\Pi$.
As such, this has a total computational complexity of $\mathcal{O}((k-1)V^{m+1})$. The key difference here is that once the computation is done for a query of length $k$, then it takes very little additional computation to evaluate similar queries on different contexts $\hist_0$ as the tensor products can be reused. For comparison, recall that for general autoregressive models the complexity of analytically solving queries in the same form is $\mathcal{O}(V^k)$ (assuming the parametric form of the model affords no simplifications).

The next section will discuss various methods for estimating probabilistic queries in the setting where analytically computing the answer is not tractable.

\section{Query Estimation Methods}\label{sec:3_methods}

Since exact query computation can scale exponentially in $K$ for generic autoregressive models $p_\theta$,
it is natural to consider approximation methods. In particular we focus on importance sampling, beam search, and a hybrid of both methods. All methods will be based on a novel proposal distribution, discussed below. 

\subsection{Proposal Distribution}
\label{sec:3_proposal}
For various estimation methods which will be discussed later, it is beneficial to have a tractable proposal distribution $q$ over $\textbf{X}$ that we can sample from with an implied measure $\q$ whose support matches that of the query $\mathcal{Q}$. For importance sampling, we will need this distribution as a proposal distribution. We will also use it as our base model for selecting high-probability sequences in beam search. 
We would like the proposal distribution to resemble our original model while also respecting the query. 
One thought is to have 
$\q(\textbf{X}_{1:K}=\textbf{x}_{1:K}) = \prob_\theta(\textbf{X}_{1:K}=\textbf{x}_{1:K} \sep \textbf{X}_{1:K}\in\mathcal{Q})$; however, computing 
this probability involves normalizing over  $\prob_\theta(\textbf{X}_{1:K}\in\mathcal{Q})$ which is exactly what we are trying to estimate in the first place. 
Instead of restricting the \emph{joint} distribution to the query, we can instead restrict every \emph{conditional} distribution to the query's restricted domain. To see this, we 
first partition $\mathcal{Q}=\cup_i\mathcal{Q}^{(i)}$ and define an autoregressive proposal distribution for each $\mathcal{Q}^{(i)}=\prod_{k=1}^K \mathcal{V}^{(i)}_k$ as follows:
\begin{align}
\q^{(i)}(\textbf{X}_{1:K}=\textbf{x}_{1:K})& := \prod_{k=1}^K q^{(i)}(x_k \sep \textbf{x}_{<k})\\
q^{(i)}(x_k \sep \textbf{x}_{<k}) & := \prob_\theta\left(X_k=x_k\sep \textbf{X}_{<k}=\textbf{x}_{<k},X_k\in\mathcal{V}_k^{(i)}\right) \label{eq:3_proposal_dist} \\    
 & = \frac{p_\theta\left(x_k\sep \textbf{x}_{<k}\right)\ind\left(x_k\in\mathcal{V}_k^{(i)}\right)}{\sum_{v\in\mathcal{V}_k^{(i)}} p_\theta(v \sep \textbf{x}_{<k})}
\end{align}
where $\ind(\cdot)$ is the indicator function. That is, we constrain the outcomes of each conditional probability to the restricted domains $\mathcal{V}_k^{(i)}$ and renormalize them accordingly. To evaluate the proposal distribution's probability, we multiply all conditional probabilities according to the chain rule. 
Since the entire distribution is computed for a single model call $p_\theta(-\sep \textbf{x}_{<k})$, it is possible to both sample a $K$-length sequence and compute its likelihood under $q^{(i)}$ with only $K$ model calls. Thus, we can efficiently sample sequences from a distribution that is both informed by the underlying model $p_\theta$ and that respects the given domain $\mathcal{Q}$. As discussed in the next section, this proposal will be used for importance sampling and for the base distribution on which beam search is conducted.

\subsection{Estimation Techniques}\label{sec:3_estimation_techniques}

\subsubsection{Sampling}
One can naively sample any arbitrary probability value using Monte Carlo samples to estimate $\prob_\theta(\textbf{X}_{1:K}\in\mathcal{Q}) = \E^{\prob_\theta}[\ind(\textbf{X}_{1:K}\in\mathcal{Q})]$; however, this typically will have high variance. This can be substantially improved upon by exclusively drawing sequences from the query space $\mathcal{Q}$ through a proposal distribution using importance sampling.
For brevity, assume that $\mathcal{Q}:=\mathcal{V}_1 \times \cdots \times \mathcal{V}_K$.\footnote{Should the query space be the more general form $\mathcal{Q}:=\cup_i \mathcal{Q}^{(i)}$, then simply apply the methods derived on each individual sub-query, $\prob_\theta(\textbf{X} \in \mathcal{Q}^{(i)})$, then sum together to get the total probabilistic query $\prob_\theta(\textbf{X}\in\mathcal{Q})$.} Here, we derive the importance sampling equivalent form of the probabilistic query utilizing the proposal distribution $q$ for query space $\mathcal{Q}$ with measure $\q$ defined in \cref{eq:3_proposal_dist}:
\begin{align}
\prob_\theta(\textbf{X}_{1:K}\in\mathcal{Q}) & = \E^\q \left[\frac{p_\theta(\textbf{X}_{1:K})}{q(\textbf{X}_{1:K})} \ind(\textbf{X}_{1:K} \in \mathcal{Q})\right] \\
& = \E^\q \left[\prod_{k=1}^K \frac{p_\theta(X_k \sep \textbf{X}_{<k})}{q(X_k \sep \textbf{X}_{<k})}\right] \\
& = \E^\q \left[\prod_{k=1}^K \frac{p_\theta(X_k \sep \textbf{X}_{<k})\sum_{v \in \mathcal{V}_k} p_\theta(v \sep \textbf{X}_{<k})}{p_\theta(X_k \sep \textbf{X}_{<k})\ind(X_k \in \mathcal{V}_k)}\right] \\
& = \E^\q \left[\prod_{k=1}^K \sum_{v \in \mathcal{V}_k} p_\theta(v \sep \textbf{X}_{<k})\right] \\
& \approx \frac{1}{M} \sum_{m=1}^M \prod_{k=1}^K \sum_{v \in \mathcal{V}_k} p_\theta(v \sep \textbf{x}_{<k}^{(i)}),
\end{align}
where the indicators disappeared due to them equaling 1 almost-surely under $\q$ and $\textbf{x}_{1:k-1}^{(i)}$ for $i=1,\dots,M$ are concrete samples drawn \textit{iid} from $q$.\footnote{While sometimes in the literature $p(X)$ is used to designate the abstract notion of the ``distribution of $X$,'' in this context it is used as a literal function of $X$. This should be interpreted in the same way as $f$ is used in $\E[f(X)]$.}
It is worth noting that this estimator could be further improved by augmenting the sampling process to produce samples without replacement from $q$ %
(e.g., \citep{meister2021conditional,pmlr-v97-kool19a,shi2020incremental}); in this chapter we restrict the focus to sampling with replacement.

\subsubsection{Search}
An alternative to estimating a query by sampling is to instead produce a lower bound,
\begin{align}
\prob_\theta(\textbf{X}_{1:K}\in\mathcal{Q}) = \sum_{\textbf{x}_{1:K}\in\mathcal{Q}}p_\theta(\textbf{x}_{1:K}) \geq \sum_{\textbf{x}_{1:K}\in\mathcal{B}}p_\theta(\textbf{x}_{1:K}),
\end{align}
where $\mathcal{B} \subset \mathcal{Q}$.
In many situations, only a small subset of sequences $\textbf{x}_{1:K}$ in $\mathcal{Q}$ have a non-negligible probability of occurring due to the vastness of the total path space $V^K$ for large $V$. As such, it is possible for $|\mathcal{B}| \ll |\mathcal{Q}|$ while still having a minimal gap between the lower bound and the actual query value.

One way to produce a set $\mathcal{B}\subset\mathcal{Q}$ is through beam search \citep{russell2010artificial}.
To ensure that beam search only explores the query space, instead of searching with $p_\theta$, we utilize $q$ for ranking beams. Since beam search is a greedy algorithm and for a given conditional $q(a\sep \textbf{x}_{<k}) \propto p_\theta(a\sep \textbf{x}_{<k})$ for $a\in\mathcal{V}_k$, the rankings will both respect the domain and be otherwise identical to using $p_\theta$ to rank.
Typically, the goal of beam search is to find the most likely completion of a sequence without having to explore the entire space of possible sequences. This is accomplished by greedily selecting the top-$B$ most likely next step continuations, or \emph{beams}, at each step into the future. Rather than finding a few high-likelihood beams, we are more interested in accumulating a significant amount of probability mass and less interested in the specific quantity of beams collected.

Traditional beam search has a fixed beam size $B$; %
however, this is not ideal for accumulating probability mass. %
As an alternative we develop \emph{coverage-based} beam search where at each step in a sequence we restrict the set of beams being considered not to the top-$B$ but rather to the smallest set of beams that collectively exceed a predetermined probability mass $\alpha$,  referred to as the ``coverage''.\footnote{This is similar to the distinction between top-$K$ and top-$p$ / nucleus sampling commonly used for natural language generation \citep{holtzman2019curious}.} More specifically, let $\mathcal{B}_k \subset \{\textbf{x}_{1:k} \sep \textbf{x}_{1:K}\in\mathcal{Q}\}$ be a set containing $|\mathcal{B}_k|$ beams for subsequences of length $k$. For brevity, we will assume that $\mathcal{Q}=\mathcal{V}_1\times \cdots \times \mathcal{V}_K$.\footnote{If $\mathcal{Q}$ requires partitioning into multiple $\mathcal{Q}^{(i)}$'s, we apply beam search to each sub query $p_\theta(\textbf{X}_{1:K}\in\mathcal{Q}^{(i)})$.} $\mathcal{B}_{k+1}$ is a subset of $\mathcal{B}_k \times \mathcal{V}_{k+1}$ and is selected specifically to minimize $|\mathcal{B}_{k+1}|$ such that $\q(\textbf{X}_{1:k+1}\in\mathcal{B}_{k+1}) \geq \alpha$. It can be shown that $\prob_\theta(\textbf{X}_{1:K}\in\mathcal{Q})-\prob_\theta(\textbf{X}_{1:K}\in\mathcal{B}_K) \leq 1-\q(\textbf{X}_{1:K}\in\mathcal{B}_K)$ (and is often significantly less).

\begin{theorem}
For a given set of decoded sequences $\mathcal{B}\subset\mathcal{Q}$ with coverage $\q(\textbf{X}_{1:K}\in\mathcal{B})$, the error between the true probabilistic query value $\prob_\theta(\textbf{X}_{1:K}\in\mathcal{Q})$ and the lower bound $\prob_\theta(\textbf{X}_{1:K}\in\mathcal{B})$ is bounded above by the complement of the coverage: $1-\q(\textbf{\textbf{X}}_{1:K}\in\mathcal{B})$.
\end{theorem}

\begin{proof}
For brevity, we will assume that $\mathcal{Q}=\mathcal{V}_1\times\dots\times\mathcal{V}_K$ (although this can easily be extended to the general case). 
Since $q(x_k\sep \textbf{x}_{<k})=\frac{p_\theta(x_k\sep \textbf{x}_{<k})\ind(x_k\in\mathcal{V}_k)}{\prob_\theta(X_k\in\mathcal{V}_k \sep \textbf{X}_{<k}=\textbf{x}_{<k})}$ and $p_\theta(\cdot) \leq 1$, it follows that $q(x_k\sep \textbf{x}_{<k}) \geq p_\theta(x_k\sep \textbf{x}_{<k})$. 
This becomes a strict inequality should $\prob_\theta(X_k\in\mathcal{V}_k \sep \textbf{X}_{<k}=\textbf{x}_{<k}) < 1$ (which is often the case should $\mathcal{V}_k \subset \V$). 
Since this holds for arbitrary $k$ and $x_k$, it then follows that $q(\textbf{x}_{1:K}) \geq p_\theta(\textbf{x}_{1:K})$  for any $\textbf{x}_{1:K} \in \mathcal{Q}$.\footnote{This can be seen by comparing their autoregressive factorizations term by term.} 
This inequality becomes strict should any sequence not in the query set have non-zero probability, i.e., $p_\theta(\textbf{x}_{1:K})>0$ for any $\textbf{x}_{1:K}\in\V^K\setminus\mathcal{Q}$. 

The target value that we are estimating can be broken up into the following terms:
\begin{align}
\prob_\theta(\textbf{X}_{1:K}\in\mathcal{Q}) & = \prob_\theta(\textbf{X}_{1:K}\in\mathcal{B}) + \prob_\theta(\textbf{X}_{1:K}\in\mathcal{Q}\setminus\mathcal{B}).
\end{align}
Rearranging these terms yields us the error of our lower bound.
We can easily derive an upper bound on the error for an arbitrary set $\mathcal{B}$:
\begin{align}
 \prob_\theta(\textbf{X}_{1:K}\in\mathcal{Q}) - \prob_\theta(\textbf{X}_{1:K}\in\mathcal{B}) &= \prob_\theta(\textbf{X}_{1:K}\in\mathcal{Q}\setminus\mathcal{B}) \\
 & = \sum_{\textbf{x}_{1:K}\in\mathcal{Q}\setminus\mathcal{B}} p_\theta(\textbf{x}_{1:K}) \\
& \leq \sum_{\textbf{x}_{1:K}\in\mathcal{Q}\setminus\mathcal{B}} q(\textbf{x}_{1:K}) \\
& = \q(\textbf{X}_{1:K}\in\mathcal{Q}\setminus\mathcal{B}) \\
& = \q(\textbf{X}_{1:K}\in\mathcal{Q}) - \q(\textbf{X}_{1:K}\in\mathcal{B}) \\
& = 1-\q(\textbf{X}_{1:K}\in\mathcal{B}) \leq 1-\alpha
\end{align}
where $\alpha$ is the targeted coverage probability used to find $\mathcal{B}$ with coverage-based beam search.\footnote{Note that $\q(\textbf{X}_{1:K}\in\mathcal{Q})=1$ because by design $q(\textbf{x}_{1:K})=0$ for every $\textbf{x}_{1:K}\notin\mathcal{Q}$.}
This inequality becomes strict should any sequence not in the query set have non-zero probability, i.e., $p_\theta(\textbf{x}_{1:K})>0$ for any $\textbf{x}_{1:K}\in\V^K\setminus\mathcal{Q}$.
\end{proof}

There is one slight problem with having $\alpha$ be constant throughout the search. Since we are pruning based on the joint probability of the entire sequence, any further continuations of $\mathcal{B}_k$ will reduce the probability $\q(\textbf{X}_{1:k+1}\in\mathbb{B}_{k+1})$. 
This may lead to situations in which every possible candidate sequence is kept in order to maintain minimal joint probability coverage.
This can be avoided by filtering by $\alpha_k$ where $\alpha_1 > \dots > \alpha_K=\alpha$, e.g., the geometric series $\alpha_k=\alpha^{k/K}$.

\subsubsection{A Hybrid Approach}\label{sec:3_hybrid_details}
Importance sampling produces an unbiased estimate, but can still experience large variance in spite of a good proposal distribution $q$ when $p_\theta$ is a heavy tailed distribution. Conversely, the beam search lower bound can be seen as a biased estimate with zero variance. We can remedy the limitations of both methods by recognizing that since $\prob_\theta(\textbf{X}_{1:K}\in\mathcal{Q}) = \sum_{\textbf{x}_{1:K}\in\mathcal{B}_K}p_\theta(\textbf{x}_{1:K}) +  \sum_{\textbf{x}_{1:K}\in\mathcal{Q}\setminus\mathcal{B}_K}p_\theta(\textbf{x}_{1:K})$, where $\mathcal{B}_K$ is the set of sequences resulting from beam search, we can use importance sampling on the latter summation. The only caveat to this is that the proposal distribution must match the same domain of the summation: $\mathcal{Q}\setminus\mathcal{B}_K$. To keep all of the nice properties of our original proposal distribution described previously, we will use $q_\mathcal{B}(\textbf{x}_{1:K}):=\prod_{k=1}^K q_\mathcal{B}(x_k \sep \textbf{x}_{<k})$ with implied measure $\q_\mathcal{B}$ where $\q_\mathcal{B}(\textbf{X}_{1:K}=\textbf{x}_{1:K}) := \q(\textbf{X}_{1:K}=\textbf{x}_{1:K} \sep \textbf{X}_{1:K}\notin\mathcal{B}_K)$ as our proposal distribution for the hybrid setting. While this does require marginalization over more than just the immediate next step, this can be accomplished through a reuse of computation done during the initial beam search phase. The steps to do so are described below. For a more in depth analysis into the resulting estimator variance for the hybrid approach and how it compares to regular importance sampling, please refer to \cref{sec:3_app_hybrid_variance}.

\paragraph{Viewing $\V^K$ and $p_\theta$ as  Trees}
In the space of all possible sequences of length $K$, $\textbf{X}_{1:K}\in\V^K$, one can represent these sequences as paths in a tree. Each node in this tree represents a single element in a sequence $X_k\in\V$ with depth $k$, with parent and children nodes representing previous and potential future values in the sequence respectively. The root node either represents the very beginning of a sequence, or a concrete history $\hist$ to condition on. See \cref{fig:3_query_example} for an example visualization.

This tree can be augmented into a probabilistic one by defining edges between nodes as the conditional probability of a child node being next in a sequence, conditioned on all ancestors of that child. These probabilities are naturally determined by $p_\theta(x_k\sep\textbf{x}_{<k})$ where $x_k$ is the child node value and $\textbf{x}_{<k}$ are the ancestors' values. 

\paragraph{Building the Tree}
Any subset of $\V^K$ can also be represented as a tree, and in fact will be a sub-tree of the one that represents $\V^K$. As such, there exists a tree that represents $\mathcal{Q}$. Our usual proposal distribution $q$ is a natural source of conditional probabilities for the edges. While none of these trees with their edge weights are fully known ahead of time, we do explore and uncover them through the process of beam search. As such, during the beam search phase of the hybrid method we keep track of any conditional distributions $q(x_k\sep \textbf{x}_{<k})$ that are computed and use them to construct a partial view of the tree for $\mathcal{Q}$. Note that the end result of this process is a tree that will likely have many paths that do not fully reach depth $K$; however, there will be at least $|\mathcal{B}|$ many that do.

For our purposes, it is also useful to keep track of $p_\theta(x_k \sep \textbf{x}_{<k})$ over this restricted set, as well as model byproducts such as hidden states to reduce computation redundancy later. 

\paragraph{Pruning the Tree}
After beam search has completed, we are left with a resulting set of beams $\mathcal{B}$ and a partial tree with weights corresponding to conditional probabilities from $q$. We would now like to alter this tree such that its weights represent $q_\mathcal{B}(X_{1:K}=x_{1:K}):=\q(\textbf{X}_{1:K}=\textbf{x}_{1:K}\sep \textbf{X}_{1:K}\notin\mathcal{B}_K)$. This alteration can be accomplished by adjusting the edge weights in the tree recursively as detailed below. New weight assignments will be denoted by $q_\mathcal{B}$ to differentiate from old weights $q$. The steps to the procedure are defined as follows:
\begin{enumerate}[leftmargin=20pt]
\itemsep0em 
\item At the final depth $K$, assign edge weights $q_\mathcal{B}(x_K\sep \textbf{x}_{<K})=0$ for all $\textbf{x}_{1:K}\in\mathcal{B}_K$. All other edge weights in the final depth will have new weights $q_\mathcal{B}(x_K \sep \textbf{x}_{<K})=q(x_K\sep\textbf{x}_{<K})$. 
\item At the next layer above with depth $k=K-1$, assign edge weights as $q_\mathcal{B}(x_k\sep \textbf{x}_{<k})=q(x_k\sep \textbf{x}_{<k})\sum_{v\in\V}q_\mathcal{B}(v\sep \textbf{x}_{\leq k}=x_{\leq k})$ for all sub-sequences $x_{1:k}\in\mathcal{B}_k$.  
\item Repeat step 2 iteratively for $k=K-2, K-3, \dots, 2, 1$. 
\item Finally, normalize every conditional distribution for every node whose children edges were altered such that they each sum to 1.
\end{enumerate}
After these steps are completed, $q_\mathcal{B}(\textbf{x}_{1:K})=\prod_{k=1}^K q_\mathcal{B}(x_k\sep \textbf{x}_{<k})$. Note that weights related to sequences that were not discovered during beam search, and are thus not in the tree, are not altered and still match the original proposal distribution $q$. As such, to sample sequences from the tree, we start at the root node and sample from each successive conditional distribution until either depth $K$ or a leaf node at depth $k<K$ is reached. In the former scenario, the sampling is complete. In the latter, the remaining values of the sequence are sampled from $q(\cdot \sep \textbf{x}_{\leq k})$ like usual.

\paragraph{Beam Search Heuristic}
Lastly, since our ultimate goal is to sample from the long tail of $p_\theta$, targeting a specific coverage $\alpha$ during beam search is no longer effective since achieving meaningfully large coverage bounds for non-trivial path spaces is generally intractable. Instead, we propose \emph{tail-splitting} beam search to better match our goals. Let $w_k^{(i)}=p_\theta(\textbf{x}^{(i)}_{1:k+1})$ for $\textbf{x}_{1:k+1}^{(i)}\in\mathcal{B}_k\times\mathcal{V}_{k+1}$ such that $w_k^{(i)} \geq w_k^{(j)}$ if $i < j$. In this regime, $\mathcal{B}_{k+1}=\{\textbf{x}^{(i)}_{1:k+1}\}_{i=1}^B$ where $B=\arg\min_{b} \sigma(\textbf{w}_k^{(1:b)})+\sigma(\textbf{w}_k^{(b+1:|\mathcal{B}_k\times\mathcal{V}_{k+1}|)})$ and $\sigma(\textbf{w}_k^{(u:v)})$ is the empirical variance of $w_k^{(i)}$ for $i=u,\dots,v$. This can be seen as performing 2-means clustering on the $w_k$'s and taking the cluster with the higher cumulative probability. 

\subsection{Saving Computation on Multiple Queries}
Should multiple queries need to be performed, such as $p_\theta(\hit(a)=k)$ for multiple values of $k$, then there is potential to be more efficient in computing estimates for all of them. The feasibility of reusing intermediate computations is dependent on the set of queries considered. %
For simplicity, we will consider two base queries $\mathcal{Q}=\mathcal{V}_1\times\dots\times\mathcal{V}_K$ and $\mathcal{Q}'=\mathcal{V}_1'\times\dots\times\mathcal{V}_{K'}'$ where $K < K'$. Due to the autoregressive nature of $p_\theta$, if $\mathcal{V}_i=\mathcal{V}_i'$ for $i=1,\dots,K-1$ then all of the distributions and sequences needed for estimating $\prob_\theta(\textbf{X}_{1:K}\in\mathcal{Q})$ are guaranteed to be intermediate results found when estimating $\prob_\theta(\textbf{X}_{1:K'}\in\mathcal{Q}')$. To be explicit, when estimating the latter query with beam search the intermediate $\mathcal{B}_{K}$ is the same as what would be directly computed for the former query. Likewise, for importance sampling if $\textbf{X}_{1:K'} \sim q$ using \cref{eq:3_proposal_dist} over $\mathcal{Q}'$ then the subsequence $\textbf{X}_{1:K}$ is also distributed under $q$ for query space $\mathcal{Q}$. This does not apply when the sample path domain is further restricted, such as with the hybrid approach, in which case we cannot  directly use intermediate results to compute other queries for ``free.''

\section{Experiments and Results}
\label{sec:3_experiments}

\subsection{Experimental Setting} 
We investigate the quality of estimates of hitting time queries across various datasets, comparing beam search, importance sampling, and the hybrid method. We find that hybrid systematically outperforms both pure search and sampling given a comparable computation budget across queries and datasets. We also investigate  the dependence of all three methods on the model entropy. 

It is worth noting that we focus almost exclusively on hitting time queries in our primary experiments, as more complex queries often decompose into operations over individual hitting times as demonstrated previously in \cref{sec:3_queries}.

\subsubsection{Datasets}
We evaluate our query estimation methods on three user behavior and two language datasets. We provide details on the preparation and utilization of each below. For all datasets, users are associated with anonymous aliases to remove personally identifiable information (PII).

\paragraph{Reviews} \citep{amazon-rev-ni-etal-2019-justifying} contains sequences of 233 million timestamped Amazon product reviews spanning from May 1996 to October 2018, with each product belonging to one of 30 product categories.  We restrict our consideration of this dataset to reviews generated by users with at least 15 product reviews and products with a defined category, retaining 63 million reviews on which the model was trained. The product category of a review is taken to be the event value $X$, and after filtering to the restricted set of users there are $V=29$ possible categories to model. This dataset is publicly available under the Amazon.com Conditions of Use License. 

\paragraph{Mobile Apps}%
\citep{app_dataset} consists of 3.6 million app usage records  from 200 Android users from September 2017 to May 2018, where each event is an interaction of an individual with an application. User behavior spans $V=88$ unique applications, which we use as the vocabulary for events for our experiments. This dataset is released under the Creative Commons License, and all users contain at least 15 mobile app interactions so no data was removed before training. 

\paragraph{MOOCs}%
\citep{kumar2019predicting} is a dataset of sequences of anonymized user interactions with online course materials from a set of massive open online courses (MOOCs). In total, the dataset includes $V=98$ unique types of interactions which we take the be the event values. Data from users with fewer than 15 interaction events are not considered, resulting in a dataset of 72\% of users and 93\% of the events (350,000 interactions) of the original dataset. The MOOCs dataset is available under the MIT License.

 \paragraph{Shakespeare}%
 We also examine character-level language models, using the complete works of William Shakespeare \citep{shakespeare_data}, comprising 125,000 lines of text and $V=67$ unique characters and released under the Project Gutenberg License. The specific characters in a sequence we what we model. 

 \paragraph{WikiText}%
 The WikiText-v2 dataset \citep{wikitext} includes word-level language data from "verified Good" and featured articles of Wikipedia. The tokenization scheme used for GPT-2 enforces a vocabulary size of $V=50257$ for this dataset, each representing a word or word-piece. All sentences are provided in English and the dataset is available under the Creative Commons Attribution-ShareAlike License.

\subsubsection{Base Models}
For all datasets except WikiText (word-level language), we trained a 2-layer long short-term memory (LSTM) network with a dropout rate of 0.3 and the ReLU activation function, e.g., $\text{ReLU}(z):=\max\{0,z\}$ \citep{hochreiter1997long,srivastava2014dropout}.
Each network was trained against cross entropy loss with the Adaptive Moments (Adam) optimizer initialized with a learning rate of 0.001. A constant learning rate decay schedule and 0.01 warm-up iteration percentage was also used. All LSTM models use a hidden state size of 512 except the model for Shakespeare, which uses a smaller hidden state size of 128 which is more appropriate give the size of the dataset. Model checkpoints were collected for all models every 2 epochs, and the checkpoint with the highest validation accuracy was selected to be used for query estimation experiments. All models were trained on NVIDIA GeForce 2080ti GPUs.

For WikiText-v2, we leveraged the GPT-2 \citep{gpt2-radford} medium (350 million parameters) architecture from \href{https://huggingface.co/}{HuggingFace} with pre-trained weights provided by OpenAI \citep{hugging-face-wolf-etal-2020-transformers}. The WikiText-v2 dataset was preprocessed using the tokenization scheme provided by HuggingFace for GPT-2, assigning numeric token indices to work pieces. No finetuning of GPT-2 is conducted.

\subsubsection{Experimental Methods} 
We investigate computation-accuracy trade-offs between 3 estimation methods (beam search, importance sampling, and the hybrid) across all datasets. Query histories $\mathcal{H}$ are defined by randomly sampling $N=1000$ sequences from the test split for all datasets except WikiText, from which we sample only $N=100$ sequences due to computational limitations. For each query history and method, we compute the hitting time query estimate $\prob_\theta(\hit(a)=K)$ over $K=3, \ldots , 11$, with $a$ determined by the $K^{th}$ symbol of the ground truth sequence. 

To ensure an even comparison of query estimators, we fix the computation budget per query in terms of model calls $p_\theta(-\sep \textbf{x}_{<k})$ to be equal across all 3 methods, repeating experiments for different budget magnitudes roughly corresponding to $O(10), O(10^2), O(10^3)$ model calls. This can be controlled directly for all of the methods except the hybrid approach. For this reason, in our experiments we typically use fix the number of samples $S$ for the importance sampling component of the hybrid method, where $S$ is the number of samples used by the hybrid method after conducting tail-splitting beam search. We then use the resulting number of total model calls used by the hybrid method (including both beam search and sampling with $S$ samples) to determine a fixed computation budget (number of model calls to compute $p_\theta(-\sep \textbf{x}_{<k})$) for all other methods. 
Should the hybrid method not be used in an experiment, the budget is set by determining the number of model calls used in drawing $S$ samples for importance sampling.
We intentionally select relatively small computation budgets per query to support systematic large-scale experiments over multiple queries up to relatively large values of $K$. Results for queries with GPT-2 are further restricted because of computational limits and are reported separately below.

To evaluate the accuracy of the estimates for each query and method, we compute the true probability of $K$ using exact computation for small values of $K \leq 4$. For larger values of $K$, we run importance sampling with a large number of samples $S$, where $S$ is adapted per query to ensure the resulting 
unbiased estimate has an empirical variance less than $\epsilon\ll1$ (see \cref{sub:ground_truth_calc}). 
This computationally-expensive %
estimate is then used as a surrogate for ground truth in error calculations.

Coverage-based beam search is not included in our results: we found   that it experiences exponential growth with respect to $K$ and  does not scale efficiently due to its probability coverage guarantees. Additional details are provided in \cref{sub:coverage_bs_ablation}.

\begin{figure}
    \centering
    \includegraphics[width=0.9\textwidth]{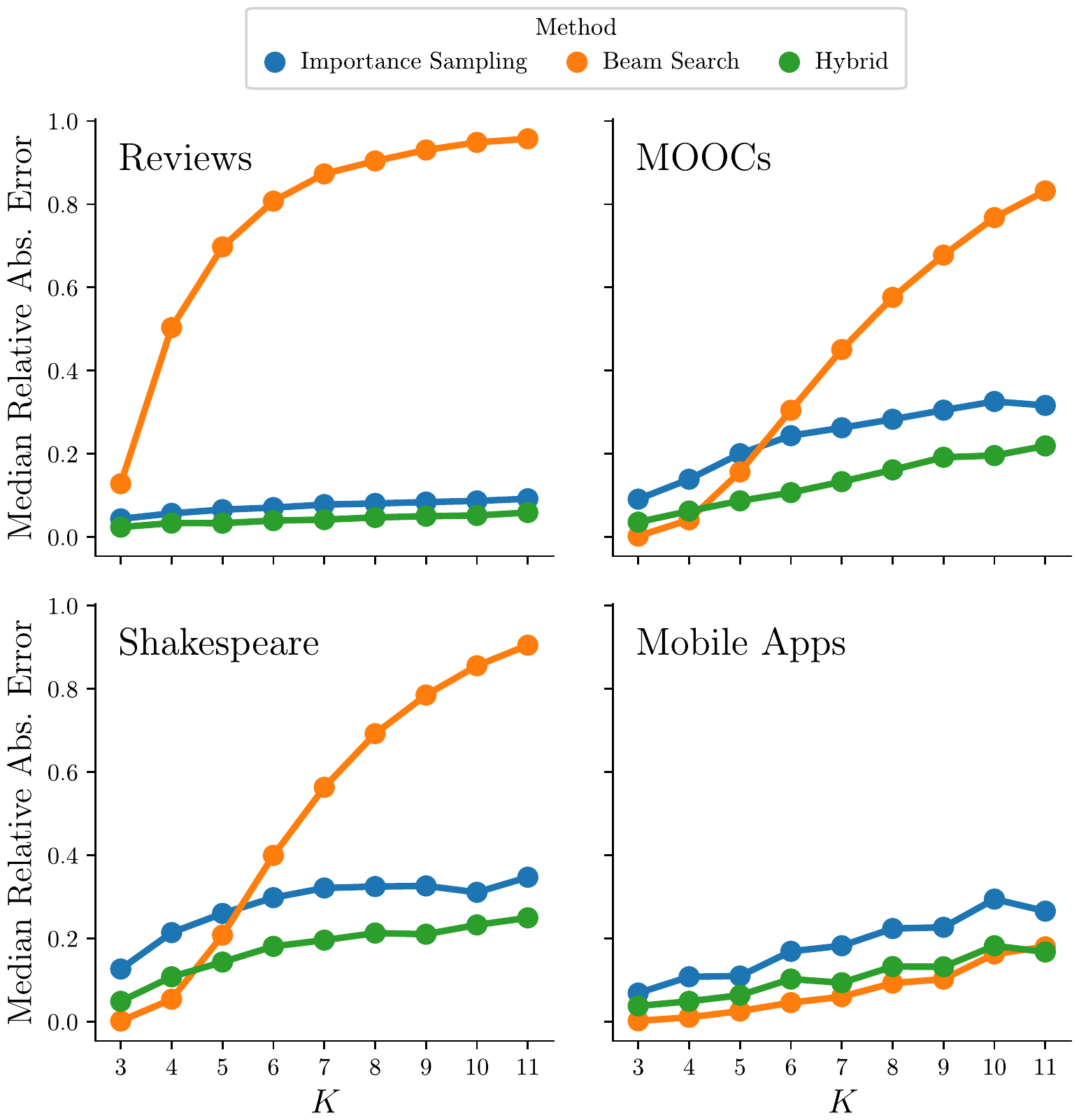}
    \caption{Median relative absolute error (RAE) between estimated probability and (surrogate) ground truth for $p_\theta(\hit(\cdot)=K)$ %
    for importance sampling, beam search, and the hybrid method. As query path space grows with $K$, beam search quickly fails to bound ground truth while sampling remains robust, with the hybrid consistently outperforming all other methods, especially for large values of $K$. Ground truth values used to determine error are exact for $K \leq 4$ and approximated otherwise.}
    \label{fig:3_err_plot}
\end{figure}

\subsection{Results}
 
\subsubsection{Accuracy and Query Horizon}
Using the methodology described above, for each query, we compute the relative absolute error (RAE) $|p - \hat{p} |/p$, where $\hat{p}$ is the estimated query probability generated by a particular method and $p$ is the ground truth probability or the surrogate estimate using importance sampling. For each dataset, for each of the 3 levels of computation budget, for each value of $K$, this yields $N=1000$ errors for the $N$ queries for each method.  

\cref{fig:3_err_plot}  shows the median RAE of the $N$ queries, per method, as a function of $K$, for each dataset, using the medium  computation budget in terms of model calls.
Across the 4 datasets the error  increases systematically as $K$ increases. However,  beam search is significantly less robust than the other methods for 3 of the 4 datasets:  the error rate increases rapidly 
compared to  importance sampling and hybrid. Beam search is also the most variable across datasets relative to other methods. The hybrid method systematically outperforms importance sampling along across all 4 datasets and  for all values of $K$.  In \cref{sec:3_additional_exps}  we provide additional results; for the lowest and highest levels of computational budget, for mean (instead of median) RAEs, and scatter plots for specific values of $K$ with more detailed error information. The qualitative conclusions are consistent across all experiments.

\begin{figure}
    \centering
    \includegraphics[width=0.88\textwidth]{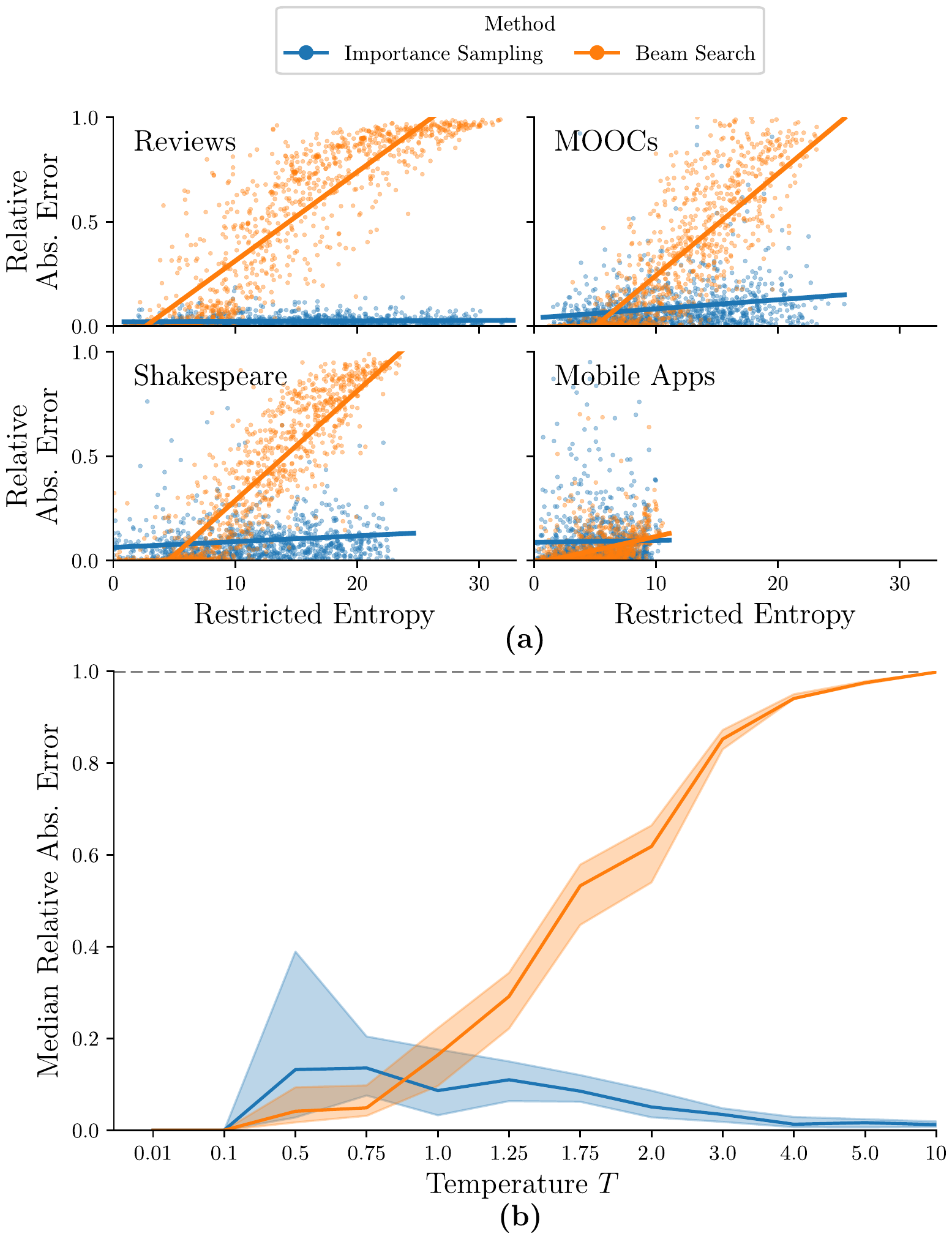}
    \caption{(a) RAE vs restricted entropy per query (with best linear fits), (b) Median RAE versus model temperature $T$ for Mobile App data. All errors computed using the same queries as in \cref{fig:3_err_plot}. Beam search errors correlate highly with model entropy even with the low-entropy Mobile Apps dataset, where increasing temperature $T$ directly induces this failure mode.
    } 
    \label{fig:3_temp_plots}
\end{figure}

\subsubsection{The Effect of Model Entropy}
We conjecture that the entropy of the proposal distribution $q$ conditioned on a given history $H^\q(\textbf{X}_{1:K})=-\E^\q[\log q(\textbf{X}_{1:K})]$, which we refer to as \emph{restricted entropy}, is a major factor in the performance of the estimation methods. 
\cref{fig:3_temp_plots}(a) shows the RAE per query (with a linear fit) as a function of estimated restricted entropy for importance sampling and beam search. The results clearly show that entropy is driving the query error in general and that the performance of beam search is much more sensitive to entropy than sampling. The difference in entropy characteristics across datasets explains the differences in errors we see in \cref{fig:3_err_plot}. In particular, the Mobile Apps dataset is in a much lower entropy regime than the other three datasets.

To further investigate the effect of entropy, we alter each model by applying a temperature $T>0$ to every conditional factor: $p_{\theta,T}(x_k\sep \textbf{x}_{<k}) \propto p_\theta(x_k\sep \textbf{x}_{<k})^{1/T}$, effectively changing the entropy ranges for the models. \cref{fig:3_temp_plots}(b) shows the median RAE, for query $p_{\theta,T}(\hit(\cdot)=4)$, as a function of model temperature for the Mobile Apps data. As predicted from   \cref{fig:3_temp_plots}(a), the increase in $T$, and corresponding increase in entropy, causes beam search's error to converge to 1, while the sampling error goes to 0. As $T$ increases, each individual sequence will approach having $1/|\mathcal{Q}|$ mass, thus needing many more beams to have adequate coverage. Results for other queries and the other three datasets (in \cref{sec:3_additional_exps}) further confirm the fundamental bifurcation of error between search and sampling (that we see in \cref{fig:3_temp_plots}(b))  as a function of entropy.

\subsubsection{Relative Efficiency of Proposal Distribution over Naive Query Estimation}

\begin{figure}
    \centering
    \includegraphics[width=0.8\textwidth]{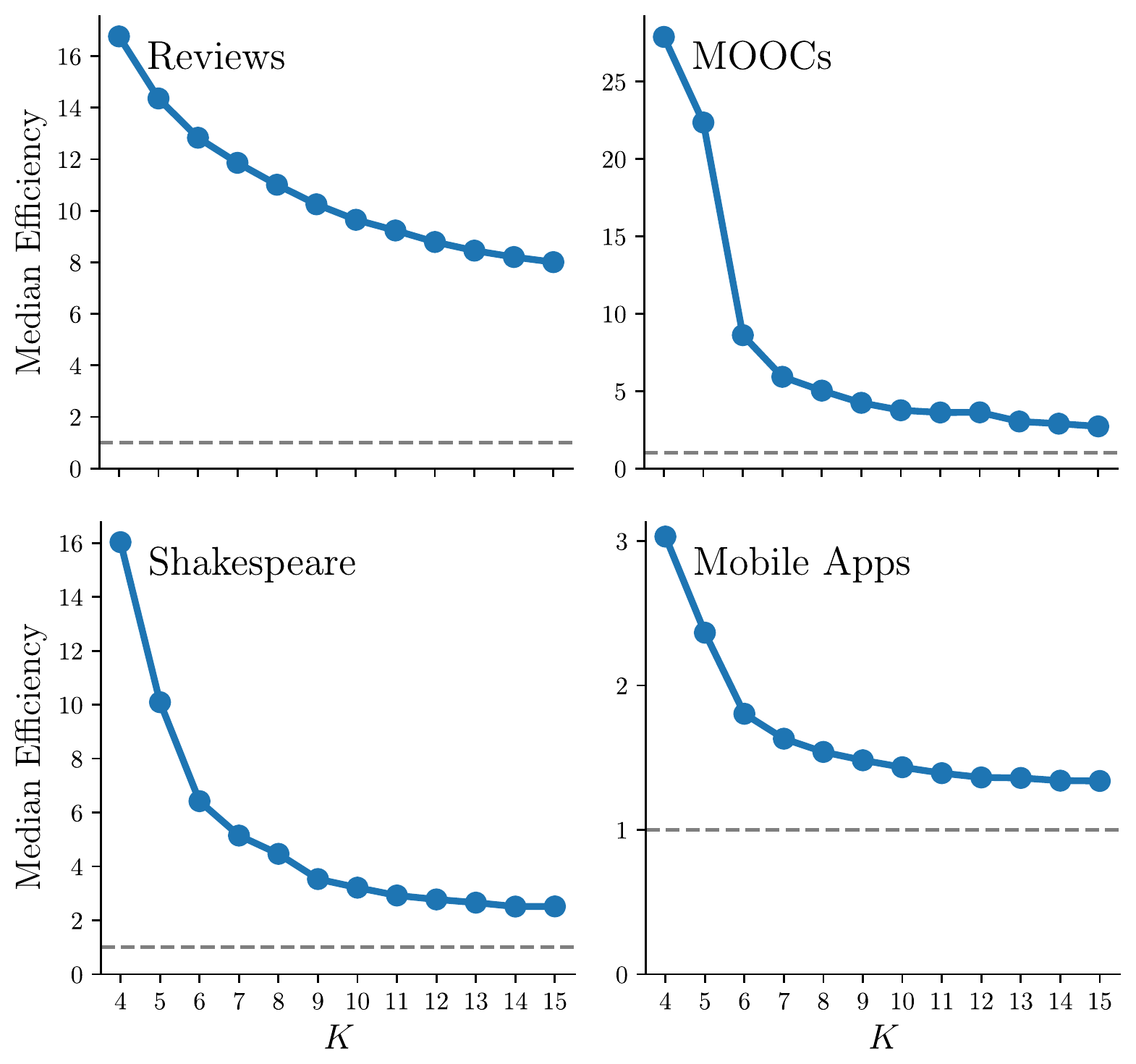}
    \caption{Median relative efficiency (over 1000 query histories and all vocabulary terms) of importance sampling estimation of the $K$-step marginal distributions for each dataset. The gray, dotted line represents 100\% relative efficiency defined by naive query estimation. Relative efficiency is documented for $ 4 \leq K \leq 15$ to highlight the regime where ground truth cannot be tractably computed.}
    \label{fig:3_relative_efficiency_abl}
\end{figure}

We also examine the relative efficiency improvements of our proposal distribution against naive Monte Carlo sampling: 
\begin{align}
p_\theta(\textbf{X}_{1:K} \in \mathcal{Q}) = \mathbb{E}^\prob[\ind(\textbf{X}_{1:K} \in \mathcal{Q})]
\end{align}

Our relative efficiency calculations in \cref{fig:3_relative_efficiency_abl} represent the variance ratio of naive query estimates and estimates from our query proposal distribution. As shown, all datasets witness improvement over naive sampling efficiency and often by a significant margin. We also observe that relative efficiency is largest for shorter query horizons, approaching naive sampling efficiency as $K$ increases.

\subsubsection{Queries with a Large Language Model}

We further explore entropy's effect on query estimation with GPT-2 and the WikiText dataset for $N=100$, $K = 3,4$, across 3 computation budgets. With a vocabulary 500x larger than the other models, GPT-2 allows us to examine queries relevant to NLP applications. 
The resulting high entropy causes beam search to fail  to match surrogate ground truth given the computation budgets (consistent with  earlier experiments), with a median RAE of 82\% (for $K=4$ and a budget of $O(10^3)$). By contrast, importance sampling's median RAE under the same setting is 13\%, \textbf{a 6x reduction}. 
Additional results  are in \cref{sec:3_gpt_exp}.

\section{Conclusion} %
\label{sec:3_discussion}

In this chapter, we formally defined a general class of probabilistic queries for discrete sequence models and derived analytical solutions for computing them for $n^\text{th}$ degree Markov models. Additionally, techniques were developed using importance sampling and beam search to approximate these queries for instances where analytic solutions are not possible or impractical. Extensive experiments were conducted to verify the potential to save on resources when using these methods compared to naive alternative approaches.

The text of this chapter is based on the publication:
\begin{center}
\vspace{-1.5em} \textit{Predictive querying for autoregressive neural sequence models} \citep{boyd2022predictive}.
\end{center}

Primary authorship was shared between the author of this dissertation and Samuel Showalter. Overall conception of the project and some of the writing was shared among all authors. The author of this dissertation determined the problem area and scope for the work, alongside deriving all analytical results. The author also developed the novel proposal distribution and importance sampling approximation technique as well as the initial coverage-based beam search method. Credit for the other heuristics for the beam search approaches, the hybrid method as a whole, and the experimental design are shared between the author and Showalter collectively. Showalter deserves credit for implementing and executing the majority of the experiments with minor oversight from the author.

\chapter{Hitting Time and Related Query Approximation for Marked Temporal Point Processes}

\noindent

Stochastic models for event data evolving in continuous time are typically referred to as temporal point processes. An important class within this general family is {\it marked temporal point processes} (MTPPs),  where each event in time is associated with a random outcome known as a mark. In general, the mark can either be discrete or continuous; in this work we focus on discrete marks. The flexibility of MTPPs has allowed them to be applied to a broad range of applications, including medical diagnosis \citep{islam2017marked, nagpal2021deep, chiang2022hawkes}, epidemic spread models \citep{marmarelis2022metric}, environmental data analysis \citep{brillinger2000}, financial data prediction \citep{zhu2021probabilistic,shi2022state,bacry2012non, hawkes2018hawkes}, communication network modeling \citep{mishra2013anonymity}, user behavior analysis \citep{mishra2016feature, kumar2019predicting, yang2021atpp, hatt2020early}, misinformation spread models \citep{zhang2021vigdet}, and activity prediction \citep{fortino2020exploiting}. 

The foundations for MTPP models have their origins in the statistical literature (e.g., \cite{lewis1972multivariate,daley2003introduction, andersen2012statistical}), with subsequent  development of specific classes of MTPPs such as multivariate self-exciting Hawkes processes \citep{hawkes1971spectra} and multivariate self-correcting processes \citep{zheng1991application}. 
More recently, there has been significant activity in  the development of machine learning methods for MTPPs, with a significant emphasis on approaches
that take advantage of neural representation learning, such as recurrent MTPPs \citep{du2016recurrent},  neural Hawkes processes \citep{mei2017neural}, stochastic variants of deep MTPPs \citep{hong2022deep}, 
scalable deep MTPPs \citep{turkmen2020fastpoint}, as well as general approaches to forecasting with deep MTPP models \citep{deshpande2021long}.
These MTPP modeling frameworks provide a general and flexible setup for making one-step-ahead predictions  such as the timing and/or type of the next event time, conditioned on a partial history of sequence.

In this chapter we look beyond one-step ahead predictions and instead investigate how to efficiently answer queries that involve more complex statements about future events and their timing. Such queries include hitting time queries (``what is the probability that at least one event of type $A$ will occur before time $t$''), queries of the form ``what is the probability that $A$ will occur  before $B$,'' 
as well as computing the marginal distribution of  event types for the $n^{\text{th}}$ next event (irrespective of time). These types of queries are useful across a variety of applications, such as making predictions conditioned on a patient's medical and treatment history, or conditioned on a customer's page view and purchase history.

However, exact computation of such queries is intractable in general except in the case of simple parametric models, such as Poisson processes. For a standard MTPP model to directly answer such queries requires that all intervening events (from current time to the event(s) of interest in the query) are marginalized over. In particular, this involves marginalizing over both the combinatorially-large space of possible  event types as well as the uncountably infinite space of possible event timings. While direct simulation of future trajectories from a model provides one avenue for answering such queries (e.g., see \cite{daley2003introduction}) these ``naive'' methods can be very inefficient (both statistically and computationally), as we will demonstrate later in the chapter. More efficient alternative approaches (to the naive simulation method) appear to be completely unexplored (to our knowledge), for both neural and non-neural MTPP models.

We develop a general query framework based on importance sampling that enables efficient estimates of various types of queries. In our approach, we first transform each query into unified forms and then derive the distribution of interest as functions of type-specific intensities (expected instantaneous rates of occurrence). Our proposed novel marginalization scheme empowers real-time computation of probabilistic queries, with proven higher efficiency compared to naive estimates. 
Furthermore, experiments on three real-world datasets in different domains demonstrate that our proposed estimation method is significantly more efficient than the naive estimate in practice. For example, for hitting time queries with neural Hawkes processes, we show an average magnitude of $10^3$ reduction in estimator variance.

Our approach for answering probabilistic queries is general-purpose in the sense that it can be integrated with any intensity-based black-box MTPP model, either parametric or neural.

\section{Related Work}
A large variety of MTPP models have been developed over recent decades, aimed at modeling sequences of marked event data with varying sorts of behaviors. This behavior has been both explicitly modeled with parametric MTPP models \citep{isham1979self, daley2003introduction}, and implicitly modeled using neural network-based methods \citep{du2016recurrent, bilovs2019uncertainty, shchur2019intensity, enguehard2020neural, zuo2020transformer}. Of particular note in these categories are the self-exciting Hawkes process \citep{hawkes1971spectra, liniger2009multivariate} and the neural Hawkes process \citep{mei2017neural}. The majority of neural MTPP models utilize some form or extension of recurrent neural networks to model conditional intensity functions (or equivalent transformations thereof).
MTPP models have %
been broadly applied to next event prediction across a number of different application areas: seismology \citep{ogata1998space}, finance \citep{bacry2012non, hawkes2018hawkes}, social media behavior \citep{mishra2016feature, rizoiu2017expecting}, and medical outcomes \citep{10.2307/2985181, andersen2012statistical}.\footnote{Survival analysis is a special case of temporal point processes where the event of interest can only occur once.} Neural-based methods have also been successful at additional tasks such as imputing missing data \citep{shchur2019intensity, mei2019imputing, gupta2021learning}, sequential representation learning \citep{shchur2019intensity, NEURIPS2020_f56de5ef}, and long-term forecasting \citep{deshpande2021long}. 

Answering probabilistic queries in some capacity has been previously explored at a model-specific level. 
Primary examples include continuous-time Markov processes \citep{shelton2014tutorial}, continuous-time Bayesian networks \citep{Nodelman+al:UAI02, fan2008sampling}, and Markovian self-exciting processes \citep{10.2307/3212408}. In this prior work, the assumed parametric form of the model allows for analytic forms of specific queries under certain conditions. For instance, the Markovian self-exciting process provides a representation that makes estimating hitting time queries directly tractable. 

However, to the best of our knowledge, apart from the naive sampling approach (e.g., \cite{daley2003introduction}),  there is no existing work on   answering general probabilistic queries  (such as hitting time of a collection of event types) for black-box MTPP models, which is the focus of this chapter.
For discrete-time models, estimating these queries has been investigated in our prior work \citep{boyd2022predictive}, and while there does not exist a direct mapping of those techniques to continuous time, this previous work will serve as a large source of inspiration for what we propose in this chapter.

\section{Preliminaries} \label{sec:4_prelims}
\subsection{Notation for Event Sequences}
Let $T_1, T_2, \dots \in \R_{\geq 0}$ be a sequence of continuous random variables with the constraint that $\forall_i: T_i < T_{i+1}$. These represent the time of occurrence for events of interest. Each event has an associated categorical value, such as a label or a location, that is referred to as a \emph{mark}. An event $X_i \in \V := \R_{\geq 0} \times \vocab$ is jointly represented as (i) a time of occurrence $T_i$ and (ii) an associated mark random variable $M_i\in\vocab$.  
In this work we will focus on the finite discrete setting of a fixed vocabulary for marks: $\vocab=\{1,2,\dots,K\}$, although more generally the mark space $\vocab$ can be defined on a variety of different domains.
\\\\
Sequences of fixed number of $n$ events can be represented via $\textbf{X}_{1:n}$; however, we are often more so interested in the events that span a given range of time. A random sequence that spans the time range $[a,b]\subset\R_{\geq 0}$ will be denoted as
\begin{align}
    \hist_{[a,b]}=\{(T_i,M_i)\sep T_i \in [a,b] \text{ for } i\in\N\}
\end{align}
with similar definitions for $\hist_{(a,b]}$ and $\hist_{[a,b)}$ over time ranges $(a,b]$ and $[a,b)$ respectively.
For simplicity, we will let $\hist_t$ and $\hist_{t-}$ be shorthand for $\hist[0,t]$ and $\hist[0,t)$ such that $X_i \in \hist_{T_i}$ and $X_i \notin \hist_{T_i-}$. The point process literature uses this notation for sequences often interchangeably with filtrations. For all intents in purposes, for this chapter we can interpret $\hist$ to be a random sequence. Please refer back to \cref{sec:mtpp_notation} for more discussion on this topic. 

We will use $\hist^k$ and $\hist^A$ for $k\in \vocab$ and $A \subset \vocab$ to refer to \emph{mark-specific} sequences, i.e., $\hist^k_t = \{(T_i,M_i) \in \hist(t) \sep M_i=k\}$ and $\hist^A_t = \{(T_i,M_i) \in \hist_t \sep M_i \in A\}$.
Realizations of these sequences of events are referred to as \emph{marked temporal point patterns} and represented with lower-case letters that mirror the notation used for random sequences, i.e., $\textbf{h}_t=\{(t_i,m_i)\sep t_i \in [0,t] \text{ for } i\in\N\}$.

\subsection{Marked Temporal Point Processes}
The generative mechanism for these point patterns are generally referred to as \emph{marked temporal point processes} (MTPPs). 
MTPP models fully define a probability measure $\prob$ over the possible events generated by the stochastic process $\mathbf{X}$, i.e., $\mathcal{F}:=\sigma(\mathbf{X})$. 
Because of this, they are capable of modeling the likelihood $\mathcal{L}$ of a given sequence $\textbf{h}_{\tau}$ of $N$ events over a time period of $[0,\tau]$ and are typically constructed in an autoregressive fashion,
\begin{align}
\mathcal{L}(&\hist_{\tau}=\textbf{h}_\tau) \\
& = \prob(T_{N+1} > \tau \sep \hist_{T_N}=\textbf{h}_{T_N})\prod_{i=1}^{N} p_{T_i \sep \hist_{T_{i-1}}}(t_i \sep \textbf{h}_{t_{i-1}}) p_{M_i \sep T_i, \hist_{T_{i-1}}}(m_i \sep t_i, \textbf{h}_{t_{i-1}}) \\
& := \prob(T_{N+1} > \tau \sep \hist_{T_N}=\textbf{h}_{T_N})\prod_{i=1}^{N} p^*_{T_i}(t_i) p^*_{M_i \sep T_i}(m_i \sep t_i) \label{eq:4_likelihood}
\end{align}
where the $*$ notation is used for brevity to indicate conditioning on the relevant preceding events. The distribution for the next event $(T_i,M_i)$ conditioned on the preceding terms is often modeled with the expected instantaneous rate of change for each mark. This is referred to as the \emph{marked intensity function} and is defined formally as
\begin{align}
\lambda_k(t\sep\hist_{t-})dt:=\E^\prob\left[\ind(|\hist^k_{[t,t+dt)}|=1)\sep \hist_{t-}\right]\label{eq:marked_intensity_def}
\end{align}
where $\ind(\cdot)$ is the indicator function and $\E^\prob$ is the expected value with respect to distribution $p$.\footnote{It is worth noting that the expectation defined in \cref{eq:marked_intensity_def} is tractable due to MTPPs being \textit{predictable processes} by construction, meaning that $\hist_t \sep \hist_{t-}$ is measurable.}
For brevity, we typically use the $*$ notation to suppress the conditional: $\lambda^*_k(t):=\lambda_k(t\sep\hist_{t-})$. Note that these functions not only condition on the preceding events, but also on the fact that no events have occurred since the last event up until time $t$, i.e., $\prob(\cdot\sep\hist_{[0,t)}) \neq \prob(\cdot\sep\hist_{[0,T_{i-1}]})$.

The total intensity function, $\lambda^*(t):=\sum_{k\in\vocab}\lambda^*_k(t)$, is sufficient to describe the timing of the next event $T_i$. The distribution of the mark conditioned on the timing of the next event is naturally described as $p^*_{M_i \sep T_i}(k\sep t)\equiv \frac{\lambda^*_k(t)}{\lambda^*(t)}$. We will be assuming that the native output of any model we are working with will produce a vector of marked intensity functions over the mark space $\vocab$ evaluated at time $t$. Any MTPP with a defined set of marked intensity functions can be easily sampled from by utilizing a thinning procedure \citep{ogata1981lewis}, if not directly.

Lastly, the likelihood of a given sequence $\hist$ of length $N$ over an observation window $[0,\tau]$ can be computed in terms of intensity values:
\begin{align}
\mathcal{L}_\prob(\hist_{\tau}=\textbf{h}_\tau)=\left(\prod_{i=1}^{N}\lambda_{m_i}^*(t_i)\right)\exp\left(-\int_0^\tau \lambda^*(s)ds\right).
\end{align}

\section{Querying MTPPs}

We are interested in evaluating probabilistic statements, or rather \emph{queries}, 
on any MTPP model, e.g., a model trained from data.
Furthermore, we are interested in evaluating queries that are conditioned on a partially observed sequence (e.g., ``what is the likelihood that at least one event of type $A$ will occur in the next year %
given a patient's medical history?''). 

Formally, we define a probabilistic query as a probability statement of the form
\begin{align}
\prob(\hist \in \mathcal{Q}) \text{ where } \mathcal{Q} \subset \Omega_{\hist}\equiv \text{Sample Space of } \hist,
\end{align}
where $\prob$ is the implied probability measure defined by a given MTPP model's distribution over future event sequences.\footnote{Similar to the previous chapter, we make a distinction between the model itself, which gives us tractable access to $\lambda_k(t)$, and the probability measure that this implies. Some forms of queries to the measure are equivalently represented directly by the model; however, there are far more queries that are not immediately answerable by direct model output. The latter of which are the kinds of queries we are most interested in.} We refer to $\mathcal{Q}$ as the \emph{query space}. The contents of the query space naturally will vary depending on the query at hand. An example query space and a subset of associated valid sequences can be seen in \cref{fig:4_example_query_space}. It is worth noting that in most contexts, the cardinality of $\mathcal{Q}$ will be uncountably infinite. 

\begin{figure}
    \centering
    \includegraphics[width=0.94\columnwidth]{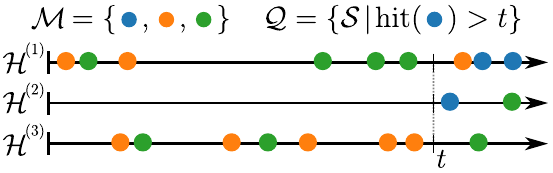}
    \caption{Example query space $\mathcal{Q}$ for the hitting time (first occurrence) of blue marks being greater than some time $t$. Sequences shown, $\hist^{(i)}$, all belong to the query space as they each do not contain a blue event occurring before time $t$.}
    \label{fig:4_example_query_space}
\end{figure}

This section will begin by discussing what probabilistic queries %
are readily available and tractable for a given model. Following this, we will present a novel class of queries, of which include hitting time and marginal mark queries, as well as an importance sampling estimation procedure. Finally, we will  discuss ``A before B'' queries and how to efficiently estimate them under our novel framework. \emph{Without loss of generality, we suppress the notation for conditioning on partially observed sequences and present all derivations and notations for unconditional queries.}

\paragraph{Contextualizing Queries} 
While we would like to be able to answer complex queries, we additionally would like to do so while conditioning on some partial sequence of previously realized events $\hist_t$. In effect, this \emph{contextualizes} whatever query is being asked. An example of this would be asking how likely a patient will experience a heart attack in the next year, conditioned on their personal medical history.

All of the methodology developed in this work allows for this. 
For the sake of brevity and consistent notation, we will be presenting all queries as not conditioning on any prior events; this allows for all queries to take place in the future over time $t\in[0,\infty)$ and all future events to start with index $i=1$. Note that all derived results can easily be extended to conditioning on sequences of realized events, as will be demonstrated empirically in \cref{sec:4_experiments}, simply by shifting the end of the realized window to end at $t=0$ and reset the index of events to end with $i=0$.

\subsection{Directly Tractable Queries}
Due to the model's autoregressive nature, queries about the immediate next event of a sequence are the only types of queries that can be directly evaluated without marginalization. We will now present the two main types for MTPPs.

\paragraph{Marginal Distribution of Next Event Time}
In general, it can be shown that $\lambda^*(t)=\frac{p_{T_i}^*(t)}{1-F_{T_i}^*(t)}$ where $t \in (T_{i-1}, T_i]$, $p_{T_i}^*$ is the probability density function (PDF) of $T_i$ conditioned on $\hist_{T_{i-1}}$, and $F_{T_i}^*$ is the cumulative density function (CDF) of $T_i$ also conditioned on $\hist_{T_{i-1}}$. By recognizing that $\lambda^*(t)=-\frac{d}{dt}\log(1-F_{T_i}^*(t))$, we find that the CDF of the next event timing $T_1$ is
\begin{align}
\prob( T_1 \leq t) := F_{T_1}(t) = 1 - \exp\left(-\int_0^t \lambda^*(s)ds\right).
\end{align}
Differentiating this result with respect to $t$ yields the PDF: $f_{T_1}(t)=\lambda^*(t)\exp\left(-\int_0^t\lambda^*(s)ds\right)$.
Note that we only immediately have access to the analytical form of the first future event timing $T_1$. To achieve the same results for $T_i$ with $i > 1$ in general would require marginalizing over all $i-1$ events, which ranges from being either cumbersome or intractable to do exactly depending on the analytic form of $\lambda$.

\paragraph{Marginal Distribution of Next Mark} 
Let $A\subset\vocab$. It follows then that the probability of the first event having a mark in $A$ is computed as follows:
\begin{align}
\prob( M_1 \in A) & = \int_0^\infty \prob( M_1 \in A \sep  T_1 = t)f_{ T_1}(t)dt \\
& = \int_0^\infty \frac{\lambda_A^*(t)}{\lambda^*(t)}\lambda^*(t)\exp\left(-\int_0^t \lambda^*(s)ds\right)dt \\
& = \int_0^\infty \lambda_A(t)\exp\left(-\int_0^t \lambda^*(s)ds\right)dt,
\end{align}
where $\lambda^*_A(t)=\sum_{k\in A}\lambda^*_k(t)$. Replacing the outer integration bounds of $[0,\infty)$ with $[a,b]$ gives the joint query $\prob( T_1\in[a,b],  M_1\in A)$.

Both of these different queries can potentially be computed analytically if the form of $\lambda^*$ permits, otherwise they can be estimated using approximate integration techniques.

\subsection{Naive Estimation of Queries}
When considering more complex queries, for example those that deal with sequences of events or those far in the future, it becomes necessary to rely on simulating potential trajectories in order to estimate their values. This is due to the fact that exactly representing a probabilistic query in terms of intensity values involves many nested integrals (for each potential interim event), potentially an infinite amount of them depending on the query.

The de facto method for approximating arbitrary probabilistic queries involves generating sequences and computing the relative frequency for which the query condition is met in the sampled sequences \citep{daley2003introduction}. This can be seen as a Monte Carlo estimate with the following formulation:
\begin{align}
\prob(\hist_\tau\in\mathcal{Q}) = \E^{\prob}\left[\ind(\hist_\tau\in\mathcal{Q})\right],
\end{align}
where $\ind(\cdot)$ is the indicator function. We refer to this procedure as ``naive'' estimation because this does not take into account any information about the query when sampling. 

\subsection{General Restricted-Mark Queries} \label{sec:4_gen_mark_queries}
One way to improve upon the naive procedure is to leverage information about the query in a proposal distribution in conjunction with importance sampling.
To do so though, we must first constrain ourselves to a specific class of query being considered. Additionally, this class of interest should take into account different aspects of sampling sequences using MTPPs for our proposal distribution. Namely, these models can easily be forced to \emph{not} sample events of specific types over a period of time (e.g., set $\lambda^*_A(t):=0$ for some time interval). Conversely, it is not immediately obvious how to \emph{encourage} or \emph{force} an event to occur within a specified time range.

As such, a natural class of queries can be seen in which over one or more specified spans of time we restrict what types of events are allowed and not allowed to occur. We term this class as ``general restricted-mark queries.''

We will now more formally define this class of queries. Consider positive real values $\alpha_1,\dots,\alpha_n$ such that $\alpha_i < \alpha_{i+1}$. These values naturally split the timeline $\R_{\geq 0}$ into $n+1$ spans: $[0,\alpha_1], (\alpha_1, \alpha_2], \dots, (\alpha_{n-1},\alpha_{n}], (\alpha_n, \infty)$.  Furthermore, let the subsets $\mathcal{M}_i \subseteq \vocab$ for $i=1,\dots,n$ represent restricted mark spaces for the first $n$ spans. The class of queries is concerned with how likely sequences spanning $[0,\alpha_n]$ respect the restricted mark spaces in each associated interval: \begin{align}
&\prob\left(\cup_{i=1}^n \{\text{No events with types } \mathcal{M}_i \text{ in } t\in(\alpha_{i-1},\alpha_i]\}\right) \nonumber \\
& = \prob\left(\land_{i=1}^n \forall_{( T,  M) \in \hist_{(\alpha_{i-1}, \alpha_i]}}  M \notin \mathcal{M}_i \right) \text{ with } \alpha_0=0. \label{eq:4_framework}
\end{align}
See \cref{fig:4_example_query} for an illustrated example query. 
This is a very flexible class of queries that includes many meaningful individual queries, which will be further discussed in \cref{sec:4_complex_queries}.

\begin{figure}
    \centering
    \includegraphics[width=0.94\columnwidth]{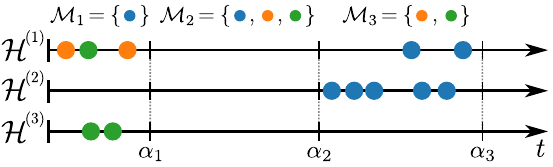}
    \caption{Three potential sequences $\hist^{(i)}$ that satisfies the condition of an example restricted-mark query. The mark space $\mathcal{M}$ in this context is equivalent to that in \cref{fig:4_example_query_space}.}
    \label{fig:4_example_query}
\end{figure}

\paragraph{Importance Sampling and Proposal Distribution}
Let $\q$ be a proposal measure with support over at least the intersection of the support of $\prob$ and the query space $\mathcal{Q}$ (i.e., $\text{supp}(\q) \supseteq \text{supp}(\prob)\cap \mathcal{Q}$). Recall that under the right conditions importance sampling allows us to convert expectations of one measure to another, e.g., $\E^\prob[f(X)]=\E^\q[L(X)f(X)]$ for $L(X)=\frac{\mathcal{L}^\prob(X)}{\mathcal{L}^\q(X)}$ with likelihood functions $\mathcal{L}^\prob$ and $\mathcal{L}^\q$\textemdash{}see \cref{sec:imp_sampling_notes} for more details. It then follows that 
\begin{align}
\E^{\prob}\left[\ind(\hist_\tau\in\mathcal{Q})\right] = \E^{\q}\left[\ind(\hist_\tau\in\mathcal{Q})\frac{\mathcal{L}^\prob(\hist_\tau)}{\mathcal{L}^\q(\hist_\tau)}\right]. \label{eq:4_imp_sampling_general}
\end{align}
It can be shown that the optimal proposal measure (i.e., lowest estimator variance) yields the following likelihood over sequences \citep{robert1999monte}:
\begin{align}
\mathcal{L}^{\q_\text{optimal}}(\hist_\tau) &:= \frac{|\ind(\hist_\tau\in\mathcal{Q})|\mathcal{L}^\prob(\hist_\tau)}{\E^{\prob}[|\ind(\hist_\tau\in\mathcal{Q})|]} \\
& = \mathcal{L}^\prob(\hist_\tau\sep\hist_\tau\in\mathcal{Q}),
\end{align}
however, this is not immediately usable since it involves computing and normalizing over the exact query that we are trying to estimate in the first place. 

The more our actual proposal distribution $q$ resembles $q_\text{optimal}$, the more efficient our estimation procedure will be. 
Since conditioning on future events is difficult for neural autoregressive models, and for that matter most MTPPs in general, we can 
instead only apply immediate ``local'' restrictions on the trajectory such that a sequence will remain within $\mathcal{Q}$. This can be accomplished by letting $\q$ be a measure implied by a MTPP with intensity
\begin{align}
\mu^*_k(t)=\ind(k\notin \mathcal{M}_i)\lambda^*_k(t)
\end{align}
for $k\in\mathcal{M}$ and $t\in(\alpha_{i-1}, \alpha_i]$. Note that this can be seen as the natural extension of the proposal distribution in \cref{sec:3_proposal} to continuous time.
This naturally leads to the likelihood of any sequence generated under $\q$ as being
\begin{align}
\mathcal{L}^\q(\hist_{\tau}) & =  \left(\prod_{i=1}^{N}\mu^*_{M_i}( T_i)\right)\exp\left(-\int_0^\tau \mu^*(s)ds\right) \\
& = \left(\prod_{i=1}^{N}\lambda^*_{ M_i}( T_i)\right)\exp\left(-\sum_{i=1}^n \int_{\alpha_{i-1}}^{\alpha_i} \lambda^*_{\mathcal{M}\setminus \mathcal{M}_i}(s)ds\right)\label{eq:4_simplified_proposal_likelihood}
\end{align}
where $N=|\hist_\tau|$. This proposal distribution was constructed so that every sample generated will always belong to the query space, which allows for the simplified form presented in \cref{eq:4_simplified_proposal_likelihood}.
Applying this to \cref{eq:4_imp_sampling_general} yields
\begin{align}
\prob(\hist_\tau\in\mathcal{Q}) & = \E^{\q}\left[\ind(\hist_\tau\in\mathcal{Q})\frac{\mathcal{L}^\prob(\hist_\tau)}{\mathcal{L}^\q(\hist_\tau)}\right] \\
& = \E^\q\left[\frac{\left(\prod_{i=1}^{N}\mu^*_{M_i}( T_i)\right)\exp\left(-\int_{0}^\tau \mu^*(s)ds\right)}{\left(\prod_{i=1}^{N}\lambda^*_{ M_i}( T_i)\right)\exp\left(-\int_0^\tau \lambda^*(s)ds\right)}\right] \text{ as } \hist_t \in \mathcal{Q} \text{ under } \q\\
& = \E^\q\left[\frac{\left(\prod_{i=1}^{N}\lambda^*_{ M_i}( T_i)\right)\exp\left(-\sum_{i=1}^n \int_{\alpha_{i-1}}^{\alpha_i} \lambda^*_{\mathcal{M}\setminus \mathcal{M}_i}(s)ds\right)}{\left(\prod_{i=1}^{N}\lambda^*_{ M_i}( T_i)\right)\exp\left(-\int_0^\tau \lambda^*(s)ds\right)}\right]\\
& = \E^{\q}\left[\exp\left(-\sum_{i=1}^n \int_{\alpha_{i-1}}^{\alpha_i} \lambda^*_{\mathcal{M}_i}(s)ds\right)\right].\label{eq:4_imp_sampling_specific}
\end{align}
Any query in this class can now be estimated in an unbiased fashion by using Monte Carlo estimation on \cref{eq:4_imp_sampling_specific}.

\paragraph{Estimator Efficiency}
Since both the naive and importance sampled estimators are unbiased, whichever has lower variance can be seen as the more \emph{efficient} estimator. 

Assume that $\mathcal{Q}$ belongs to a general restricted-mark query and that $\pi=\prob(\hist_\tau\in\mathcal{Q})$. Let
\begin{align}
\hat{\pi}_\text{Naive}(\hist_\tau) & = \ind(\hist_\tau \in \mathcal{Q}), \\
\hat{\pi}_\text{Imp.}(\hist_\tau) & = \exp\left(-\sum_{i=1}^n \int_{\alpha_{i-1}}^{\alpha_i} \lambda_{\mathcal{M}_i}^*(s)ds \right),
\end{align}
where both are unbiased estimators of $\pi$ under $\prob$ and $\q$ respectively.
Note that $\hat{\pi}_\text{Imp.}(\cdot)\in[0,1]$ as $\lambda^*_k(\cdot) \geq 0$. Finally, let relative efficiency of the two estimators be defined as
\begin{align}
\text{eff}(\hat{\pi}_\text{Imp.}, \hat{\pi}_\text{Naive}):=\frac{\var^{\prob}\left[\hat{\pi}_\text{Naive}(\hist_\tau)\right]}{\var^{\q}\left[\hat{\pi}_\text{Imp.}(\hist_\tau)\right]}.
\end{align}

\begin{theorem}
If $\pi \in (0,1)$ and $\lambda^*(t) < \infty$ for all $t \in [0,\tau]$, then $\text{eff}(\hat{\pi}_\text{Imp.}, \hat{\pi}_\text{Naive}) > 1$. In other words, under these conditions $\hat{\pi}_\text{Imp.}$ is \underline{always} more efficient than $\hat{\pi}_\text{Naive}$. \label{thm:eff}
\end{theorem}
\begin{proof}
Since the naive estimator is unbiased and binary, then it follows that $\hat{\pi}_\text{Naive}(\hist_\tau) \sim \text{Bern}(\pi)$. Thus, $\var^{\prob}\left[\hat{\pi}_\text{Naive}(\hist_\tau)\right]=\pi-\pi^2$.

To approach the variance of the importance sampling estimator, we note that
\begin{align}
\var^{\prob}\left[\hat{\pi}_\text{Imp.}(\hist_\tau)\right] & = \E^{\q}\left[\hat{\pi}^2_\text{Imp.}\right] - \E^{\q}\left[\hat{\pi}_\text{Imp.}\right]^2 \\
& = \E^{\q}\left[\hat{\pi}^2_\text{Imp.}\right] - \pi^2 \\
& \leq \E^{\q}\left[\hat{\pi}_\text{Imp.}\right] - \pi^2 \text{ since } \hat{\pi}_\text{Imp.}\in[0,1] \\
& = \pi-\pi^2  
\end{align}
The equality only holds if $\pi\in\{0,1\}$ or $\hat{\pi}_\text{Imp.} \sim \text{Bern}(\pi)$. The latter condition is due to the fact that for $[0,1]$ bounded random variables with mean $\pi$, if the variance is equal to $\pi - \pi^2$ then this implies it is Bernoulli (see \cref{thm:4_eff_lemma} below). However, when $\pi \in (0,1)$ then unless $\lambda^*(t) = \infty$ for some subset of $[0,\tau]$ it is impossible for $\hat{\pi}_\text{Imp.}(\hist_\tau)$ to equal $0$. Thus, outside of those circumstances the inequality is strict and $\text{eff}(\hat{\pi}_\text{Imp.},\hat{\pi}_\text{Naive}) > 1$.
\end{proof}

\begin{lemma} \label{thm:4_eff_lemma}
If a bounded random variable $X\in[0,1]$, with mean $\pi$ and variance $\pi(1-\pi)$, then $X\sim\text{Bern}(\pi)$.
\end{lemma}
\begin{proof}
Let $X$ be a random variable with support $[0,1]$ on the probability space $(\Omega, \mathcal{F}, \prob)$ with known mean $\E^\prob[X]=\pi$ and variance $\var^\prob[X]=\pi(1-\pi)$. It then follows that:
\begin{align}
\var^\prob \left[X\right] & = \E\left[X^2\right] - \E\left[X\right]^2 \\
\implies \pi(1-\pi) & = \E\left[X^2\right] - \pi^2 \\
\implies \pi & = \E\left[X^2\right] \\
 & = \int_{[0,1]} x^2 dF_X(x) \\
 & = \int_{\{0,1\}} x^2dF_X(x) + \int_{(0,1)} x^2 dF_X(x) \\
 & = \prob(X=1)  + \int_{(0,1)} x^2 dF_X(x)
\end{align}
$\int_{(0,1)} x^2 dF_X(x) > 0$ if and only if $\prob(X \in (0,1)) > 0$. If we assume that $\prob(X \in (0,1)) > 0$, then it follows that:
\begin{align}
\pi & = \prob(X=1)  + \int_{(0,1)} x^2 dF_X(x) \\
& < \prob(X=1)  + \int_{(0,1)} x dF_X(x) \\
& = \prob(X=1) + (\pi - \prob(X=1)) \\
& = \pi,
\end{align}
however, $\pi \nless \pi$. Hence, by contradiction $\prob(X \in (0,1)) = 0$ which implies that $\prob(X=1)=\pi$ and $\prob(X=0)=1-\pi$ since $\E^\prob\left[X\right]=\pi$. Thus, it can be concluded that $X \sim \text{Bern}(\pi)$.
\end{proof}

\subsection{Practical Estimation of Complex Queries} \label{sec:4_complex_queries}
We will now apply our findings from \cref{sec:4_gen_mark_queries} to produce estimators for three different complex, probabilistic queries.

\paragraph{Marginal Distribution of Hitting Time}
Let $A\subset \mathcal{M}$ and $A \neq \emptyset$. The first occurrence of an event with type $k\in A$, regardless of events of other types, is referred to as the \emph{hitting time} of $A$ or $\hit(A)$. The probabilistic query of the CDF of the hitting time of $A$ at a specific time $t$ can be seen as a query under the general restricted-mark class:
\begin{align}
\prob(&\hit(A) \leq t) = 1 - \prob(\hit(A) > t) \\
& = 1 - \prob(\{\text{No events of types } A \text{ in } [0,t] \}) \\
& = 1 - \prob(\forall_{( T, M)\in\hist_t} M\notin A) \nonumber \\
& = 1- \E^{\q}\left[\exp\left(- \int_{0}^t \lambda^*_{A}(s)ds\right)\right]. 
\end{align}
Note that this derivation relies on this query being a special case of the general framework outlined in \cref{eq:4_framework} where $n=1$, $\alpha_0=0$, $\alpha_1=t$, and $\mathcal{M}_1=A$. 
Interestingly, the importance sampled result of this query greatly resembles the CDF of the general first event timing: $F_{ T_1}(t)=1-\exp\left(-\int_0^t \lambda^*(s)ds\right)$.\footnote{It is important to remember that in the general case, we must marginalize over possible trajectories for other types of events $A'$ as these can all either potentially influence the intensity of events of type $A$.} Furthermore, should $A=\mathcal{M}$ then we recover $F_{ T_1}(t)$ as the estimator becomes deterministic (due to $\mu^*(t)=0 \implies q(\hist)\propto \ind(\hist=\emptyset)$).

\paragraph{Marginal Distribution of $n^\text{th}$ Mark}
Let $A\subset \mathcal{M}$ and $n\geq 1$. The distribution of the marginal $n^\text{th}$ mark describes how likely it is that the $n^\text{th}$ event has a mark $k\in A$, irrespective of the timing of itself or of any of the $n-1$ events that occurred prior. In contrast to hitting time queries, we do not fix the integration bounds but rather sample them to be the timings of the $ T_{n-1}$ and $ T_n$. In doing so, this query falls under the general mark-restricted framework:
\begin{align}
\prob( M_n\in A) & = \prob(\{\text{No events of types } A' \text{ in } ( T_{n-1}, T_{n}]\}) \\
& = \prob(\forall_{( T, M)\in \hist_{(T_{n-1}, T_n]}} M\notin A') \\
& = \E^{\q}\left[\exp\left(-\int_{ T_{n-1}}^{ T_n}\lambda_{A'}^*(s)ds\right)\right]
\end{align}
where $A'=\mathcal{M}\setminus A$. This can be seen as a special case under \cref{eq:4_framework} where $\alpha_0=0$, $\alpha_i= T_i$ for $i=1,\dots,n$, $\mathcal{M}_1,\dots,\mathcal{M}_{n-1} = \emptyset$, and $\mathcal{M}_n=A'$. Tying the values of the boundaries $\alpha_i$ to the random event times $ T_i$ effectively ensures that each span with a restricted vocabulary $\mathcal{M}_i$ only pertains to the occurrence of one event. In doing so, we actually recover the ability to estimate queries purely concerning the marks, similar to the discrete sequence setting \citep{boyd2022predictive}.

It is worth noting that, interestingly, we can also compute the complement under the same framework as $\prob( M_n\in A) = 1-\E^{\q}\left[\exp\left(-\int_{ T_{n-1}}^{ T_n}\lambda_{A}^*(s)ds\right)\right]$.

\paragraph{``$A$ before $B$'' Queries}
The last class of queries we will discuss are what we refer to as ``$A$ before $B$'' queries. To be precise, we are interested in the probability of an event with some type $k\in A$ occurring before an event with some type $k\in B$ where $A\cap B=\emptyset$ and non-empty $A,B\subset\mathcal{M}$. In math, this is formally represented as $\prob(\hit(A) < \hit(B))$.

Surprisingly, with our previous developments we can actually estimate this query using importance sampling in conjunction with proposal distribution $\q$. For the proposal distribution, let $\mu_k^*(t)=\ind(k \notin A\cup B)\lambda^*_k(t)$. It then can be shown that
\begin{align}
\prob&(\hit(A)<\hit(B)) \\
& = 1-\E^{\q}\left[\int_0^\infty \lambda_B^*(t) \exp\left(-\int_0^t \lambda^*_{A\cup B}(s) ds\right)dt\right]  \\
& = \E^{\q}\left[\int_0^\infty \lambda_A^*(t) \exp\left(-\int_0^t \lambda^*_{A\cup B}(s) ds\right)dt\right] \label{eq:4_a_before_b}
\end{align}
with both expressions being equal due to the complement $1-\prob(\hit(A)>\hit(B))$ also being estimable under this derivation. See \cref{sec:4_a_before_b_derivation} for the derivations of this expression. 

Interestingly, just like the parallels between the hitting time CDF and the first event time CDF, there exist similar comparisons for \cref{eq:4_a_before_b} and the analytical form of the marginal distribution for the first mark $\prob( M_1\in A)=\int_0^\infty \lambda^*_A(t)\exp\left(-\int_0^t \lambda^*(s)\right)dt$. Additionally, should $B=A'$ then the estimator becomes deterministic and we recover the form of $\prob( M_1\in A)$. 

Note that the expectations in \cref{eq:4_a_before_b} are with respect to $\hist_\infty \sim \q$, which is naturally not possible to evaluate; however, since the integrands are non-negative we can compute natural lower and upper bounds by sampling $\hist_\tau\sim\q$ and integrating over $[0,\tau]$ instead of $[0,\infty)$. Lastly, since these bounds utilize the same proposal distribution, we can actually compute both at the same time for a little extra computation. It then follows that a good estimate for $\prob(\hit(A) < \hit(B))$ would be an average of the upper and lower bounds:
\begin{align}
\prob(\hit(A) < \hit(B)) \approx  \frac{1}{2} + \E^{\q}\left[\int_0^\tau \frac{\lambda_A^*(t)-\lambda^*_B(t)}{2} \exp\left(-\int_0^t \lambda^*_{A\cup B}(s) ds\right)dt\right], \label{eq:4_ab_estimator}
\end{align}
where $\tau>0$ can either be set as a constant or could be dynamically determined on a per sequence basis based on some precision threshold.
It should be noted that because $\tau < \infty$ and the integration bounds are truncated, this estimate is no longer unbiased.

\section{Experiments} \label{sec:4_experiments}

We investigate the effectiveness of our novel importance sampling regime in the context of estimating hitting time, ``A before B,'' and marginal mark distribution queries, while conditioning on partially observed sequences. We find that across both synthetic and real settings as well as parametric and neural-network-based models that our importance sampling estimator dramatically reduces variance compared to naive sampling and results in a much lower error on average. Furthermore, we demonstrate that, on average, these gains in performance outweigh any potential increases in computation time.

\paragraph{Ground Truth}
Computation of any arbitrary query $\prob(\hist_\tau\in\mathcal{Q})$ to arbitrary precision is  intractable in the general case. Given this, in our experiments we compute our queries with an unbiased estimator to high precision  using a large amount of computation, with much higher precision than any of the methods and scenarios evaluated for a given experiment. We refer to the result of this high-precision computation as ``ground truth'' below.

\paragraph{Metrics of Interest}
There are two primary metrics with which we judge query estimation procedures: mean relative absolute error and relative efficiency (or variance reduction should one of the estimators be biased). The former is defined as the mean of $|\pi-\hat{\pi}|/\pi$, where $\pi=\prob(\hist_\tau\in\mathcal{Q})$ and $\hat{\pi}$ is some estimator of $\pi$, over different queries (and potentially models). This particular form of error is chosen to offset the fact that $\pi\in[0,1]$, which can lead to naturally closer estimates should $\pi$ be close to $0$ or $1$. The latter metric of interest is the relative efficiency (or variance reduction) of importance sampling compared to naive sampling. This is calculated by dividing the variance of the naive estimator (calculated using ground truth: $\pi(1-\pi)$) by the variance of the importance sampled estimator (calculated empirically). As an example, a value of 5 for this metric indicates that, on average, 5 times as many samples are needed for naive estimation to achieve an estimator variance as low as that of importance sampling.

\subsection{Real-world Experiments}

\paragraph{Datasets}

\begin{table}
\centering
\caption{Real-world Dataset Summary Statistics}
\begin{tabular}{l @{\hspace{0.8cm}} c c c}
\toprule
Dataset &  \# Sequences & $\tau_{max}$ & \# Marks\\
\midrule
MovieLens & 34,935 & 43,000 & 182 \\
MOOC & 6,863 & 715 & 97 \\
Taobao & 17,777 & 192 & 1,000 \\
\bottomrule
\end{tabular}
\label{tab:datasets}
\end{table}

We conduct our real-world experiments on three sequential user-behavior datasets. In all three, a sequence is defined as the records generated by a single individual.
The \textbf{MovieLens 25M} dataset \citep{harper2015movielens} contains records of user-generated movie reviews alongside a rating. Marks represent the categories under which a reviewed movie can be classified as.
The \textbf{MOOC} dataset \citep{kumar2019predicting} is a collection of online user-behaviors for students taking an online course. Marks represent the type interaction a student has performed.
Lastly, the \textbf{Taobao} user behavior dataset \citep{zhu2018learning} contains page-viewing records from users on an e-commerce platform. Marks are defined as the category of the item being viewed, with categories outside of the top 1,000 most frequent being discarded.
All datasets were split into 75\% training, 10\% validation, and 15\% test splits for model fitting and experiments.
Summary statistics for these datasets can be found in \cref{tab:datasets}. All preprocessing details for these three datasets can be found in \cref{sec:4_exp_appendix}.

\paragraph{Models}
All real-world experiments use neural Hawkes models \citep{mei2017neural}, one trained for each dataset. Each model was trained to convergence on the training split with stability/generality ensured via the validation split. All training and model details can be found in \cref{sec:4_exp_appendix}.

\paragraph{Hitting Time Queries:}
For each dataset, we randomly sample 1,000 different sequences $\hist_\tau$. For each sequence, we condition on the first five events, $\hist_{[0,T_5]}$, and evaluate a hitting time query for the remaining future.\footnote{All experiments evaluate necessary integrals with the trapezoidal rule. For more details, see \cref{sec:4_exp_appendix}.} The specific hitting time query asked is $\prob(\hit(k)\leq t\sep\hist_{[0,T_5]})$ where $k:=M_6$ and $t:=10\times T_6$ for $(T_6,M_6)\in\hist(T)$. 

We compared estimating this query with naive sampling and importance sampling using varying amounts of samples: $\{2,4,10,25,50,250,1000\}$. Mean RAE compared to ground truth (estimated using importance sampling with 5000 samples) can be seen in \cref{fig:4_hit_err}. We witness roughly an order of magnitude of improvement in performance for the same amount of samples. Primarily, we attribute this improvement to the fact that naive sampling only collects binary values, whereas our proposed procedure collects much more dense information over the entire span $[T_5, t]$.

We also analyze the relative efficiency of our estimator compared to naive sampling. For each query asked, the efficiency was estimated using 5000 importance samples. The results can be seen in \cref{fig:4_hit_eff}.  
We achieve a dramatic decrease in variance by several orders of magnitude, in the majority of contexts, across all datasets. Interestingly, it appears that the efficiency is correlated with the underlying ground truth value $\pi$. We believe this may be due to the form of the importance sampling estimator: $1-\exp\left(-\int_0^t \lambda_k^*(s)ds\right)$. Since the intensity function is non-negative, it is simple for the model to produce estimates close to 0; however, to producing values close to 1 requires the integral to tend towards infinity. 

\begin{figure}
\centering{\phantomsubcaption\label{fig:4_hit_err}\phantomsubcaption\label{fig:4_hit_eff}}
\includegraphics[width=0.85\columnwidth]{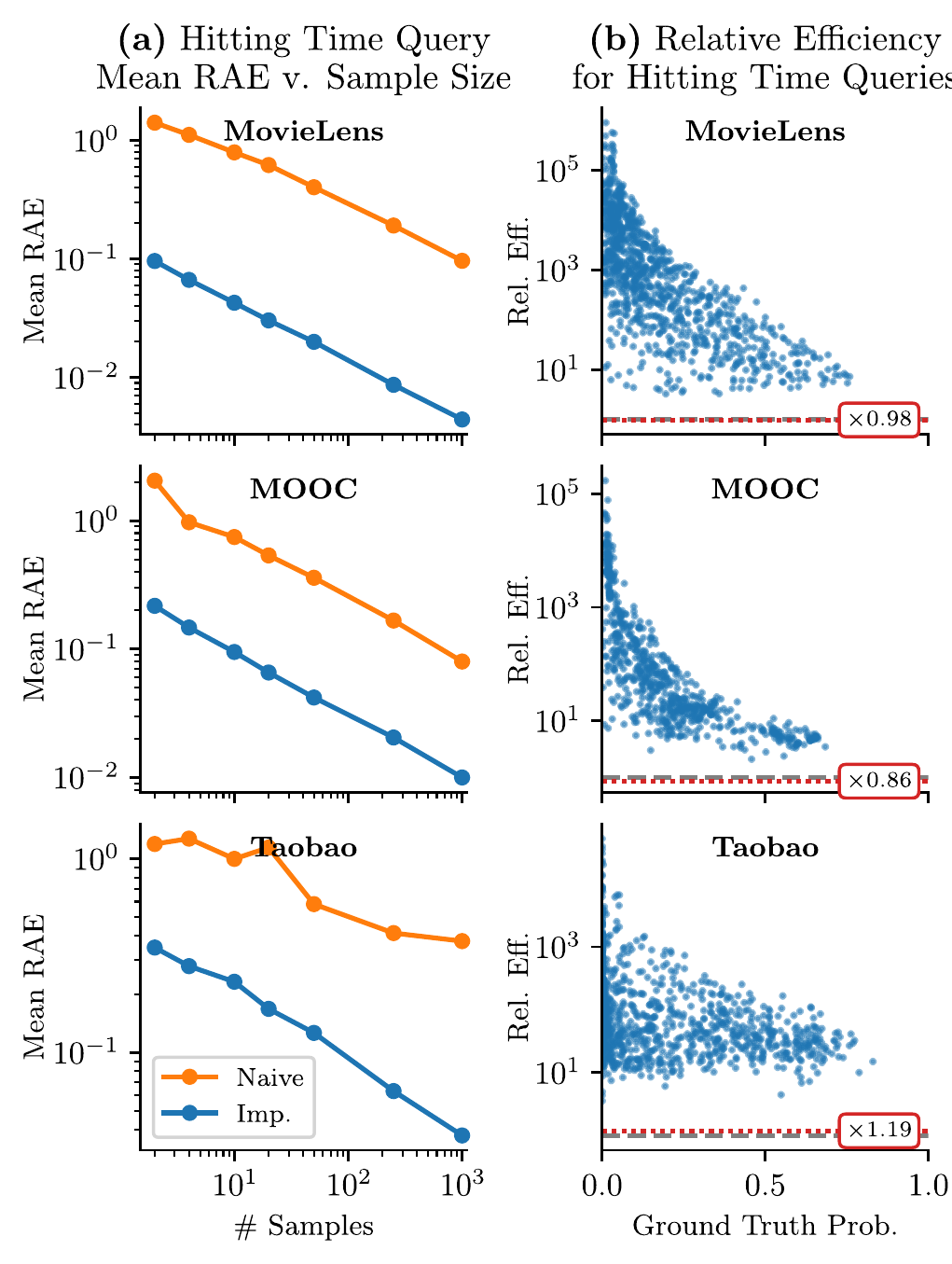}
\caption{Results from 1000 different hitting time queries evaluated on models trained on three different datasets. (a) Average relative absolute error for naive and importance sampling shown in comparison to number of sampled sequences used. (b) Estimated relative efficiency values for importance sampling versus naive sampling plotted against ground truth hitting time query values. Gray dashed lines indicate an efficiency of 1. Red lines with text box show the mean multiplicative increase in computation time for importance sampling.}
    \label{fig:4_hit_plots}
\end{figure}

\paragraph{``A before B'' Queries:}
Similar to the hitting time experiments, for ``A before B'' queries we similarly sample 1000 random test sequences and condition on the first five events $\hist_{[0,T_5]}$. Then, we estimate the query $\prob(\hit(A) < \hit(B)\sep\hist_{[0,T_5]})$ where $A$ and $B$ are randomly chosen to contain one third of the mark space $\vocab$.  

We compared estimating this query with naive sampling and importance sampling using varying amounts of samples: $\{2,4,10,25,50,250\}$. We utilized the truncated importance sampled estimator, \cref{eq:4_ab_estimator}, where $\tau$ is chosen dynamically for each sequence such that a maximum difference of 0.01 is allowed between the upper and lower bounds.
Mean RAE compared to ground truth (estimated using naive sampling with 5,000 samples) can be seen in \cref{fig:4_ab_err}.\footnote{Importance sampling would have been used for ground truth here; however, it is more sound to use an unbiased estimator for ground truth.} Like the hitting time results, we can see roughly an order of magnitude improvement in performance. Some results indicate that the limiting factor is the precision threshold for choosing $\tau$ (e.g., see MovieLens results). We also see a similar variance reduction relative to previous experiments, shown in \cref{fig:4_ab_eff}. Here, the runtime cost is much greater as we have to accumulate an integral over an indefinite amount of time; however, we can see that on average it is still very much ``worth it'' to utilize this framework over naive sampling as evidenced by all of the blue dots above the red line.

\begin{figure}
\centering{\phantomsubcaption\label{fig:4_ab_err}\phantomsubcaption\label{fig:4_ab_eff}}
\includegraphics[width=0.85\columnwidth]{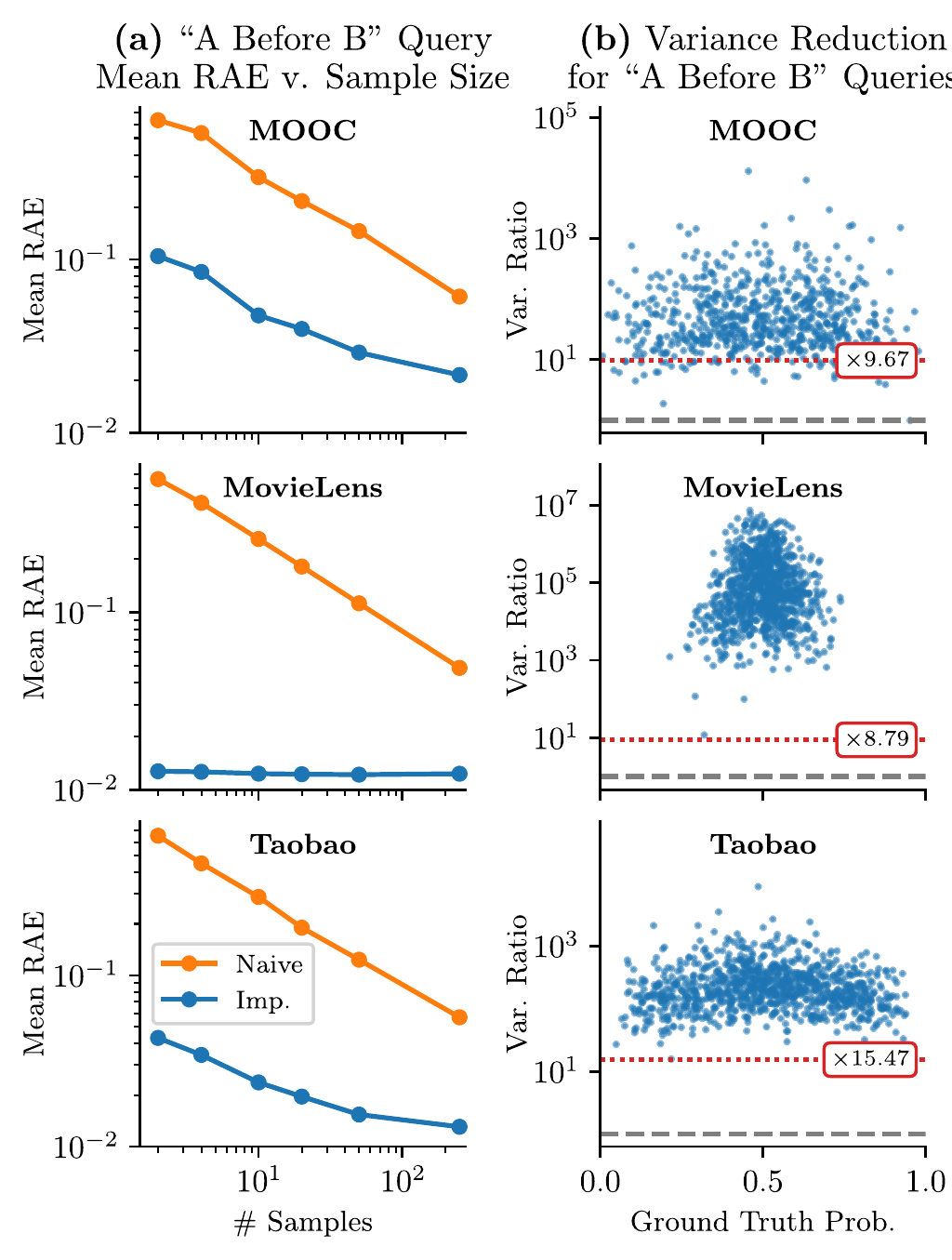}
\caption{Same setup as seen in \cref{fig:4_hit_plots} with the same models and datasets only applied to ``A before B'' queries. Results for (b) are presented as ``variance reduction'' instead of ``relative efficiency'' since our derived estimator for importance sampling is biased due to truncating the integral in \cref{eq:4_ab_estimator}.}
\label{fig:4_ab_plots}
\end{figure}

\paragraph{Marginal Mark Distribution Queries:}
We additionally performed $n^\text{th}$ marginal mark distribution queries in much the same vein as the hitting time queries. Due to space limitations, the majority of the details and results can be found in \cref{sec:4_exp_appendix}. That being said, we found that the resulting relative efficiencies for these queries to be much less than those of the other queries, but still more efficient than naive as \cref{thm:eff} suggests. Across the datasets, the median relative efficiency ranged from 1.9 to 2.7.
We speculate this to be due to the fact that the bounds of integration in the estimator are tied to sampled event times rather than being static values, inducing quite a lot of potential variance in the estimator. 

\subsection{Synthetic Experiments}
\begin{figure}
    \centering
    \includegraphics[width=0.85\columnwidth]{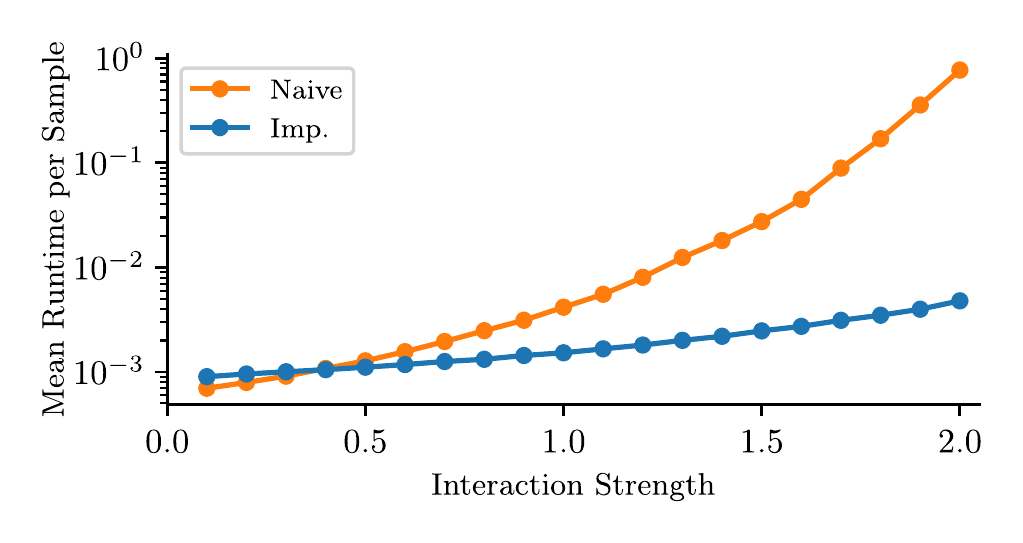}
    \caption{Average wall-clock time taken to generate a \underline{single} sample under naive and importance sampling for hitting time queries with 1,000 randomly instantiated parametric exponential-kernel Hawkes models with different scalings of ``interaction strength'' (i.e., amount of modeled cross-mark interaction).}
    \label{fig:4_interaction_strength}
\end{figure}

For artificial experiments, we wanted to investigate the trends of our estimation procedures for a variety of queries over \emph{many different} models\textemdash{}something that is difficult to do with real-world data as we typically only have access to one model trained on a given dataset.
As such, we primarily focus on randomly instantiated parametric Hawkes processes with both exponential kernels and Gamma decay kernels.
Under this setting, we were able to recreate similar findings in terms of estimation error and efficiency for the types of queries evaluated with real-world data. Refer to \cref{sec:4_exp_appendix} for more details and in depth results.

In addition to these expected results, we also sought to investigate how different aspects of the underlying model affect the estimation procedure. In particular, we measured the average wall-clock time taken to generate samples for naive estimation and importance sampling as a function of how much cross-mark interaction is present. We modulate the \emph{interaction strength} in these generated models by changing the scale of the randomly generated mark-to-mark intensity parameters. The results can be seen in \cref{fig:4_interaction_strength} where we evaluated both estimation procedures on random hitting time queries for 1,000 different generated models with each across a span of interaction strengths. As more cross-mark excitement is encouraged by a model, the runtime it takes to sample a sequence over the same observation window becomes much longer in general; however, importance sampling counters this trend due to zeroing out (potentially several) marked intensities via $\q$, thus barring events from happening. From these results, we can see that our importance sampling procedure is more robust to the underlying dynamics of a model over sampling windows fixed in time.

\section{Conclusion}

In this chapter, we extended the class of queries described in Chapter 3 to the continuous-time setting. Additionally, a similar importance sampling technique was also introduced for marked temporal point processes with relative efficiency guarantees derived. Extensive experiments were conducted that showcase a definitive reduction in estimator variance, often with several orders of magnitude difference compared to naive approaches. 

The text of this chapter is based on the publication:
\begin{center}
\vspace{-1.5em} \textit{Probabilistic Querying of Continuous-Time Event Sequences
} \citep{boyd2022probabilistic}.
\end{center}

The author of this dissertation is the primary author for this publication, and was responsible for all theoretic contributions. Additionally, the author was responsible for implementing and executing all experiments that involved neural MTPPs. Yuxin Chang contributed the implementation and execution of the classical parametric-based MTPPs with oversight from the author. Overall conception of the project and some of the writing was shared among all authors.

\chapter{Marginalization of Marked Temporal Point Processes to Account for Censoring}

\noindent

An important practical aspect of working with real-world event data is that censoring of observations can occur in a number of different ways. 
For example, a common example of right-censoring often occurs in survival analysis (a sub-field of temporal point processes) in which a patient's event of interest is unobserved due to the end of a data collection period. This particular type of censoring is  well-studied and there are well-known methods for accommodating this during training and inference. More recently, there has been work on handling broader categories of censoring for neural MTPP models, for example, censoring where each event has a type-specific probability of being missing \citep{mei2019imputing}.

In this chapter we focus on a different problem, the problem of making predictions when some, or all, marks are censored over (potentially open-ended) intervals of time, i.e., there is partial censoring of a specific subset of marks. We will refer to this type of censoring as {\it mark-censoring}. 
To our knowledge, there has been no prior work that addresses this problem of adapting MTPPs to mark-censored sequences at inference time.
The problem is motivated by the real-world scenario where an MTPP model has been trained on a known set of marks with fully-observed data, but where at prediction time some of the marks (and their associated timing) are no longer observable. 
For example, in medical data analysis, certain types of events that were measured in the training dataset at a particular hospital might no longer be recorded when the model is deployed at a different hospital. %
Or, in system monitoring, all events of a certain type could be censored over a window of time  due to events such as network and power outages, and accommodating such gaps is important for modeling future dynamics once outages are resolved.

Previous work such as \citet{linderman2017bayesian} focuses on special cases of missingness patterns and/or only applies to specific model architectures (as will be discussed in more detail in Section \ref{sec:5_related}). In contrast, our work is able to handle all of the scenarios shown in Figure \ref{fig:5_example_censoring}.
The basis of our approach is a novel marginalization technique that can correct the intensity for the censoring of marks. Our proposed method is {\it model-agnostic} in that it can be applied to any MTPP with a well-defined intensity function. We demonstrate this by employing our method on different types of MTPP models and evaluating predictive performance and simulation behavior under a censored-mark regime.

\begin{figure}
    \centering
    \includegraphics[width=0.95\linewidth]{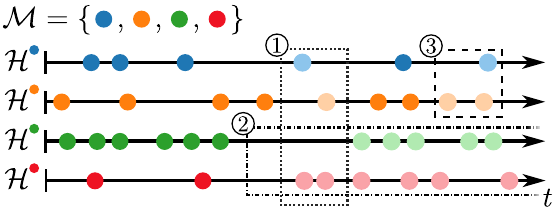}
    \caption{Visualization of an example sequence with four possible marks. $\vocab$ is the vocabulary of possible event types, $\hist^k$ is the history of events with types equal to $k$. Boxes  over  sequences represent different  modes of censoring that could occur during generation: (1) mark-agnostic censoring for a particular interval, (2) censoring of green and red marks over an open interval, and (3) censoring of blue and orange marks over a finite interval. The occurrence of an event or the total count of events during an interval is not known, differentiating our scenarios from the typical ``interval censoring'' in survival analysis or MTPPs.}%
    \label{fig:5_example_censoring}
\end{figure}

\section{Related Work}\label{sec:5_related} 

A broad range of temporal censoring scenarios 
have been studied in the literature, such as asynchronous event times \citep{upadhyay2018deep, trouleau2019learning} and interval-censored point process data \citep{fan2009local, rizoiu2022interval}. 
Here we focus the discussion of related work to MTPPs where the marks are from a fixed vocabulary.
Existing work on missingness %
in this context can broadly be divided into three categories.

The first category considers various incomplete intervals, regardless of event types, and focuses on novel tasks such as imputing missing events and sequential representation learning. For example, \citet{shchur2019intensity} proposed a flow-based mixture model that enables closed-form sampling and handles missing data through imputation. \citet{xu2017learning} assumes that a proportion of each short doubly-censored event sequence is observed, and in turn proposes a sampling-stitching data synthesis method based on parametric Hawkes processes to sample long training sequences that improve predictions. 

The second category considers the scenario in which each individual event, regardless of mark or time of occurrence, has a chance of being censored.
For the Hawkes process, for example, sampling methods were developed to identify latent structure in the data \citep{shelton2018hawkes} or to correct for biased marks that are underrepresented \citep{zhou2021multivariate}. In  neural settings, \citet{gupta2021learning, gupta2022modeling} proposed the use of two MTPPs to model missing events in order to make better predictions. 
\citet{mei2019imputing} proposed bidirectional-LSTM models that are conditioned on future observations to apply particle smoothing to impute unobserved events.

The third category of prior work assumes that events are observed but the mark and/or the exact event time is unknown. For instance, \citet{deutsch2020abc} developed an approximate Bayesian
 algorithm to fit Hawkes processes in the presence of noisy event times, and \citet{calderon2021linking} 
 addressed partially interval-censored Hawkes processes, where the total event counts on the censored intervals are available. For the case of Hawkes models, \citet{linderman2017bayesian} imputed latent marks and developed a sequential Monte Carlo approach for latent Hawkes processes that can also be applied to multiple types of censoring.

In summary, previous approaches to censoring in MTPPs either focus on specific types of missingness mechanisms during training time or focus on one specific type of model such as parametric or neural Hawkes process models. In contrast, our approach considers a broad range of interval- and mark-censoring mechanisms (see \cref{fig:5_example_censoring}) and is model-agnostic in that it can work with any MTPP model with a marked intensity function at prediction time. Furthermore, the results of our method yield a well-defined intensity function of a MTPP that can be used just the same as any other MTPP, meaning various statistics can be computed such as expected next event (time and mark), log likelihood of partially observed sequences, etc.

\section{Mark-Censored Temporal Point Processes}\label{sec:5_censoring}

For relevant preliminary information such as notation and details on marked temporal point processes (MTPPs), please refer to \cref{sec:4_prelims}.

\subsection{Problem Statement} %
Assume that we have access to a trained MTPP with intensity functions $\lambda^*_k(t)$ for $k\in\vocab$ which defines a probability measure $\prob$. We are interested in performing inference on such a model in the presence of censoring. In particular, we are interested in a type of censoring we term \emph{mark-censoring} in which only events of types $k\in\obs\subset\vocab$ are observed, while all events of types $k\in\cen:=\vocab\setminus\obs$ are censored and unobserved. 
In particular, we assume in mark-censoring that we know (a) the time-interval where censoring occurs and (b) which kinds of marks are missing (e.g., knowing the time intervals and colors of marks in the censoring boxes displayed in \cref{fig:5_example_censoring}). 
Below we develop the framework for the case when censoring takes place over all of time (i.e., $t\in[0,\infty)$); however, as we will discuss later in this section, the general approach can be directly applied to a range of more complicated censoring schemes (such as those illustrated in \cref{fig:5_example_censoring}).

\paragraph{On Censoring}
The term ``censoring'' can be quite a loaded concept with regards to statistical models. 
In our work we assume the absence of certain marks over a time interval to correspond to \emph{missing completely at random} (MCAR) \citep{heitjan1996distinguishing}, i.e., we assume that the realized sequence $\hist$ (both observed and unobserved portions) are independent of \emph{why} it is censored in the first place. We leave handling of more informative censoring to future work.

\begin{figure*}
    \centering
    \includegraphics[width=0.99\linewidth]{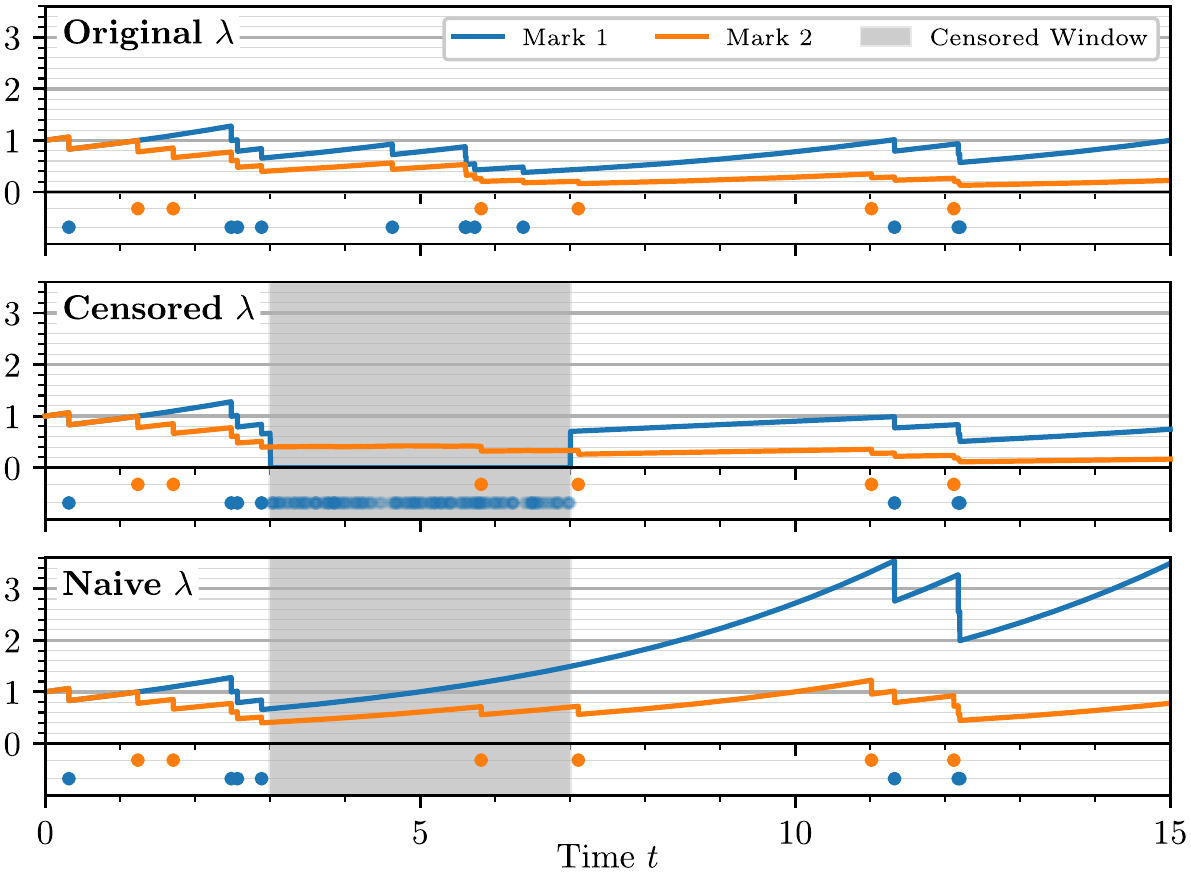}
    \caption{Intensity visualizations (lines) alongside conditioned sequences (dots) for a sequence sampled from a self-correcting point process (top), the same process with blue marks censored from time 3 to 7 (middle), and the naive intensity results for the censored sequence (bottom). The middle sequence displays both the observed sequence as opaque dots and the various censored continuations sampled from the importance distribution as transparent dots. Note that the intensity of the censored mark (blue) after the censoring interval (at time 7) does not necessarily equal the intensity before censoring (at time 3).}
    \label{fig:5_example_intensities}
\end{figure*}

\subsection{Censored Intensity Function} 
Since we have access to the original MTPP, which models the entire distribution for event sequences as a whole, embedded within this model is a well-defined sub-process that represents an MTPP that only observes events of types $k\in\obs$. We refer to this embedded model as a \emph{mark-censored sub-process}. This sub-process can be thought of as the original model with the censored information marginalized out of it. Had this sub-process been our intended model from the beginning, we could have achieved comparable results by censoring the original training data and training a model on what remains. There is one key difference, however, which is that the mark-censored sub-process still allows for conditioning on events of types $\cen$ even if they are censored moving forward in time (e.g., in the case that the censoring interval only started at time $t > 0$ instead of at $t=0$\textemdash{}see case 3 in \cref{fig:5_example_censoring}).

The censored sub-process is a fully-fledged MTPP, and as such it has its own set of marked intensity functions. We will denote these as $\underline{\lambda}^*_k(t)$ for $k \in \obs$ (should $k\in\cen$ then $\underline{\lambda}^*_k(t)=0$). Likewise, the total intensity for a censored sub-process is defined as $\underline{\lambda}^*(t):=\underline{\lambda}^*_\obs(t)$. These will be referred to as the \emph{censored intensity} from here forward. Note that for any MTPP with well-defined intensity functions $\lambda_k^*$, by the point process superposition property it is justified for the censored intensity $\underline{\lambda}_k^*$ to exist for any arbitrary censoring \citep{daley2003introduction}.

\paragraph{High Level Intuition for Censored Intensity}
Later in this section we will present a formal definition of the censored intensity, as well as a tractable estimator for it that solely relies on the original underlying MTPP with likelihood $\mathcal{L}^\prob$ and intensity $\lambda^*_k(t)$ functions for $k\in\vocab$. However, prior to presenting these, we will first give an informal overview to help understand the arguments at a high level.

We start by recognizing that we are interested in obtaining the intensity at time $t$ for a censored point process where we only observe events of types $k\in\obs$ and no events of types $k\in\cen$. To accomplish this, we would prefer to directly marginalize out all possible sequences of $\hist^\cen_t$; however, for most MTPPs this is unobtainable analytically. Instead, we can approximate the censored intensity $\underline{\lambda}^*_k(t)$ for $k\in\obs$ with the original intensity by simply sampling a possible sequence $\tilde{\hist}^\cen_{t-}$ from the original point process:
\begin{align}
\underline{\lambda}^*_k(t) \approx \lambda_k(t\sep \hist^\obs_{t-}, \tilde{\hist}^\cen_{t-}),
\end{align}
where $\tilde{\hist}^\cen_{t-} \sim \prob(\cdot \sep \hist^\obs_{t-})$. Naturally, we cannot directly perform this sampling, so we will do the next best thing and simply sample from the model as usual except that we will prevent any new event with types $k\in\obs$ from occurring (i.e., set $\lambda_k^*(t)=0$ when sampling).

To get a better approximation, this should be done many times with different sampled trajectories: $\tilde{\hist}^{\cen,(i)}_{t-}$ for $i = 1,\dots,n$. One could simply compute a standard average where $\underline{\lambda}_k^*(t)\approx 1/n \sum_{i=1}^n \lambda_k(t\sep \hist^\obs_{t-}, \tilde{\hist}^{\cen,(i)}_{t-})$; however, since we did not sample $\tilde{\hist}^{\cen,(i)}_{t-}$ perfectly from the model without adjustments we must account for the fact that some samples will be more likely under the original model than others.

As such, we can instead perform a weighted average:
\begin{align}
\underline{\lambda}_k^*(t) \approx \frac{\sum_{i=1}^n \lambda_k\left(t\sep \hist^\obs_{t-}, \tilde{\hist}^{\cen,(i)}_{t-}\right)\omega\left(\tilde{\hist}^{\cen,(i)}_{t-}\right)}{\sum_{i=1}^n \omega\left(\tilde{\hist}^{\cen,(i)}_{t-}\right)}
\end{align}
where $\omega(\cdot)$ determines the weight of a sampled trajectory. We define this weight to be the probability of the imposed sampling restriction (i.e., no \emph{new} events of types $k\in\obs$ allowed) being satisfied under the original model. This can be computed for a given sample and is equal to
\begin{align}
\omega(\tilde{\hist}^\cen_{t-}) = \exp\left(-\int_0^t \lambda_\obs(s \sep \hist^\obs_{s-}, \tilde{\hist}^\cen_{s-}ds \right).
\end{align}

As an illustration of this censored intensity $\underline{\lambda}_k^*(t)$,   \cref{fig:5_example_intensities}  shows the original, censored, and naive intensities for an example sequence sampled from a self-correcting process. After the censoring interval (in gray) ends at $t=7$, the censored intensity tracks the original true intensity (top) much more closely than the naive intensity (bottom) does. In this context, naive intensity is referring to the original intensity being computed while treating the partially observed sequence $\hist^\obs$ as if it were the fully observed sequence $\hist$.

The approximation of $\underline{\lambda}_k^*(t)$ is for finite samples and is a ratio estimator \citep{tin1965comparison}. Taking the limit as $n\rightarrow\infty$ converts each summation into an expected value with respect to the proposal distribution, as ratio estimators are consistent. This description matches what will formally be derived below in \cref{eq:5_censored_intensity_result}.

\paragraph{Formal Definition of $\underline{\lambda}$}
Without loss of generality, we will assume that any prior events being conditioned on have been shifted to end at $t=0$ such that $\hist_0$ contains all of the previous events. It can be shown that the censored intensity function for the sub-process is just a specific marginalization of the original intensity function:
\begin{align}
& \underline{\lambda}^*_k(t) := \underline{\lambda}_k(t\sep\hist_0,\hist^\obs_{(0,t)}=\emptyset) \text{ for } k \in \obs \\
& = \lim_{\Delta \downarrow 0} \frac{1}{\Delta} \Prob(\hit(k)\in [t, t+\Delta) \sep \hist_0,\hist^\obs_{(0,t)}=\emptyset) \\
& = \lim_{\Delta \downarrow 0} \frac{1}{\Delta} \E^\prob_{\hist^\cen_{(0,t)}\sep \hist_0, \hist^\obs_{(0,t)}=\emptyset} \left[\Prob(\hit(k)\in [t, t+\Delta) \sep \hist_0,\hist^\obs_{(0,t)}=\emptyset, \hist^\cen_{(0,t)})\right] \\
& = \lim_{\Delta \downarrow 0} \frac{1}{\Delta} \E^\prob_{\hist_{t-}\sep \hist_0, \hist^\obs_{(0,t)}=\emptyset} \left[\Prob( T_i\in [t, t+\Delta),  M_i=k \sep \hist_{t-})\right], \text{ where } |\hist_{t-}| = i-1\\
& = \E^\prob_{\hist_{t-}\sep \hist_0, \hist^\obs_{(0,t)}=\emptyset} \left[ \lim_{\Delta \downarrow 0} \frac{1}{\Delta} \Prob( T_i\in [t, t+\Delta),  M_i=k \sep \hist_{t-})\right] \text{ by DCT} \\
& = \E^\prob_{\hist_{t-}\sep \hist_0, \hist^\obs_{(0,t)}=\emptyset} \left[\lambda^*_k(t)\right] 
\end{align}
where in this context, $\hit(k)$ refers to the first occurrence time of event $k$ after $t=0$, and $\hist_{t-}:=\hist_0\cup\hist^\obs_{(0,t)}\cup\hist^\cen_{(0,t)}$. The Dominated Convergence Theorem (DCT) holds true because we assume that there exists some value $D$ that is greater than $\lambda^*_k(t)$ for any given $t$. Note that this assumption is typically made to sample from arbitrary MTPPs.

\subsection{Tractable Estimation of Censored Intensity}

To approximate the censored intensity function $\underline{\lambda}_k^*(t)$, we need to perform a Monte Carlo estimation on the above derived expected value, $\E^\prob_{\hist_{t-}\sep \hist_0, \hist^\obs_{(0,t)}=\emptyset} \left[\lambda^*_k(t)\right]$. The only issue is that we cannot directly sample $\hist_{t-}\sep \hist_0, \hist^\obs_{(0,t)}=\emptyset$ under $\prob$ due to the autoregressive nature of MTPPs.\footnote{Due to the conditionals, sampling $\hist_{t-}\sep \hist_0, \hist^\obs_{(0,t)}=\emptyset$ is equivalent to sampling $\hist^\cen_{(0,t)}\sep \hist_0, \hist^\obs_{(0,t)}=\emptyset$.}

Consider the proposal measure $\q$ defined by a MTPP with intensity function
\begin{align}
\mu^*_k(t) = \begin{cases} 0 & \text{ if } k \in \obs \text{ and } t \geq 0 \\
\lambda^*_k(t) & \text{ otherwise.} \label{eq:5_proposal_intensity}
\end{cases}
\end{align}
This can essentially be thought of as the original MTPP prior to censoring, and then during sampling it only produces sequences of events that cannot be observed. The likelihood for a sequence under this distribution is computed as follows:
\begin{align}
& \mathcal{L}^\q(\hist_{t-} \sep \hist_0) \label{eq:5_proposal_likelihood} \\
& = \left[\prod_{i=1}^N \mu^*_{ M_i}( T_i)\right]\exp\left(-\int_0^t \mu^*(s)ds \right)  \\
& = \left[\prod_{i=1}^N \lambda^*_{ M_i}( T_i) \ind( M_i \in \cen)\right]\exp\left(-\int_0^t \lambda^*_\cen(s)ds \right)
\end{align}
where $|\hist_{(0,t)}|=|\hist^\cen_{(0,t)}|=N$.
Note that the proposal distribution has the same support as $\prob$ for $\hist_{t-}\sep\hist_0,\hist^\obs_{(0,t)}=\emptyset$.\footnote{It follows that $\E^\q_{\hist_{t-} \sep \hist_0}\left[\ind(\hist^\obs_{(0,t)}=\emptyset)\right]=1$, which becomes useful for subsequent derivations.}

Using importance sampling with this proposal distribution, we can see that the censored intensity becomes tractable:
\begin{align}
\underline{\lambda}^*_k(t) & = \E^\prob_{\hist_{t-}\sep \hist_0, \hist^\obs_{(0,t)}=\emptyset} \left[\lambda^*_k(t)\right] \\
& = \E^\q_{\hist_{t-} \sep \hist_0} \left[\lambda^*_k(t)\frac{\mathcal{L}^\prob(\hist_{t-} \sep \hist_0, \hist^\obs_{(0,t)}=\emptyset)}{\mathcal{L}^\q(\hist_{t-} \sep \hist_0)}\right] \\
& = \E^\q_{\hist_{t-} \sep \hist_0} \left[\lambda^*_k(t)\frac{\ind(\hist^\obs_{(0,t)}=\emptyset)\mathcal{L}^\prob(\hist_{t-} \sep \hist_0)}{\prob(\hist^\obs_{(0,t)}=\emptyset\sep \hist_0)\mathcal{L}^\q(\hist_{t-} \sep \hist_0)}\right] \\
& = \frac{\E^\q_{\hist_{t-} \sep \hist_0} \left[\lambda^*_k(t)\frac{\left[\prod_{i=1}^N \lambda^*_{ M_i}( T_i) \right]\exp\left(-\int_0^t \lambda^*(s)ds \right)}{\left[\prod_{i=1}^N \lambda^*_{ M_i}( T_i) \ind( M_i \in \cen)\right]\exp\left(-\int_0^t \lambda^*_\cen(s)ds \right)}\right]}{\prob(\hist^\obs_{(0,t)}=\emptyset\sep \hist_0)} \\
& = \frac{\E^\q_{\hist_{t-} \sep \hist_0}  \left[\lambda^*_k(t)\exp\left(-\int_0^t \lambda^*_\obs(s)ds \right)\right]}{\prob(\hist^\obs_{(0,t)}=\emptyset\sep \hist_0)}.
\end{align}

Now the expected value can be approximated with easy-to-access Monte Carlo samples. The only immediate problem is evaluating $\prob(\hist^\obs_{(0,t)}=\emptyset\sep \hist_0)$ as this does not have a closed form solution; however, this statement takes the form of a restricted-mark probabilistic query introduced in \cref{sec:4_gen_mark_queries}. 
Conveniently, we can actually utilize the exact same proposal measure $\q$ as specified in \cref{eq:5_proposal_intensity,eq:5_proposal_likelihood} to represent $\prob(\hist^\obs_{(0,t)}=\emptyset\sep \hist_0)$ as a tractable expected value:
\begin{align}
\prob(\hist^\obs_{(0,t)}=\emptyset\sep \hist_0) & = \E^\prob_{\hist_{t-}\sep \hist_0}\left[\ind(\hist^\obs_{(0,t)}=\emptyset)\right] \\
& = \E^\q_{\hist_{t-}\sep \hist_0} \left[\ind(\hist^\obs_{(0,t)}=\emptyset)\frac{\mathcal{L}^\prob(\hist_{t-}\sep \hist_0)}{\mathcal{L}^\q(\hist_{t-} \sep \hist_0})\right] \\
& =\E^\q_{\hist_{t-}\sep \hist_0}\left[\exp\left(-\int_0^t \lambda^*_\obs(s)ds\right)\right].
\end{align}
Thus, the censored intensity can be ultimately represented as a ratio of two expected values:
\begin{align}
\hspace{-0.975em}\implies \underline{\lambda}^*_k(t) & = \frac{\E^\q_{\hist_{t-}\sep \hist_0} \left[\lambda^*_k(t)\exp\left(-\int_0^t \lambda^*_\obs(s)ds \right)\right]}{\E^\q_{\hist_{t-}\sep \hist_0}  \left[\exp\left(-\int_0^t \lambda^*_\obs(s)ds \right)\right]}. \label{eq:5_censored_intensity_result}
\end{align}
In practice, this censored intensity can be approximated using Monte Carlo (MC) estimates for both the numerator and denominator.

It is worth reiterating that this estimator, which accounts for the censoring of marks $\cen$ at inference time, only requires a  trained MTPP along with samples from it. No further training, additional models, or specific architectures are required to properly deal with the censoring.

\paragraph{More Complex Censoring Regimes}
All of the derivations thus far have been focused on having a static set of marks $\cen$ being censored for an indefinite amount of time; however, there are many other types of censoring that can occur for a given MTPP. For example, the censoring could occur over a specific window of time for either some or all marks $\vocab$. This could occur, for instance, in settings where the connection is briefly lost to some or all sensors in a system. Furthermore, censoring could occur multiple times over different windows, and the marks being censored across each window need not be the same from censoring to censoring. See \cref{fig:5_example_censoring} for example censoring scenarios.

We can easily extend our previous results to cover the most general case allowing for censoring over arbitrarily many time windows and arbitrarily different censored marks. To do so, first we will define the censoring schedule. The observed and censored marks, $\obs$ and $\cen$, are no longer static and will potentially change over time. This will be represented via the set-valued functions $\obs(t), \cen(t) \subset \vocab$ for $t \geq 0$. This results in the proposal MTPP now being characterized by the intensity function $\mu_k^*(t)=\lambda^*_k(t)\ind(k\in\cen(t))$. Lastly, the resulting censored intensity estimate also accommodates this dynamic censoring:
\begin{align}
\underline{\lambda}^*_k(t) & = \frac{\E^\q_{\hist_{t-}\sep \hist_0} \left[\lambda^*_k(t)\exp\left(-\int_0^t \lambda^*_{\obs(s)}(s)ds \right)\right]}{\E^\q_{\hist_{t-}\sep \hist_0} \left[\exp\left(-\int_0^t \lambda^*_{\obs(s)}(s)ds \right)\right]}. 
\end{align}
This result is achieved effectively for free as the censored intensity $\underline{\lambda}_k^*(t)$ in the static setting is technically defined individually for any given moment in time $t$, making the swap from $\obs$ to $\obs(t)$ and $\cen$ to $\cen(t)$ for each $t$ well defined.  

\paragraph{More Complex Mark Spaces $\vocab$}
Our setting of interest has the marks being modeled come from some discrete, finite mark space $\vocab:=\{1,\dots,M\}$; however, that does not have to be the case. We can easily extend our method to apply for more complex mark spaces. Consider an arbitrary mark space $\vocab$ which could be finite, continuous, high-dimensional, etc. and let $\nu$ be a reference measure for $\vocab$ (e.g., the Lebesgue measure for $\vocab\equiv\mathbb{R}$). Assume we have a MTPP model with marked intensity function $\lambda^*(t, m)$ for $m \in \vocab$, and that under our framework we know the observed and censored portions of the mark space at any given time, $\obs(t) \subset \vocab$ and $\cen(t):=\vocab\setminus\obs(t)$ respectively. From this, the censored intensity defined in \cref{eq:5_censored_intensity_result} can be readily used by letting $\lambda_{\obs(t)}^*(t):=\int_{\obs(t)}\lambda^*(t,m)d\nu(m)$ which can either be computed analytically or estimated with Monte-Carlo samples. The proposal distribution stays the same as previously defined and samples from it can be achieved easily using either rejection sampling on top of the typical thinning procedure.

\paragraph{Bias and Variance Analysis of Censored Intensity Estimator}
In practice, the numerator and denominator of \cref{eq:5_censored_intensity_result} 
are estimated with Monte-Carlo samples, resulting in the following approximation:
\begin{align}
\underline{\lambda}^*_k(t) & \approx \frac{\frac{1}{M}\sum_{i=1}^M \lambda_k(t \sep \hist^{(i)}_t)\exp\left(-\int_0^t\sum_{k'\in\obs}\lambda_{k'}(s\sep \hist^{(i)}_s)ds\right)}{\frac{1}{M'}\sum_{j=1}^{M'} \exp\left(-\int_0^t\sum_{k'\in\obs}\lambda_{k'}(s\sep \hist^{(j)}_s)ds\right)}
\end{align}
where  $\hist^{(i)}_t, \hist^{(j)}_t \overset{\text{iid}}{\sim} \mathcal{L}^\q$ for $i=1,\dots,M$ and $j=1,\dots,M'$.\footnote{These are complete histories and the notation for indexing different samples should not be confused with the notation for mark-specific histories. To be explicit, $\hist^{(i)}_t \equiv \hist^{\vocab,(i)}_t$.} For simplicity, we typically set $M=M'$. This estimator is what is typically referred to as a ratio estimator, and while it is consistent unfortunately for finite samples it is biased. 

To see in what way this is biased, we will recast this form into a more general format. Consider random variables $X, \{X_i\}_{i=1}^M, \{X_j'\}^{M'}_{j=1} \overset{\text{iid}}{\sim} p_X$ with support $\mathcal{X}$ under measure $\prob$, and functions $f: \mathcal{X} \rightarrow \mathbb{R}^{+,0}$ and $g: \mathcal{X} \rightarrow [0,1]$. We assume the mean and variance of both $f(X)g(X)$ ($\mu_{fg}$ and $\sigma^2_{fg}$ respectively) and $g(X)$ ($\mu_g$ and $\sigma^2_g$) exist and that $\mu_{fg}>0$ and $\mu_g \in (0,1)$. This implies that the quantity of interest $\frac{\mu_{fg}}{\mu_g}:=\frac{\E[f(X)g(X)]}{\E[g(X)]}$ is well defined. We now can investigate the bias of a finite sample ratio estimator through a second-order Taylor series expansion around $\frac{\mu_{fg}}{\mu_g}$:
\begin{align}
\E^\prob& \left[\frac{\frac{1}{M}\sum_{i=1}^M f(X_i)g(X_i)}{\frac{1}{M'}\sum_{j=1}^{M'} g(X_j')}\right]  \approx \frac{\E^\prob\left[\frac{1}{M}\sum_{i=1}^M f(X_i)g(X_i)\right]}{\E^\prob\left[\frac{1}{M'}\sum_{j=1}^{M'} g(X_j')\right]} \nonumber\\
& \quad - \frac{\cov^\prob\left(\frac{1}{M}\sum_{i=1}^M f(X_i)g(X_i), \frac{1}{M'}\sum_{j=1}^{M'} g(X_j')\right)}{\E^\prob\left[\frac{1}{M'}\sum_{j=1}^{M'} g(X_j')\right]^2}\nonumber \\
& \quad + \frac{\var^\prob\left(\frac{1}{M'}\sum_{j=1}^{M'} g(X_j')\right)\E^\prob\left[\frac{1}{M}\sum_{i=1}^M f(X_i)g(X_i)\right]}{\E^\prob\left[\frac{1}{M'}\sum_{j=1}^{M'} g(X_j')\right]^3} \\
& = \frac{\mu_{fg}}{\mu_g} - \frac{\sum_{i=1}^M\sum_{j=1}^{M'}\cov^\prob\left(f(X_i)g(X_i), g(X_j')\right)}{MM'\mu_g^2} + \frac{\var^\prob\left(g(X)\right)\mu_{fg}}{M'\mu_g^3} \\
& = \frac{\mu_{fg}}{\mu_g} + \frac{\sigma^2_g\mu_{fg}}{M'\mu_g^3} \text{ since } X_i \perp X_j' 
\end{align}
Likewise, the variance of the ratio estimator can also be approximated with a second-order Taylor series expansion around $\frac{\mu_{fg}}{\mu_g}$:
\begin{align}
\var^\prob & \left(\frac{\frac{1}{M}\sum_{i=1}^M f(X_i)g(X_i)}{\frac{1}{M'}\sum_{j=1}^{M'} g(X_j')}\right) \approx \frac{\var^\prob\left(\frac{1}{M}\sum_{i=1}^M f(X_i)g(X_i)\right)}{\E^\prob\left[\frac{1}{M'}\sum_{j=1}^{M'} g(X_j')\right]^2} \nonumber\\
& \quad - \frac{2\cov^\prob\left(\frac{1}{M}\sum_{i=1}^M f(X_i)g(X_i), \frac{1}{M'}\sum_{j=1}^{M'} g(X_j')\right)\E^\prob\left[\frac{1}{M}\sum_{i=1}^M f(X_i)g(X_i)\right]}{\E^\prob\left[\frac{1}{M'}\sum_{j=1}^{M'} g(X_j')\right]^3} \nonumber\\
& \quad + \frac{\var^\prob\left(\frac{1}{M'}\sum_{j=1}^{M'} g(X_j')\right)\E^\prob\left[\frac{1}{M}\sum_{i=1}^M f(X_i)g(X_i)\right]^2}{\E^\prob\left[\frac{1}{M'}\sum_{j=1}^{M'} g(X_j')\right]^4} \\
& = \frac{\sigma^2_{fg}}{M\mu_g^2} - \frac{2\mu_{fg}\sum_{i=1}^M\sum_{j=1}^{M'} \cov^\prob\left(f(X_i)g(X_i), g(X_j')\right)}{MM'\mu_g^3} + \frac{\sigma_g^2\mu_{fg}^2}{M'\mu_g^4} \\
& = \frac{\sigma^2_{fg}}{M\mu_g^2} + \frac{\sigma_g^2\mu_{fg}^2}{M'\mu_g^4} \text{ since } X_i \perp X_j'.
\end{align}

It can be tempting to consider reusing samples for both the numerator and the denominator (i.e., $M=M'$ and $X_i=X_i'$ for $i=1,\dots,M$) as this would save in the amount of computations needed for computing the ratio estimate. This would result in the following expected value and variance of the estimator:
\begin{align}
\E^\prob&\left[\frac{\frac{1}{M}\sum_{i=1}^M f(X_i)g(X_i)}{\frac{1}{M}\sum_{j=1}^M g(X_j)}\right] \\
& \approx \frac{\mu_{fg}}{\mu_g} - \frac{\sum_{i=1}^M\sum_{j=1}^{M}\cov^\prob\left(f(X_i)g(X_i), g(X_j)\right)}{M^2\mu_g^2} + \frac{\sigma^2_g\mu_{fg}}{M\mu_g^3} \\
& = \frac{\mu_{fg}}{\mu_g} - \frac{\cov^\prob\left(f(X)g(X), g(X)\right)}{M\mu_g^2} + \frac{\sigma^2_g\mu_{fg}}{M\mu_g^3} \\
\var^\prob & \left(\frac{\frac{1}{M}\sum_{i=1}^M f(X_i)g(X_i)}{\frac{1}{M}\sum_{j=1}^{M} g(X_j)}\right) \\
& \approx  \frac{\sigma^2_{fg}}{M\mu_g^2} - \frac{2\mu_{fg}\sum_{i=1}^M\sum_{j=1}^{M} \cov^\prob\left(f(X_i)g(X_i), g(X_j)\right)}{M^2\mu_g^3} + \frac{\sigma_g^2\mu_{fg}^2}{M\mu_g^4}  \\
& = \frac{\sigma^2_{fg}}{M\mu_g^2} - \frac{2\mu_{fg} \cov^\prob\left(f(X)g(X), g(X)\right)}{M\mu_g^3} + \frac{\sigma_g^2\mu_{fg}^2}{M\mu_g^4}.
\end{align}

Either forms of the expected values of the estimators can be used to help us correct for the bias by simply moving all terms on the right that are not $\frac{\mu_{fg}}{\mu_g}$ to the left. Interestingly, we can see that there is potential for reusing samples to not only save on computation, but to also reduce the variance of the estimator. Should $\cov^\prob(f(X)g(X),g(X)) > 0$, which is often the case in practice, then the variance will be reduced.

\section{Experiments}\label{sec:5_experiments}
We investigate experimentally the impact that mark-censoring has on various MTPP models and the ability of our proposed marginalization method to handle such censoring relative to baseline. Our investigations are carried out across both classical parametric models and neural network-based models on both synthetic and real-world data, respectively. We find, as a whole, that in the presence of mark-censoring, the inference ability of a model (i.e., assigning likelihood to observed sequences) suffers significantly in comparison to properly accounting for the missing data via our method. Not surprisingly, we also find that our method yields larger improvements as the information being censored becomes more influential with respect to the information observed. 

We also investigate the effect that mark-censoring has on next event (time and mark) prediction. We observe in general systematic differences that our mark-censored model has on these predictions, with positive improvements in real-world settings. Lastly, we also perform a sensitivity analysis on the effect of both the number of sequences sampled as well as the resolution used in estimating integrals has on our method. We find that our method is typically fairly robust to these hyperparameters. More details and exact results for both of these experiments can be found in \cref{sec:5_exp_appendix}.

\paragraph{Censoring}
In each of the experiments, we analyze the performance of models using various sequences $\hist_\tau$ of differing lengths $\tau$.
For the synthetic setting, we utilize sequences that have been drawn from the given models. For the real-world data, we use held-out sequences from the dataset that a given model was trained on.

For every sequence being used, we filter out events according to a particular censoring scheme that is selected for each sequence individually to produce $\hist^\obs_\tau$. To ensure that the chosen censoring scheme is relevant for a given sequence, we randomly select a non-empty subset $\cen(t)$ of the unique marks that actually appear in $\hist_\tau$ for $t\in[0,T]$. The proportion of marks to censor, relative to the total number of unique marks in each sequence $\hist_\tau$, which we will refer to as $\gamma$, is varied based on the particular sequence for the experiment being conducted.

It is important to note that since information in $\hist_\tau$ is informing the censoring scheme $\cen$ that we technically no longer have data that is MCAR. As we will see, in spite of violating this assumption the mark-censored model still yields substantial performance gains.

\paragraph{Methods \& Metric of Interest}
For the main set of experiments, we primarily compared two approaches. Both rely on an existing MTPP and are used to calculate the likelihood of a given observed sequence $\hist^\obs_\tau$.\footnote{Previous works are not compared against in these experiments due to them largely having different goals and setups (such as learning from censored data during training time or imputing missing data), as well as typically not having a proper likelihood.}

The first approach is our proposed mark-censored model, using $\underline{\lambda}_k^*$ for $k \in \obs$. Since this is a well-defined MTPP, we can calculate the likelihood of $\hist^\obs_\tau$ using \cref{eq:4_likelihood} in conjunction with the censored intensity. Results for this method will be labeled as ``Censored.'' 
Synthetic and real-world experiments use 128 and 64 MC samples to estimate the censored intensity respectively;
both use 1024 integration points for numerically estimating integrals.

The second approach is a baseline method for comparison, based on a slight adaptation to the original model that takes advantage of knowing what marks are being 
censored for a given sequence. This method uses the original intensity $\lambda^*_k(t)$ for $k\in\obs$ and sets the intensity to be 0 when $k \in \cen$. Results for this method will be labeled as ``Baseline.'' In general, we expect the two methods to be comparable should $\prob(\hist^\cen_\tau=\emptyset\sep\hist^\obs_\tau) \approx 0$ as the two methods would produce similar intensity values.

We do not include results where we evaluate the likelihood of the observed sequence $\hist^\obs_\tau$ as if it were a fully observed, uncensored sequence under the original model. Since intensity values are always non-negative, likelihood values using this approach will \emph{never} be better than the baseline. Because of this, we only compare against the baseline as it effectively captures the original model's inference capabilities while still managing to leverage information about the mark-specific censoring scheme to some degree. Note also that none of the methods discussed earlier in Related Work are used as baselines since none are applicable to mark-censoring and model-agnostic. %

Results are reported as likelihood ratios between the censored method and the baseline method for individual observed sequences. These ratios directly quantify how much more likely the censored method perceives a sequence to be relative to the baseline. Values above 1 are evidence in favor of the censored method, and below 1 for the baseline. It should be noted that the sequences used in these experiments are censored over the entire observation window $[0,\tau]$.

\subsection{Experiments on Synthetic Data}

\begin{figure}
    \centering    \includegraphics[width=0.9\linewidth]{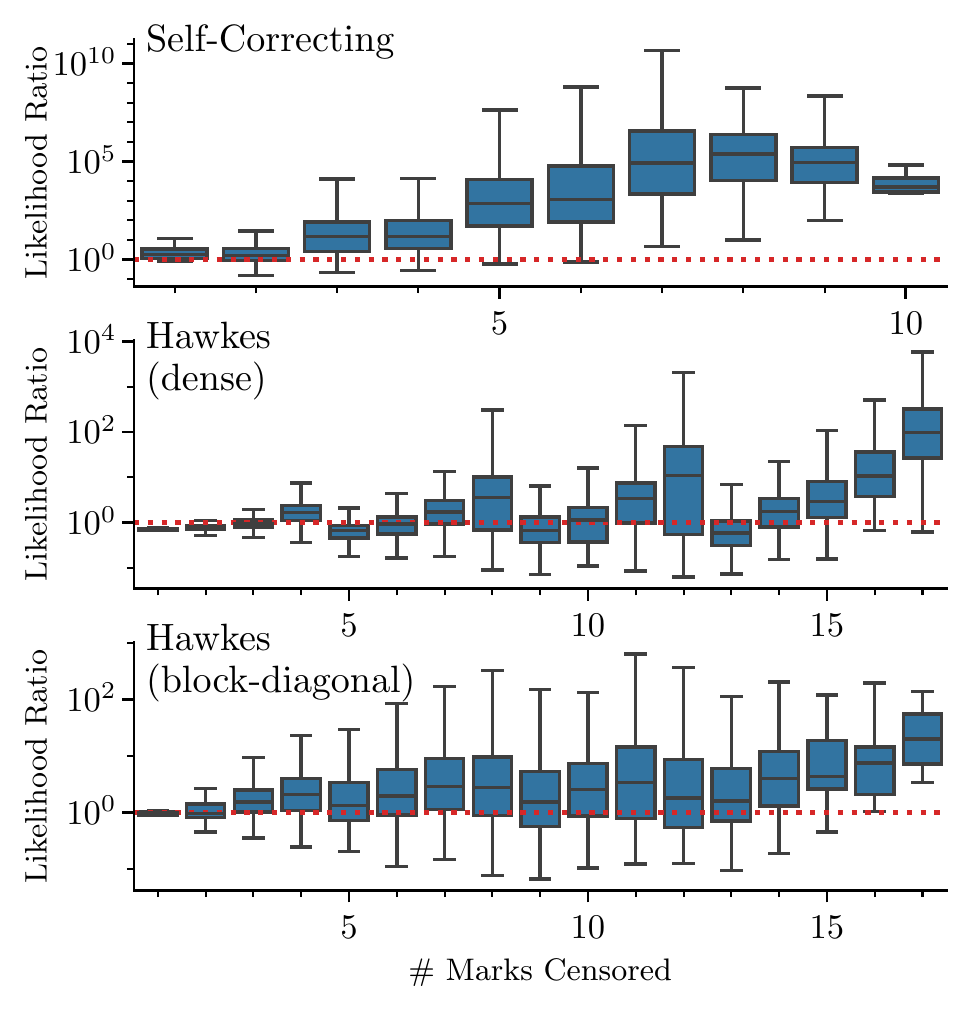}
    \caption{Distributions of likelihood ratios across number of marks censored for the duration of the sequences used for synthetic experiments with self-correcting, Hawkes (dense), and Hawkes (block-diagonal) models. Values greater than 1 indicate higher likelihoods under the mark-censored model.}
    \label{fig:5_synth_log_likelihood}
\end{figure}

\begin{figure}
    \centering    \includegraphics[width=0.9\linewidth]{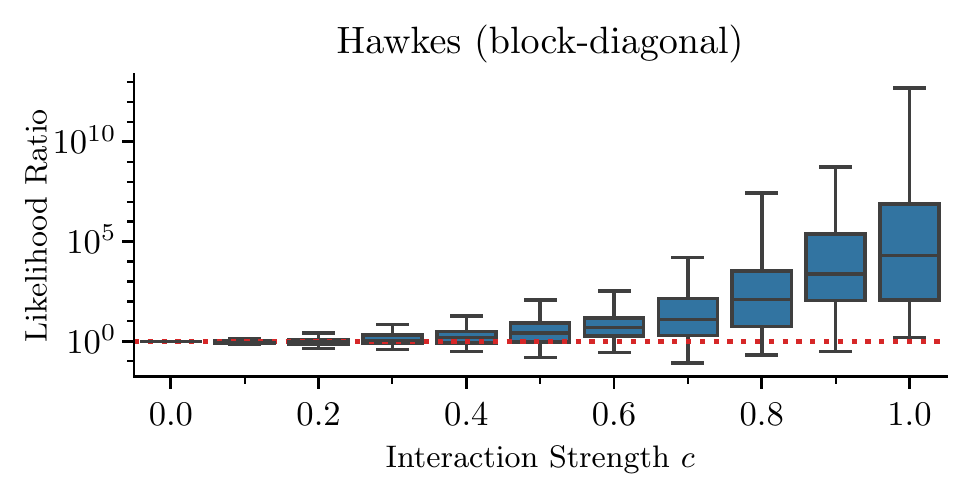}
    \caption{Distributions of likelihood ratios for a block-diagonal Hawkes model with varying interaction strengths applied to off-diagonal $\alpha$ terms. Values greater than 1 indicate higher likelihoods under the mark-censored model compared to the baseline.}
    \label{fig:5_synth_interaction}
\end{figure}

\begin{figure*}
    \centering
    \includegraphics[width=0.9\linewidth]{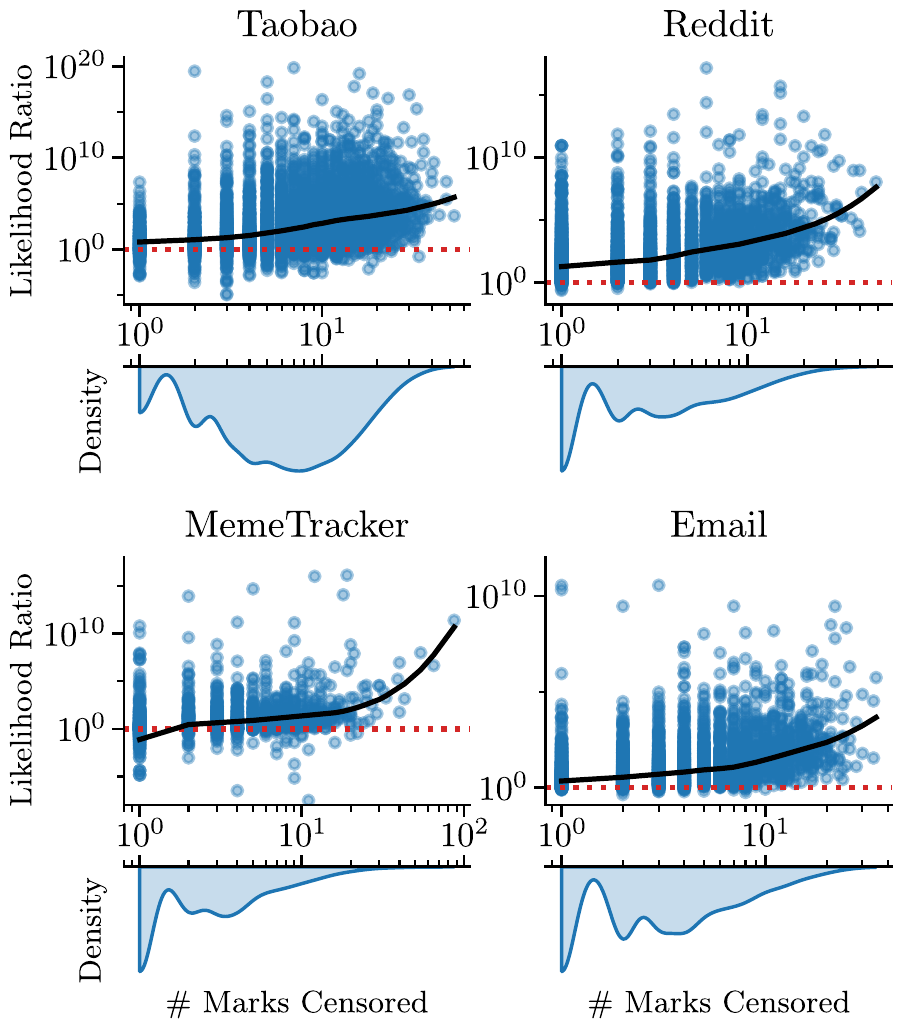}
    \caption{Same setup as \cref{fig:5_synth_log_likelihood} except with results produced on four real-world datasets with trained neural Hawkes models. Note that we display the results with respect to the absolute amount of marks censored rather than the percentage censored as we suspect this has a more direct impact on the likelihood ratios, especially when dealing with sequences that naturally have few unique marks compared to the total mark space $\vocab$\textemdash{}as is typical in real datasets.}
    \label{fig:5_log_likelihood}
\end{figure*}

\paragraph{Models}
We evaluate our method on randomly instantiated parametric MTPPs including Hawkes processes \citep{hawkes1971spectra} and self-correcting processes \citep{isham1979self} (also known as stress release model \citep{zheng1991application}), where the sampled sequences are evaluated on the same model.

For Hawkes processes with exponential kernels, the intensity has the form 
$\lambda_k^*(t) = \mu_k + \sum_{ T, M\in\hist_\tau} \phi_{ M,k}(t -  T)$. 
The kernel can be expressed as $\phi_{i,j}(z) = \alpha_{ij} \exp(-\beta_{ij} z)$, where parameters $\alpha_{ij}, \beta_{ij} > 0$ for $i,j\in\vocab$ specify the excitation effects and decay rates respectively that events of type $i$ have on events of type $j$. We consider two different instantiations of this type of model; both with 20 marks. We refer to the first type   as ``Hawkes (dense)'' with all parameters drawn from the following distributions: $\alpha_{ij}\stackrel{iid}{\sim}\text{Unif}[0.075, 0.2]$, $\beta_{ij}\stackrel{iid}{\sim}\text{Unif}[0.4, 1.2]$, and $\mu_k\stackrel{iid}{\sim} \text{Unif}[0.1,0.5]$. To better emulate realistic settings in which events correlate strongly with other events of certain types, we also consider a sparsely-parameterized version which we refer to as ``Hawkes (block-diagonal)'' \citep{wu2020diagnostics}. This model is instantiated by drawing $\alpha_{ij}\stackrel{iid}{\sim}\text{Unif}[0.3, 0.8]$ when $\lfloor \frac{i-1}{5} \rfloor = \lfloor \frac{j-1}{5} \rfloor$ and $\alpha_{ij}=0$ otherwise.\footnote{Note that different ranges of values were chosen for $\alpha$ between the dense and block-diagonal Hawkes models to normalize the effective rate of events overall. This is done by, in expectation, having the same values for $\sum_{i\in\vocab} \alpha_{ij}$.} This effectively imposes a block-diagonal structure on the matrix $\{\alpha_{ij}\}$, resulting in four subgroups of correlated marks.
Values for $\mu$ and $\beta$ are drawn similarly to the dense model. 

In contrast, self-correcting processes are characterized by a growing intensity that is dampened when an event occurs. This is formally defined by the intensity function $\lambda_k^*(t) = \exp\left(\eta_k t - \sum_{ T,  M \in \hist_\tau} \delta_{ M k}\right),$
 where $\delta_{ M k}>0$ determines the inhibition that past events of type $ M$ have on future events of type $k$. The model used for this class also has 20 marks and is instantiated by drawing weights $\delta_{ij} \stackrel{iid}{\sim} \text{Unif}[0.3, 0.8]$.
Values for $\eta$ are drawn similarly to $\mu$.

\paragraph{Results}
We evaluated the likelihood ratio of 1000 censored sequences on all three models with interaction strength fixed at 0.5 (a scalar that controls the interaction between events of different types) for each value $\gamma\in\{0.2, 0.4, 0.6, 0.8\}$. Each sequence prior to censoring was sampled from each model (self-correcting, Hawkes (dense), and Hawkes (block-diagonal)) over the time window $t\in[0,2]$ and contains at most 200 events. These results are shown in \cref{fig:5_synth_log_likelihood}, where the likelihood ratio of the censored method compared to the baseline is visualized with respect to the number of marks censored. 
We see a systematic improvement in the estimated likelihood when using the mark-censored model. Furthermore, the improvement increases as more information is censored; however, it is clear that the improvement depends on the relationship between events and the underlying model dynamics (i.e., the form of $\lambda$) as noted by the difference in results between models.

To further investigate this, %
for the block-diagonal Hawkes model we artificially modulate the interaction strength between events of different types. To do this, we performed the same likelihood ratio evaluation on 1000 sequences with $\gamma=0.5$ using the same block-diagonal Hawkes model but with $\alpha^\prime_{ij}:=
c\alpha_{ij}$ if  $i\neq j$ and
$\alpha_{ij}$ otherwise for each value of $c\in\{0.1, 0.2, \dots, 1.0\}$. 
This results in 10 different models that have the same diagonal values in $\boldsymbol{\alpha}$ but different scales of off-diagonal values. 
The results  in \cref{fig:5_synth_interaction}  clearly demonstrate that properly 
accommodating mark-censored sequences yields the biggest impact when there is high correlation between observed and censored events.

\subsection{Experiments on Real-World Data}

\paragraph{Models}
Many real-world data involve working with large vocabularies of possible marks, $|\vocab|=K$. Because of this, it can often be more parameter efficient to train a neural network based MTPP rather than a classical parametric one. The model architecture of choice for our experiments is the neural Hawkes process, a continuous-time RNN that takes inspiration from the parametric Hawkes process \citep{mei2017neural}. Details on model hyperparameters, optimizer, training regime, etc. can be found in \cref{sec:5_exp_appendix}.

\paragraph{Datasets}
We evaluate our censoring method on neural Hawkes models that have been trained individually on four different datasets.
The \textbf{Taobao} user behavior dataset \citep{zhu2018learning} contains page-viewing records of different categories ($K=1000$) from users on an e-commerce platform.
The \textbf{Reddit} dataset \citep{baumgartner2020pushshift} contains comments that users have made on various communities ($K=1000$) on the social media website \url{reddit.com}. 
\textbf{MemeTracker} \citep{leskovec2009meme} contains records of what websites ($K=5000$) a common phrase, or meme, was mentioned on over time.
Lastly, the \textbf{Email} \citep{paranjape2017motifs} dataset contains sequences of sender addresses of incoming emails ($K=808$) for each recipient within a research organization. More information on various aspects of these datasets and details of data preprocessing can be found in \cref{sec:5_exp_appendix}. The following results are achieved using models that have been sufficiently trained on their respective datasets.

\paragraph{Results}
We evaluated the likelihood ratio of 1000 held out, censored sequences for each dataset for each value $\gamma\in\{0.2, 0.4, 0.5, 0.6, 0.8\}$. The results are shown in \cref{fig:5_log_likelihood}. Similar to the results in the synthetic experiments, we see a systematic trend towards a large improvement in likelihood over censored sequences across the board. This improvement increases significantly as more marks are censored.

\section{Conclusion}\label{sec:5_conclusion}

In this chapter, we derived a tractable method for marginalizing over missing information with MTPPs which is the first to do so for generic intensity-based MTPPs. Empirical investigations were conducted to confirm the predictive abilities of MTPPs when adjusting for missing information.

The text of this chapter is based on the publication:
\begin{center}
\vspace{-1.5em} \textit{Inference for mark-censored temporal point processes} \citep{pmlr-v216-boyd23a}.
\end{center}

The author of this dissertation is the primary author for this publication, and was responsible for all theoretic contributions. Additionally, the author was responsible for implementing and executing all experiments that involved neural MTPPs. Yuxin Chang contributed the implementation and execution of the classical parametric-based MTPPs with oversight from the author. Overall conception of the project and some of the writing was shared among all authors.

\chapter{Generalized Hitting Times and their Approximated Distributions for Stochastic Jump Processes}

\noindent

Many real-world phenomena exhibit both continuous trends and sudden jumps. 
For example, stock prices often display gradual fluctuations with occasional spikes or dips due to economic events. Similarly, molecular simulations track particle movements and energy states over time with sudden changes exhibited during chemical reactions or phase transitions. 
Climate trends demonstrate comparable patterns, with gradual variations punctuated by natural disasters causing drastic and immediate shifts in measurement readings. 
Stochastic jump processes model these dynamics, combining continuous processes for the "between-jump" periods with marked temporal point processes for the jumps themselves. 
Within these various fields, understanding long-term dynamics is crucial for various forecasting and prediction tasks. 
Traditional hitting times provide valuable insights, but their limitations often oppose a more comprehensive analysis. 

This chapter introduces generalized hitting times, extending their scope and expressiveness for autoregressive models like stochastic jump processes.
Following prior work from earlier chapters on hitting times, we address their inability to capture diverse events like the $n^\text{th}$ passage time or exiting a region of interest. 
We propose novel estimators, including extensions of the importance sampling techniques from previous chapters, to answer various queries involving these generalized hitting times. 
These queries can even involve compositions of hitting times, allowing us to estimate, for example, the likelihood of specific regions being hit in a particular order. 
Notably, we achieve tractable estimation of joint distributions over multiple hitting times.

Building upon the efficient estimation methods developed in this dissertation, this chapter significantly expands the range of information accessible from autoregressive models. 
The framework and results presented here pave the way for deeper exploration of complex dynamics in various real-world applications.
Crucially, as with prior results, all proposed estimators only require access to the predicted next-step distributions and the ability to sample from the process itself, ensuring applicability without introducing additional learning or training burdens.

\section{Related Work}\label{sec:6_related_work}

\paragraph{Stochastic Jump Processes}
Due to their inherent expressivity and flexibility, stochastic jump processes have been found to be useful across a variety of diverse domains. By encompassing both continuous stochastic processes and marked temporal point processes as special cases, they offer a powerful framework for modeling phenomena exhibiting both smooth trends and abrupt jumps \citep{woyczynski2022diffusion}.

Some prominent examples include:
\begin{itemize}
\item \textbf{Finance:} Analyzing stock prices and options leverage jump processes to capture sudden fluctuations and market crashes \citep{kijima2002stochastic,rolski2009stochastic,platen2010numerical,zhu2021probabilistic}.
\item \textbf{Molecular Biology:} Simulating molecular and particle interactions utilize jump processes to model discrete events like chemical reactions and phase transitions \citep{goss1998quantitative,gillespie2007stochastic,szekely2014stochastic,ganguly2015jump}.
\item \textbf{Climate Science:} Predicting climate and weather patterns employs jump processes to represent sudden shifts due to extreme events like storms or heatwaves \citep{majda2008applied,tang2016semi,alinsato2023climate}.
\item \textbf{Network Science:} Forecasting network usage benefits from jump processes to model bursts and congestion \citep{267444,KLEMM2003149}.
\item \textbf{Transportation Engineering:} Modeling traffic and transportation flow often involves jump processes to represent accidents, lane changes, and other discrete events \citep{6682896,doi:10.1080/23249935.2013.769648}.
\item \textbf{Epidemiology:} Understanding epidemic outbreaks utilizes jump processes to capture the discrete nature of transmission events \citep{greenwood2009stochastic,dieu2020asymptotic,albani2024stochastic}.
\end{itemize}

\paragraph{Analysis of Stochastic Processes Hitting Times}
\textit{Hitting times}, also known as first passage times, are the random moments when a process enters a specified region for the first time. 
While some works focus on the moments of hitting times (e.g., \citet{weiss1983order,SCHAL1993131,aspandiiarov1996passage,gitterman2000,szpruch2010comparing}), many tackle the more difficult task of analyzing their distribution across time.

For specific classes of stochastic processes, analytical solutions or bounds exist (e.g., \citet{ricciardi1988,salminen1988first,rogers2000diffusions,byczkowski2013hitting,aurzada2016first,lefebvre2023first}).
However, for many other types of stochastic processes, analytical solutions or bounds are not available. For instance, hitting time densities in diffusion models can be expressed as Volterra integral equations, which lack guaranteed solutions \citep{jaimez1995,gutierrez1997,lipton2020first}.
In stochastic jump processes, analytic solutions are even rarer beyond exceptional cases like single and double exponential jump diffusions \citep{lefebvre2021exact,kou2003first}.

Similar to hitting times, \textit{exit times} capture the first departure from a previously entered region. Their analysis is even more complex due to dependence on the process history.\footnote{One can see this dependence by noting that a process cannot leave a region that it has not entered before.} 
This often necessitates additional assumptions for obtaining analytical or numerical results \citep{gartner1982location,yin2013exit,burch2014exit}.

Numerical approaches become necessary where analytics fail. Existing methods often impose restrictive assumptions: enforcing independence between continuous segments and jumps \citep{herrmann2023exact}, limiting the class of process considered \citep{metwally2002using,ross2010efficient,zhang2017efficiently}, assuming full processes are fully Markov \citep{pollock2013some}, or excluding jumps entirely \citep{sym12111907}.

To our knowledge, no general framework aside from naive Monte Carlo simulation exists for numerically estimating hitting time distributions in stochastic jump processes, nor for exit times (leaving a region after entering it) and other more complex random times. This chapter bridges this gap by proposing sampling-based estimation procedures applicable to a broad range of stochastic jump processes with minimal assumptions made.

\section{Preliminaries \& Setting of Interest}

\subsection{Stochastic Jump Process}

Let $(\Omega, \filter, \prob)$ be a probability space equipped with a complete filtration $(\filter_t)_{t\geq 0}$ such that $\filter_\infty := \filter$. Under this probability space, let $X:\Omega \times [0,\infty] \rightarrow \mathcal{X}$ be a predictable stochastic jump process adapted to $(\filter_t)_{t\geq 0}$ where $\mathcal{X}$ is the domain of the process (typically finite and real-valued, e.g., $\mathbb{R}^d$). The process itself will be referred to as either $X_t(\omega)$ or $X_t$ for $t \in [0,\infty]$ and $\omega\in\Omega$.

We focus on a specific class of stochastic jump processes. These processes exhibit continuous segments described by stochastic differential equations (SDEs) and discrete jumps governed by marked temporal point processes. 
Mathematically, the instantaneous changes are represented by:
\begin{align}
    dX_t := \mu(t, X_{t-})dt + \sigma(t, X_{t-})dW_t + \nu(t, X_{t-}, M_{N_t})dN_t.
\end{align}
Here, $\mu$ denotes the drift function, $\sigma$ the scaling function, and $\nu$ the jump function.
The process $X$ is driven by Brownian motion $W:=(W_t)_{t\geq0}$ with the same dimensionality as the state space $|\mathcal{X}|$. 
Additionally, it is influenced by the counting process $N:=(N_t)_{t\geq 0}$ which generates a mark $M_i\in\mathcal{M}$ at jump time $S_i$ for $i \in \mathbb{N}$.
The counting process corresponds to a marked temporal point process with a marked intensity process $\lambda(m) := (\lambda_t(m))_{t\geq0}$ defined for $m \in \mathcal{M}$ within mark-space $\mathcal{M}$. 
We posit that the intensity process, like $X$, is adapted to the filtration $(\mathcal{F}_t)_{t\geq 0}$.
The process $X$ as a whole can be similarly described via
\begin{align}
X_t = X_0 + \int_0^t \mu(s, X_{s-})dt + \int_0^t \sigma(s, X_{s-})dW_t + \sum_{i=1}^{N_t}\nu(S_i, X_{T_i-}, M_i) 
\end{align}
where $X_0$ is the initial value of the process which we assume to be constant.\footnote{In prior chapters, the random event times modeled by a MTPP were denoted as $T_i$ instead of $S_i$; however, we reserve $T_i$ for random times which are described in the next subsection.}

\subsection{Random Times}
A \textit{random time}, denoted by $T(\omega)$, is a random variable defined on the time domain $[0,\infty]$.
We often encounter random times associated with an existing stochastic process $X$, where $X_T$ carries specific meaning (e.g., the first time the process reaches a specific value).
If the condition defining the random time is never met for a given realization $\omega\in\Omega$ (e.g., the process never reaches the desired value), we set $T=\infty$. 
Importantly, $T$ is always measurable with respect to $\filter$.

\paragraph{Stopping Times}
\textit{Stopping times} are a specific type of random time.
They possess the unique property that, based solely on the process history up to time $t$, we can determine whether the stopping time has occurred (i.e., if the process has ``stopped'').
Mathematically, a random time $ T$ is a stopping time if and only if:
\begin{align}
\{\omega\in\Omega \sep  T(\omega) \leq t\} \text{ is } \filter_t \text{ measurable for all } t \geq 0. 
\end{align}
This implies that knowing the process behavior from time $0$ to time $t$ allows us to determine if the stopping time $T$ falls within the interval $[0, t]$ or not.

\paragraph{Hitting Times}
\textit{Hitting times}, also known as \textit{first passage times}, form a  subclass of stopping times.
They represent the specific times at which a process first reaches a particular state or value.
Formally, for a measurable set $A \subset \mathcal{X}$, the hitting time of $A$ is defined as:
\begin{align}
 T_A(\omega) & = \inf \{t \sep X_t(\omega) \in A\}. \label{eq:hitting_time}
\end{align}
While the stochastic process literature typically defines hitting times as random variables adapted to the same filtration as the underlying process, it is important to note that they can be equivalently considered as functionals of the stochastic process: $T_A = T_A[(X_t)_{t \geq 0}] = \inf \{t \sep X_t \in A\}$.

Some transformations of hitting times are themselves hitting times. For instance, the minimum of countably many hitting times is itself a hitting time:
\begin{align}
\min_i\{ T_{A_i}(\omega)\} & = \inf \{t \sep X_t(\omega) \in A_1 \text{ or } X_t(\omega) \in A_2 \text{ or } ... \} \\
& = \inf \{t \sep X_t(\omega) \in \cup_i A_i \} \\
& = \inf \{t \sep X_t(\omega) \in A\} \\
& =  T_A(\omega)
\end{align}
for $A:=\cup_i A_i$.

\paragraph{Generalized Hitting Times}
While hitting times capture important information for a process, they are often a bit limiting in what events they can capture. That being said, there are many other random times of interest where after the process satisfies a sufficient number of conditions the random times resemble a hitting time in how it is realized.
For instance, consider the time $T$ defined to be the maximum of two hitting times for regions $A$ and $B$ where $A \cap B = \emptyset$. 
This random time captures the time at which the process has visited both regions $A$ and $B$ at least once.
This clearly is not a hitting time; however, after the process has first visited one of the regions, say $A$, then the actions the process need to take to realize $T$ are identical to the actions needed for a typical hitting time of region $B$. 
Put differently and in a slightly more general way, $\max \{T_A, T_B\}$ is a hitting time after conditioning on $\filter_{\min \{T_A, T_B\}}$ with a hitting region of either $A$ or $B$ depending on which has not been visited yet.

This example random time and many others are formalized under a new class of random times we call \textit{generalized hitting times}, which exists as a subclass of stopping times and a superclass of hitting times. 
Taking inspiration from the original definition of the hitting time in \cref{eq:hitting_time}, consider a random time $T$ defined as
\begin{align}
 T(\omega) = \inf \{t \sep X_t(\omega) \in A_t(\omega)\}
\end{align}
where $A:\Omega\times[0,\infty]\rightarrow\sigma(\mathcal{X})$ (or put differently, for $t\in[0,\infty]$ and $\omega\in\Omega$, $A_t(\omega)$ is a measurable subset of $\mathcal{X}$). $A$ can be thought of as a set-valued stochastic process. 

We say that $ T$ is a \textit{generalized hitting time} if and only if $A$ satisfies the following properties:
\begin{enumerate}
\item $A_t(\omega) \subset A_s(\omega)$ for $t \leq s$ with the exception of  and $A_t(\omega) = \emptyset$ for $ T(\omega) < t$.
\item There exist a collection of $K$ random times, $ T_1, \dots,  T_K$, such that:
\begin{enumerate}
\item $ T_k(\omega) <  T(\omega)$ for all $k=1,\dots,K$ and $\omega\in\Omega$.
\item Each $ T_k$ for $k=1,\dots,K$ is either a hitting time or another generalized hitting time.
\item Let $ T_{(k)}(\omega)$ be the $k^\text{th}$ largest time of the collection for a given $\omega\in\Omega$. It is required that $A_t(\omega) \equiv A_{ T_k(\omega)}(\omega)$ for $\omega\in\Omega$ and $ T_{(k)}(\omega) \leq t <  T_{(k+1)}(\omega)$ where $ T_{(K+1)}:=\infty$.
\end{enumerate}
\item Lastly, $A_t(\omega) \neq \emptyset$ for $ T_{(K)}(\omega) \leq t \leq  T(\omega)$ and $\omega \in \Omega$.
\end{enumerate}
More intuitively, a generalized hitting time $T$ can be thought of as a random time for a process where after reaching some other predictable time $T'$, specifically $T_{(K)}$ from 2(c) above, then the conditional $T \sep \filter_{T'}$ becomes a standard hitting time. Put even more plainly, after $t=T'$, the condition that $T$ needs to realize involves the process reaching some hitting region $A_t$.

\cref{table:example_times} illustrates examples of random times that correspond to each of these different classes. Most often, we consider the simpler case of $A_t=\emptyset$ for $t <  T_{(K)}$; however, certain compositions of generalized hitting times are also generalized hitting times, such as the minimum of multiple generalized hitting times.

\begin{table}
\centering
\begin{tabular}{l c c c c} 
 \hline
 Example & Random Time & Stopping Time & Gen. Hitting Time & Hitting Time \\ 
 \hline\hline
$ T=\argmax_t X_t$ & \cmark & \xmark & \xmark & \xmark \\
$ T= T_A+c$ & \cmark & \cmark & \xmark & \xmark \\
$ T= T_A$ & \cmark & \cmark & \cmark & \cmark \\
$ T= T_A-c$ & \cmark & \xmark & \xmark & \xmark \\
$ T=\min_i  T_{A_i}$ & \cmark & \cmark & \cmark & \cmark \\
$ T=\max_i  T_{A_i}$ & \cmark & \cmark & \cmark & \xmark \\
$ T= T_A; T_{A'}< T_A$ & \cmark & \cmark & \cmark & \xmark \\
$ T= T_A; T_{A}<T_{A'}$ & \cmark & \cmark & \cmark & \xmark \\
 \hline
\end{tabular}
\caption{Examples of random times $ T$ and their classifications, with $c>0$, $A,$ $A',$ $A_1,$ $A_2,\dots \subset \mathcal{X}$, and $ T_A$ being the time of reaching state $A$ for the first time by $X$. The ``$;(\cdot)$'' notation implies that unless the $(\cdot)$ is satisfied then $ T(\omega)=\infty$. For instance, the random time $T= T_A; T_{A'}<T_A$ is concerned with the time the process reaches $A$ for the first time so long as it has reached $A'$ prior. Should $A$ be reached before $A'$, then $T=\infty$. Likewise, $T=T_A;T_{A}<T_{A'}$ covers the other scenario of the process reaching $A$ before $A'$ first. It should be noted the minimum of these two random times is simply $T_A$.}
\label{table:example_times}
\end{table}

\paragraph{Additional Extensions} Sometimes there are random times that by themselves cannot be considered generalized hitting times, but they can be converted to one by transforming the process that they correspond to. For example, consider the stopping time $ T(\omega) = \inf \{ t \sep X_t(\omega) \in [t, \infty] \}$.
While this could be considered as a first passage time with a moving boundary, it does not strictly fall under the definition given above in \cref{eq:hitting_time} as the hitting state $A$ is constantly changing. It can, however, be straightforwardly converted to a hitting time (under definition in \cref{eq:hitting_time}) by transforming the associated process:
\begin{align}
 T(\omega) = \inf \{ t \sep X_t(\omega) \in [t, \infty] \} \equiv \inf \{ t \sep X_t(\omega) - t \in [0, \infty]\}.
\end{align}
Similar transformations can be carried out when the random time itself is being shifted:
\begin{align}
 T(\omega) & = \inf \{t \sep X_t(\omega) \in A\} + c  \text{ for some } c \in \mathbb{R} \\
& = \inf \{t + c \sep X_t(\omega) \in A \} \\
& = \inf \{t \sep X_{t-c}(\omega) \in A \}.
\end{align}
In this example, $T$ as defined with $(X_t)_{t\geq 0}$ and $A$ is not a hitting time, but it is equivalent to the hitting time of $A$ with process $(X_{t-c})_{t \geq c}$. Both of these types of transformations hold for both hitting times and generalized hitting times.
\\\\
For all generalized hitting times $T$ considered in this chapter, we assume there are no shifts to the time itself (meaning $c=0$ as used in the example previously) and that the associated set-valued process $A:=(A_t)_{t\geq 0}$ either directly satisfies the previous three conditions outlined or it does so after applying a simple transformation $Y_t:=X_t + f(t)$ with respect to $t$.

\subsection{Objective}
Our goal henceforth is given an arbitrary stochastic jump process $X$ and generalized hitting time $T$, find a pair of stochastic processes $F:=(F_t)_{t\geq 0}$ and $f:=(f_t)_{t\geq 0}$ that affords us unbiased and efficient estimates of the distribution of $T$ via Monte Carlo samples:
\begin{align}
\prob( T \leq t)& = \E^\prob\left[F_t\right] \approx \frac{1}{N} \sum_{i=1}^N F_t^{i} \\
\text{ and } \prob( T \in [t, t+dt)) & = \E^\prob\left[f_t\right]dt  \approx \frac{1}{N} \sum_{i=1}^N f_t^{i}dt
\end{align}
where $F$ and $f$ are adapted to the natural filtration of our process of interest $X$. 
We denote realized samples of these processes as $F_{t}^{i} \overset{\text{iid}}{\sim} \prob(F_{t})$ and $f_t^i \overset{\text{iid}}{\sim} \prob(f_t)$ for $i=1,\dots,N$. 
Given the connection to the cumulative distribution function (CDF) and probability distribution function (PDF) of $ T$, $F$ and $f$ will be referred to as the CDF and PDF processes henceforth.
This connection affords us the perspective that $f$ can be considered the partial derivative of $F$ with respect to time $t$ (i.e., $\frac{\partial F_t(\omega)}{\partial t} = f_t(\omega)$). 
It should be noted that the main goal is to find an appropriate CDF process $F$ as the corresponding PDF process $f$ may not always be well defined, e.g., see \cref{sec:6_naive_estimators}. 
These processes will sometimes be referred to as estimator processes due to their finite sample average being an unbiased estimate of the probability statements of interest, $\prob(T \leq t)$ and $\prob(T \in [t, t+dt))$.

\paragraph{Estimator Efficiency}
Naturally, given enough time and resources any unbiased estimator can reach arbitrary precision due to the law of large numbers; however, there can be a vast disparity in the rate at which different estimators converge. 
For our purposes, we will focus on estimators that converge more quickly and will quantify this comparison between estimators via \emph{relative efficiency}. 
Since we are primarily considering unbiased estimators for $\prob( T \leq t)$, say between CDF processes $F$ and $\underline{F}$, the relative efficiency of the finite sample estimates is the ratio of variances of the processes:
\begin{align}
\text{eff}(F_t, \underline{F}_t) & := \frac{\var^\prob\left(\frac{1}{N}\sum_{i=1}^N \underline{F}_t^i\right)}{\var^\prob\left(\frac{1}{N}\sum_{i=1}^N F_t^i\right)} \\
& = \frac{\var^\prob\left(\underline{F}_t\right)}{\var^\prob\left( F_t\right)}.
\end{align}
If $\text{eff}(F_t, \underline{F}_t) > 1$, the estimator $F_t$ is more efficient (i.e., converges more quickly) than $\underline{F}_t$. 
Note that it need not be the case that the efficiency stays the same value, or even consistently above or below 1, across different values of $t$.

\paragraph{Outline}
First, we investigate simple empirical estimators utilizing indicator functions for these distributions in \cref{sec:6_naive_estimators} and discuss why these are not optimal. 

Next, these indicator-based estimators are extended using the tower rule in conjunction with the fact that stochastic jump processes possess an intensity function in \cref{sec:6_tower_rule_estimators}. The intensity function encodes the infinitesimal likelihood of the process jumping to other values, which when accumulated over time for jumps that lead the process to the hitting region yields a potentially lower variance estimator.

Unfortunately, while this extension has the potential to be more efficient it also can be worse due to being unbounded. To remedy this, we further apply importance sampling on this estimator in \cref{sec:6_imp_sample_estimators} to accumulate the hitting probability over time by actively preventing the process from reaching the hitting region of interest in the first place. In doing so, the importance weights contain the probability of these restricted outcomes and ensure that the estimator remains both unbiased and bounded, leading to improved estimator variance.

Restricting specific outcomes of the process at hand is also advantageous for other scenarios, namely for estimating distributions of different compositions of generalized hitting times. Extensions to the importance sampling estimator are proposed in \cref{sec:6_comp_ght} for both orderings of hitting times and joint distributions.

All of the estimators proposed in this chapter are explored empirically in \cref{sec:6_experiments} for various use cases. The results in this section illustrate in general that for estimator efficiency, the proposed importance sampling estimators are significantly more efficient than the simple indicator-based estimators and more reliable than the tower rule-derived unbounded estimators.

\section{Naive Empirical Estimators}\label{sec:6_naive_estimators}

To the best of our knowledge, computing $\prob( T \leq t)$ for most stochastic processes $X$ and some stopping time $ T$ is typically done either one of two ways. The first is to leverage the specific forms of $ T$ and $X$ in order to analytically determine the distribution of $ T$. This often only works for a narrow class of processes and times and needs to be considered on a case-by-case basis. The second approach is to use what we will refer to as \emph{naive empirical estimators}, which relies on the empirical cumulative distribution function:
\begin{align}
F^\text{NE}_t(\omega) := \ind( T(\omega) \leq t)
\end{align}
for $\omega \in \Omega$ where $\ind$ is the indicator function returning 1 when the argument is true and 0 when false. Note that to evaluate $F^\text{NE}_t$ for a given sample $\omega$, we do not need to know the exact value of $ T(\omega)$ but rather just if by time $t$ the process has stopped or not (i.e., if $ T \leq t$ or $ T > t$). Since $ T$ is a stopping time, by definition this can be determined purely via $(X_{s}(\omega))_{0\leq s\leq t}$. Because of this, we can often more explicitly write what the naive empirical estimator is as a transformation of $X$. Take for example having $ T$ represent the hitting time of reaching the state $A \subset \mathcal{X}$ for the first time. In this scenario, the exact form of the estimator would be
\begin{align}
F^\text{NE}_t := \ind(\exists_{s \in [0,t]} X_s \in A).
\end{align}
In the naive setting, there is not a well-defined corresponding PDF process $f^\text{NE}$\textemdash{}at least not one that is useful for working with empirically. Mathematically, the corresponding $f^\text{NE}$ would be defined as a Dirac-delta process $\delta( T(\omega) - t)$ which for finite samples will return 0 almost surely for any predefined time $t$. 

\paragraph{Estimator Variance}
As mentioned previously, estimator efficiency is the primary means for evaluating different estimators. Since $F^\text{NE}_t$ is unbiased with respect to $\prob( T \leq t)$, we only need to quantify the variance of the process to compare it to other estimators in later sections. We can easily determine the variance of the naive estimator CDF process:
\begin{align}
\var^\prob\left(F^\text{NE}_t\right)  & = \var^\prob(\ind( T \leq t)) \\
& = \prob( T \leq t)-\prob( T \leq t)^2
\end{align}
with the last statement being true due to $F_t^\text{NE}$ being binary-valued and unbiased, i.e., $F_t^\text{NE} \sim \text{Bern}(\prob( T \leq t))$. The variance of the process reaches a maximum value of 0.25 when $\prob( T \leq t)=0.5$ and is minimized to a value of 0 when $\prob( T \leq t) \in \{0, 1\}$. 
We can see that the variance of the estimate itself (with \emph{any} estimator, not just the naive one) scales $\mathcal{O}\left(N^{-1}\right)$:
\begin{align}
\var^\prob\left(\frac{1}{N}\sum_{i=1}^N F_t^i\right) & = \frac{1}{N^2}\sum_{i=1}^N \var^\prob\left(F_t^i\right) \\
& = \frac{1}{N}\var^\prob\left(F_t\right).
\end{align}

\section{Tower Rule Estimators}\label{sec:6_tower_rule_estimators}

The naive estimator CDF process $F^\text{NE}$ is our starting point in this work. While simple, it is an important foundation to build from. This is due to the fact that it produces unbiased estimates of $\prob(T \leq t)$. Through careful manipulation and further derivations of this process we can remain unbiased yet construct estimators with potentially lower variances.

Our first step in order to construct low-variance estimators lies in recognizing that the process $X$ can jump at random times offers a unique characteristic not present in typical continuous stochastic processes. Namely, with the presence of jumps it is no longer guaranteed that the left and right limits of a process will be equivalent, either marginally or conditionally:
\begin{align}
\prob(X_{t-} \neq X_{t}) > 0 \text{ and } \prob(X_{t-}\neq X_{t} \sep \filter_{t-}) > 0 \text{ for all } t > 0 \text{ such that } \lambda_t > 0.
\end{align}
Recall that in our model, when a jump occurs at some time $t$ there is an associated mark $m$ that accompanies this event. Furthermore, we define the resulting change in the process $X_t$ at a jump time via $\Delta X_{t} := \nu(t, X_{t-}, m)$. 
Also recall that for the generalized hitting time $T$ there is a time- and history-dependent hitting region $A_t$ that once reached by the process $X$ at time $t$ indicates $T=t$.
Correspondingly, for every time $t$, we can define a random subset of the mark-space $B_t(\omega)\subseteq \mathcal{M}$ that all elements in it $m\in B_t$ would lead to $X_{t-}$ immediately realizing the hitting time $T$ via $X_{t-}+\nu(t, X_{t-}, m) \in A_t$. 
We denote the total intensity of all jumps at time $t$ that lead to this realization via
\begin{align}
\lambda^T_t(\omega) & := \lambda_t(\omega, B_t(\omega)) \equiv \int_{B_t(\omega)} \lambda_t(\omega, m) dm \text{ for } \omega \in \Omega 
\end{align}
Since $T$ is a predictable random time, this quantity can be interpreted as a valid intensity for $T$'s associated counting process, $N^T_t:=\ind(T\leq t)$, and as such holds the following equivalence:
\begin{align}
\lambda^ T_t dt \equiv \prob( T \in [t, t+dt) \sep \filter_{t-}).
\end{align}
From here, we have all the tools we need to extend our previous estimator; however, we first start by deriving a tractable \emph{density} estimator first:
\begin{align}
\prob( T \in [t, t+dt)) &= \E^\prob\left[\ind( T \in [t, t+dt))\right] \\
& = \E^\prob\left[\ind( T \in [t, t+dt))\ind( T \notin [0, t))\right] \\
& = \E^\prob\left[\E^\prob\left[\ind( T \in [t, t+dt)) \sep \filter_{t-}\right]\ind( T \notin [0, t))\right] \text{ by Tower Rule}\\
& = \E^\prob\left[\prob( T \in [t, t+dt) \sep \filter_{t-})\ind( T \notin [0, t))\right]\\
& = \E^\prob\left[\lambda_t^ T \ind( T \notin [0, t))\right]dt \\
\implies f^\text{TR}_t(\omega)dt & := \lambda_t^ T(\omega) \ind( T(\omega) \notin [0, t))dt \text{ for } \omega \in \Omega.
\end{align}
Due to the use of the tower rule in the derivation, we will refer to this as the tower rule estimator process \citep{casella2021statistical}.
Note that technically, the indicator in this estimator process $f^\text{TR}$ is unnecessary as the hitting time intensity process $\lambda^ T_t$ is equal to 0 for $t >  T$. That being said, we still include the indicator to emphasize this fact. 

From here, we can derive the corresponding tower rule cumulative distribution estimator process:
\begin{align}
\prob( T \leq t) & = \int_0^t \prob( T \in [s, s+ds)) \\
& = \int_0^t \E^\prob\left[f^\text{TR}_s\right]ds \\
& = \int_0^t \E^\prob\left[\lambda_s^ T \ind( T \notin [0,s))\right]ds \\
& = \E^\prob\left[\int_0^t \lambda_s^ T \ind( T \notin [0,s))ds \right] \text{ by Fubini's Theorem}\\
& = \E^\prob\left[\int_0^{t\wedge  T} \lambda_s^ T ds \right] \\
\implies F^\text{TR}_t(\omega) & := \int_0^{t\wedge  T(\omega)} \lambda_s^ T(\omega) ds \text{ for } \omega \in \Omega
\end{align}
where $t \wedge  T$ is the minimum of $t$ and $ T$. The tower rule cumulative distribution estimator process $F^\text{TR}_t$ can be described as the stopped compensator for the generalized hitting time (i.e., $F^\text{TR}_t(\omega) \equiv \Lambda^ T_{t \wedge  T(\omega)}(\omega)$).\footnote{The compensator of a temporal point process is the cumulative intensity: $\Lambda_t(\omega) := \int_0^t \lambda_s(\omega) ds$ \citep{daley2003introduction}. For more details on this, please refer back to Chapter 2.} Put more plainly, prior to the process $X$ entering the hitting region $A_t$ (for $t<T$), the tower rule estimator process accumulates all of the infinitesimal likelihoods ($\lambda^T_t$) of opportunities that the process had to jump into the hitting region. When these estimator processes are averaged over many trajectories, they yeild an unbiased estimate of the true underlying CDF of the generalized hitting time $T$.

\paragraph{Estimator Variance}
As the cumulative distribution $\prob( T \leq t)$ is our primary object of interest, we will limit our analysis to just the associated estimator and not the estimator for the density. 

It should be noted that the form of the tower rule estimator was defined in a way to be the most clear in terms of its properties; however, it is perfectly acceptable to also write it in the following form:
\begin{align}
F_t^\text{TR}(\omega) := \int_0^t \lambda^ T_s(\omega)ds \equiv \Lambda_t^ T(\omega) \text{ for } \omega \in \Omega.
\end{align}
This is true due to the fact that $\lambda^ T_t=0$ for $t >  T$ and that the generalized hitting time is assumed to only be able to occur once (these two facts are actually one and the same). It then follows that
\begin{align}
\var_\prob\left(F_t^\text{TR}\right) & = \var_\prob\left(\Lambda^ T_t\right) \\
& = \E^\prob\left[(\Lambda^ T_t)^2\right] - \E^\prob\left[\Lambda^ T_t\right]^2 \\
& = \E^\prob\left[(\Lambda^ T_t)^2\right] - \E^\prob\left[N^ T_t\right]^2 \text{ by properties of compensators} \\
& = \E^\prob\left[(\Lambda^ T_t)^2\right] - \E^\prob\left[\ind( T \leq t)\right]^2 \\
& = \E^\prob\left[(\Lambda^ T_t)^2\right] - \prob( T \leq t)^2,
\end{align}
where $N^T_t$ is the associated counting process for the generalized hitting time $T$. Comparing this to the variance of the naive estimator $F^\text{NE}_t$, we can see that they differ solely in the first term: $\E^\prob\left[(\Lambda^ T_t)^2\right]$ versus $\prob( T \leq t)$. Depending on the model and hitting time, either term could be dominant; however, it is worth noting that the second moment for the hitting time compensator is unbounded and thus has the potential to result in a significantly worse variance than the naive estimator, i.e., in practice the relative efficiency of the tower estimator to the naive estimator may be less than 1, which is undesirable. Luckily, both estimators utilize the same sampling distribution, so we can efficiently compute both with the same samples and select which one to use post-hoc based on estimates of the sample variance. 

\section{Importance Sampling Estimators}\label{sec:6_imp_sample_estimators}
As mentioned, the potential gains in efficiency when using the tower rule estimator over the naive estimator are dampened due to the former being unbounded compared to the latter being bounded between 0 and 1. Because of this, it is clear that there are samples that are a detriment to the resulting estimator variance, and specifically they are the samples that result in values $F_t^\text{TR} > 1$. This can occur for realizations in which the hitting time was likely to occur earlier (i.e., higher intensity values $\lambda_t^T$) but did not due to random chance.
\\\\
One potential way to eliminate this effect would be to apply importance sampling and limit the range of behaviors the process can exhibit when sampling. In doing so, we can potentially rein in these detrimental occurrences while possibly also achieving a bounded estimator. One might think that a good approach to achieve this should be to enforce that $ T < t$ during sampling; however, this has a number of drawbacks:
\begin{enumerate}
\item Properly enforcing this requires marginalizing over future events, which would lead to inefficient sampling. Put differently, we could no longer simply sample future trajectories autoregressively.
\item Any samples drawn would be done so with respect to a specific point in time $t$, which means that getting a sense of the distribution over a set of values $t \in [a, b]$ would require many independent estimations.
\end{enumerate}
Counterintuitively, we can still employ importance sampling with our original plan in mind and avoid both of these issues\textemdash{}the trick is to \emph{prevent} the hitting time from occurring instead of enforcing it to (in a manner similar to what we did in Chapters 3 and 4 for other processes earlier in this thesis). Preventing an action can be done through immediate intervention, not requiring information of future events. Furthermore, sampling a sequence with an event prevented over times $[0, a]$ is a prerequisite to sampling over $[a, b]$ so the samples can be reused over multiple different values of $t$ for $\prob( T \leq t)$, leading to potentially significant gains in computational efficiency.
\\\\
Let $\q$ be a measure that is a part of a probability space $(\Omega, \filter, \q)$ and is equipped with a complete filtration $\mathbb{F} := (\filter_t)
_{t\geq 0}$. Under this measure exists another process which we will refer to as $\underline{X}:=(\underline{X}_t)_{t\geq0}$ and is defined to belong to the same class of processes as $X$ does under $\prob$:
\begin{align}
\underline{X}_t := \underline{X}_0 + \int_0^t \underline{\mu}(t, \underline{X}_{t-})dt + \int_0^t \underline{\sigma}(t, \underline{X}_{t-})dW_t + \sum_{i=1}^{{N}_t} \underline{\nu}(\underline{S}_i, \underline{X}_{\underline{t}_i-}, \underline{M}_i)
\end{align}
where $\underline{S}_i$ and $\underline{M}_i$ are times and marks generated by an associated marked counting process $\underline{N}_t$ characterized by intensity $\underline{\lambda}_t$. Let $T$ be a generalized hitting time with the same restrictions for $\underline{X}$ under $\q$ as it is for $X$ under $\prob$. Our goal is to appropriately define $\underline{\mu}, \underline{\sigma}, \underline{\nu},$ and $\underline{\lambda}$ such that
\begin{align}
\prob( T \leq t) & = 1 - \E^\prob\left[\ind( T > t)\right] \label{eq:naive_neg_estimator}\\
& = 1 - \E^\q\left[L_t \ind(T > t)\right] \text{ where } \E^\prob\left[L_t\right]=1\label{eq:gen_imp_sampling}\\
& \approx 1-\frac{1}{n} \sum_{i=1}^n  L_t^{(i)} \ind(T^{(i)} > t) \text{ where } (L_t^{(i)}, T^{(i)}) \overset{iid}{\sim} \q
\end{align}
and ideally that $L_t \ind(T > t)$ has low variance under $\q$. $L_t$ is referred to as the likelihood ratio process and can be thought of as $\frac{d\prob}{d\q}((\underline{X}_s)_{0\leq s \leq t})\equiv L_t$. Note that we are now working with the complement of $\prob( T \leq t)$ as we are aiming to enforce that $T$ occurs less often (ideally never almost surely) and thus we prevent some scenarios where the expectant would be equal to 0. Additionally, \cref{eq:gen_imp_sampling} relies implicitly on the fact that there exists a well-defined likelihood ratio $\frac{d\prob}{d\q}(\underline{X})$ to allow for the change of measures in the first place.
\\\\
When our process $X$ progresses forward in time, the hitting condition of the generalized hitting time can be triggered in one of two ways\textemdash{}the process could either (i) jump directly into the hitting region or (ii) it could continuously drift into the region. The next two subsections will focus on designing a valid proposal measure $\q$ such that neither situation is possible. Preliminary experiments indicate that choosing $\underline{\mu}$ and $\underline{\sigma}$ to be something other than the original $\mu$ and $\sigma$ yield an \textit{increase} in estimator variance. We suspect this is due to the underlying distribution of $T$ being shifted too much from the original under $\prob$, leading to poor importance samples. As such, for this work we choose to leave these functions untouched, relegating this direction for future work, and instead focus solely on altering the jump distributions for the process at hand.

\subsection{Altering Jumps}\label{sec:jump_is}
When altering how the model can jump in a trajectory, determining the forms of $\underline{\nu}$ and $\underline{\lambda}$ are immediately pertinent. Our primary goal is to prevent $\underline{X}$ from triggering the hitting condition via jumping. Because of this, we do not need to change the effect that the jump time $\underline{S}_i$ and marks $\underline{M}_i$ have on the process and can instead directly manipulate the underlying intensity. This means that $\underline{\nu} := \nu$. 
\\\\
The Girsanov Theorem states that for a counting process under $\prob$ with intensity $\lambda$ and another under $\q$ with intensity $\underline{\lambda}$, the likelihood ratio $\frac{d\prob}{d\q}$ over $[0, t]$ is
\begin{align}
L_t := \left[\prod_{i=1}^{N_t} \frac{\lambda_{\underline{S}_i}(\underline{M}_i)}{\underline{\lambda}_{\underline{S}_i}(\underline{M}_i)} \right]\exp\left(\int_0^t \underline{\lambda}_s - \lambda_s ds\right)
\end{align}
with respect to $\filter_t$ \citep{privault2022introduction}.
Now we must determine what exactly $\underline{\lambda}$ should be.
\\\\
Earlier, for the base process $X$ we denoted the intensity of the generalized hitting time $ T$ to be $\lambda^T_t$. Naturally, this implies the existence of the corresponding counting process:
\begin{align}
N^T_t(\omega) := \ind( T(\omega) \leq t).    
\end{align}
A well known property of counting processes is that they can be aggregated and/or decomposed via the superposition property. This means that if we denote $N^{\bcancel{ T}}_t$ as the counting process for all jumps that \emph{did not} trigger the generalized hitting time with intensity $\lambda^{\bcancel{ T}}_t$, then it can be shown that:
\begin{align}
N_t(\omega) & \equiv N^T_t(\omega) + N^{\bcancel{ T}}_t(\omega) \\
\text{and } \lambda_t(\omega) & \equiv \lambda^T_t(\omega) + \lambda^{\bcancel{ T}}_t(\omega).
\end{align}
This decomposition is true also for our proposal intensity $\underline{\lambda}_t \equiv \underline{\lambda}^T_t + \underline{\lambda}^{\bcancel{ T}}_t$. To prevent hitting times being triggered via jumps, we zero out the appropriate intensity:
\begin{align}
\underline{\lambda}^T_t & := 0  \\
\text{ and } \underline{\lambda}^{\bcancel{ T}}_t & := \lambda^{\bcancel{ T}}_t \\
\implies \underline{\lambda}_t & := \lambda^{\bcancel{ T}}_t \text{ for all } t \geq 0 
\end{align}
We do not alter $\underline{\lambda}^{\bcancel{ T}}$ from the original $\lambda^{\bcancel{ T}}$ to try and make the proposal process resemble the original as much as possible outside of prohibiting the undesirable jumps. 
\\\\
It should be noted that due to the previously noted decomposition, we really have only performed a change of measure with respect to $N^T$. As such, the resulting likelihood ratio purely pertains to this subprocess:
\begin{align}
L_t & := \left[\prod_{i=1}^{N_t^T} \frac{\lambda_{\underline{t}_i^T}(\underline{m}_i)}{\underline{\lambda}_{\underline{t}_i}^T(\underline{m}_i)} \right]\exp\left(\int_0^t \underline{\lambda}_s^T - \lambda_s^T ds\right) \\
& = \exp\left(-\int_0^t \lambda_s^T ds\right) \text{ since } N^T_t=0 \text{ for all } t\geq 0 \text{ under }\q.
\end{align}
Should the counting process for the hitting time be Poisson (i.e., $\lambda_s^T(\omega)=f(s)$ for $\omega\in\Omega$), then the likelihood ratio derived is exactly equal to $\prob(N^T_t = 0) \equiv \prob( T > t)$, assuming that the process is not able to continuously walk into the hitting region, i.e., $\mu=\sigma=0$.

Using this likelihood ratio results in the following estimator process:
\begin{align}
    \prob(T \leq t) & = \E^\q\left[F^\text{IS}_t\right] \\
    F^\text{IS}_t & := 1 - L_t\ind(T > t) = 1 - \exp\left(-\int_0^t\lambda^T_s ds\right)\ind(T>t). \label{eq:6_is}
\end{align}
This estimator, under $\q$, can be roughly thought of as one minus the probability of the process not experiencing any jumps that result in the process $X$ entering the hitting region $A_t$.

\subsection{Alternative Approach} \label{sec:alt_is_approach}
We can utilize the same proposal distribution and measure to derive a different, unbiased estimator by choosing a different starting point for our derivations. Recall that the Tower Rule PDF estimator process takes on the form:
\begin{align}
    \prob( T \in [t, t+dt)) = \E^\prob \left[f^\text{TR}_t\right] dt = \E^\prob\left[\lambda_t^T \ind( T > t)\right] dt
\end{align}
Coincidentally, just like in \cref{eq:naive_neg_estimator}, $f^\text{TR}_t = 0$ if $ T \leq t$. Because of this, we can apply the same importance sampling change of measure as we did before. This yields the following alternative importance sampling estimators that will be denoted by $\text{IS}'$:
\begin{align}
    \prob( T \in [t, t+dt)) & = \E^\q\left[L_t\lambda^T_t\ind( T > t)\right]dt \\
    & := \E^\q\left[f^{\text{IS}'}_t\right]dt \\
    \prob( T \leq t) & = \int_0^t \prob( T \in [t, t+dt)) \\
    & = \int_0^t \E^\q\left[f^{\text{IS}'}_s\right]ds \\
    & = \E^\q\left[\int_0^t f^{\text{IS}'}_s ds\right] \text{ by Fubini's Theorem} \\
    & = \E^\q\left[\int_0^{t\wedge T} L_s\lambda^T_s ds\right] := \E^\q\left[F^{\text{IS}'}_t\right] \label{eq:6_alt_is}
\end{align}
Curiously, should the proposal measure $\q$ result in trajectories that never reach the hitting state (i.e., $\ind(T > 1) = 1$ almost surely), then it can be shown that $F^{\text{IS}}_t(\omega)=F^{\text{IS}'}_t(\omega)$ for all $t>0$ and $\omega \in \Omega$. 
An example of this occurring is in pure jump processes where $\mu=\sigma=0$ as $\q$ ensures that process will never jump into the hitting region of interest.
Under this condition, it follows then that
\begin{align}
    F_t^{\text{IS}'} & = \int_0^t L_s\lambda^T_s ds \\
    & = \int_0^t \lambda_s^T \exp\left(-\int_0^s \lambda^T_u du\right) ds \\
    & = -\left[\exp\left(-\int_0^s \lambda_u^T du\right)\right]^t_0 \text{ by Fundamental Theorem of Calculus} \\
    & = 1 - \exp\left(-\int_0^t \lambda_s^T ds\right) \\
    & = 1 - L_t \\
    & = F_t^\text{IS}.
\end{align}
In practice, should this situation arise it is recommended to opt for $F_t^\text{IS}$ over $F_t^{\text{IS}'}$ to avoid potential compounding errors from numerically approximating nested integrals.

\section{Compositions of Generalized Hitting Times}\label{sec:6_comp_ght}
The goal of this section is to illustrate that the previous derived estimators also have simple extensions for estimating probabilistic statements involving compositions of multiple generalized hitting times. Some compositions are directly estimable due to the composition itself being a generalized hitting time, e.g., $\min\{ T_1,  T_2\}$ or $\max\{ T_1,  T_2\}$, whereas others require a bit more manipulation. We will focus on two such cases: (i) comparisons or ordering of hitting times $\prob( T_1 <  T_2)$ and (ii) joint distributions $\prob( T_1 < t_1,  T_2 < t_2)$. The latter of which we will also consider the generalization to the joint distribution of $n$ different generalized hitting times. 

\subsection{Hitting Time Ordering}
The probabilistic statement $\prob( T_1 <  T_2)$ is typically intractable for a generic process $X_t$, both analytically and with unbiased estimations, due to the fact that $ T_1$ could potentially take on values all across $\R_{>0}$ and still satisfy the statement $ T_1 <  T_2$. One approach to address this issue is to estimate a slightly different statement: $\prob( T_1 <  T_2 \leq t) := \prob( T_1 <  T_2,  T_2 \leq t),$ where $t \in \R_{>0}$ determines the time to sample the underlying process out to. Naturally, we recover the original query as $t \rightarrow \infty$.
\\\\
Now, while there does exist a naive estimator for this statement, $\E^\prob\left[\ind( T_1 <  T_2 \leq t)\right]$, it is not clear how to manipulate this form directly to further apply importance sampling. From our previous derivations in the prior section, we know that there are two conditions that the integrand needs to satisfy in order to easily apply our strategy for importance sampling. Namely, (i) preventing samples, all or a subset of them, from yielding an integrand of 0 should be immediately actionable and (ii) if there are multiple conditions for integrands equal to 0 they must have an OR structure instead of an AND structure. The former condition simply means that any change we impose must be done instantaneously and not condition on future information. To give an example for the latter condition, say we have two expectations that are equivalent: $\E[f_1(X)]=\E[f_2(X)]$. Furthermore, we assume that $f_1(X(\omega))=0$ if $\omega \in A \cup B$ and $f_2(X(\omega))=0$ if $\omega \in A \cap B$ for $A,B \subset \Omega$. Here, $A$ and $B$ represent actionable conditions that can be immediately prevented. By (ii), it is straightforward to apply our importance sampling approach to $\E[f_1(X)]$; however, it is non-trivial to do the same to $\E[f_2(X)]$.
\\\\
With this in mind, it becomes clear that the integrand $\ind( T_1 <  T_2 \leq t)$ violates condition (i), as imposing $ T_2 \leq t$ is not immediately actionable. In other words, without conditioning on future information, we cannot ensure that $ T_2$ occurs before time $t$ through instantaneous interventions.\footnote{Note that technically from just a sampling point of view, we could ensure this by forcing $ T_2$ to occur at time $t$ if it had not already for a given sequence. This has a number of different consequences though. Namely, it places a point-mass on the distribution of $ T_2$ at time $t$ under the proposal measure, which would make $\q$ and $\prob$ not absolutely continuous. Additionally, even if this property was not violated, implementing this estimator in practice would be inefficient as it is not clear how to parallelize it for multiple values of $t$.} In \cref{sec:6_imp_sample_estimators}, we initially got around this restriction by considering the complement of the initial integrand. For this scenario, that would correspond to $\ind( T_1 >  T_2 \text{ or }  T_2 > t)$. Unfortunately, this integrand violates condition (ii), as it is equal to 0 when $ T_1 <  T_2$ \emph{and} $ T_2 \leq t$.
\\\\
Taking a cue from \cref{sec:alt_is_approach}, applying importance sampling becomes tractable if we first start from a PDF process:
\begin{align}
\prob( T_1 <  T_2 \in [t, t+dt)) & := \prob( T_1< T_2,  T_2 \in [t, t+dt)) \\
& = \E^\prob\left[\ind( T_1 < t,  T_2 \in [t, t+dt))\right] \\
& = \E^\prob\left[\ind( T_1 < t < T_2)\lambda_t^{ T_2}\right]dt \text{ by Tower Rule} \\
& = \E^\q\left[\ind( T_1 < t < T_2)L_t\lambda_t^{ T_2}\right]dt \\
\prob( T_1 <  T_2 \leq t) & = \int_0^t \prob( T_1 <  T_2 \in [s, s+ds)) \\
& = \E^\q\left[\int_0^t \ind( T_1 < s < T_2) L_s \lambda_s^{ T_2} ds\right] \text{ by Fubini's Thm.} \\
& = \E^\q\left[\ind(T_1 < T_2)\int_{ T_1 \wedge t}^{t\wedge T_2} L_s \lambda_s^{ T_2} ds\right], \label{eq:ordered_is}
\end{align}
where the measure $\q$ and $L$ correspond to the proposal model with jumps that avoid the hitting region for $T_2$. This can be easily generalized to multiple orderings of hitting times:
\begin{align}
\prob( T_1 <  T_2 < \dots <  T_K \leq t) & = \E^\q\left[\ind(T_1 < T_2 < \dots < T_K)\int_{ T_{K-1} \wedge t}^{t\wedge T_K} L_s \lambda_s^{ T_K} ds\right]. \label{eq:mult_ordered_is}
\end{align}
where $\q$ and $L$ here correspond to avoiding the following hitting region process:
\begin{align}
    A_t := \begin{cases}
        \cup_{i=k+1}^K A_t^{ T_i} & \text{if } k=1,\dots,K-1 \\
        A_t^{ T_K} & \text{otherwise}
    \end{cases}
\end{align}
where $ T_{k-1} < t \leq  T_{k}$ for $k=1,\dots,K$ and $ T_0=0$. Note that this enforces sampled jumps to respect the ordering of generalized hitting times (with the last, $ T_K$, never occurring) almost surely. This is important as without this enforcement, we could produce a sample with the hitting times out of the desired order, yielding an integrand value of 0 no matter how far into the future we simulate. Practically speaking, for $\prob( T_1 <  T_2 < \dots <  T_K \leq t) > 0$, this would lead to an (unbiased) estimator with higher variance.

\subsection{Joint Hitting Time Distributions}
Consider two generalized hitting times $ T_1$ and $ T_2$. We are interested in the joint CDF of the two:
\begin{align}
\prob( T_1 \leq t_1,  T_2 \leq t_2) = \E^\prob\left[\ind( T_1 \leq t_1,  T_2 \leq t_2)\right].
\end{align}
In general, $t_1,t_2 \in [0,\infty)$; however, we will assume $0 \leq t_1 < t_2 < \infty$ without loss of generality. For our purposes, we are interested in a form of an estimator for the joint CDF that can easily produce multiple estimates for various values of $t_1$ and $t_2$, similar to the earlier estimators in \cref{sec:6_imp_sample_estimators} allowed for multiple values of $t$ for $\prob(T \leq t)$ by simply simulating further in time. Since time is one-dimensional, that means for a given joint CDF estimate it stands to reason that we will only be able to manipulate or restrict one of the hitting times in order to satisfy our desiderata.
\\\\
As a brief tangent, consider the CDF of the maximum of two hitting times: $\prob(\max\{ T_1,  T_2\} \leq t)$. While it is true that the max of generalized hitting times is itself a generalized hitting time, this CDF could also be separately estimated through compositions of orderings:
\begin{align}
\prob(\max\{ T_1,  T_2\} \leq t) & = \prob( T_2 <  T_1 \leq t) + \prob( T_1 <  T_2 \leq t).
\end{align}
A very similar decomposition can be done for the joint CDF of hitting times:
\begin{align}
\prob( T_1 \leq t_1,  T_2 \leq t_2) & = \prob( T_2 <  T_1 \leq t_1) + \prob( T_1 < t_1,  T_1 <  T_2 \leq t_2).
\end{align}
We can then perform importance sampling on each of these terms individually:
\begin{align}
\prob( T_1 < t_1,  T_1 <  T_2 \leq t_2) & = \int_0^{t_2} \prob( T_1 < t_1,  T_1 <  T_2 \in [s, s+ds)) \\
& = \int_0^{t_2} \prob( T_1 < t_1 \wedge s,  T_2 \in [s, s+ds)) \\
& = \int_0^{t_2} \E^\prob\left[\ind( T_1 < t_1 \wedge s,  T_2 \in [s, s+ds))\right] \\
& = \int_0^{t_2} \E^\prob\left[\ind( T_1 < t_1 \wedge s,  T_2 > s)\lambda_s^{ T_2}\right]ds \text{ by Tower Rule}\\
& = \int_0^{t_2} \E^\q\left[\ind( T_1 < t_1 \wedge s, T_2 > s)L_s\lambda_s^{ T_2}\right]ds \label{eq:joint_cdf_is} \\
& =  \E^\q\left[\int_{ T_1 \wedge t_2}^{t_2 \wedge T_2}\ind( T_1 < t_1 )L_s\lambda_s^{ T_2}ds\right] \text{ by Fubini's Theorem} 
\end{align}
where the $\q$ and $L$ in \cref{eq:joint_cdf_is} correspond to the same used in \cref{eq:ordered_is}.
\\\\
Just as before for the ordered composition, the joint CDF estimator can be generalized to $K$ hitting times: $\prob( T_1 \leq t_1, \dots,  T_K \leq t_K)$. Deriving an estimator for this query can be accomplished by considering all possible orderings of $K$ hitting times:
\begin{align}
\prob( T_1 \leq t_1, \dots,  T_K \leq t_K) & = \sum_{\pi \in P_K} \prob(\{ T_i \leq t_i\}_{i=1}^K,  T_{\pi_1} < \dots <  T_{\pi_K}) \\
\prob(\{ T_i \leq t_i\}_{i=1}^K,  T_{\pi_1} < \dots <  T_{\pi_K}) & = \E^\q\left[\int_{ T_{\pi_{K-1}}\wedge t_{\pi_K}}^{t_{\pi_K} \wedge T_{\pi_K}} \ind(\{ T_{\pi_i} \leq t_{\pi_i}\}_{i=1}^{K-1})L_s\lambda_s^{ T_{\pi_K}} ds \right]
\end{align}
where $P_K$ is the set of all permutations of $\{1, \dots, K\}$, $\pi_i$ is the $i^\text{th}$ element of the permutation $\pi \in P_K$, and $\q$ and $L$ is the same used in \cref{eq:mult_ordered_is} for the ordering of $ T_{\pi_1} <  T_{\pi_2} < \dots <  T_{\pi_K}$. Note that to produce an estimate for this joint CDF will require $K!$ individual estimates for each of the potential orderings. One potential trade-off that can be made is to not enforce the specific orderings of hitting times in $\q$ and $L$, and instead just have them prevent the final hitting time $ T_{\pi_K}$ from occurring. This corresponds to the following decomposition:
\begin{align}
\prob( T_1 \leq t_1, \dots,  T_K \leq t_K) & = \sum_{j=1}^K \prob(\{ T_i \leq t_i\}_{i=1}^K, \max_i\{ T_i\} =  T_j) \\
\prob(\{ T_i \leq t_i\}_{i=1}^K, \max_i\{ T_i\} =  T_j) & = \E^\q\left[\int_{\max_{i\neq j} \{ T_i\} \wedge t_j}^{t_j \wedge T_j} \ind(\{ T_{i} \leq t_{i}\}_{i\neq j})L_s\lambda_s^{ T_{j}} ds \right].
\end{align}
This results in $K$ different estimates for the different possible maximum hitting times. Depending on the query of interest, number of hitting times, and process may result in one of these forms being more efficient usage of computing than the other. 

\section{Experiments}\label{sec:6_experiments}

In this section, we present five different example scenarios involving different kinds of stochastic jump processes and generalized hitting time-based queries. 
In particular, we investigate the relative efficiencies of various approaches and demonstrate that overall our proposed importance sampling approach is significantly more efficient than naive estimation and is more reliable than using the tower rule estimators.
While this selection of use cases and results are by no means exhaustive, they do demonstrate the versatility and effectiveness of the proposed techniques.

Experiments using purely point processes with no Brownian motion simulate trajectories exactly using \cref{alg:mtpp_sampling} (either with the original process for naive and tower rule estimation or the proposal MTPP for the importance sampling estimator). For those that do involve Brownian motion, a simple Euler scheme is utilized as outlined in \citet{bruti2007approximation}.\footnote{Applying importance sampling in conjunction with the Euler discretization approach allows for preventing samples from triggering hitting times almost surely under $\q$, unlike in continuous time where only jumps can be prevented. More information on this can be found in \cref{sec:6_discrete_appendix}.} For all scenarios, a discretization time step $\Delta t$ of 0.01 is used for both simulating trajectories with the Euler scheme and numerically approximating integrals for estimators. 

\subsection{First Passage Time}\label{sec:6_hitting_exp}

For this experiment, we utilize the Merton jump diffusion model \citep{matsuda2004introduction}, which is commonly used for modeling stock options. The scalar-valued stochastic jump process is defined by
\begin{align}
dX_t := \left(r-\frac{1}{2}\sigma^2-\lambda k\right)X_{t-}dt + \sigma X_{t-}dW_t + (Y_t-1)X_{t-}dN_t \label{eq:6_exp_0}
\end{align}
where $\log Y_t \sim \mathcal{N}(\mu, \delta^2)$ with $Y_t \perp Y_{t'}$ for $t \neq t'$, $k=\exp\left(\mu+\frac{1}{2}\delta^2\right)-1$, $r$ is treated as the risk-free rate of return for a stock, $\sigma>0$ is interpreted as the stock volatility, and $\lambda > 0$ is the constant intensity for the Poisson counting process $(N_t)_{t\geq 0}$. For experimental results concerning this process, we will assume the following parameter values: $X_0=1, r=0.02, \mu=0, \delta=0.3, \lambda=1,$ and $\sigma=0.2$.

We will investigate the efficiency of computing the time at which the process first crosses over different fixed barriers (or prices) of $cX_0$ for $c \in \{1.25, 1.5, 2, 3, 5, 7, 10\}$: $\prob(\inf \{s \sep X_s \geq cX_0\} \leq t):=\prob(T_c \leq t)$. Note that the queries can be interpreted as the distribution of when a stock will multiply in price by $c$. We evaluate this query for $t \in [0, 5]$ using naive (NE), tower rule (TR), and importance sampling (IS) estimators to compare resulting estimator variances. Each estimate was generated using $10^5$ sampled trajectories.

Results for this experiment can be seen in \cref{fig:6_exp_0}. As a general trend, we note a large reduction in variance, often many orders of magnitude, for both tower rule and importance sampling estimators compared to naive estimation across all times and hitting regions with higher relative efficiencies noted during particularly lower hitting time likelihoods. For example, both estimators achieve the highest relative efficiencies for any barrier $c$  when $t$ is close to 0, which prohibits the process from having any time to reasonably cross over the boundary of interest. For barriers closer to the process' starting point $X_0$, it appears that importance sampling is favored with lower variance, as seen in the bottom plot in \cref{fig:6_exp_0}; however, for the barriers further out and closer to $10$, it is less clear and the comparison between the two is far more noisy.

\begin{figure}
    \centering
    \includegraphics[width=0.95\textwidth]{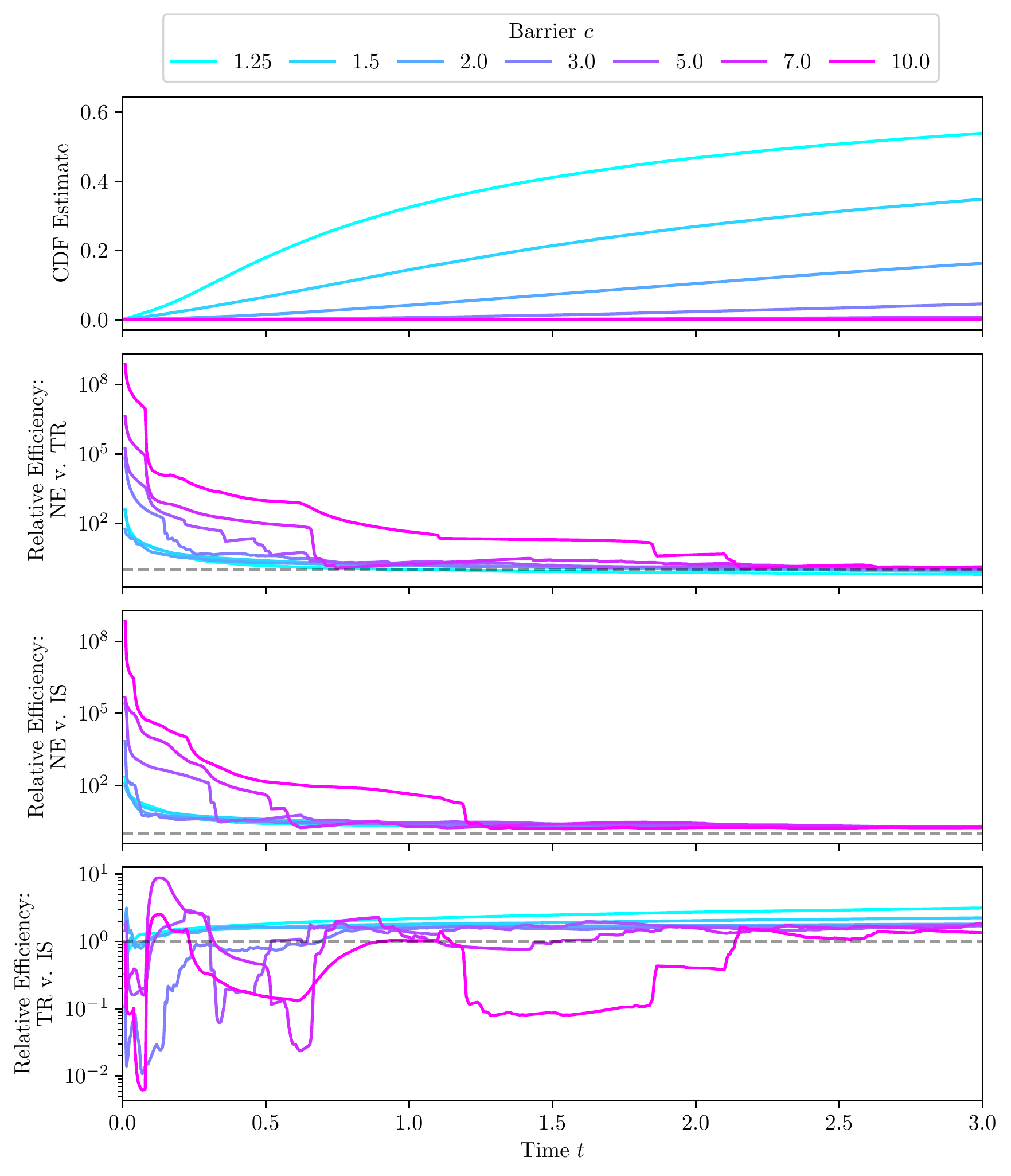}
    \caption{Estimates for the distribution of the hitting time of entering the region $[c, \infty)$ with $X_t$ from \cref{eq:6_exp_0} are shown in the top plot where $c$ corresponds to the different colored lines in the plots ranging from $1.25$ to $10$. Bottom three plots show estimated relative efficiencies comparing naive (NE), tower rule (TR), and importance sampling (IS) estimators, with dashed horizontal lines indicating an  efficiency of 1. For plots labeled ``XX v. YY'', efficiency values greater than 1 favor the ``YY'' method.}
    \label{fig:6_exp_0}
\end{figure}

\subsection{First Exit Time}

Consider the univariate stochastic jump process
\begin{align}
dX_t & := btX_{t-}dt + X_{t-}dW_t - X_{t-}U_tdN_t \label{eq:6_exp_1}
\end{align}
where $X_0=1$, $U_t \sim \text{Uniform}(0,1)$ with $U_t \perp U_{t'}$ for $t\neq t'$, $b\geq 0$ is a constant drift-rate of increase, and $(N_t)_{t\geq 0}$ is the counting process to a self-correcting temporal point process characterized by conditional intensity
\begin{align}
\lambda_t & := \exp\left(t - N_{t-}\right).
\end{align}
This process can be described as a geometric Brownian motion with drift $bt$ and jumps that regress the process closer to 0. 

We are interested not in the time at which the process first crosses over a boundary, but rather the time at which it crosses \textit{back} over. Put more plainly, for a boundary at value $c>1$, then $T_c$ defines the time at which the process returns to the region $[0, c)$ after first entering $[c, \infty)$. We evaluate the query $\prob(T_c \leq t)$ for $t \in [0, 5]$ with $c=3$ using naive, tower rule, and importance sampling estimators for different drift-rates $b$ ranging from 0 to 50. Each estimate was generated using $10^5$ sampled trajectories.

Results are shown in \cref{fig:6_exp_1} with specific rates for $b$ can be seen in the legend. We see that for both extremely low and extremely high values of $b$, the CDF plateaus early on in time. For the former, this is due to the process not being likely to cross into  $[c, \infty)$ in the first place, whereas in the latter the process is unlikely to return to $[0, c)$ after entering$[c, \infty)$ due to the aggressive rate of growth. The resulting estimator efficiencies, for both tower rule and importance sampling, do tend to correlate with this rate of growth with low $b$ values yielding low efficiencies and high $b$ values yielding higher efficiencies. Interestingly, the estimators relative efficiency is highest at both the beginning of time and when the CDF has saturated. This can be seen by noting when both the CDF and the efficiency plots flatten out for a given colored line.

Unlike in the previous experiment, there are instances where the tower rule estimator is less efficient than naive estimation. However, in the worst case scenario for importance sampling it is just as efficient as naive estimation. This is further exemplified in the bottom plot of \cref{fig:6_exp_1} where, after an initial turbulent period for small values of $t$, importance sampling is favored over the tower rule estimator in general. 

\begin{figure}
    \centering
    \includegraphics[width=0.9\textwidth]{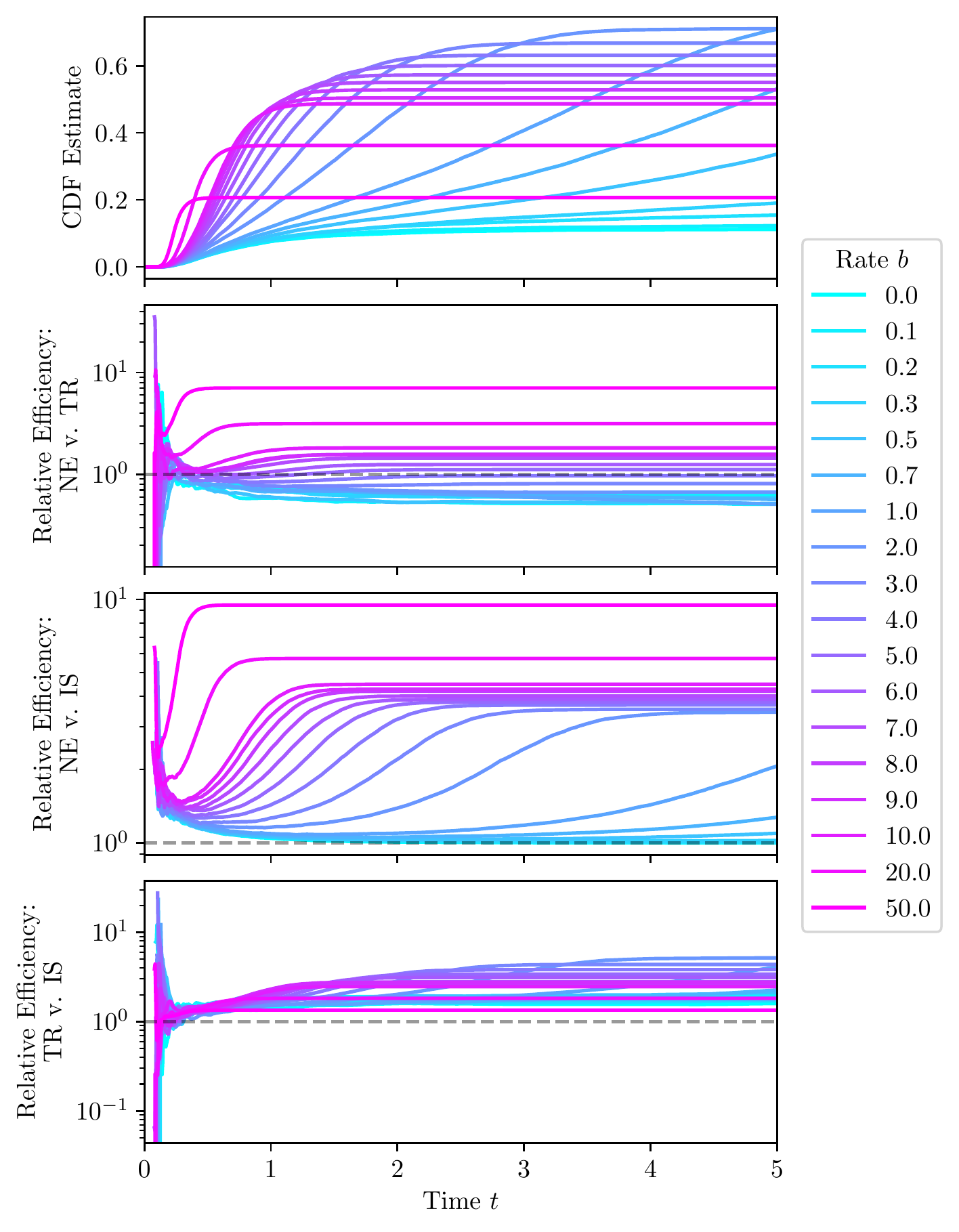}
    \caption{Results shown in the same format as \cref{fig:6_exp_0} for comparing different first exit time distributions of the region $[3, \infty)$ over time $t\in[0,5]$ as they vary across drift-rates $b$ which range from $0$ to $50$ for $X_t$ from \cref{eq:6_exp_1}.}
    \label{fig:6_exp_1}
\end{figure}

\subsection{Cover Time}\label{sec:6_cover_time}

Consider a graph $G=(V,E)$, which is finite and connected, and an accompanying stochastic process $(X_t)_{t\geq 0}$ that governs a random walk over the graph in continuous time. The \textit{cover time} of a graph is the time at which the random walk has visited each vertex in $V$. The cover time of a graph is known to have connections to the \textit{coupon collector problem} \citep{ferrante2014coupon} and \textit{the (double) dixie cup problem} \citep{newman1960double}. Please refer to \citep{lovasz1993random} for a more on these connections and details of cover times in general.

Letting $T_G$ denote the cover time for a graph $G$, this is appropriately decomposed as a maximum over individual hitting times for each vertex $T_v$. Formally, this is defined as
\begin{align}
    T_G & := \max \{T_v \sep v \in V\} \\
    T_v & := \inf \{s \sep X_s = v\} \text{ for } v \in V.
\end{align}
We assume the random walk to be a continuous-time Markov model that is described in the following manner:
\begin{align}
    X_t & := M_{N_t} \label{eq:6_exp_2}\\
    \lambda_t(v) & = \mu p_{X_t,v}
\end{align}
where $\mu > 0$ describes the base rate of a jump occurring and $p_{i,j} \in [0,1]$ describes is probability of transitioning from vertex $i$ to vertex $j$ during a jump. Note that the graph structure is respected by forcing $p_{i,j}=0$ for when the edge $(i,j)$ is not present in the graph; however, for our experiments we will assume that the graph is complete which will allow the process to jump from any vertex to any other vertex (including itself). 

The query $\prob(T_G \leq t)$ is evaluated for a fixed value $t=10$, where the relative efficiency between the naive and importance sampling estimator is measured across 100 different processes, instantiated with parameters sampled according to $p_{i,j}=\frac{\exp(z_{i,j}\tau^{-1})}{\sum_{j'\in V} \exp(z_{i,j'}\tau^{-1})}$ for $z_{i,j}\sim\text{Uniform}[0,1]$ for $i,j\in V$ with temperature $\tau > 0$ and $\mu_v = 1$ for $v \in V$. While using complete graphs with $|V|=6$, we investigate the effect of sparsity in transitioning states by setting temperature $\tau$ to $0.1, 1,$ and $10$. Low temperature values yield sparser transition probabilities and higher temperatures yield more uniform transition distributions. Each estimate was generated using $10^4$ sampled trajectories.

Resulting relative efficiencies for different temperatures $\tau$ across 10 different random instantiations for each can be seen in \cref{fig:6_exp_2}. The relative efficiencies shown start at varying time $t$ values due to there being insufficient sampled support for variance estimates prior to that time $t$.\footnote{Given enough samples, the estimators will typically give non-zero estimates prior to these points in time; however, it is difficult to know ahead of time just how many are necessary for this.}
Similar to the basic hitting time experiment in \cref{sec:6_hitting_exp}, we see several orders of magnitude improvement over naive estimation when using both the tower rule and importance sampling estimators. Sparser transition matrices seem to lead to a bit more unpredictable and chaotic estimator efficiencies, noted by the variety and jaggedness of results for $\tau=0.1$ compared to $\tau=10$. This is most likely due to there being less immediately accessible possibilities for the process to jump from its current state to the last remaining vertex when the transition matrix is sparser. 

In summary,  whether importance sampling or the tower rule approach is superior varies from process to process across all temperatures. However, importance sampling is always better than naive estimation which cannot be said for the tower rule estimator. Because of this, importance sampling is likely to be a good default technique for this setting.

\begin{figure}
    \centering
    \includegraphics[width=0.99\textwidth]{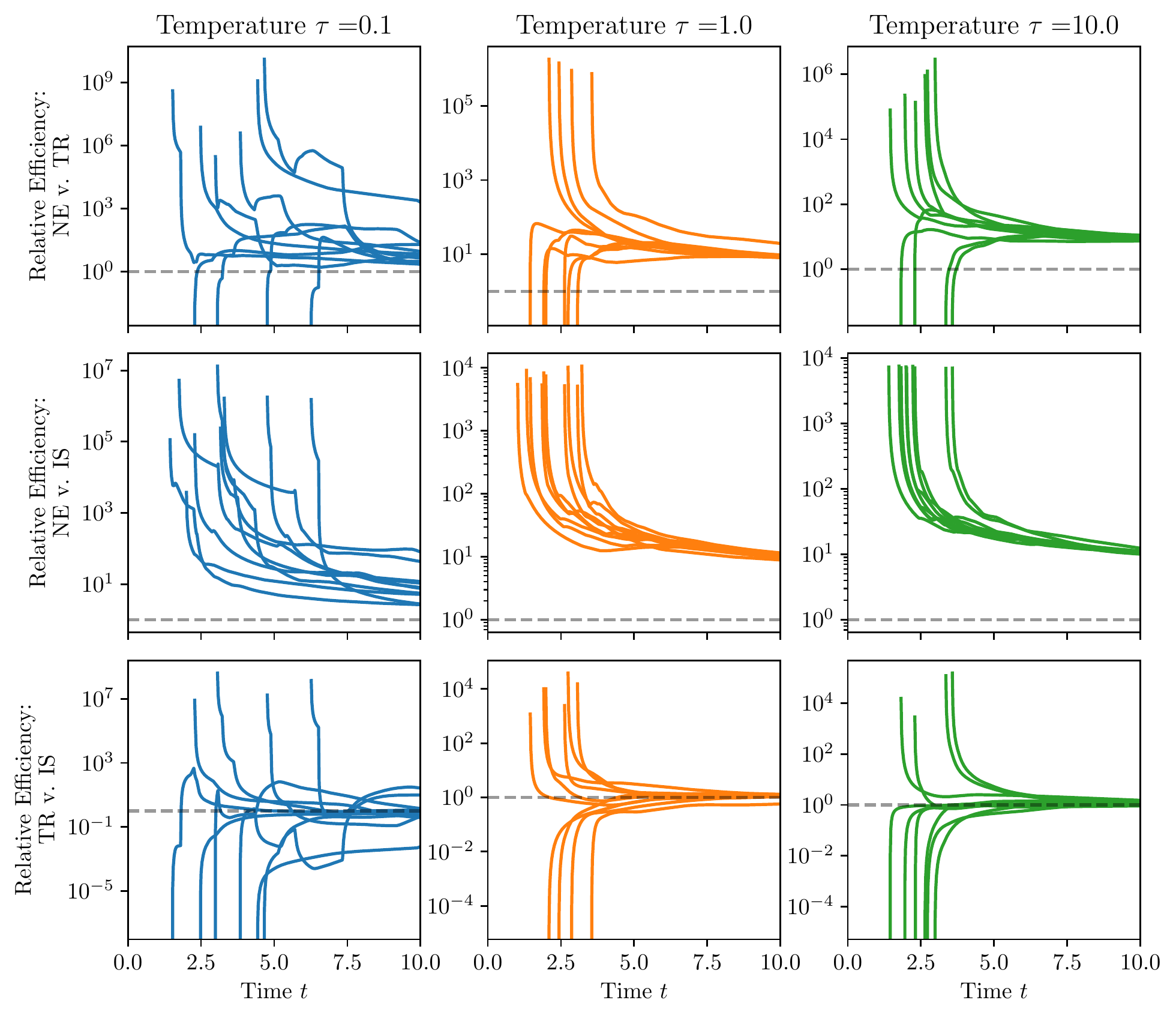}
    \caption{For each temperature $\tau\in\{0.1, 1, 10\}$, ten different processes defined by \cref{eq:6_exp_2} are initialized with random parameters according to details in \cref{sec:6_cover_time}. Relative efficiencies between naive (NE), tower rule (TR), and importance sampling (IS) are shown, similar to \cref{fig:6_exp_0}, for computing the distribution of the cover time $\prob(T_G \leq t)$ for $t \in [0, 10]$. Plotted lines start at varying time $t$ values due to the cover time for that sampled process requiring being so rare that the estimators prior to that time all equal $0$ exactly.}
    \label{fig:6_exp_2}
\end{figure}

\subsection{Joint Hitting Time Distributions}

Let $(X_t)_{t\geq 0}$ be a Hawkes process with 3-dimensional continuous-valued marks that follow the following constraints:
\begin{align}
    X_t & := M_{N_t} \\
    \lambda_t(m) & := \mu \cdot \mathcal{N}_3(m; 0, 1) + \sum_{(S,M) \in \hist_{t-}} \exp\left(-(t-S)\right)\mathcal{N}_3(m; M, (t-S)^2) 
\end{align}
for $m \in \R^3$, $\mathcal{N}_3(\cdot; a, b)$ is the probability density function of a 3-dimensional normal distribution with mean $a$ and covariance matrix $bI$. 
This marked intensity encourages new marks to occur close to previous ones, so long as they happen also close enough in time. Previous events have less influence on the timing and placement of future ones due to the $\exp(-(t-S))$ scaling term and the variance $(t-S)^2$ respectively. Note that this process is a variant of ones commonly used for modeling earthquake and subsequent aftershocks \citep{ogata1998space}.

Consider the following subsets of $\R^3$:
\begin{align}
    R_1 &:= [0.5, \infty) \times [0.5, \infty) \times [0.5, \infty) \\
    R_2 &:= (-\infty, -0.5] \times [0.5, \infty) \times [0.5, \infty) \\
    R_3 &:= (-\infty, -0.5] \times (-\infty, -0.5] \times [0.5, \infty) \\
    R_4 &:= [0.5, \infty) \times (-\infty, -0.5] \times [0.5, \infty) \\
    R_5 &:= [0.5, \infty) \times (-\infty, -0.5] \times (-\infty, -0.5]
\end{align}
and let the associated hitting time of each region be denoted $T_{R_i}$ for $i=1,\dots,5$. We are interested in calculating and comparing the joint CDFs of varying amounts of these hitting times. Estimates are computed for $\prob(\{T_{R_i} \leq t\}_{i=1}^d)$ for $d=1,\dots,5$ and $t\in [0,10]$ using naive and importance sampling estimators.\footnote{In practice, we can actually estimate the joint CDF using different values of $t$ for each individual hitting time, such as $\prob(\{T_{R_i} \leq t_i\}_{i=1}^d)$; however, for these results all $t_i$'s are fixed to the same $t$ for visualization purposes.} As mentioned in \cref{sec:6_comp_ght}, for importance sampling we can either enforce all possible orderings of $\{T_{R_i} \leq t\}_{i=1}^d$ or simply enforce which hitting time occurs last. We investigate the efficiency for both methods, and refer to these estimators as ``ordered importance sampling'' and ``unordered importance sampling'' respectively. Each estimate was generated using $3000$ sampled trajectories.

The resulting CDF estimates and relative efficiencies for these three methods can be seen in \cref{fig:6_exp_3}. It is worth noting that the two importance sampling procedures are identical for $d=1$ and $d=2$, and that $d=1$ is similar to the simple hitting time scenario explored in \cref{sec:6_hitting_exp}. When $d > 1$, the relative efficiencies are only well-defined after some time $t$ has passed. This is due to there being insufficient support to estimate the estimators' variance from the finite set of samples generated over those regions. When they are well-defined, we see similar relative efficiency values across values of $d>1$ and for both importance sampling estimators when compared against naive estimation. While this needs to be confirmed on a case-by-case basis, this does suggest that the simpler and easier to implement unordered importance sampling estimator is sufficient for efficient joint distribution estimation. 

\begin{figure}
    \centering
    \includegraphics[width=0.9\textwidth]{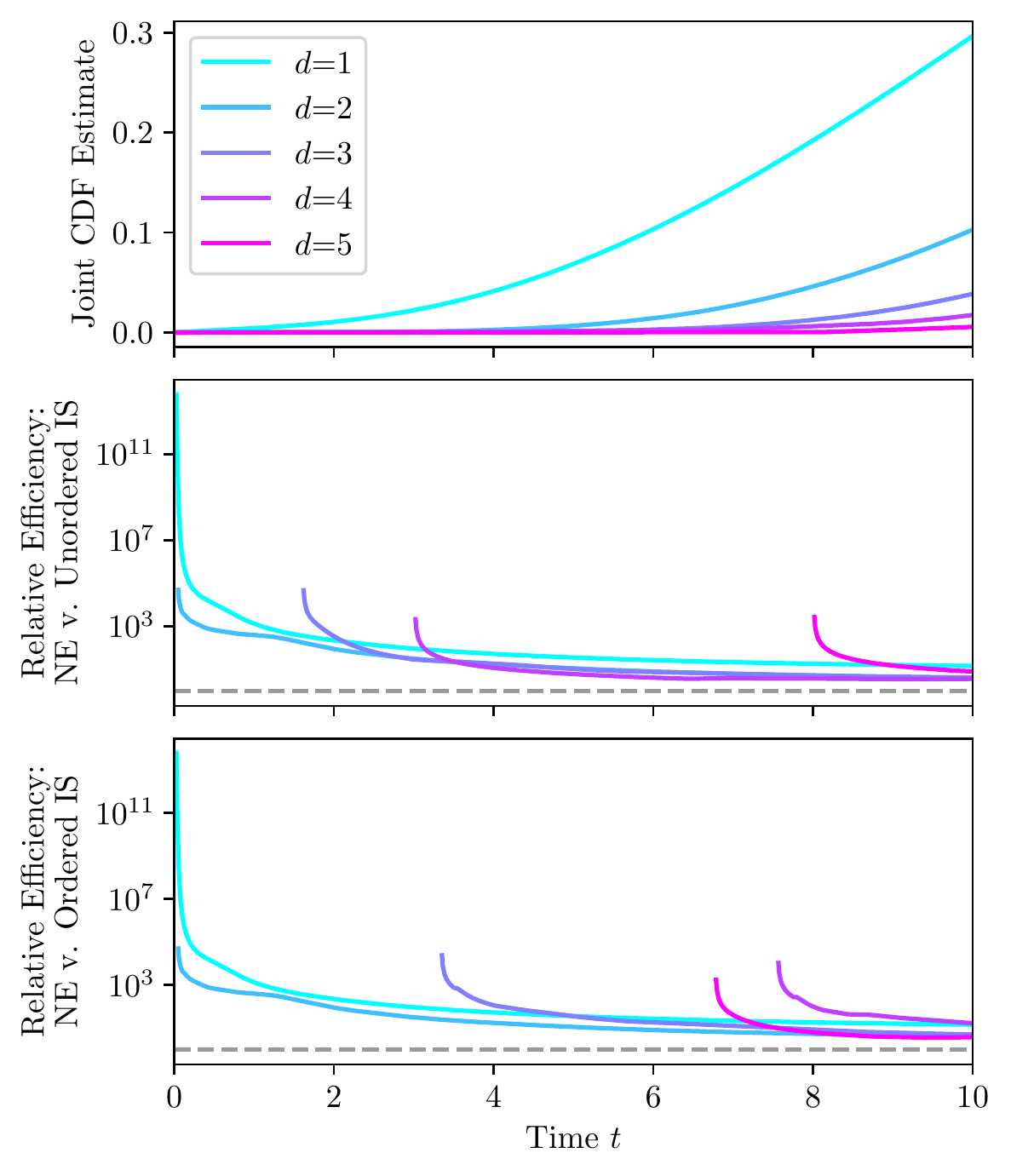}
    \caption{Estimated values and efficiencies for the joint distribution $\prob(\{T_{R_i}\leq t\}_{i=1}^d)$ for $t \in[0,10]$, with different colors for $d=1,\dots,5$, computed with naive estimation (NE) and importance sampling (IS), both ordered and unordered, all shown in a similar format as \cref{fig:6_exp_0}. Efficiency values sometimes start at varying values of time $t$ due to the same reasons as in \cref{fig:6_exp_2}. For that reason, the relative efficiency between ordered and unordered importance sampling is not shown due to a lack of sampled support.}
    \label{fig:6_exp_3}
\end{figure}

\section{Conclusion}\label{sec:6_conclusion}

In this chapter, we generalize the approaches from Chapter 4 to apply to the more general setting of stochastic jump processes. Additionally, estimators are derived for a novel class of random times we termed generalized hitting times. Specific estimators are given for compositions and joint distributions for these times. All of these results are empirically verified in a variety of different use cases and settings, sometimes achieving orders of magnitude reduction in estimator variance compared to naive Monte Carlo estimation.

\chapter{Conclusion}

\noindent

This dissertation explores the potential of extracting probabilistic beliefs from autoregressive sequence models, extending their capabilities beyond immediate next-step predictions. 
As modern machine learning models continue to evolve, fueled by increasing data availability and enhanced computational power, we can build trust in their ability to capture longer-range dependencies. 
This necessitates efficient methods for accessing these models' ``beliefs,'' enabling practitioners to expand their use cases and rely on models beyond their original training objectives.

This work provides a direct solution to this need. 
Our methods empower practitioners to leverage pre-trained autoregressive models and efficiently estimate their beliefs for various probabilistic queries, without modifying the underlying architecture. 
This approach offers significant financial and environmental benefits by utilizing existing models, eliminating the need for additional resource-intensive training for specific queries.

\section{Reiteration of Contributions}

\paragraph{Chapter 3:} We establish a framework for formulating diverse probabilistic queries on discrete sequence models. 
We develop efficient estimation methods, leveraging importance sampling and beam search, to answer these queries.

\paragraph{Chapter 4:} We extend the importance sampling approach to marked temporal point processes (MTPPs), while theoretically guaranteeing improved sampling efficiency. 
Empirical results demonstrate significant variance reduction in query estimators across MTPP models.

\paragraph{Chapter 5:} We utilize the proposal distribution from Chapter 4 to address missing information in MTPPs. 
This offers the first principled approach, independent of specific parameterization, for handling missing data in MTPPs.

\paragraph{Chapter 6:} We further extend our importance sampling methods to handle highly expressive stochastic jump processes. 
This enables us to answer queries on a novel class of random times termed "generalized hitting times." 
Notably, these methods generalize to Chapters 3-5 due to the inherent generality of stochastic jump processes.

By providing efficient and principled methods for extracting long-range probabilistic beliefs, this work opens doors for broader and more informed applications of autoregressive models.

\section{Future Research Directions}
While the presented methods offer significant benefits, they also have limitations inherent in both the underlying models and the approximation procedures. 
This section explores key areas for future research to further expand the impact and applicability of this work.

\paragraph{Long-Range Belief Distillation}
Current methods require repeated sampling from the autoregressive model for each query. 
For frequently asked queries, it might be more computationally efficient to directly model the query instead of estimating it through sampling. 
Knowledge distillation could be used, where the original model acts as a ``teacher'' and a new model directly approximates the query as a ``student'' \citep{gou2021knowledge}. 
This would involve sampling only during distillation, with the student model then providing faster estimates at inference time. 
However, careful consideration is needed for student model parameterization and the targeted query class.

\paragraph{Probabilistic Queries as Training Objectives}
This work focused on efficient extraction rather than predictive performance of the resulting queries. 
Evaluating the quality of predictions compared to real data is a crucial area for future research. 
If the goal is high-quality predictions for a specific query class, incorporating this into training or fine-tuning stands to significantly improve inference performance. 
For instance, instead of solely minimizing the cross-entropy of the next-step distribution, the model could additionally be trained on the hitting time distribution of specific events. 
However, current methods struggle with non-differentiable samples from autoregressive models, posing a challenge for using them as training objectives. 
New differentiable sampling techniques or alternative approaches, potentially using reinforcement learning, are needed to overcome this hurdle.

\paragraph{Querying Higher Order Beliefs}
The current methods are limited in the granularity of the modeling scope. 
For models handling fine-grained details, like individual characters in text or DNA base pairs for genomic sequences, the estimable queries are not very informative. 
For example, knowing when the next A-T pair appears in DNA is not particularly insightful. 
Similarly, for character-level language models, knowing if a period will occur before an exclamation point is not very helpful. 
Instead, we often want to ask questions about higher-order semantics, such as when the mood or topic will shift in text or when cancer-associated genes may appear next in an individual's DNA.
This requires methods that can handle softer constraints and potentially incorporate other latent, semantic information.

Addressing these future research directions will further enhance the capabilities and impact of extracting probabilistic beliefs from autoregressive models, enabling them to answer more complex and informative questions across diverse applications.

\clearpage
\phantomsection

\bibliographystyle{plainnat}
\bibliography{thesis}

\captionsetup[figure]{list=no}
\captionsetup[table]{list=no}

\begin{appendices}
\appendix
\chapter{Supplemental Material for Chapter 3}

\addtocontents{toc}{\protect\setcounter{tocdepth}{0}}

\section{Variance of Estimates from the Hybrid Method}
\label{sec:3_app_hybrid_variance}

We are assuming to be under the hybrid method regime where a collection of sequences $\mathcal{B}\subset\mathcal{Q}$ relevant to answering $\prob_\theta(\textbf{X}_{1:K}\in\mathcal{Q})$ have been deterministically found and are interested in using sampling methods to estimate the remainder $\prob_\theta(\textbf{X}_{1:K}\in\mathcal{Q}\setminus\mathcal{B})$. For brevity, we will assume that $\mathcal{Q}=\mathcal{V}_1\times\dots\times\mathcal{V}_K$. As mentioned in the previous section, we leverage our originally presented proposal distribution $q$ by further restricting the domain to $\mathcal{Q}\setminus\mathcal{B}$ in order to be used in this scenario. This will be represented by
\begin{align}
q_\mathcal{B}(\textbf{x}_{1:K}):\!&=\q(\textbf{X}_{1:K}=\textbf{x}_{1:K}\sep \textbf{X}_{1:K}\notin\mathcal{B}) \\
&=\frac{q(\textbf{x}_{1:K})\ind(\textbf{x}_{1:K}\in\mathcal{Q}\setminus\mathcal{B})}{1-\q(\textbf{X}_{1:K}\in\mathcal{B})}.
\end{align}
Note that an associated autoregressive form $q_\mathcal{B}(x_k\sep \textbf{x}_{<k})$ exists and is well defined; however, the exact definition is a bit unwieldy. Please refer to \cref{sec:3_estimation_techniques} for more details.

Using this proposal distribution, we can easily estimate the remaining probability:
\begin{align}
\prob_\theta(\textbf{X}_{1:K}\in\mathcal{Q}\setminus\mathcal{B}) & = \E_{\textbf{x}_{1:K}}^{\q_\mathcal{B}}\left[\frac{p_\theta(\textbf{X}_{1:K})}{q_\mathcal{B}(\textbf{X}_{1:K})}\right] \\
& \approx \frac{1}{M} \sum_{m=1}^M \frac{p_\theta(\textbf{X}_{1:K}^{(m)})}{q_\mathcal{B}(\textbf{X}_{1:K}^{(m)})} \quad\quad\text{ for } \textbf{X}_{1:K}^{(1)},\dots,\textbf{x}_{1:K}^{(M)}\overset{iid}{\sim} q_\mathcal{B}.
\end{align}
Let $\omega_\mathcal{B}(\textbf{x}_{1:K}):=\frac{p_\theta(\textbf{x}_{1:K})}{q_\mathcal{B}(\textbf{x}_{1:K})}$ and for brevity we will refer to $\prob_\theta(\textbf{X}_{1:K}\in\mathcal{Q}\setminus\mathcal{B})$ as  $\bar{\omega}_\mathcal{B}$.

If we assume $\mathcal{B}'=\mathcal{B}\cup\{\hat{\textbf{x}}_{1:K}\}$ for some $\hat{\textbf{x}}_{1:K}\in\mathcal{Q}\setminus\mathcal{B}$, then it is interesting to determine when exactly there will be a reduction in sampling variance for $\prob_\theta(\textbf{X}_{1:K}\in\mathcal{Q}\setminus\mathcal{B}')$ versus $\prob_\theta(\textbf{X}_{1:K}\in\mathcal{Q}\setminus\mathcal{B})$ as this will give insight into when the hybrid method is successful. 
In other words, we would like to show when the following inequality holds true:
\begin{align}
\Delta_{\text{Var}}:=\var_{\textbf{X}_{1:K}}^{\q_{\mathcal{B}}} \left[\omega_{\mathcal{B}}(\textbf{X}_{1:K})\right] - \var_{\textbf{X}_{1:K}}^{\q_{\mathcal{B}'}} \left[\omega_{\mathcal{B}'}(\textbf{X}_{1:K})\right] \geq 0
\end{align}
If this is true for a given $\hat{\textbf{x}}_{1:K}$, then this finding can be applied recursively for more general (but still possibly restricted) $\mathcal{B}'\supset\mathcal{B}$.
\begin{align}
\var_{\textbf{X}_{1:K}}^{\q_{\mathcal{B}}} \left[\omega_{\mathcal{B}}(\textbf{X}_{1:K})\right] & = \E_{\textbf{X}_{1:K}}^{\q_\mathcal{B}} \left[\left(\omega_\mathcal{B}(\textbf{X}_{1:K}) - \bar{\omega}_\mathcal{B}\right)^2\right] \\
& = \E_{\textbf{X}_{1:K}}^{\q_\mathcal{B}} \left[\omega_\mathcal{B}(\textbf{X}_{1:K})^2\right] - \bar{\omega}_\mathcal{B}^2 \\
& = \sum_{\textbf{x}_{1:K}\in\mathcal{Q}\setminus\mathcal{B}} q_{\mathcal{B}}(\textbf{x}_{1:K})\omega_\mathcal{B}(\textbf{x}_{1:K})^2 - \bar{\omega}_\mathcal{B}^2
\end{align}
It then follows that
\begin{align}
\Delta_\text{Var} & =
\bar{\omega}_{\mathcal{B}'}^2 - \bar{\omega}_\mathcal{B}^2 +
\sum_{\textbf{x}_{1:K}\in\mathcal{Q}\setminus\mathcal{B}}  q_{\mathcal{B}}(\textbf{x}_{1:K})\omega_\mathcal{B}(\textbf{x}_{1:K})^2 - \sum_{\textbf{x}_{1:K}\in\mathcal{Q}\setminus\mathcal{B}'} q_{\mathcal{B}'}(\textbf{x}_{1:K})\omega_{\mathcal{B}'}(\textbf{x}_{1:K})^2 \nonumber \\
& = (\bar{\omega}_{\mathcal{B}'}^2 - \bar{\omega}_\mathcal{B}^2) +
q_\mathcal{B}(\hat{\textbf{x}}_{1:K})\omega_{\mathcal{B}}(\hat{\textbf{x}}_{1:K})^2 + \nonumber \\
& \quad\quad \sum_{\textbf{x}_{1:K}\in\mathcal{Q}\setminus\mathcal{B}'}\left[q_{\mathcal{B}}(\textbf{x}_{1:K})\omega_\mathcal{B}(\textbf{x}_{1:K})^2- q_{\mathcal{B}'}(\textbf{x}_{1:K})\omega_{\mathcal{B}'}(\textbf{x}_{1:K})^2\right] \label{eq:3_three_terms}
\end{align}
We will now analyze each of the three terms in \cref{eq:3_three_terms} to determine when $\Delta_\text{Var}\geq0$. For the first term, it follows that:
\begin{align}
\bar{\omega}_{\mathcal{B}'} + \bar{\omega}_\mathcal{B} & = \prob_\theta(\textbf{X}_{1:K}\in\mathcal{Q}\setminus\mathcal{B}') +  \prob_\theta(\textbf{X}_{1:K}\in\mathcal{Q}\setminus\mathcal{B})\nonumber  \\
& = 2\prob_\theta(\textbf{X}_{1:K}\in\mathcal{Q}\setminus\mathcal{B}') + p_\theta(\hat{\textbf{x}}_{1:K}) \nonumber \\
\bar{\omega}_{\mathcal{B}'} - \bar{\omega}_\mathcal{B} & = \prob_\theta(\textbf{X}_{1:K}\in\mathcal{Q}\setminus\mathcal{B}') - \prob_\theta(\textbf{X}_{1:K}\in\mathcal{Q}\setminus\mathcal{B}) \nonumber \\
& = -p_\theta(\hat{\textbf{x}}_{1:K}) \nonumber \\
\implies \bar{\omega}_{\mathcal{B}'}^2 - \bar{\omega}_\mathcal{B}^2 & = -p_\theta(\hat{\textbf{x}}_{1:K})\left(2\prob_\theta(\textbf{X}_{1:K}\in\mathcal{Q}\setminus\mathcal{B}') - p_\theta(\hat{\textbf{x}}_{1:K})\right) \label{eq:3_neg_terms}
\end{align}
The other two terms in \cref{eq:3_three_terms} must sum to a positive value with as large or larger magnitude to \cref{eq:3_neg_terms} for $\Delta_\text{Var}\geq 0$. Looking at the second term, we see that
\begin{align}
q_\mathcal{B}(\hat{\textbf{x}}_{1:K})\omega_{\mathcal{B}}(\hat{\textbf{x}}_{1:K})^2 & = \frac{p_\theta(\hat{\textbf{x}}_{1:K})^2}{q_\mathcal{B}(\hat{\textbf{x}}_{1:K})} \\
& = p_\theta(\hat{\textbf{x}}_{1:K})(1-\q(\textbf{X}_{1:K}\in\mathcal{B}))\prod_{k=1}^K \sum_{v_k \in \mathcal{V}_k} p_\theta(v_k\sep \hat{\textbf{x}}_{<k}) \geq 0.
\end{align}
This inequality becomes strict should $p_\theta(\hat{\textbf{x}}_{1:K})>0$. We will now look at the final summation in \cref{eq:3_three_terms}:
\begin{align}
& \sum_{\textbf{x}_{1:K}\in\mathcal{Q}\setminus\mathcal{B}'}\left[q_{\mathcal{B}}(\textbf{x}_{1:K})\omega_\mathcal{B}(\textbf{x}_{1:K})^2- q_{\mathcal{B}'}(\textbf{x}_{1:K})\omega_{\mathcal{B}'}(\textbf{x}_{1:K})^2\right] \\
& = \sum_{\textbf{x}_{1:K}\in\mathcal{Q}\setminus\mathcal{B}'}\left(\frac{p_\theta(\textbf{x}_{1:K})^2}{q_{\mathcal{B}}(\textbf{x}_{1:K})} - \frac{p_\theta(\textbf{x}_{1:K})^2}{q_{\mathcal{B}'}(\textbf{x}_{1:K})}\right) \\
& = \sum_{\textbf{x}_{1:K}\in\mathcal{Q}\setminus\mathcal{B}'}p_\theta(\textbf{x}_{1:K})\left(\q(\textbf{X}_{1:K}\in\mathcal{B}') - \q(\textbf{X}_{1:K}\in\mathcal{B})\right)\prod_{k=1}^K \sum_{v_k \in \mathcal{V}_k} p_\theta(v_k \sep \textbf{x}_{<k}) \\
& = q(\hat{\textbf{x}}_{1:K})\sum_{\textbf{x}_{1:K}\in\mathcal{Q}\setminus\mathcal{B}'}p_\theta(\textbf{x}_{1:K})\prod_{k=1}^K \sum_{v_k \in \mathcal{V}_k} p_\theta(v_k \sep \textbf{x}_{<k})\\
& = \frac{p_\theta(\hat{\textbf{x}}_{1:K})}{\prod_{k=1}^K \sum_{v_k \in \mathcal{V}_k} p_\theta(v_k \sep \hat{\textbf{x}}_{<k})}\sum_{\textbf{x}_{1:K}\in\mathcal{Q}\setminus\mathcal{B}'}p_\theta(\textbf{x}_{1:K})\prod_{k=1}^K \sum_{v_k \in \mathcal{V}_k} p_\theta(v_k \sep \textbf{x}_{<k}).
\end{align}
Since all terms in \cref{eq:3_three_terms} have a common factor $p_\theta(\hat{\textbf{x}}_{1:K})$, we can see that $\Delta_\text{Var}\geq 0$ iff the following holds:
\begin{align}
& 2\prob_\theta(\textbf{X}_{1:K}\in\mathcal{Q}\setminus\mathcal{B}') -\frac{1}{\rho(\hat{\textbf{x}}_{1:K})}\sum_{\textbf{x}_{1:K}\in\mathcal{Q}\setminus\mathcal{B}'}p_\theta(\textbf{x}_{1:K})\rho(\textbf{x}_{1:K}) \nonumber \\
& \quad\quad \leq (1-\q(\textbf{X}_{1:K}\in\mathcal{B}))\rho(\hat{\textbf{x}}_{1:K}) + p_\theta(\hat{\textbf{x}}_{1:K}) \label{eq:3_final_ineq}
\end{align}
for $\rho(\textbf{x}_{1:K})=\prod_{k=1}^K \sum_{v_k \in \mathcal{V}_k} p_\theta(v_k \sep \textbf{x}_{<k})$. Should this hold true, then by taking $\mathcal{B}'$ instead of $\mathcal{B}$ during the beam search segment of the hybrid approach would the variance of the sampling subroutine reduce. Generalizing this further, it is not guaranteed that the hybrid estimate will have a lower variance than regular importance sampling; however, our experimental results across a variety of settings (see \cref{fig:3_err_plot}) seem to indicate that the variance is reduced on average.

All of the terms to the left of the inequality in \cref{eq:3_final_ineq} are quantities that would require either expansive computations or estimation in order to know their values. Conversely, all the values to the right of the inequality are readily available as a byproduct of beam search. Incorporating this into decision making for our hybrid method is left for future work.

\section{Determining Ground Truth for Experiments}
\label{sub:ground_truth_calc}\label{sec:3_dataset_model_prep}

Exact computation of ground truth is intractable for queries with large path spaces. We circumvent this issue via the law of large numbers by computing \textbf{surrogate ground truth} query estimates with a large computational budget, leveraging the variance of the query's samples as a convergence criterion. Specifically, our algorithm is conducted as follows. We first specify a \textit{minimum} number of samples $S_{\text{low}}=10000$ to be drawn for the surrogate ground-truth estimate. Once $S_{\text{low}}$ samples have been drawn, we compute the variance of our estimate
$\hat{p}^{*}_\theta(X_{1:K}\in\mathcal{Q}) := \frac{1}{S} \sum_{i=1}^S \frac{p^{*}_\theta(X_{1:K} = x^{(i)}_{1:K})}{q(X_{1:K}=x_{1:K}^{(i)})}$ for $x^{(1)}_{1:K},\dots,x^{(S)}_{1:K} \overset{iid}{\sim} q(X_{1:K})$:
\begin{align}
\widehat{\var}_{q}\left[\hat{p}_\theta^*(X_{1:K}\in\mathcal{Q})\right] & = \frac{1}{S}\sum_{i=1}^S \left(\frac{p^{*}_\theta(X_{1:K} = x^{(i)}_{1:K})}{q(X_{1:K}=x_{1:K}^{(i)})} - \hat{p}^*_\theta(X_{1:K}\in\mathcal{Q})\right)^2
\end{align}

We then evaluate $\widehat{\var}_{q}\left[\hat{p}_\theta^*(X_{1:K}\in\mathcal{Q})\right]$ every 1000 additional samples until either it drops below tolerance $\delta=1e^{-7}$ or $S$ meets our maximum sample budget $S_{\text{high}}=100000$. This procedure is done in all of our experiments in which a method's performance is being compared to a query's ground truth value and exact ground truth cannot be computed due to resource constraints (typically when $K>4$).

\section{Additional Experimental Details and Results}
\label{sec:3_additional_exps}
This section  discusses additional experimental details and results that were not included in section 5 of our main paper due to space constraints. In all experiments, all means and medians reported are with respect to  
$N_Q = 1000$ randomly selected sequence locations/histories/queries per datapoint in each plot, unless stated otherwise. For each randomly selected current location, the event $a$ used in the query for $k$ steps ahead corresponds to the actual observed event $a$ for $k$ steps ahead.

\begin{figure}
    \centering
    \includegraphics[width=1.0\textwidth]{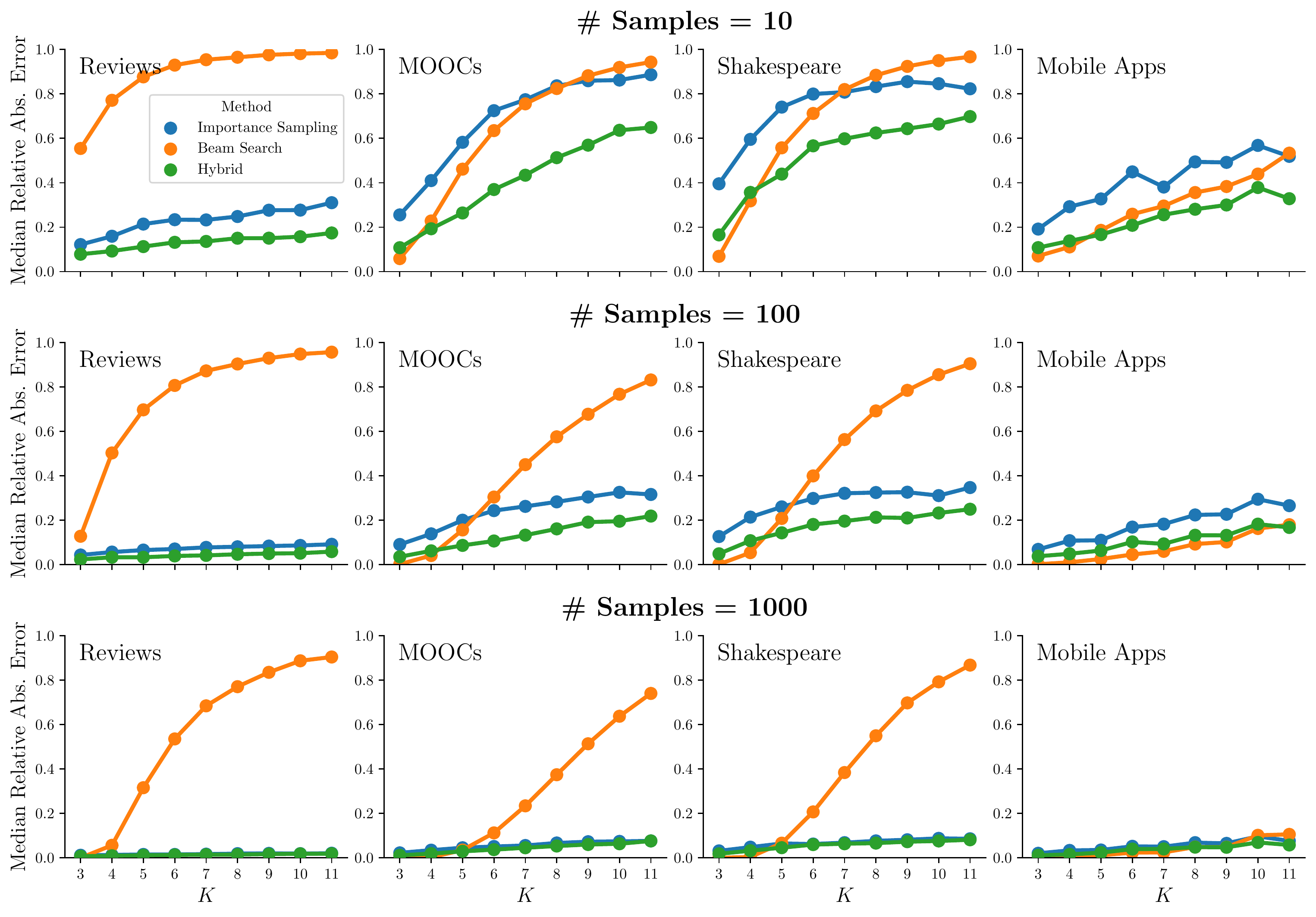}
    \caption{\textbf{Median estimation error versus query horizon $K$}: median relative absolute error (RAE) between estimated probability and (surrogate) ground truth for $\prob_\theta(\hit(\cdot)=K)$, for importance sampling, beam search, and the hybrid method, with varying computation budgets determined by 10, 100, and 1000 samples for the hybrid method. Ground truth values used to determine error in these plots are exact for $K \leq 4$ and approximated otherwise.}. 
    \label{fig:3_med_err_all_samples}
\end{figure}

\begin{figure}
    \centering
    \includegraphics[width=1.0\textwidth]{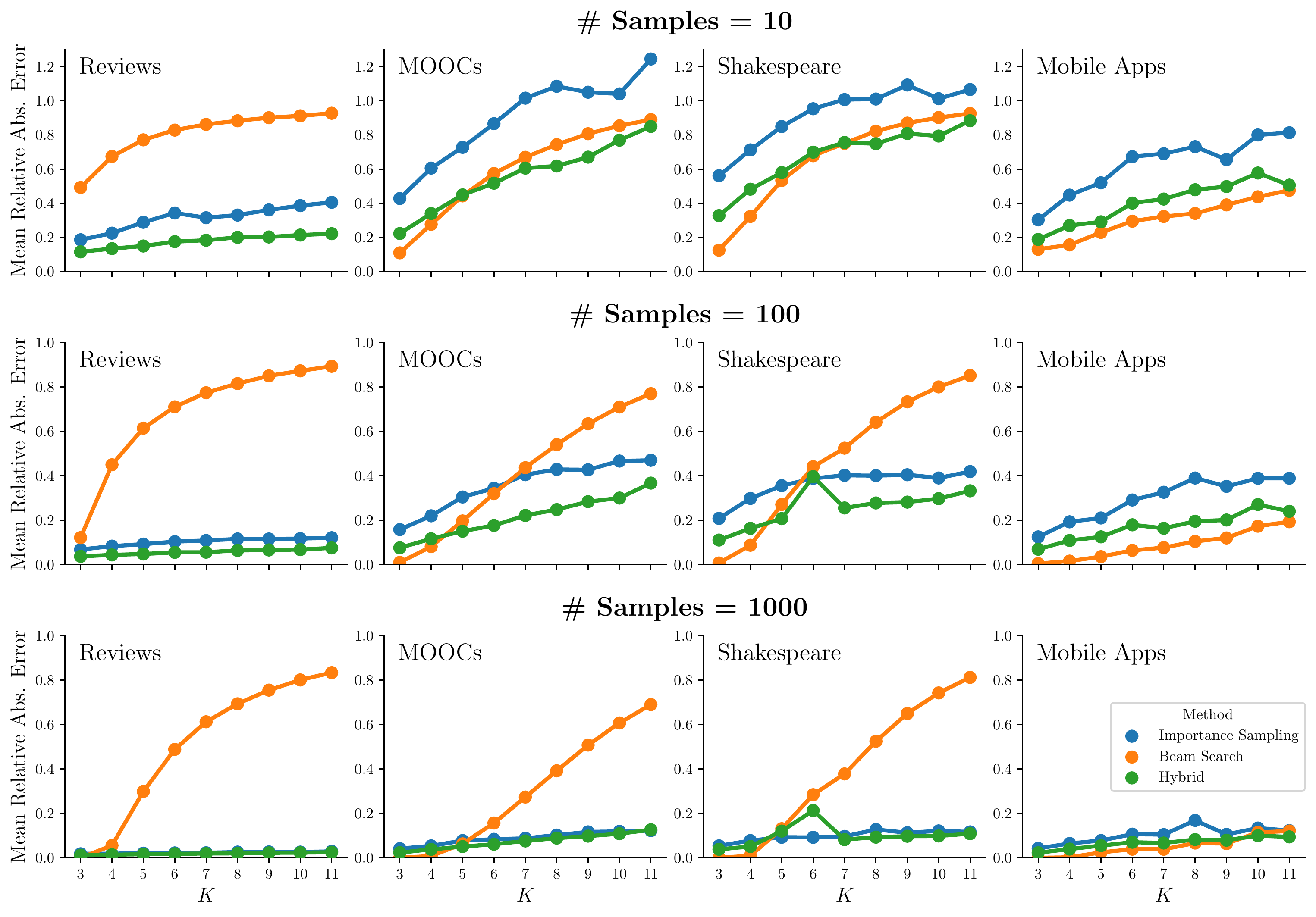}
    \caption{\textbf{Mean estimation error versus query horizon $K$}:   same format as   \cref{fig:3_med_err_all_samples}. Ground truth values used to determine error in these plots are exact for $K \leq 4$ and approximated otherwise.}
    \label{fig:3_mean_err_all_samples}
\end{figure}

\subsection{Query Estimation Error as a Function of Horizon $K$}
 Fig. 3 in the main paper shows the median relative absolute error (RAE) (across $N_Q=1000$ queries) as a function of query horizon $K$ for 4 datasets, comparing beam search, importance sampling, and the hybrid method, with a computation budget fixed   at $S=100$ hybrid samples. Here we provide a number of extensions of these results.

\cref{fig:3_med_err_all_samples} shows the median   RAE, for three levels of computation budget: $S=10, 100, 1000$ and \cref{fig:3_mean_err_all_samples} shows the same results but now reporting mean RAE on the y-axis. While the details differ across different settings, the qualitative conclusions in these Figures agree with those for Fig. 3 in the main paper, namely that beam search is more sensitive (in its error) to both the horizon query $K$ and to individual datasets, compared to both importance sampling and the hybrid method. More granular perspectives of this information for one of the datasets (MOOCs) can be seen in \Crefrange{fig:3_err_scatter_10}{fig:3_err_scatter_1000} in the form of scatter plots of each of the individual query estimates against (surrogate) ground truth for that query. Different budgets are shown in different figures, and the results for $K=3, 5, 7, 9, 11$ are shown in each column.

\begin{figure}
    \centering
    \includegraphics[width=1.0\textwidth]{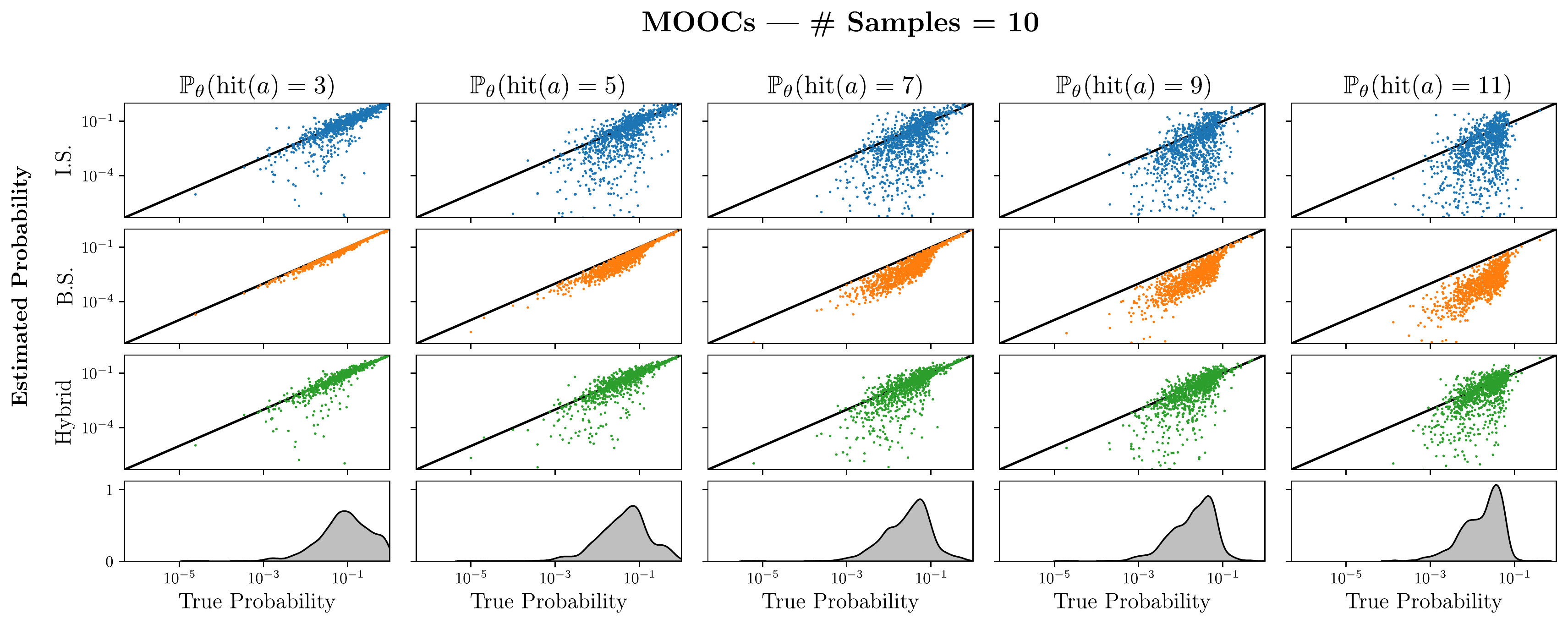}
    \caption{\textbf{Scatterplots of individual query estimates versus (surrogate) ground truth, computation budget of 10 hybrid samples:}  Comparison of importance sampling (I.S.), beam search (B.S.), and the hybrid method for the MOOCs dataset with the budget determined by the hybrid method using 10 samples. The x-axis corresponds to the surrogate ground truth values for a given query result. Density plots at the bottom are for the surrogate ground truth values. Ground truth values used to determine error in these plots are exact for $K \leq 4$ and approximated otherwise.}
    \label{fig:3_err_scatter_10}
\end{figure}

\begin{figure}
    \centering
    \includegraphics[width=1.0\textwidth]{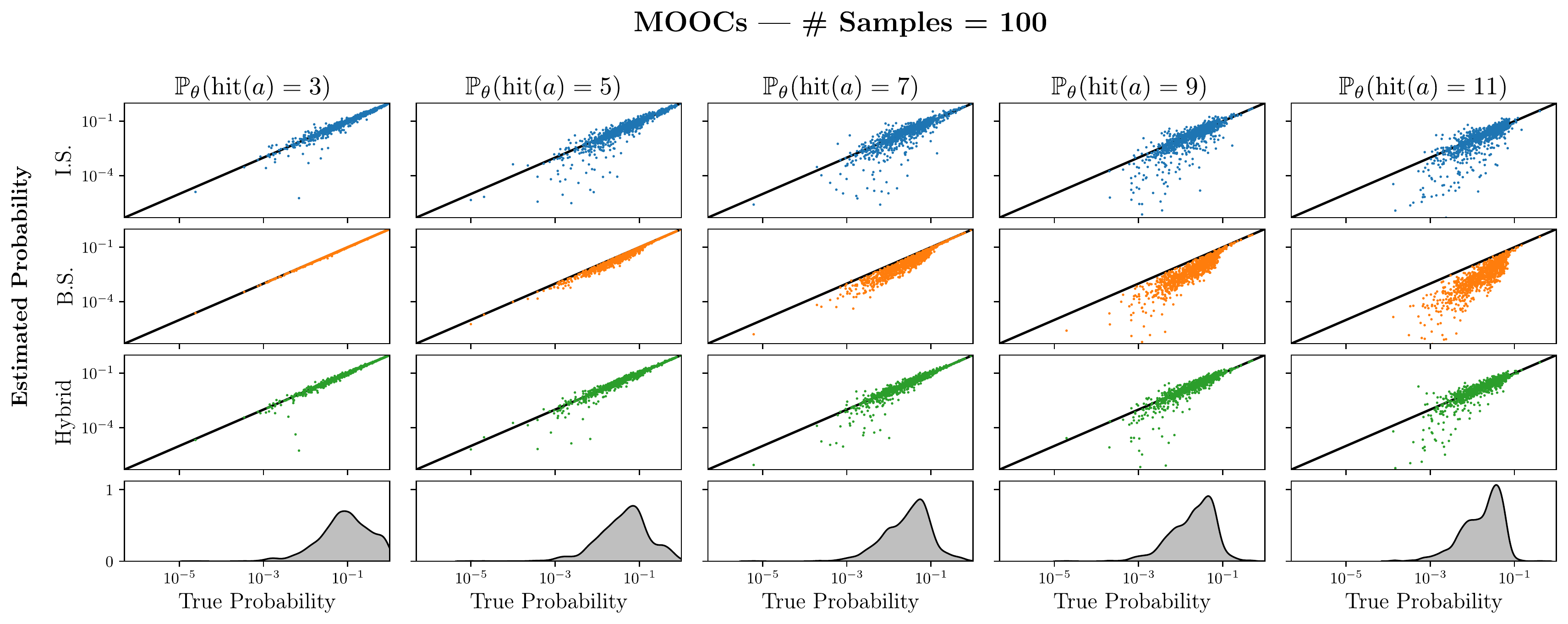}
    \caption{\textbf{Scatterplots of individual query estimates versus (surrogate) ground truth, computation budget of 100 hybrid samples:} Same format as \cref{fig:3_err_scatter_10}.}
    \label{fig:3_err_scatter_100}
\end{figure}
\begin{figure}
    \centering
    \includegraphics[width=1.0\textwidth]{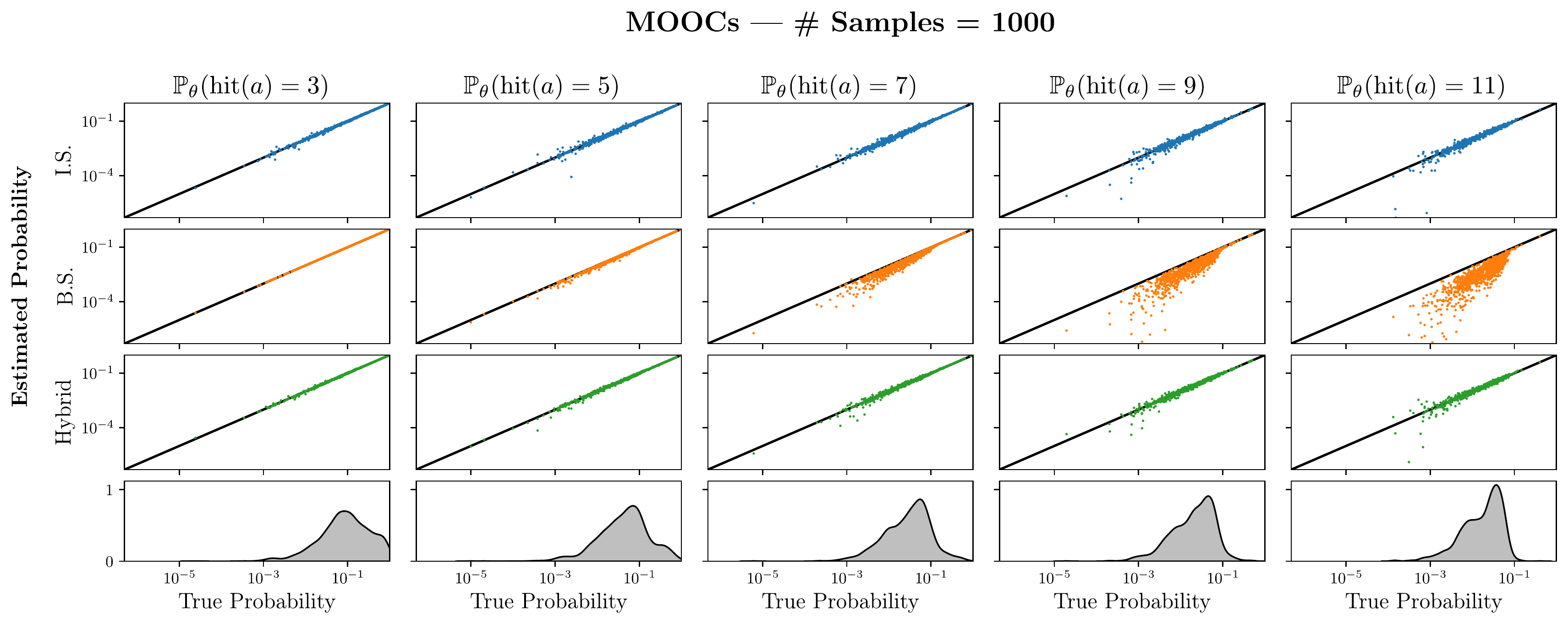}
    \caption{\textbf{Scatterplots of individual query estimates versus (surrogate) ground truth, computation budget of 1000 hybrid samples:} Same format as \cref{fig:3_err_scatter_10}.}
    \label{fig:3_err_scatter_1000}
\end{figure}

\subsection{Query Estimation Error as a Function of Computation Budget}

In addition to identifying optimal query estimation methodologies for a low and fixed computation budget, we also explore the impact of increasing computation budget for query lengths $K=3,7,11$, roughly corresponding to short, medium, and long horizon queries. These experiments are conducted in the same manner as the query estimation experiments with a fixed model budget, but are then repeated for many different budgets derived from $S=10,30,50,100,300,500,1000,3000,5000,10000$ hybrid samples. The intention with these experiments is to observe if any query estimation methods disproportionately benefit from increased computation and exhibit behavior that was not present at lower computation budgets. We also include two additional baselines. The full set of methods explored is listed below.
\begin{enumerate}
  \item Importance sampling (informative proposal distribution $q$ derived from model $p_\theta$)
  \item Beam search 
  \item Hybrid search and sampling 
  \item Monte-Carlo sampling with a uniform proposal distribution
  \item Naive model sampling (direct MC sampling for $\E^{\prob_\theta} \ind(X_{1:K}\in\mathcal{Q})$)
\end{enumerate}

As a clarifying point, naive sampling is conducted by sampling sequences from the model and determining if they fall within the query set $\mathcal{Q}$; the proportion of samples that exist in $\mathcal{Q}$ serves as a naive means of determining the query estimate in question. In addition, Monte-Carlo sampling with a uniform proposal distribution samples sequences in $\mathcal{Q}$ uniformly and the estimates the query probability for that sample. These two methods were not included in the main paper due to their consistently poor performance.

\begin{figure}
    \centering
    \includegraphics[width=1.0\textwidth]{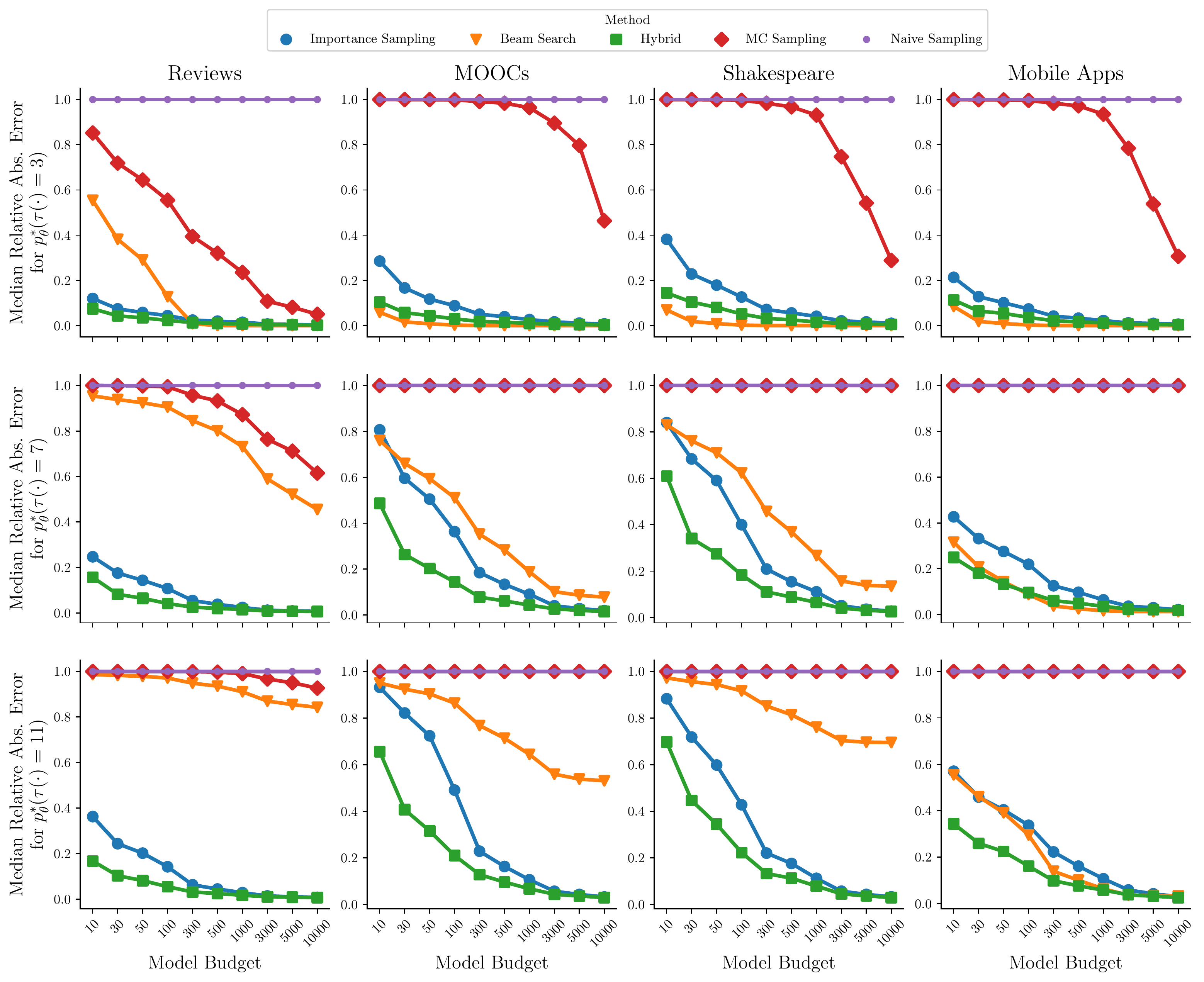}
    \caption{\textbf{Median estimation error versus model budgets:}  Median relative absolute error (RAE) for importance sampling, beam search, the hybrid method, MC sampling, and naive sampling across three different queries, $\prob_\theta(\hit(\cdot)=k)$, for $k=3,7,11$, over all four main datasets, as a function of different model budgets. For cases where the MC sampling results are not visible (e.g., see rightmost plot on the bottom row), the results coincide with the naive sampling results. Ground truth values used to determine error in these plots are exact for $K \leq 4$ and approximated otherwise.}
    \label{fig:3_budget}
\end{figure}

\begin{figure}
    \centering
    \includegraphics[width=1.0\textwidth]{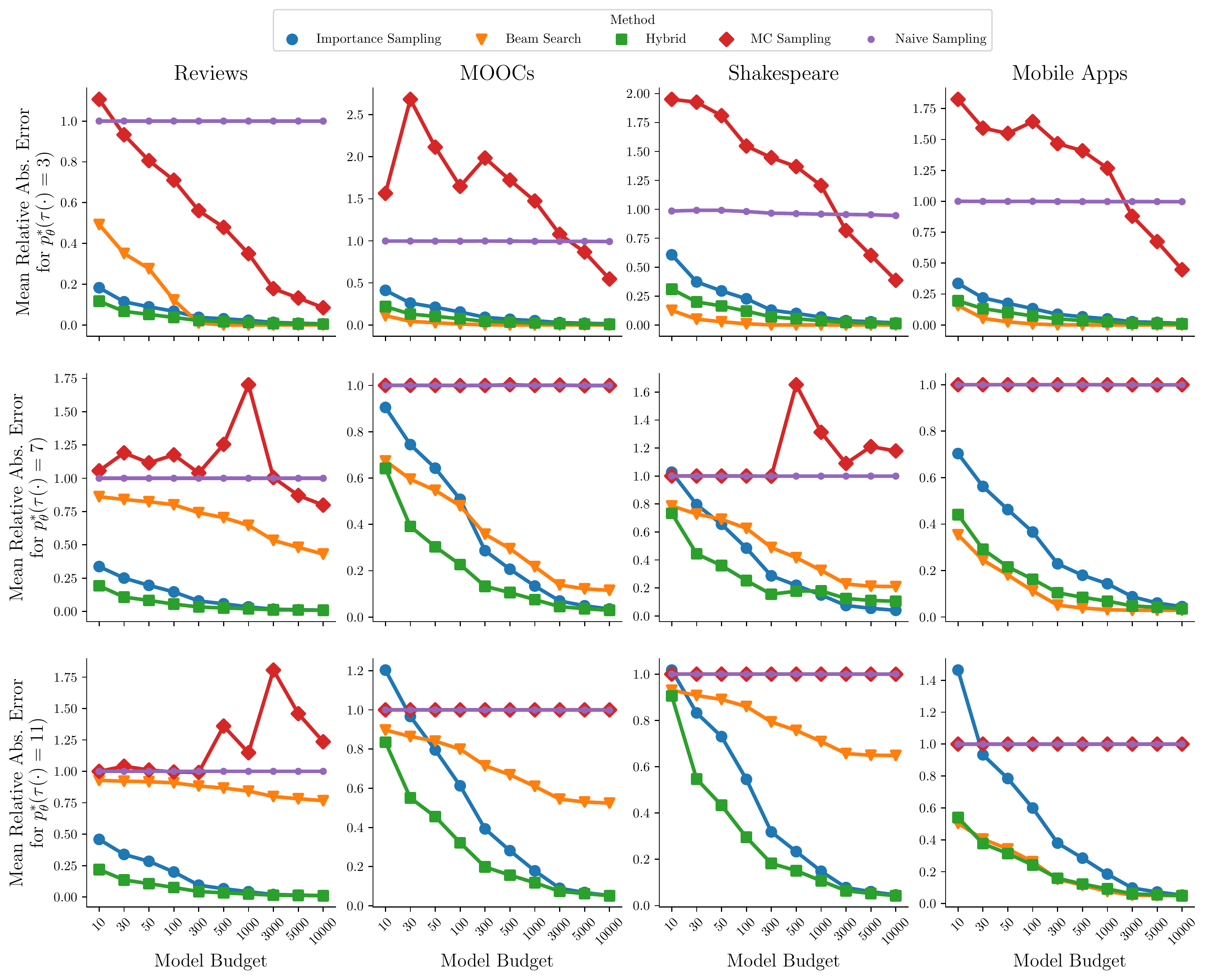}
    \caption{\textbf{Mean estimation error versus model budgets:} Same format as \cref{fig:3_budget}. Ground truth values used to determine error in these plots are exact for $K \leq 4$ and approximated otherwise.}
    \label{fig:3_med_budget}
\end{figure}

\cref{fig:3_budget} (median error) and \cref{fig:3_med_budget} (mean error), show that increasing the computation budget by an order of magnitude roughly corresponds to a three times reduction in RAE for both importance sampling and the hybrid method.  Naive sampling sees almost no benefit from the increased budget regardless of the size of the query path space. Monte Carlo sampling sees some reduction in error from  increased computation budget for some configurations, but also often sees no benefit. 

For the provided budgets, the query estimates resulting from naive model sampling are consistently 0, resulting in an RAE of 1. While naive model sampling can be useful in some contexts in general,  these results indicate that it is not well-suited for estimating query probabilities. This is likely because many of the ground truth probabilities for these queries have values on the order of $10^{-1}$ or smaller. 
For queries that are highly unlikely under the model, the probability of even a single sampled sequence belonging to $\mathcal{Q}$ is very low. 

Monte-Carlo sampling also includes high error estimates, but for a different reason. Since the Monte-Carlo estimate can be decomposed into an expectation over $\prob_\theta(X_{1:K}=x_{1:K}, x_{1:K} \in \mathcal{Q})$ that is then re-scaled by $|\mathcal{Q}| \gg 0$, the scaling term magnifies any error in the expectation dramatically, inducing extremely high variance and the potential to produce query estimates that exceed 1. This high variance can persist even for high computation budgets. By contrast, beam search improves with an increased computational budget, but only as a function of the total path space. The larger the path space and the higher the entropy of the distribution, the worse the beam search estimates are as measured by RAE.

\begin{figure}
    \centering
    \includegraphics[width=\textwidth]{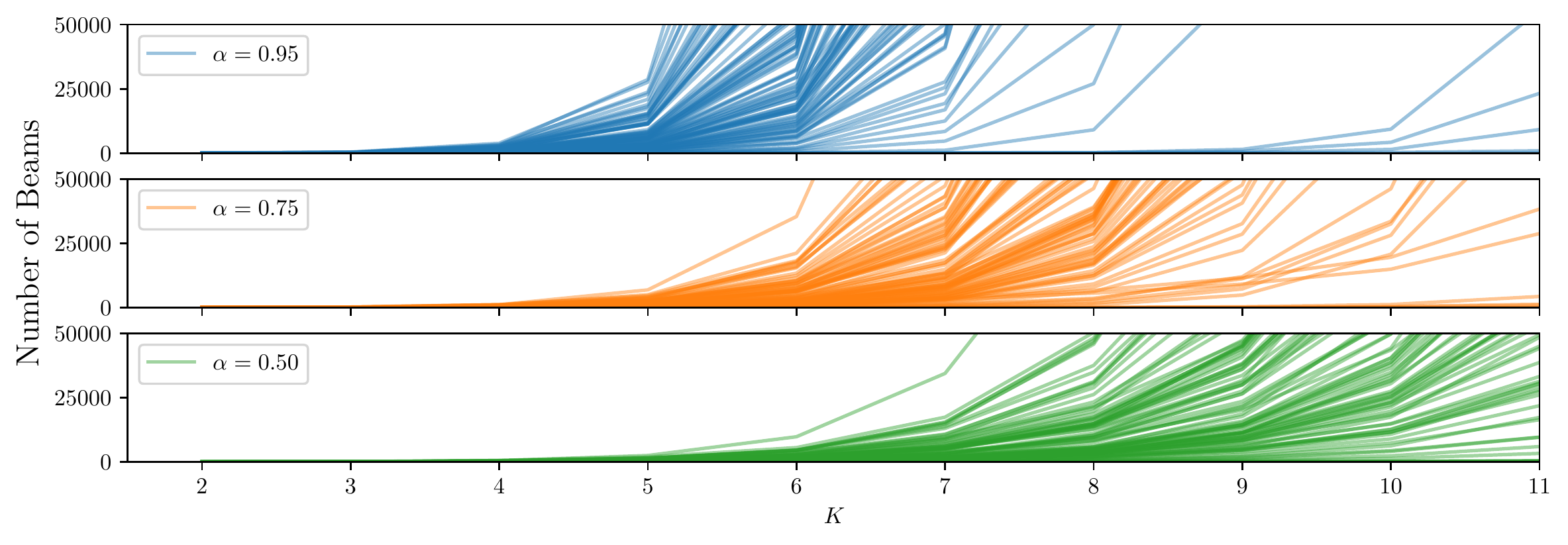}
    \caption{Number of beams as a function of $K$ steps into decoding sequences for coverage-based beam search estimating $\prob_\theta(X_K\in A)$ with $\alpha\in\{0.95, 0.75, 0.5\}$ over a variety of different starting histories and subsets $A \subset \mathcal{X}$ in the Shakespeare dataset.}
    \label{fig:3_coverage}
\end{figure}

\subsection{Coverage-based Beam Search Ablation}%
\label{sub:coverage_bs_ablation}
As described in the main paper, a variant of beam search well-suited to query estimation is \textit{coverage-based beam search}. However, this method was not explored further in our analysis as it could not scale to non-trivial query types and large query path spaces. Depicted in \cref{fig:3_coverage}, the minimum number of beams needed to cover 50\%, 75\%, and 90\% of the query path space increases exponentially with the query length $K$. Though coverage-based beam search comes with desirable coverage guarantees and is a consistent (albeit biased) query estimator, naively applying coverage-based beam search to queries of practical interest is ill-posed and computationally intractable. However, the hybrid method preserves some of beam search's desirable traits while making it suitable for queries of practical interest. See \cref{sec:3_hybrid_details} for more information.

\subsection{Long-horizon Query Estimation}

The empirical results of our experiments tell us much of the general performance of the proposed methods, but we still do not have clarity on the extent to which importance sampling can yield low-error estimates in the limit of large sequence lengths. To that end, we conducted an experiment to directly assess the performance of importance sampling when estimating long horizons. Our primary conclusion was that (i) with reasonable sample sizes the error remains below 25\%, even at $K=100$, and (ii) the error does not increase substantially as $K$ increases beyond $K=5$.

In more detail, a set of 100 history sequences $\mathcal{H}$ of length 5 was collected from each dataset. We then compute ground truth or pseudo ground truth (PGT) for query 2 (probability of event at $K$ without restriction) over sequence lengths $K =[2, 5, 10, 20, 40, \dots, 100]$. Pseudo-ground truth estimates are generated with a tolerance $\delta=1e^{-7}$ and a maximum sampling budget of $250,000$. The same estimates are then computed using importance sampling with sampling budgets $[10, \dots, 10000]$. Similar to the experiments in Section 5, we aggregate and present the results as the median relative absolute error, shown in the tables below. Since we are not focusing on a specific event type  of interest, this median also marginalizes over all event types  in the vocabulary as well as the sampled query sequences. All results are reported as percentages

\begin{table}[!ht]
    \centering
    \begin{tabular}{l|cccccccc}
    \toprule
        \# Samples & 2 & 5 & 10 & 20 & 40 & 60 & 80 & 100 \\
        \midrule
        10 & 31.35 & 40.33 & 44.39 & 43.48 & 45.98 & 45.84 & 49.44 & 50.05 \\ 
        100 & 13.29 & 17.12 & 18.73 & 19.46 & 20.41 & 21.65 & 21.82 & 21.86 \\ 
        1000 & 4.94 & 6.38 & 7.06 & 7.33 & 7.26 & 6.99 & 7.13 & 6.82 \\ 
        10000 & 1.60 & 2.05 & 2.21 & 2.34 & 2.20 & 2.13 & 2.06 & 2.10 \\ 
    \bottomrule
    \end{tabular}
    \caption{Median RAE across query estimations methods (1000 samples) and query horizons $K=2,5,10,20, ..., 100$ and sample budgets $10, 100, 1000, 10000$ for Amazon Reviews.}
    \label{table:3_long_horizon_amazon_reviews}
\end{table}

\begin{table}[!ht]
    \centering
    \begin{tabular}{l|cccccccc}
    \toprule
        \# Samples & 2 & 5 & 10 & 20 & 40 & 60 & 80 & 100 \\ 
        \midrule
        10 & 66.21 & 84.93 & 92.92 & 96.96 & 98.41 & 99.00 & 99.19 & 99.32 \\ 
        100 & 62.36 & 80.94 & 89.79 & 94.00 & 96.16 & 97.05 & 97.38 & 97.66 \\ 
        1000 & 47.96 & 59.14 & 69.10 & 75.45 & 74.79 & 71.34 & 65.12 & 59.75 \\ 
        10000 & 15.21 & 20.51 & 23.38 & 24.00 & 20.66 & 18.78 & 16.48 & 15.55 \\ 
        \bottomrule
    \end{tabular}
    \caption{Median RAE across query estimations methods (1000 samples) and query horizons $K=2,5,10,20, ..., 100$ and sample budgets $10, 100, 1000, 10000$ for Mobile Apps.}
    \label{table:3_long_horizon_apps}
\end{table}

\begin{table}[!ht]
    \centering
    \begin{tabular}{l|cccccccc}
    \toprule
        \# Samples & 2 & 5 & 10 & 20 & 40 & 60 & 80 & 100 \\  
        \midrule
        10 & 79.84 & 79.99 & 83.53 & 85.29 & 83.53 & 84.00 & 85.82 & 85.15 \\ 
        100 & 28.66 & 32.85 & 39.25 & 40.66 & 38.78 & 41.51 & 39.71 & 40.07 \\ 
        1000 & 8.61 & 10.99 & 13.29 & 13.52 & 13.78 & 13.89 & 13.64 & 14.10 \\ 
        10000 & 2.77 & 3.48 & 4.28 & 4.34 & 4.32 & 4.33 & 4.34 & 4.30 \\ 
    \bottomrule
    \end{tabular}
    \caption{Median RAE across query estimations methods (1000 samples) and query horizons $K=2,5,10,20, ..., 100$ and sample budgets $10, 100, 1000, 10000$ for Shakespeare.}
    \label{table:3_long_horizon_shakespeare}
\end{table}

\begin{table}[!ht]
    \centering
    \begin{tabular}{l|cccccccc}
    \toprule
        \# Samples & 2 & 5 & 10 & 20 & 40 & 60 & 80 & 100 \\ 
        \midrule
        10 & 89.62 & 97.30 & 98.73 & 99.19 & 99.08 & 99.06 & 98.91 & 98.87 \\ 
        100 & 63.80 & 83.34 & 84.77 & 79.15 & 66.89 & 62.99 & 60.28 & 61.30 \\ 
        1000 & 30.39 & 43.37 & 36.32 & 26.07 & 19.78 & 18.11 & 17.53 & 17.38 \\ 
        10000 & 11.92 & 15.15 & 11.35 & 7.86 & 5.90 & 5.66 & 5.43 & 5.40 \\
        \bottomrule
    \end{tabular}
    \caption{Median RAE across query estimations methods (1000 samples) and query horizons $K=2,5,10,20, ..., 100$ and sample budgets $10, 100, 1000, 10000$ for MOOCs.}
    \label{table:3_long_horizon_moocs}
\end{table}

In general, we see (not surprisingly) that the increase in sequence length leads to a consistent and non-trivial increase in error for most sampling budgets. In addition, as expected the increase in sampling budget consistently reduces the query estimation error. However, we do witness the interesting phenomenon that the error occasionally \textit{decreases} as the sequence length increases. We conjecture that this may be happening because as the sequence length increases, the relevance of the history context $\mathcal{H}$ decreases and the distribution may regress to a base stationary distribution (as if no history context were provided at all) indicating that the conditional model entropy may be the main driving factor in estimation complexity. This intuition is further supported by the fact that in many datasets, budgets exist where query estimation error first increases but then begins to decrease again. Regardless, as the largest budget of $10,000$ we witness that median RAE remains at or under 25\% in all cases, often significantly so.

\subsection{Query Estimation with Large-Scale Language Models}%
\label{sec:3_gpt_exp}

In order to explore the feasibility of applying our query estimation methods to real-world sequence data, we also analyze a subset of our query estimation methods against GPT-2 and WikiText language data. GPT-2 decomposes English words into $V=50257$ work pieces, a vocabulary over 500 times larger than our other datasets. For this reason, we only conduct experiments using only 100 sequence histories per dataset due to computational limitations. For the same reason, we do not explore the hybrid method and restrict ourselves to analyzing the following query estimators:
\begin{enumerate}
  \item Importance sampling (informative proposal distribution $q$ derived from model $p_\theta$)
  \item Beam search 
  \item Monte-Carlo sampling with a uniform proposal distribution
\end{enumerate}

\begin{table}
\centering
\begin{tabular}{cccccc}
\toprule
K &  Importance Samp. &  Beam Search &  MC Samp. &  Entropy &   Entropy \% \\
\midrule
         3 &                     \textbf{11.41} &              51.25 &                  99.36 &    12.89 &        41.28 \\
        4 &                     \textbf{13.35} &              82.42 &                  99.95 &    19.09 &        40.74 \\
        5 &                      \textbf{13.53} &              93.59 &                  99.99 &    25.23 &        40.38 \\
\bottomrule
\end{tabular}
\caption{\textbf{Median RAE across query estimations methods (1000 samples) and query horizons $K=3,4,5$ for GPT-2 and Wikitext}. Entropy values estimated as the mean (over 100 queries) of the restricted proposal $q$ and entropy \% is the entropy in percentage relative to its potential maximum value $Klog(V)$.}
\label{table:3_wikitext_and_entropy}
\end{table}

Fixed-budget query experiments with GPT-2 are conducted identically to those on the other 4 datasets. We find query estimation error closely mirrors the results we see in datasets with smaller vocabulary sizes, suggesting our findings may generalize well to practical domains. Our analysis is reported in \cref{table:3_wikitext_and_entropy} and includes estimates of the restricted model entropy $H(q)$ for different query lengths $K$, with the entropy increasing much faster than small-vocabulary models, as expected. With that said, there is still much exploration to be done on large-scale sequence models and is a promising avenue for future work.

\begin{figure}
    \centering
    \includegraphics[width=\textwidth]{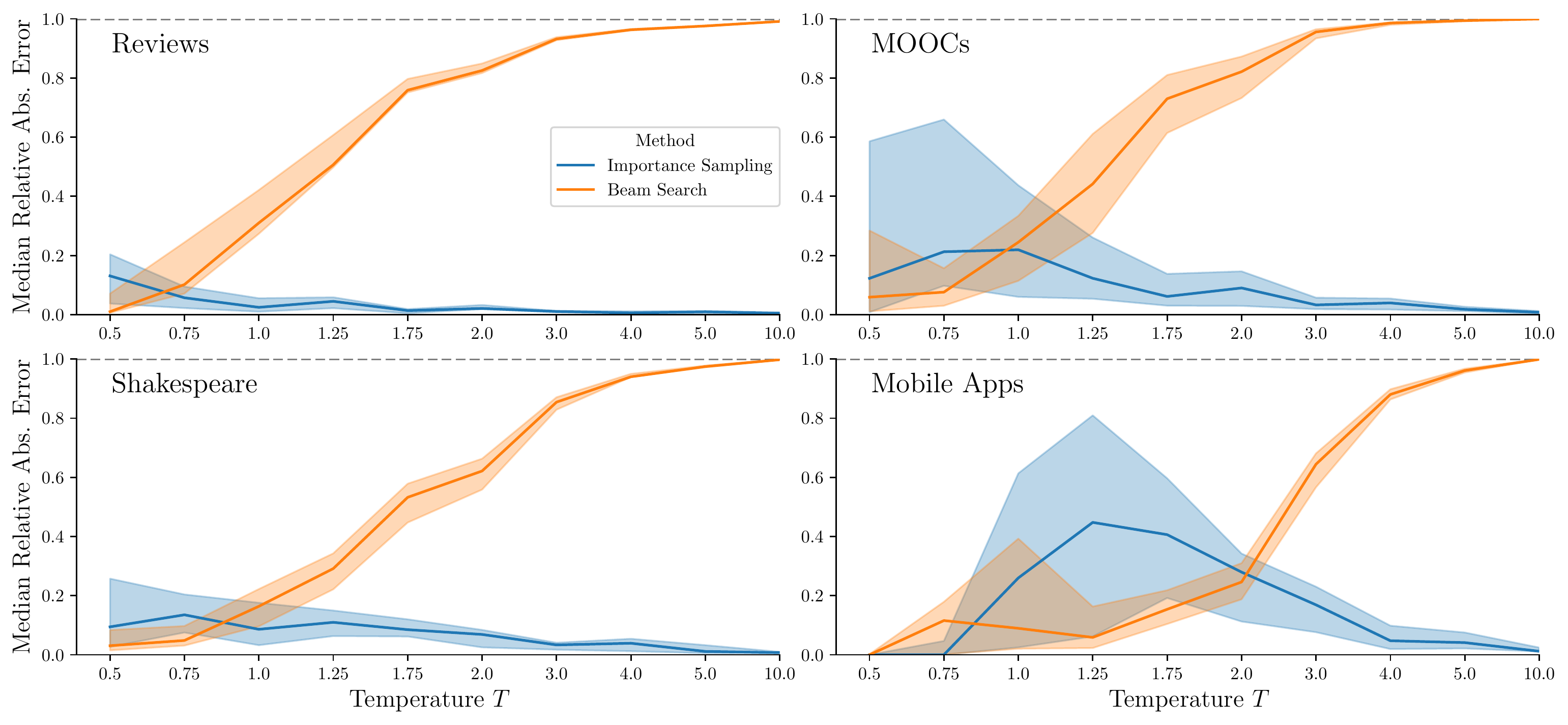}
    \caption{\textbf{Error as a function of entropy:} Median relative absolute error for estimating $\prob_{\theta,T}(\hit(\cdot)=4)$ over all four datasets for a variety of imposed temperatures $T$. Shaded regions indicate interquartile ranges.}
    \label{fig:3_temp_all_datasets}
\end{figure}

\subsection{Entropy Relationship with Query Estimation Error}

Fig. 4(b) in the main paper demonstrated how indirectly controlling the entropy of a given model through an applied temperature affected the performance of beam search and importance sampling. In \cref{fig:3_temp_all_datasets} we can see similar plots for all of four of our main datasets.

\subsection{Investigation of Query 4 (``A" before ``B")}

Most of our analysis was conducted on hitting time queries, and justifiably so as more advanced queries decompose into individual operations on hitting times. This includes Q4, colloquially stated as the probability an item from token set $A$ occurs before an item in token set $B$. More formally:
\begin{align}
\prob_\theta(\hit(A) < \hit(B) ) & = \sum_{k=1}^\infty \prob_\theta(\hit(A)=k, \hit(B)>k) \\
& = \sum_{k=1}^\infty \sum_{a\in A}\prob_\theta(X_k=a, X_{<k}\in (\V\setminus(A\cup B))^{k-1}) 
\label{eq:3_hit_comp_sum_supp}
\end{align}
While this cannot be computed exactly, a lower bound can. The other option is to produce a lower bound on this expression by evaluating the sum in \cref{eq:3_hit_comp_sum_supp} for the first $K$ terms. We can achieve error bounds on this estimate by noting that $\prob_\theta(\hit(A) < \hit(B) ) + \prob_\theta(\hit(A) > \hit(B) ) = 1$. As such, if we evaluate \cref{eq:3_hit_comp_sum_supp} up to $K$ terms for both $\prob_\theta(\hit(A) < \hit(B) )$ and $\prob_\theta(\hit(A) > \hit(B) )$, the difference between the sums will be the maximum error either lower bound can have. This difference will be referred to as \emph{unaccounted probability} and will approach 0 as $K\rightarrow\infty$.

A natural question to ask is what is the minimum value of $K$ sufficient to compute these lower bounds to in order to have negligible unaccounted probability. Though this will surely vary based on the entropy of the model and the specific query in question, we explore this question across all datasets except WikiText and note some general trends.

\begin{figure}
    \centering
    \includegraphics[width=\textwidth]{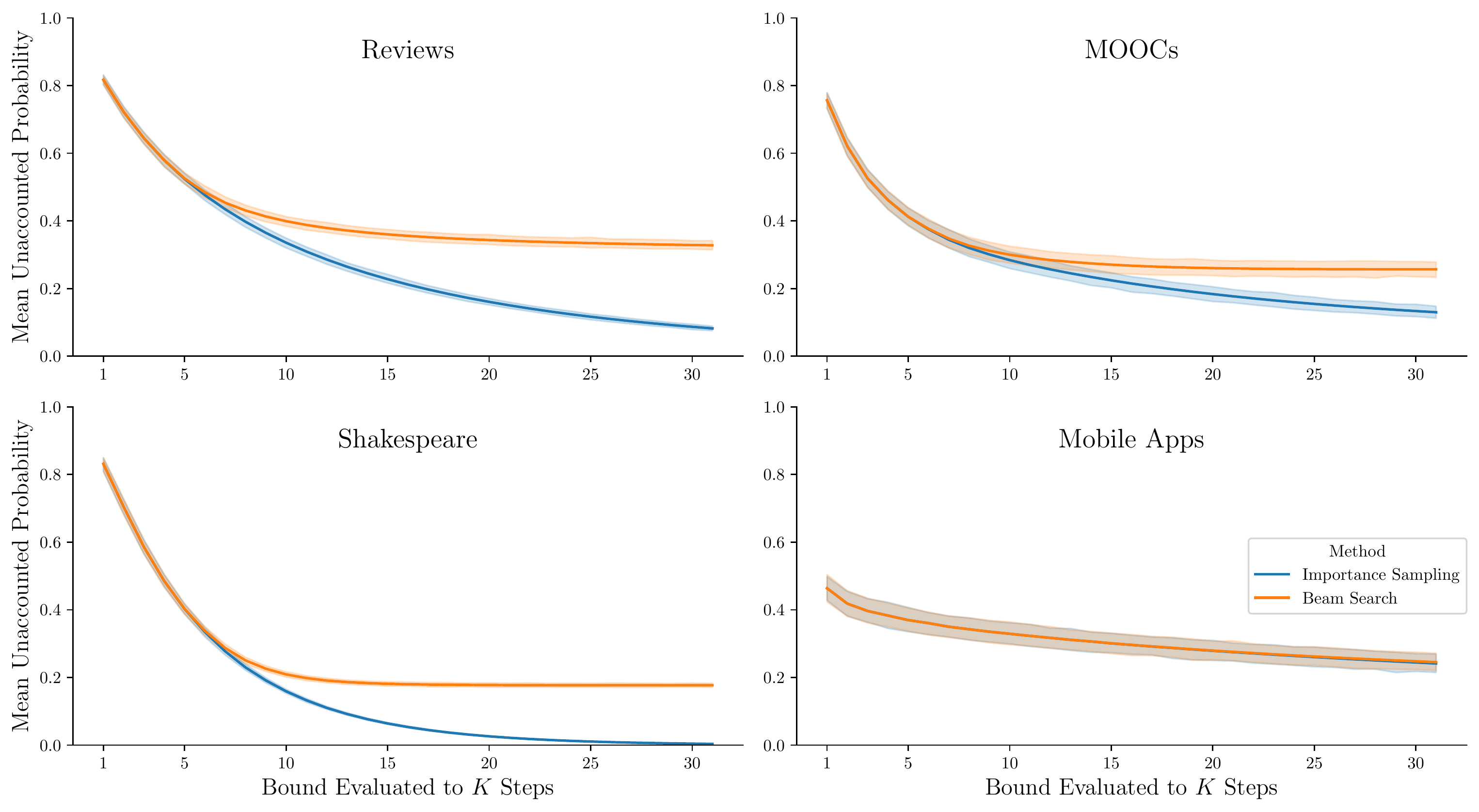}
    \caption{Mean unaccounted probability ($1-(\hat{\prob}_\theta(\hit(A)<\hit(B)) + \hat{\prob}_\theta(\hit(B)<\hit(A)))$) when evaluating the pair of queries for $\hit(A) < \hit(B)$ and $\hit(B) < \hit(A)$ up to $K$ steps into the future over all four main datasets. Shaded regions indicate 99\% confidence intervals and estimates were computed with a fixed model budget based on 1000 samples.}
    \label{fig:3_a_before_b_investigation}
\end{figure}

\cref{fig:3_a_before_b_investigation} plots the unaccounted probability $1-\hat{\prob}_\theta(\hit(A) < \hit(B))-\hat{\prob}_\theta(\hit(A)>\hit(B))$ as a function of query length $K$. We observe that for many datasets, a query horizon of $k=30$ is largely sufficient to reduce the remaining probability to under 10\%. One notable exception is the Mobile Apps dataset, which, due to its lower entropy and high self-transition rate, maintains a much longer query horizon. This discovery implies that a successful partition of probability space with a given Q4 query can be sensitive to the model distribution, but also that approximate partitions are possible for relatively low values of $K$.

\subsection{Linearly Compounding Errors in Complex Query Estimation}

Hitting time queries can often be seen as components of more involved queries, such as “a” before “b” queries $\prob_\theta(\hit(a) < \hit(b))$ or counting-style queries $\prob_\theta(N_a(K)=n)$. These queries can be rewritten as summations of more basic hitting time queries. For instance, $\prob_\theta(\hit(a) < \hit(b)) = \sum_{k=1}^\infty \prob_\theta(\hit(a)=k,  \hit(b)>k)$. In practice, each summand probability is estimated using our proposed techniques so it can be seen that the error compounds additively with respect to the different basic hitting time queries. 

More generally, our framework proposes representing general queries as $\mathcal{Q} = \cup_i \mathcal{Q}_i = \cup_i \prod_j \mathcal{V}_j^{(i)}$ such that the $\mathcal{Q}_i$’s form a minimal partition on $\mathcal{Q}$. With this representation, we can see that for an arbitrary query, the error when estimating will compound additively and scale linearly with respect to the number of different $\mathcal{Q}_i$. Note that the actual values of $\prob_\theta(\textbf{X}\in\mathcal{Q}_i)$ do have an impact on the errors when estimating due to the values $\in[0,1]$. Lastly, as mentioned in the paper we can additionally control this error either by utilizing the coverage-based beam search, or by leveraging the Central Limit Theorem with importance sampling.

\subsection{Qualitative Exploration and Practical Applications}%

In addition to systematic quantitative analysis of query estimators, we also qualitatively explore specific applications of our methods that lend practical insights. First, we consider the question of ``given a partial sentence, predict when the sentence will end". 
Using our query estimation methods, we can not only answer this question with relatively low error but also effectively re-use intermediate computational results. 
This necessarily correlates query estimates over steps $K$, but this confers little negative impact upon the analysis since our sampling methods are unbiased estimators and beam search, though biased, offers a deterministic lower bound. The results of our analysis are seen in Fig. 1 %
in the main paper and align with basic intuition about English sentences, with open ended prefixes possessing a long-tailed end-of-sentence probability distribution relative to more structured and declarative phrases.

\begin{figure}
    \centering
    \includegraphics[width=\textwidth]{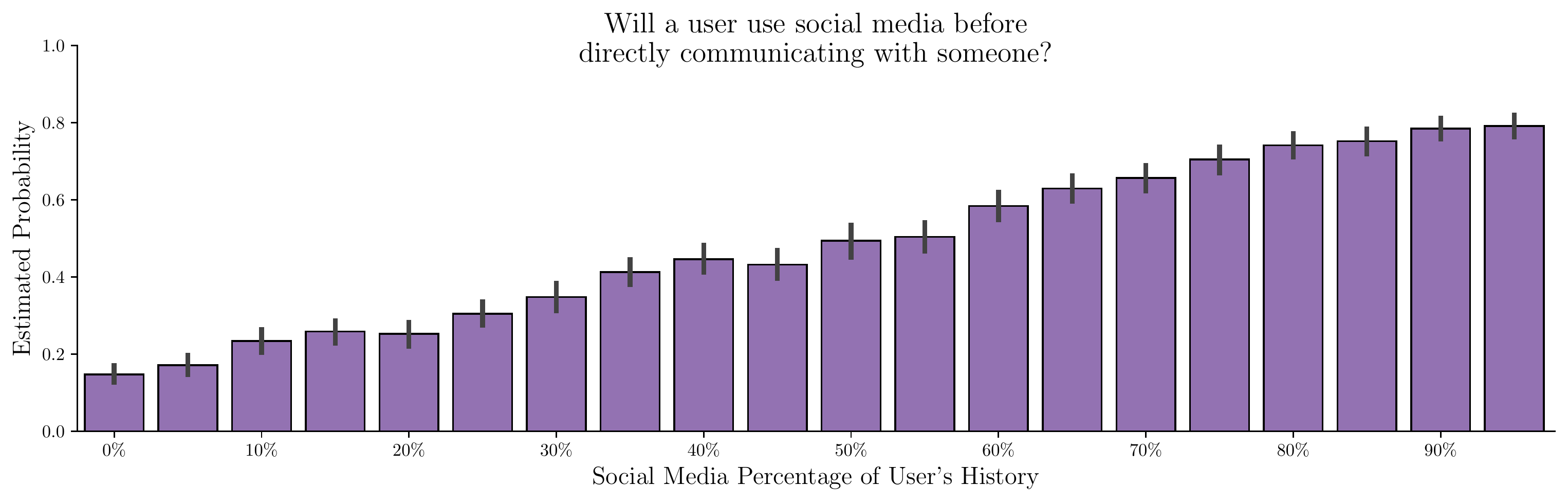}
    \caption{A case study of evaluating how likely it is that a mobile app user will use social media before they directly communicate with someone via text messaging, email, phone call, etc. as a function of how much of their recent history consists of social media usage. This is a $\prob_\theta(\hit(A)<\hit(B))$ type query that has been estimated and averaged over 500 different histories for each social media history percentage. Error bars indicate 90\% confidence intervals.}
    \label{fig:3_flashy_ab}
\end{figure}

A second practical question that can be asked with our query estimation methods utilizes the Mobile Apps dataset: given someone's mobile usage history, predict what will occur first: the individual will go on social media or directly interact with someone via video chat, call, text, or email. This question is a practical application of query 4, $\prob_\theta(\hit(a) < \hit(b))$. Other, equally interesting equivalents of this question include ``will an online shopper purchase something before leaving the website?". Below, we have conducted an experiment where we synthetically generate mobile app behavior histories with specific percentages of social media activity present. By computing these queries over several histories and then averaging the estimates, we obtain the results seen in \cref{fig:3_flashy_ab}. As expected, we see a clear linear trend between a user's social media usage and the likelihood they will return to it before conducting other tasks like directly communication. Such a result, though contrived and purely demonstrative in this setting, could be applied to many practical applications that analyze the characteristics of online user behavior.

\chapter{Supplemental Material for Chapter 4}

\addtocontents{toc}{\protect\setcounter{tocdepth}{0}}

\section{Deriving ``A before B'' Estimator}\label{sec:4_a_before_b_derivation}
Let $A,B\subset\mathcal{M}$ and $A\cap B = \varnothing$. Recall that $\hist^A_{[0,t]}:=\hist^A_t$ is the sequence of events over times $[0,t]$ with the restriction that the marks must all belong to $A$. Finally, let $\q$ describe a proposal distribution with $\mu_k^*(t) = \ind(k \notin A \cup B)\lambda^*_k(t)$. With this in mind, we derive the expected value expression for the ``A before B'' queries:
\begin{align*}
\prob&\left(\hit(A) < \hit(B)\right) = \int_0^\infty \prob\left(\hit(A) < \hit(B), \hit(A) \in [t,t+dt)\right)\\
& = \int_0^\infty \sum_{k\in A}\prob\left(|\hist_{[t,t+dt)}^k|=1, \hist^A_{t-}=\varnothing, \hist^B_{t-}=\varnothing\right) \\
& = \int_0^\infty \sum_{k\in A} \prob\left(|\hist_{[t,t+dt)}^k|=1, \hist^{A\cup B}_{t-}=\varnothing \right) \\
& = \int_0^\infty \sum_{k\in A}\E^{\prob}_{\hist_{t-}}\left[\prob\left(|\hist^k_{[t,t+dt)}|=1, \hist^{A\cup B}_{t-}=\varnothing\sep\hist_{t-}\right)\right] \\
& = \int_0^\infty \sum_{k\in A}
\E^{\prob}_{\hist_{t-}}\left[\prob\left(|\hist^k_{[t,t+dt)}|=1 \sep  \hist^{A\cup B}_{t-}=\varnothing, \hist_{t-}\right)\prob\left(\hist^{A\cup B}_{t-}=\varnothing\sep\hist_{t-}\right)\right] \\
& = \int_0^\infty \sum_{k\in A}
\E^{\prob}_{\hist_{t-}}\left[\prob\left(|\hist^k_{[t,t+dt)}|=1 \sep  \hist_{t-}\right)\ind\left(\hist^{A\cup B}_{t-}=\varnothing\right)\right] \\
& = \int_0^\infty \sum_{k\in A}
\E^{\prob}_{\hist_{t-}}\left[\lambda^*_k(t)\ind\left(\hist^{A\cup B}_{t-}=\varnothing\right)\right]dt \\
& = \int_0^\infty 
\E^{\prob}_{\hist_{t-}}\left[\lambda^*_A(t)\ind\left(\hist^{A\cup B}_{t-}=\varnothing\right)\right]dt \\
& = \int_0^\infty \E^\q_{\hist_{t-}}\left[\lambda^*_A(t)\ind\left(\hist_{A\cup B}(t)=\varnothing\right)\frac{\mathcal{L}^\prob\left(\hist_{t-}\right)}{\mathcal{L}^\q\left(\hist_{t-}\right)}\right] dt \\
& = \int_0^\infty \E^\q_{\hist_{t-}}\left[\lambda^*_A(t)\exp\left(-\int_0^t \lambda^*_{A \cup B}(s) ds\right)\right] dt \\ 
& = \int_0^\infty \E^\q_{\hist_{\infty}}\left[\lambda^*_A(t)\exp\left(-\int_0^t \lambda^*_{A \cup B}(s) ds\right)\right] dt \\ 
& = \E_{\hist_{\infty}}^ \q\left[\int_0^\infty \lambda^*_A(t)\exp\left(-\int_0^t \lambda^*_{A \cup B}(s) ds\right)\right] dt 
\end{align*}
where the last line is justified due to the Dominated Convergence Theorem. The prerequisites for this theorem are satisfied by noting that:
\begin{align*}
\int_0^\infty \lambda^*_{A}(t)\exp&\left(-\int_0^t \lambda^*_{A \cup B}(s) ds\right)dt \leq  \int_0^\infty \lambda^*_{A\cup B}(t)\exp\left(-\int_0^t \lambda^*_{A \cup B}(s) ds\right)dt \\
& = -\int_0^\infty \frac{d}{dt}\exp\left(-\int_0^t \lambda^*_{A \cup B}(s) ds\right)dt \\
& = \exp\left(-\int_0^0 \lambda^*_{A \cup B}(s) ds\right) - \exp\left(-\int_0^\infty \lambda^*_{A \cup B}(s) ds\right) \\
& = 1 - \exp\left(-\int_0^\infty \lambda^*_{A \cup B}(s) ds\right) \\
& \leq 1.
\end{align*}

\section{Further Experimental Details and Results}\label{sec:4_exp_appendix}

\subsection{Dataset Preprocessing}
We evaluate our methods for probabilistic querying on three real-world user-behavior datasets in different application domains that are publicly available. All datasets do not include personally identifiable information, where users are identified by unique integer IDs. For all our experiments, sequences are defined as the event histories of each user, where events have timestamps in seconds. We changed the time resolution from seconds to hours for better interpretability of our query implications. Additionally, we only consider sequences with at least 5 events and at most 200 events. We use 75\% of the sequences for training, 10\% for validation, and 15\% for testing.

\paragraph{MovieLens}
The MovieLens 25M dataset \citep{harper2015movielens} contains 25 million movie ratings by 162,000 users. The movie category (genre) associated with each rating is modeled as marks, and the exact rating value is ignored.\footnote{A single movie in this dataset can possibly have multiple categories associated with it. To accommodate this, if a movie has multiple categories we randomly select a subset of two categories to represent the movie. Note this highlights the benefits of formulating queries as sets of marks instead of just singular marks. To evaluate the hitting time of the next ``comedy'' movie reviewed, then we would need to evaluate the hitting time of the set of all pairs of categories where one element is the comedy genre. This is essentially describing marginalizing over a hierarchical structure for the marks.} For each sequence, the start and the end time are defined as the first and the last event time of each user respectively, because the time span for different users ranges from seconds to years. The first event is discarded in the sequence of history and is only used to indicate $t=0$. For consistent dynamics across the dataset, we filter the data to only contain reviews at or after the year 2015. This leaves 34,935 remaining sequences, each from a unique user.

\paragraph{MOOC}
The MOOC user action dataset \citep{kumar2019predicting} represents user activities on a massive open online course (MOOC) platform. It consists of 411,749 course activities in 97 different types modeled as marks for 7,047 users, out of which 4,066 users dropped out after an activity. Timestamps are standardized to start from timestamp 0. We use the last event time for drop-out users as the end of their sequences, and the maximum timestamp for the other users. 

\paragraph{Taobao}
The Taobao user behavior dataset \citep{zhu2018learning} was originally intended for recommendations for online shopping, which includes four behaviors: page viewing, purchasing, adding items to the chart, and to wishlist. We focus on page viewing of users as events, and model the item category as the event mark, which has marketing implications such as click through rate of recommending some types of items. Due to the large scale of the dataset, we use a subset of 2,000,000 events on 8 consecutive calendar days inclusive (November 25th, 2017 - December 2nd, 2017), as well as the most frequent 1,000 marks (item categories) to demonstrate query answering. All user sequences have the same length.

\subsection{Modeling Details}
For each of the real-world datasets, a neural Hawkes process model \citep{mei2017neural} was trained with a batch size of 128, a learning rate of 0.001, a linear warm-up learning rate schedule over the first 1\% of training iterations, a max allowed gradient norm of $10^4$ for training stability, and the Adam stochastic gradient optimization algorithm \citep{DBLP:journals/corr/KingmaB14} with default hyperparameters. Specific datasets had specific model hyperparameters due to differences in the amount of data and total possible marks. The details for these can be found in \cref{tab:4_model_details}. All models were trained for a fixed amount of epochs; however, each one was confirmed to have converged based on average held-out validation log-likelihood.

\begin{table}
    \centering
    \caption{Model Hyperparameters for Real-World Datasets}
    \begin{tabular}{lccc}
    \toprule
    Hyperparameter & MovieLens & MOOC & Taobao \\
    \midrule
    \# Training Epochs & 100 & 100 & 300 \\
    Mark Embedding Size & 32 & 32 & 64 \\
    Recurrent Hidden State Size & 64 & 64 & 128 \\
    \bottomrule
    \end{tabular}
    \label{tab:4_model_details}
\end{table}

\subsection{Integration Approximation}
For the real-world experiments, many integrals need to be evaluated in order to produce estimates for various queries. Since we use essentially black-box MTPP models, we do not have access to an analytical form for integration. As such, we must estimate every integral at play.

To do this, we utilize the trapezoidal rule. For reference, this involves estimating integrals with the following summation:
\begin{align*}
\int_a^b f(x)dx \approx \sum_{i=1}^N \left(f(x_i)+f(x_{i-1})\right)\frac{x_i-x_{i-1}}{2}
\end{align*}
where the points $x_{i-1} < x_i$ span the interval $[a,b]$ with $x_0=a$ and $x_N=b$. For hitting time queries and marginal mark queries, we utilize $N=1000$ integration points with equal spacing. It is likely that we could get by with much less for these queries, however, for the sake of high precision for experimental results we utilized a large amount of sample points.

For the ``A before B'' queries, we found that the resolution at which the estimator is evaluated at is of much more importance than the other queries. As such, for this query we estimate integrals in an online fashion during the sampling procedure for each proposal distribution sample sequence in conjunction with a very high proposal dominating rate (see \citet{ogata1981lewis} for details). This allowed for a much more efficient procedure (in both computation and memory consumption) compared to integrating results after sampling. 

\subsection{Marginal Mark Query Experiments}
Similar to the hitting time experiments, for the marginal mark queries we similarly sample 1000 random test sequences and condition on the first five events $\hist_{T_5}$. Then, we estimate the query $\prob(M_{8}\in A\sep\hist_{\tau_5})$ where $A$ is a randomly selected subset of all of the unique marks that appear in the entire sequence $\hist$. This is done to ensure that $A$ contains relevant marks for the given sequence.  

\begin{figure}    \centering{\phantomsubcaption\label{fig:4_marg_mark_err}\phantomsubcaption\label{fig:4_marg_mark_eff}}
    \includegraphics[width=0.78\textwidth]{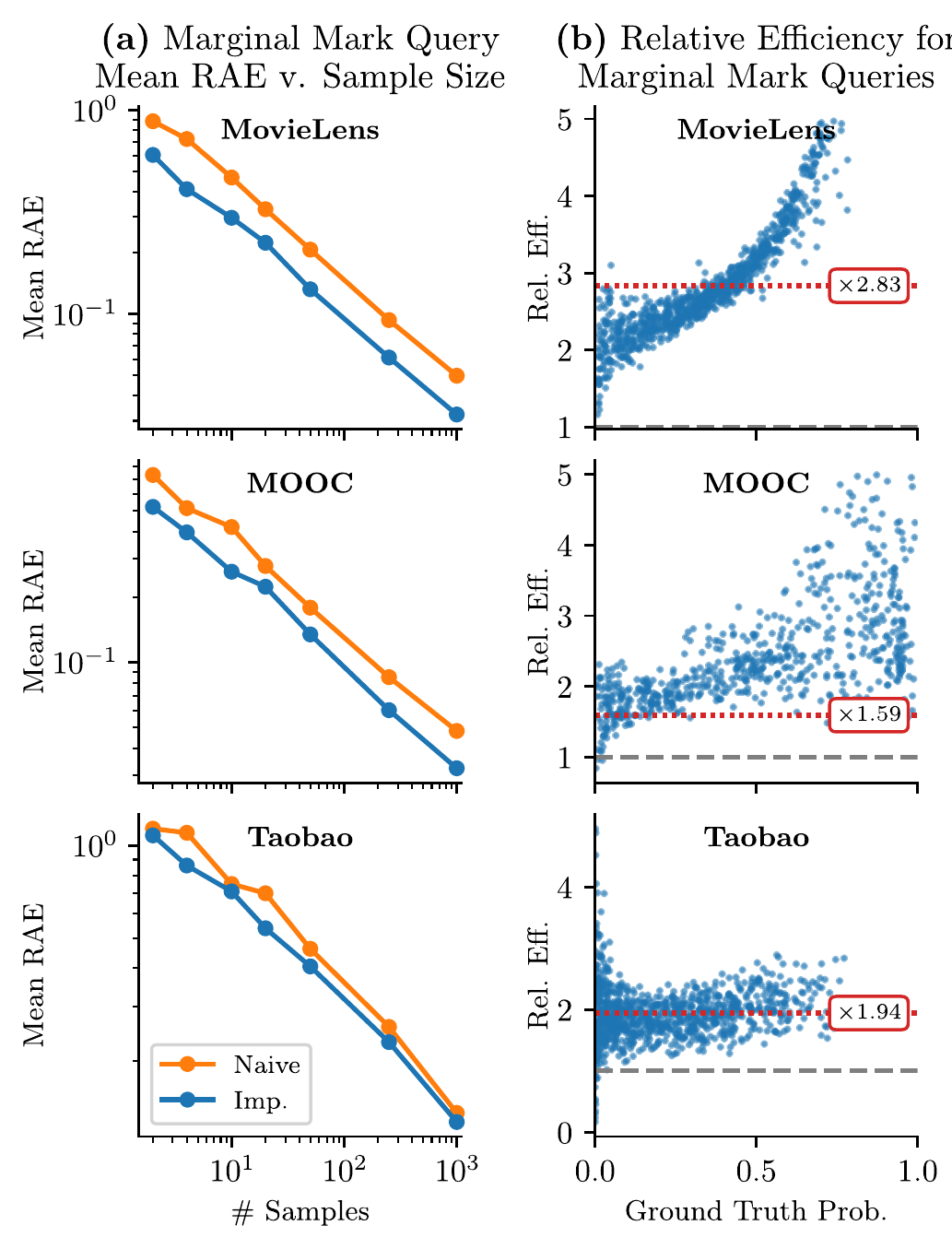}
    \caption{Results from 1,000 different marginal mark queries evaluated on models trained on three different datasets. (a) Average relative absolute error for naive and importance sampling shown in comparison to number of sampled sequences used. (b) Estimated relative efficiency values for importance sampling compared to naive sampling plotted against ground truth marginal mark query values. Gray dashed lines indicate an efficiency of 1. Red lines with associated text box indicate the average multiplicative increase in computation time for importance sampling.}
    \label{fig:4_marg_mark_plots}
\end{figure}

We compared estimating this query with naive sampling and importance sampling using varying amounts of samples: $\{2,4,10,25,50,250,1000\}$. Mean RAE compared to ground truth (estimated using importance sampling with 5,000 samples) can be seen in \cref{fig:4_marg_mark_err}. We witness roughly $1.5$ to $3$ times improvement in performance for the same amount of samples. 
Similar to hitting time query results, we attribute this improvement to the fact that naive sampling only collects binary values, whereas our proposed procedure collects much more dense information over the entire span from $T_5$ to $T_8\sim q$.

We also analyze the relative efficiency of our estimator compared to naive sampling. For each query asked, the efficiency was estimated using 5,000 importance samples. The results can be seen in \cref{fig:4_marg_mark_eff}.  
We achieve a decent decrease in variance, in the majority of contexts, across all datasets. Like the hitting time query results, we also note a pretty strong correlation between underlying ground truth values and the relative efficiency of this estimator. 

Notably, these results do not appear to be as drastic as the hitting time query results. We believe this is due to the fact that the estimator's bounds of integration are sampled from the proposal distribution to be between $\tau_{N-1}$ and $\tau_N$ for each sequence (whereas the bounds for the hitting time query $p(\hit(k)\leq t)$ is always the span of $[0,t]$). This added variability seems to dampen the impact of the integration in the first place.

\subsection{Synthetic Data Experiments}
We also perform experiments on hitting time queries and ``A before B'' queries using self-exciting parametric Hawkes processes \citep{hawkes1971spectra}. The intensity for Hawkes processes with exponential kernels has the explicit form:
\begin{align}
    \lambda_k^*(t) &= \mu_k + \sum_{m=1}^K \int_0^t \phi_{m k}(t - u) d N_m(u) \nonumber \\
    &= \mu_k + \sum_{m=1}^K \sum_{T_{m, i} < t} \phi_{m k}(t - T_{m, i}) \label{eq:4_mutual_density2},
\end{align}
where $T_{m, i}$ refers to the time when the $i^\text{th}$ event of type $m$ occurs, $\phi(\boldsymbol{x}) = \boldsymbol{\alpha} e^{-\boldsymbol{\beta} \boldsymbol{x}}$ with $\boldsymbol{\alpha}, \boldsymbol{\beta} > \boldsymbol{0}$, and Equation \ref{eq:4_mutual_density2} can be expressed in matrix form. The first term $\boldsymbol{\mu}$ is referred to as the \textit{base intensity} or \textit{background intensity} in literature. Each event instantaneously increases the intensity by the corresponding value of $\alpha$ and its influence decays exponentially with $\beta$ and over time. 
Under this parametric form, the integrals for query estimates can be computed in closed forms. 

We also conduct both experiments on hitting time and ``A before B'' queries using Hawkes processes with Gamma kernels. The Gamma kernel has the form of $\phi(\boldsymbol{x}) = \boldsymbol{x} e^{- \boldsymbol{x}}$, and the corresponding Hawkes processes do not have closed-form solutions to these queries.

We evaluate our methods on (i) hitting time queries $\prob(\hit(k)\leq t)$
and (ii) ``A before B'' queries $\prob(\hit(A) < \hit(B))$. All results are averaged over 1,000 different randomly initiated parametric self-exciting Hawkes models that are not feasible for real-world datasets. These random models have different total amounts of marks ranging from $K=3$ to $K=10$, and have different inter-event effects as well as exponential rates of decay. We use 10 integration points for hitting time queries and 1,000 integration points for ``A before B'' queries. \footnote{For the ``A before B'' queries, using 1,000 integration points after sampled provided sufficient precision and we did not need to employ the online integration approach used with the real-world experiments. This is most likely attributable to the well-behaved dynamics exhibited by the parametric Hawkes intensity. This is also why we used a reduced amount of integration points for the synthetic hitting time queries as well compared to the real-world experiments.} 

For each hitting time query, we fix $t=1$ and $k=0$, because the model is randomly generated.
For the ``A before B'' queries, like the real-world experiments we let them be randomly sampled subsets of the vocabulary such $|A|=|B|\approx K/3$.
We evaluate the hitting time queries using varying amounts of samples: $\{2,4,10,25,50,250,1000\}$. For ``A before B'' queries, we only use $\{2,4,10,25,50,250\}$ number of samples because the query estimates take longer. Ground truth probabilities are calculated using 5,000 samples with importance sampling for hitting time queries and with naive method for ``A before B'' queries respectively. 

The plots in \cref{fig:4_supp_hawkes_hitting,fig:4_supp_hawkes_ab} reveal  similar patterns and illustrate that our method is more efficient than the naive estimates averaged over a range of different model settings.

\begin{figure}
    \centering
    \includegraphics[width=0.8\textwidth]{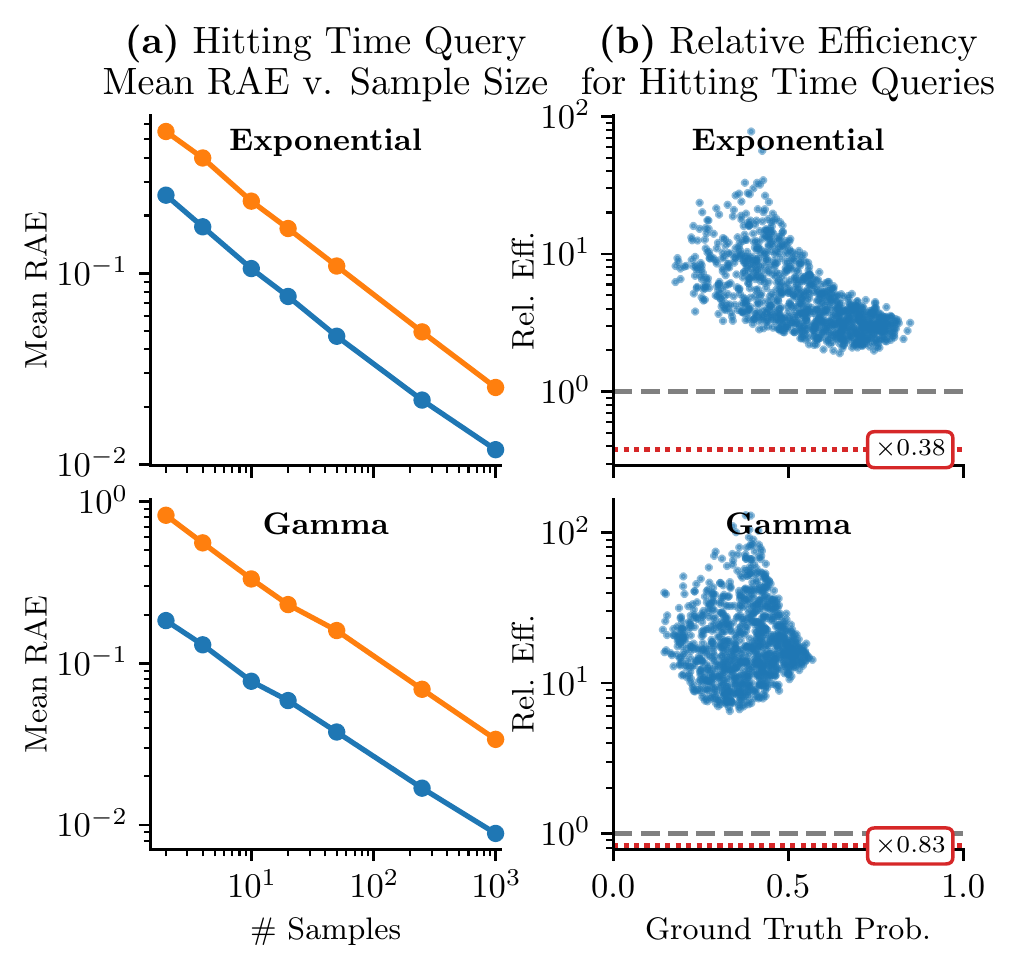}
    \caption{Synthetic experiments for hitting time queries evaluated on parametric self-exciting Hawkes processes with both exponential and Gamma kernels. (a) Average relative absolute error for naive and importance sampling shown in comparison to number of sampled sequences used. (b) Estimated relative efficiency values for importance sampling compared to naive sampling plotted against ground truth hitting time query values. Gray dashed lines indicate an efficiency of 1. Red lines with associated text box indicate the average multiplicative increase in computation time for importance sampling.
    }
    \label{fig:4_supp_hawkes_hitting}
\end{figure}

\begin{figure}
    \centering
    \includegraphics[width=0.8\textwidth]{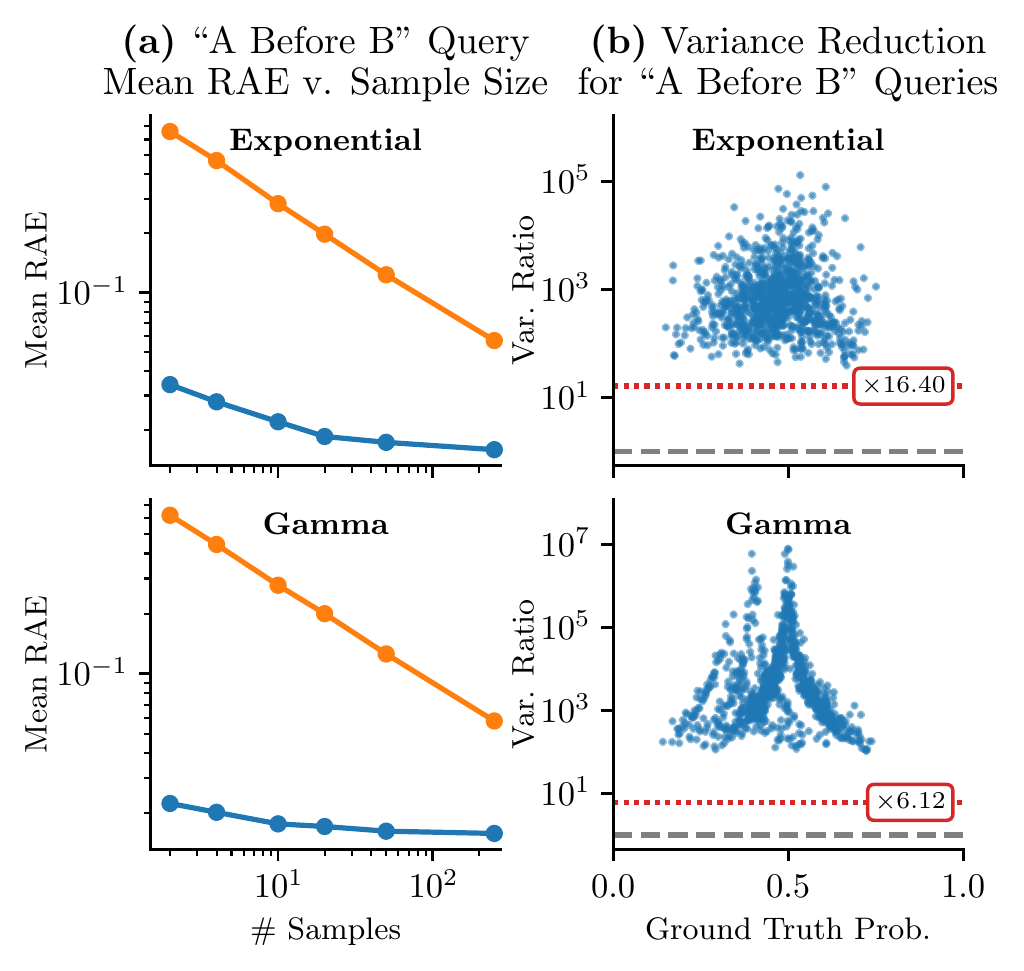}
    \caption{Synthetic experimental results evaluated on 1,000 different random models and ``A before B'' queries for parametric self-exciting Hawkes processes with both exponential and Gamma kernels. (a) Average relative absolute error for naive and importance sampling shown in comparison to number of sampled sequences used. (b) Estimated relative efficiency values for importance sampling compared to naive sampling plotted against ground truth ``A before B'' query values. Gray dashed lines indicate an efficiency of 1. Red lines with associated text box indicate the average multiplicative increase in computation time for importance sampling.}
    \label{fig:4_supp_hawkes_ab}
\end{figure}

\chapter{Supplemental Material for Chapter 5}

\addtocontents{toc}{\protect\setcounter{tocdepth}{0}}

\section{Further Experimental Details and Results}\label{sec:5_exp_appendix}

\subsection{Datasets}

The following are more in depth descriptions on the different real-world datasets used in experiments. All sequences used for both training and inference are preprocessed to only allow sequences with at least 5 events and at most 200. Summary statistics can be found in \cref{tab:data_statistics}.

\paragraph{Taobao}
The Taobao user behavior dataset \citep{zhu2018learning} was originally intended for recommendations during online shopping sessions, which includes four different behaviors: page viewing, purchasing, adding items to the chart, and adding items to a wishlist. We focus on modeling the page viewing of users as events, and let the item category be the associated event mark. Modeling this information has various marketing implications such as click through rate of recommending some types of items. Due to the large scale of the dataset, we use a subset of 2,000,000 events on 8 consecutive calendar days inclusive (November 25th, 2017 - December 2nd, 2017), as well as the most frequent 1,000 marks (item categories). All user sequences have the same time length of $T=192$ hours.

\paragraph{Reddit} The Reddit comments dataset \citep{baumgartner2020pushshift} contains records of comments made by different users on various posts listed in the social media site \url{reddit.com}. One month's worth of data (October 2018) was used to extract user sequences, and the mark vocabulary was defined as the top 1000 communities (subreddits) determined by marginal comment volume. The month was divided into multiple week-long sequences for each user, with event times in units of hours ($T=178$ hours).

\paragraph{MemeTracker} The MemeTracker dataset \citep{leskovec2009meme} tracks to common phrases (memes) as they appear on various websites. We compile these records into sequences, each pertaining to a single meme with events defined as the time of mention and the website they appeared on as the mark. Only the mentions in the top 5000 websites by marginal volume were considered. Sequences were defined as one-week-long chunks spanning August 2008 to April 2009, and event times were measured in hours ($T=178$ hours).   

\paragraph{Email} 
Lastly, the Email dataset \citep{paranjape2017motifs} contains the email records for a research organization over the course of 803 days. Sequences were defined as the collection of incoming emails for a given user where each mark was the address of the original sender. These sequences were defined over four week intervals and event times were measured in days ($T=28$ days). After preprocessing the sequences, we were left with 808 different unique addresses.

\begin{table}
    \caption{Summary Statistics for the Four Real-World Datasets}
    \label{tab:data_statistics}
    \centering
    \begin{tabular}{llrrrrr}
    \toprule
    & & & Mean & \multicolumn{3}{c}{\# Sequences}  \\
    Dataset & \mc{$T$} & \mc{$M$} & \mc{$|\mathcal{H}|$} &  Train & Valid & Test \\
    \midrule
    Taobao & 8 Days & 1000  & 62.6 & 13.3K & 1.8K  & 2.7K  \\
    Reddit & 1 Week  & 1000 & 65.2 & 343K & 15K & 34K \\
    MemeTracker & 1 Week  & 5000 & 23.4 & 271K & 9K  & 21K  \\
    Email & 28 Days   & 808   & 31.1 & 6.9K & 1.5K & 1.5K \\
    \bottomrule
    \end{tabular}
\end{table}

\subsection{Model \& Training Details}

For each of the real-world datasets, a neural Hawkes process model \citep{mei2017neural} was trained on fully observed sequences for a given dataset. 
Each model was trained using the Adam stochastic gradient optimization algorithm \citep{DBLP:journals/corr/KingmaB14} with default hyperparameters, a learning rate of 0.001, and a linear warm-up learning rate schedule over the first 1\% of training iterations. Each iteration optimized the parameters against the average log-likelihood for a batch of 128 training sequences. Gradients were clipped to have a maximum norm of $10^4$ for stability. All models were trained for a fixed amount of epochs; however, each one was confirmed to have converged based on average held-out validation log-likelihood.

Models possessed different hyperparameters depending on the dataset due to differences in the amount of data and total possible marks. Details can be found in \cref{tab:5_model_details}.

\begin{table}
    \centering
    \caption{Model Hyperparameters for Real-World Datasets}
    \begin{tabular}{lcccc}
    \toprule
    Hyperparameter & Taobao & Reddit & MemeTracker & Email \\
    \midrule
    \# Training Epochs & 300 & 50 & 50 & 300 \\
    Mark Embedding Size & 64 & 64 & 64 & 32 \\
    Recurrent Hidden State Size & 128 & 128 & 128 & 64 \\
    \bottomrule
    \end{tabular}
    \label{tab:5_model_details}
\end{table}

\subsection{Next Event Prediction}

Alongside likelihood, we are also interested in making predictions for next events in the presence of censored data. The following section details the prediction experiments conducted for both synthetic and real-world settings.

\paragraph{Setup}
We follow the same settings for the next event prediction as \citet{du2016recurrent, mei2017neural} on both event time and event mark. The predicted time is chosen to be the expected time of the next event occurrence, which is defined as 
\begin{align}
\hat{T}_i=\mathbb{E}^\prob\left[ T_i \sep \hist_{T_{i-1}}\right]=\int_{ T_{i-1}}^{\infty} t \lambda^*(t) \exp\left(-\int_{ T_{i-1}}^t \lambda^*(s) ds\right) d t.
\end{align}
We measure predictive performance for this with the mean absolute error between predicted and true next event time. Without the knowledge of the event time $ T_i$, the predicted next event type set as
\begin{align}
\hat{M}_i & := \argmax_{k\in\mathcal{M}} \prob(M_i = k \sep \hist_{T_{i-1}}) \\
\text{ for }  \prob(M_i = k \sep \hist_{T_{i-1}}) & = \int_{T_{i-1}}^\infty \lambda^*_k(t) \exp\left(-\int_{T_{i-1}}^t \lambda^*(s)ds\right)dt,
\end{align}
and is evaluated via top-10 accuracy (i.e., the proportion of predictions in which true mark $M_i$ appears in the set of top-10 highest probability predicted marks). Both predictions can be achieved by approximating integrals numerically, for both the censored and baseline methods.

Similar to the likelihood ratio experiments, we evaluated these methods on sequences that have been artificially censored. For the synthetic experiments, we evaluate 1000 sequences $\hist(T)$ sampled from their respective models and then randomly choose a subset of unique marks that appear in each sequence to be censored $\cen$, the proportion of which is determined for each value $\gamma \in \{0.2, 0.4, 0.6, 0.8\}$. For real-world experiments, the same is done except the sequences originate from held-out sets and $\gamma$ is also allowed to be $0.5$.

We condition each method on the occluded sequence $\hist^\obs_{T_{\lfloor\frac{n}{2}\rfloor}}$ where $|\hist(T)|=n$ and have each produce predictions for the next time $\hat{ T}_{\lfloor\frac{n}{2}\rfloor+1}$ and the next mark $\hat{ M}_{\lfloor\frac{n}{2}\rfloor+1}$.

\paragraph{Synthetic Results}
Figure \ref{fig:5_synth_next_event} reports the results evaluated on three parametric point process models with 20 distinct marks. When predicting next time to event, both versions of Hawkes processes achieve less error under our framework compared to the baseline. The performance gap between methods widen as more information is censored. However, the baseline outperforms our method for self-correcting models, which may be due to the fact that the occurrence of an event has an inhibiting effect on future events. This results in the baseline always overestimating the intensity as it lacks the censored events to correct it. For this model, this leads to always underestimating the next time to event which is favorable as this will be bounded between $ T_{\lfloor\frac{n}{2}\rfloor}$ and $ T_{\lfloor\frac{n}{2}\rfloor+1}$. This can be seen as a systematic bias inherent to the specific model parameterization.

As for the prediction of the next event type, both self-correcting processes and Hawkes processes with dense interaction between events have similar performances as random guesses that will have an accuracy of around 0.5 for top-10 accuracy. This is expected for both models, as there is not much imposed correlation between events of different types due to how the models were instantiated. 
However, the Hawkes processes with block-diagonal interactions better model the structure in sequential events, where the prediction accuracy is much higher than 0.5, which in general decreases as more marks are censored. It is clear that our method is less sensitive to the amount of censored information and significantly outperforms random guesses, as long as the model is able to capture the underlying structured dynamics of the event sequences.

\paragraph{Real-World Results}
Real-world datasets naturally have more meaningful structures and larger vocabulary sets compared to synthetic experiments. We evaluate the results on all four datasets that have different numbers of marks ranging from 808 to 5000. The prediction of the next event time of our method is on par with the baseline, while we see consistent improvements in the next event prediction evaluated by top-10 accuracy. Furthermore, the accuracy in general, regardless of method, tends to decrease with more information being censored which is expected.

\begin{figure}
    \centering
    \includegraphics[width=\linewidth]{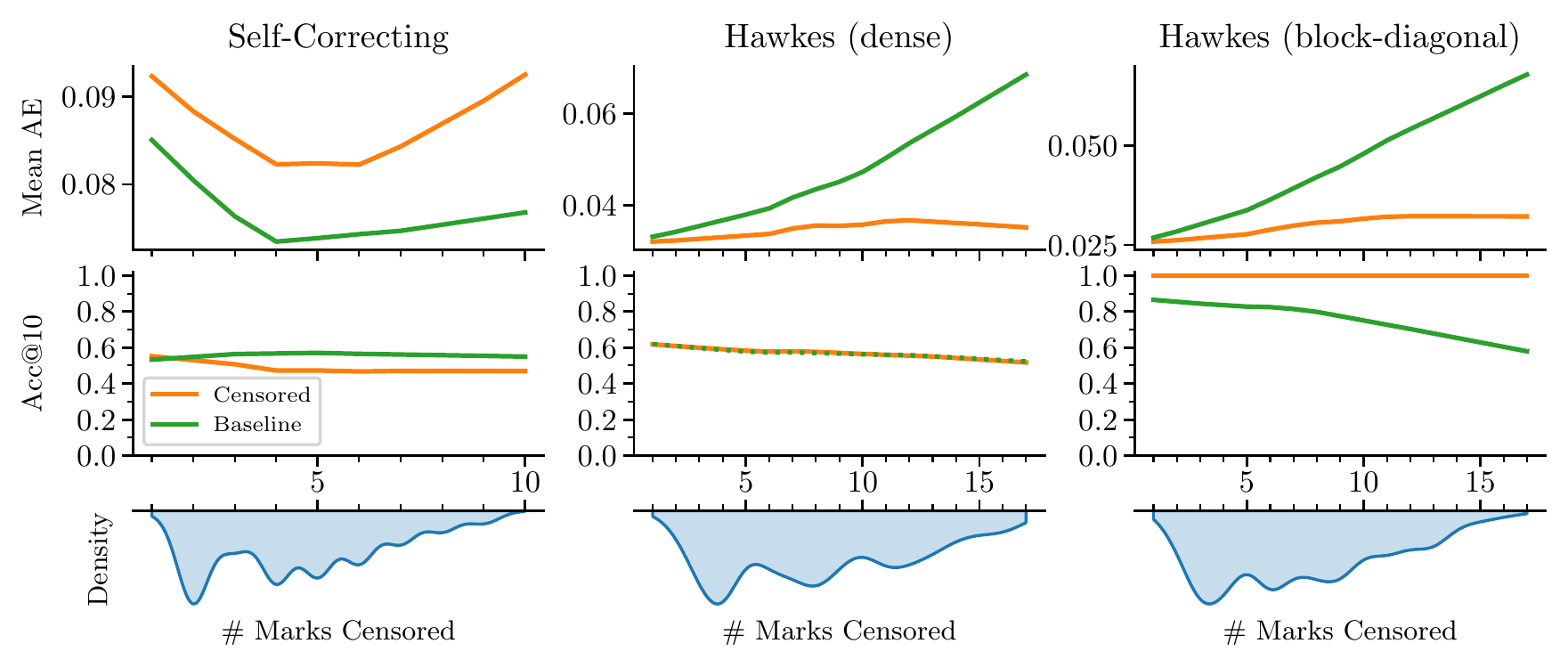}
    \caption{Next event prediction results for censored and baseline methods across the three different parametric MTPPs. Top plots indicate the mean absolute error in next time prediction, middle plots indicate top-10 accuracy in next mark prediction, and bottom plots show density of the number of marks censored across the sequences used for the experiments.}
    \label{fig:5_synth_next_event}
\end{figure}

\begin{figure}
    \centering
    \includegraphics[width=\linewidth]{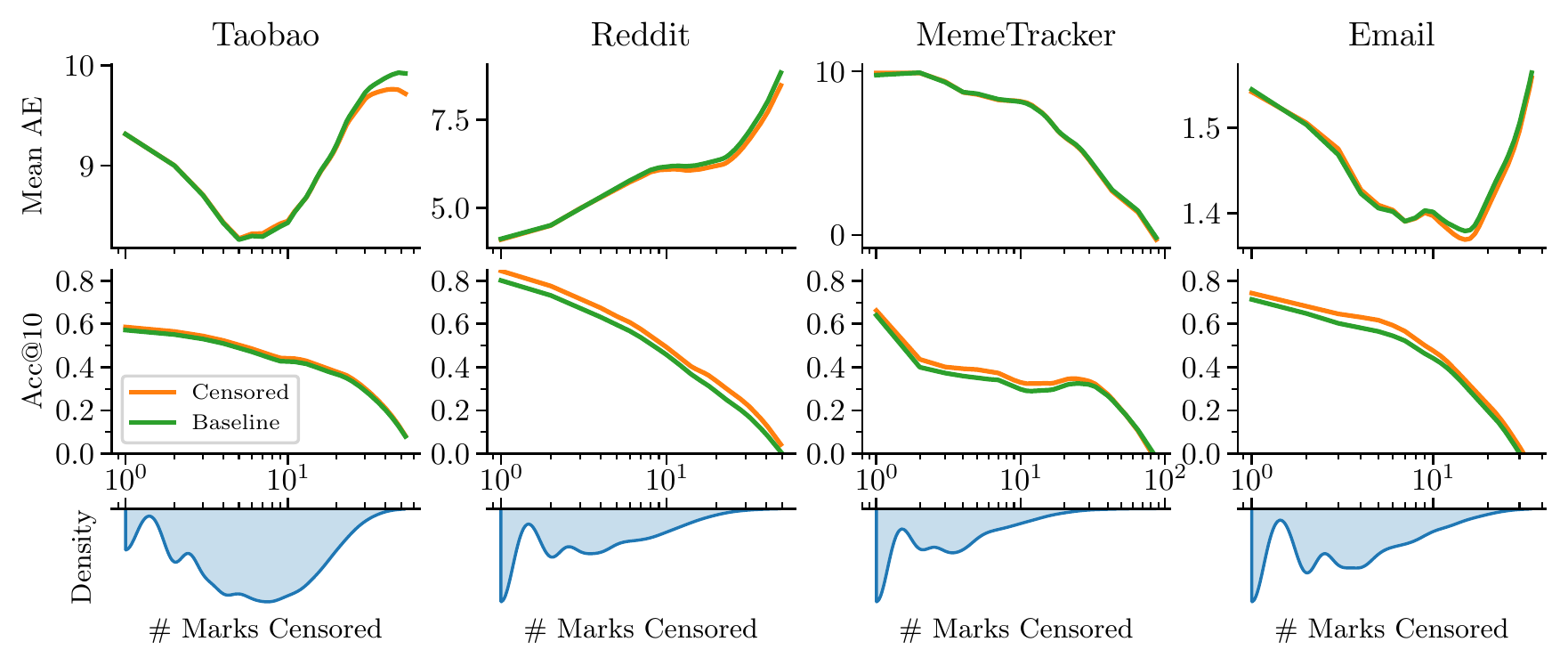}
    \caption{Same format as \cref{fig:5_synth_next_event} except using held-out sequences from real-world datasets with respectively trained neural-based models.}
    \label{fig:5_next_event}
\end{figure}

\subsection{Model Misspecification} \label{sec:5_appendix_misspecification}
Recall in the synthetic experiments that we evaluated the log-likelihood for the mark-censored model $\prob_\text{Cen}$ and the baseline method $\prob_\text{Base}$ on censored sequences that were originally sampled from the same model $\prob$ used in both methods. Under this setting, for a given mark-censoring scheme $\cen$ and $\obs$ and sampled sequences $\hist^\obs_t \sim \prob$, it is guaranteed that
\begin{align*}
\E^{\prob}_{\hist^\obs_t}\left[\mathcal{L}^{\prob_\text{Cen}}(\hist^\obs_t)\right] \geq \E^{\prob}_{\hist^\obs_t}\left[\mathcal{L}^{\prob_\text{Base}}(\hist^\obs_t)\right]
\end{align*}
with the inequality being strict so long as $\prob(\hist^\cen_t=\emptyset \sep \hist^\obs_t) > 0$. This is due to the fact that the mark-censored model is simply a marginalized version of the original model, thus resulting in no model misspecification for this setup.

That being said, we no long have this guarantee once we start considering sequences that are drawn from a different distribution from the model we are using. This is inherently the same scenario that was evaluated in the real-world data experiments, as all of the sequences used there came from some other source $\prob_\text{data}$ whereas the models $\prob$ were learned to best approximate this distribution. Naturally, the closer $\prob$ is to $\prob_\text{data}$ (i.e., the less model misspecification there is) the more we can start to trust that the censored method will produce superior results to the baseline.

\subsection{Sensitivity Analysis}
We perform an ablation study for synthetic experiments using different numbers of samples and integration points. The parameters of the Hawkes process are drawn from the same distributions as described in Section 4.1, where we used 128 MC samples and 1024 integration points. Figure \ref{fig:5_diff_seqs} shows the results of the same experiment but varies the number of Monte Carlo sampled sequences and keeps the number of integration points as 1024, while \ref{fig:5_diff_int} shows the same results but varies the number of integration points while keeping the number of Monte Carlo samples fixed to 128. Aside from slight deviations on the lower end of the values tested (e.g., number of sampled sequences $=2$ and number of integration points $=8$), the results across the board are roughly consistent. This indicates that our method is fairly robust and does not necessitate prohibitive amounts of computing resources to employ.

That being said, we do recommend evaluating this on a case-by-case basis as each model and dataset are different. In critical applications, this concern can be taken care of by iteratively sampling sequences and monitoring the convergence of the resulting censored intensity function. 

\begin{figure}
    \centering
    \includegraphics[width=\linewidth]{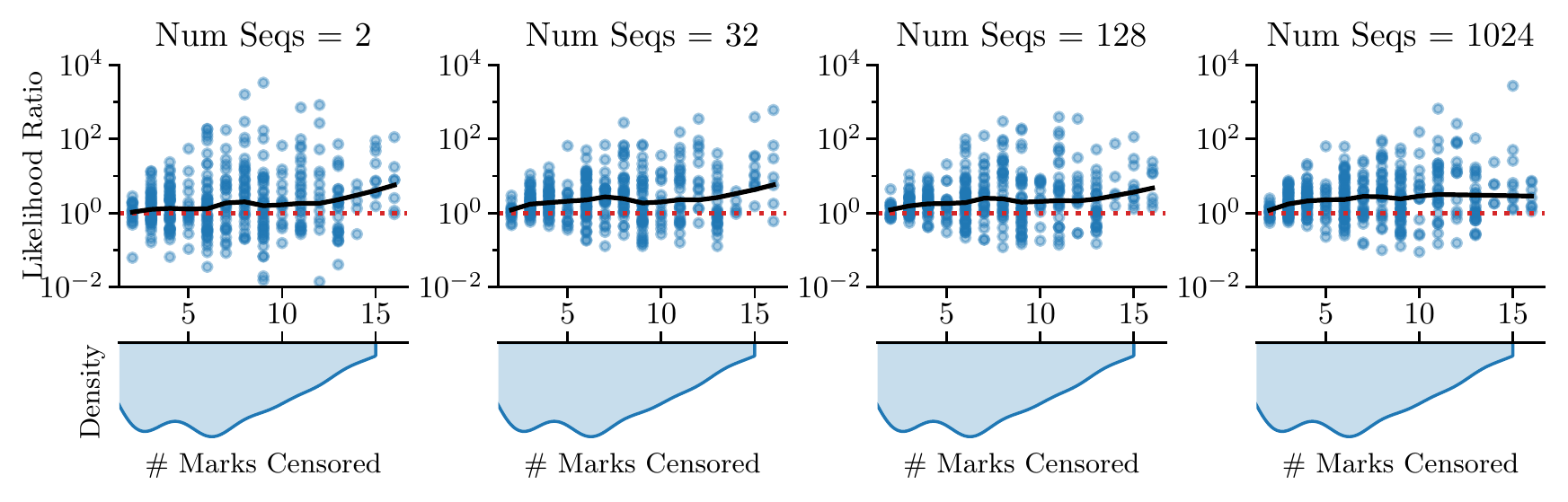}
    \caption{Distributions of likelihood ratios across number of marks censored for the duration of the sequences used for synthetic experiments. Integration points is fixed as 1024, with varying numbers of MC samples used for estimation. Values greater than 1 indicate higher likelihoods under the mark-censored model.}
    \label{fig:5_diff_seqs}
\end{figure}

\begin{figure}
    \centering
    \includegraphics[width=\linewidth]{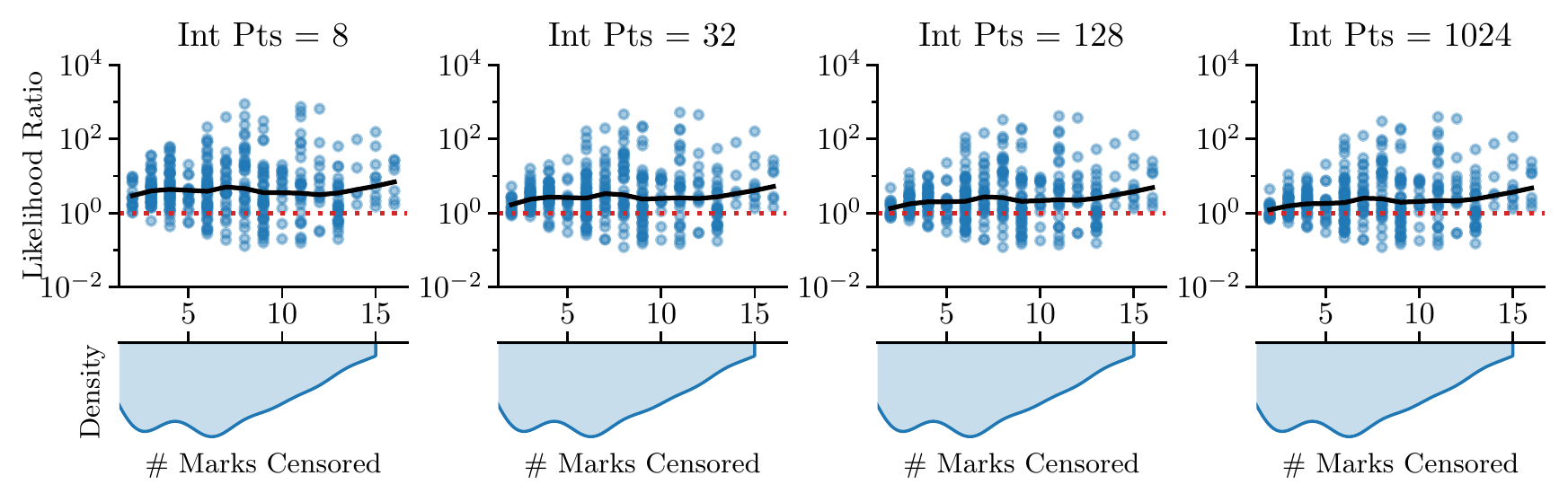}
    \caption{Same format as Fig. \ref{fig:5_diff_seqs} except using 128 MC samples and different numbers of integration points for estimation.}
    \label{fig:5_diff_int}
\end{figure}

\chapter{Supplemental Material for Chapter 6}

\addtocontents{toc}{\protect\setcounter{tocdepth}{0}}

\section{Estimators for Discrete-Time}\label{sec:6_discrete_appendix}
As outlined in \citet{bruti2007approximation}, any stochastic jump process defined in continuous-time,
\begin{align}
dX_t = \mu(t, X_{t-})dt + \sigma(t, X_{t-})dW_t + dJ_t
\end{align}
for some jump process $J:=(J_t)_{t\geq 0}$ with increments of either $0$ if no jump occurs or $M_{N_t}$ for the $N_t^\text{th}$ jump, can be approximated by a Euler discretization scheme via
\begin{align}
X_{t+\Delta}-X_{t} \approx \mu(t, X_{t})\Delta + \sigma(t, X_t)(W_{t+\Delta}-W_t) + (J_{t+\Delta}-J_t)
\end{align}
where 
\begin{align}
    W_{t+\Delta}-W_t \sep \filter_t & \overset{d}{=} W_{t+\Delta}-W_t \sim \mathcal{N}(0, \Delta)
\end{align}
and
\begin{align}
J_{t+\Delta}-J_t \sep \filter_t \sim \begin{cases}
    0 & \text{ with probability } 1 - \lambda_t\Delta \\ 
    p(M_{N_t+1}\sep \filter_t) & \text{ with probability } \lambda_t\Delta 
\end{cases}
\end{align}
for $t=0, \Delta, 2\Delta, \dots$ and $X_0$ being some constant. For $\Delta > 0$, it then follows that $\prob(X_{t+\Delta} \neq X_t) = 1$ regardless if a jump occurs or not (assuming that $\sigma > 0$). When working with the original process, using importance sampling we could only intervene directly on the potential jumps that would lead the process to landing in the hitting region $A_t$ for some generalized hitting time $T$; however, now it is as if the process jumps at every increment of time and thus we can prevent all actions that lead to reaching the hitting region.

We will briefly outline how the tower rule and importance sampling CDF estimators are defined for discretized processes. For notational convenience, we index by steps $i=1,2,\dots$ rather than equivalent time $t=\Delta, 2\Delta, \dots$ and use this for both processes $X_{1:i}$ and the filtrations $\filter_{i}$. While only the general CDF estimators are shown, it is straight forward to extend these derivations to the other estimators derived in \cref{sec:6_comp_ght} as well.

\subsection{Discrete-Time Tower Rule Estimator}
Let $X:=(X_i)_{i\in\mathbb{N}}$ be a process over discrete-time $i\in\mathbb{N}$ adapted to filtration $(\filter_i)_{i\in\mathbb{N}}$ with autoregressive conditional distributions $\prob(X_i \sep \filter_{i-1})$ for $i\in\mathbb{N}$. Given a generalized hitting time $T$ with hitting time process $A:=(A_i)_{i\in\mathbb{N}}$, then the associated tower rule CDF estimator is defined as follows:
\begin{align}
\prob(T \leq t) & = \sum_{i=1}^t \prob(T = i) \\
& = \sum_{i=1}^t \E^\prob\left[\ind(T=i)\right] \\
& = \sum_{i=1}^t \E^\prob\left[\ind(T=i)\ind(T > i-1)\right] \\
& = \sum_{i=1}^t \E^\prob\left[\prob(T=i\sep \filter_{i-1})\ind(T > i-1)\right] \text{ by Tower Rule }\\
& = \sum_{i=1}^t \E^\prob\left[\prob(X_i \in A_i \sep \filter_{i-1})\ind(T > i-1)\right]\\
& =  \E^\prob\left[\sum_{i=1}^t \prob(X_i \in A_i \sep \filter_{i-1})\ind(T > i-1)\right]\\
& =  \E^\prob\left[\sum_{i=1}^{t\wedge T} \prob(X_i \in A_i \sep \filter_{i-1})\right]\\
& =: \E^\prob\left[F_t^\text{TR}\right].
\end{align}
Compare this to the continuous-time version, $\int_0^{t\wedge T} \lambda^T_s ds$, we can see that the two are similar as both the integral and intensity are simply limiting cases of the summation and $\prob(X_i \in A_i \sep \filter_{i-1})$ as $\Delta \downarrow 0$.

\subsection{Discrete-Time Importance Sampling Estimator}
Just like in continuous-time, the tower rule estimator can be further developed into an importance sampling estimator. We present the discrete-time equivalent of \cref{eq:6_alt_is} due to it being more versatile for compositions of hitting times, however, the discrete-time variant of \cref{eq:6_is} is similarly straight forward to show.

First, we define discrete-time autoregressive proposal distribution $\q$:
\begin{align}
    \q(X_i \sep \filter_{i-1}) &:= \prob(X_i \sep \filter_{i-1}, X_i \notin A_i) \\
    & = \frac{\prob(X_i \sep \filter_{i-1}) \ind(X_i \notin A_i)}{\prob(X_i \notin A_i \sep \filter_{i-1})},
\end{align}
with likelihood ratio
\begin{align}
    \frac{d\prob}{d\q}(X_{1:t}) := \prod_{i=1}^t \prob(X_i \notin A_i \sep \filter_{i-1})
\end{align}
for $X \sim \q$.
Note that this is of a similar form to the proposal distribution in \cref{sec:3_proposal}. The importance sampling CDF estimator is then derived as such:
\begin{align}
\prob(T \leq t) & = \E^\prob\left[F_i^\text{TR}\right] \\
& = \sum_{i=1}^t \E^\prob\left[\prob(X_i \in A_i \sep \filter_{i-1})\ind(T > i-1)\right] \\
& = \sum_{i=1}^t \E^\q\left[\prob(X_i \in A_i \sep \filter_{i-1})\prod_{j=1}^i \prob(X_j \notin A_j \sep \filter_{j-1})\right] \\
& = \E^\q\left[\sum_{i=1}^t\prob(X_i \in A_i \sep \filter_{i-1})\prod_{j=1}^i \prob(X_j \notin A_j \sep \filter_{j-1})\right] \\
& =: \E^\q\left[F_t^\text{IS'}\right]. 
\end{align}
Unlike the tower rule estimator, the discrete-time importance sampling estimator does not appear as closely related to the continuous-time version, $\int_0^t \lambda^T_s \exp\left(-\int_0^s \lambda^T_v dv\right)ds$, at least at first glance. To see the connection more clearly, first note that for a general function $f(x)$ with the region $[a,b]$ discretized into evenly spaced points $\{x_0, x_1, \dots, x_N\}$ where $x_0=a$ and $x_N=b$, it follows that
\begin{align}
    \lim_{N\rightarrow\infty} \prod_{i=0}^N (1+f(x_i)(\Delta x)) := \prod_{a}^b (1+f(x)dx) := \exp\left(\int_a^b f(x)dx\right)
\end{align}
where $\Delta x$ is the discretization width \citep{slavik2007product}. With this, we can now see that
\begin{align}
F_t^\text{IS'} & = \sum_{i=1}^t\prob(X_i \in A_i \sep \filter_{i-1})\prod_{j=1}^i \prob(X_j \notin A_j \sep \filter_{j-1}) \\
& = \sum_{i=1}^t\prob(X_i \in A_i \sep \filter_{i-1})\prod_{j=1}^i \left(1 - \prob(X_j \in A_j \sep \filter_{j-1})\right) \\
& = \sum_{i=1}^t\prob(X_i \in A_i \sep \filter_{i-1})\exp\left(-\sum_{j=1}^i \prob(X_j \in A_j \sep \filter_{j-1})\right) \text{ for small } \Delta,
\end{align}
which can be seen as the continuous-time version estimator just with each component swapped out with a discretized version.

\section{Generalized Hitting Time as a Point Process}
Let $T$ be a generalized hitting time with hitting region process $A:=(A_t)_{t>0}$. In the general setting, it was noted that $T$ has a corresponding counting process $N_t^T(\omega):=\ind(T(\omega) \leq t)$ and the corresponding conditional intensity function $\lambda_t^T(\omega)$ was defined to be the general intensity function $\lambda_t$ integrated over all marks that would lead to the ground process $X_t$ landing in the hitting region $A_t$.
\\\\
This information is great if we would like to leverage these facts for other estimators; however, it is not suited for directly characterizing the distribution of $T$. While it is true that for a generic counting process $N_t$ with a \emph{Poisson} intensity $\lambda_t$, the CDF of the next event takes the following form,
\begin{align}
\prob(N_{t'} \geq N_{t}+1 \sep \filter_t) = 1 - \exp\left(-\int_t^{t'}\lambda_s ds\right),
\end{align}
this does not hold for general intensity processes that are conditional on $\filter_t$. It turns out that taking the expected value of the above equation with the appropriate measure yields our previously derived importance sample estimator; however, there exists an alternative direction for estimation\textemdash{}namely, estimate the intensity of $N^T_t$ with respect to the natural filtration $\filter^{N^T}_t:=\sigma(N^T_s \sep s \leq t) \subset \filter_t$. We will refer to this as a \emph{marginal intensity} and denote it with $\eta_t$. 
\\\\
Note that since the filtration $\filter^{N^T}_t$ is with respect to the counting process $N^T$ and we assume that the hitting time can only occur once, the filtration essentially holds information of either $T \notin [0,t]$ or the value of $T$ in the range $[0,t]$. We will now derive an estimator for this marginal intensity:
\begin{align}
\eta_t dt & := \prob(T \in [t, t+dt) \sep T \notin [0,t)) \\
& = \E_\prob\left[ \prob(T \in [t, t+dt) \sep \filter_t) \sep T \notin [0,t)\right] \\
& = \E_\q\left[\frac{d\prob(\cdot \sep T \notin [0, t))}{d\q} \lambda^T_t \right]dt \\
& = \frac{1}{\prob(T > t)}\E_\q\left[L_t \lambda^T_t \right]dt \\
& = \frac{\E_\q\left[L_t \lambda^T_t \right]dt}{\E_\q\left[L_t\right]} \\
\\
\implies \prob(T \in [t, t+dt)) & = \eta_t \exp\left(-\int_0^t \eta_s ds\right) \\
& = \frac{\E_\q\left[L_t \lambda^T_t \right]}{\E_\q\left[L_t\right]}\exp\left(-\int_0^t \frac{\E_\q\left[L_s \lambda^T_s \right]}{ \E_\q\left[L_s\right]} ds\right) \\
\text{and } \prob(T \leq t) & = 1 - \exp\left(-\int_0^t \eta_s ds \right) \\
& = 1 - \exp\left(-\int_0^t \frac{\E_\q\left[L_s \lambda^T_s \right]}{ \E_\q\left[L_s\right]} ds \right) \\
& \approx 1 - \exp\left(-\int_0^t \frac{\sum_{i=1}^n L_s^{(i)}\lambda^{T,(i)}_s}{\sum_{i=1}^n L_s^{(i)}} ds \right)
\end{align}
Due to the presence of a ratio estimator (and non-linear transformation of the integrand), this is a biased, yet consistent, estimator. 

\section{Variance Reduction of Estimator through Importance Sampling}
Let $f(\omega)$ be a random variable with finite variance under a probability space $(\Omega, \filter, \prob)$ with domain $[0,\infty)$. The null set under $\prob$ will be denoted as $\Theta\subset\Omega$ (i.e., for all $A\subset\Omega$ where $\prob(A)=0$  it holds that $A \subseteq \Theta$). Denote the set of outcomes where $f=0$ as $F_0$ where $F_0:=\{\omega \in \Omega\setminus\Theta\sep f(\omega)=0\}$. Elements in this set will be referred to as \emph{zero-events}. We will assume that this scenario has a non-zero probability of occurring, i.e., $\prob(F_0)>0$. 
\\\\
We are ultimately interested in the mean of $f$, which we will denote as $\pi$, and finding low variance estimators for it. This can be achieved by applying importance sampling with the change of measure $\q$:
\begin{align}
\pi & := \E^\prob\left[f\right] \\
& = \E^\q\left[Lf\right]
\end{align}
where $L(\omega):=\frac{d\prob(\omega)}{d\q(\omega)}$. The only requirement is that for all $A \subset \Omega$ where $\q(A) = 0$ then $\prob(A \setminus F_0)=0$. In words, this means that if an event is possible under $\prob$ then it either is possible under $\q$ as well or is a zero-event. Aside from this constraint, the exact measure $\q$ is left up to design, with the goal being to choose one that leads to low variance, $\var^\q(Lf)$, especially compared to $\var^\prob(f)$.
\\\\
Consider a subset of these zero-events, $F \subset F_0$, where $\prob(F_0) \geq \prob(F) > 0$. Let $d\q(\omega) = d\prob(\omega \sep \Omega\setminus F)$. At a high level, this choice of measure simply forbids all of the zero-events in $F$ and renormalizes the remaining distribution according to the original measure $\prob$. It can be shown that $d\q(\omega)\propto d\prob(\omega)\ind(\omega\notin F)$.
\\\\
Note that in general, $\var(Z) = \E\left[Z^2\right]-\E\left[Z\right]^2$. Because importance sampling produces an unbiased estimate of $\pi$, to compare the variance of the original and the new estimator we simply need to compare their second moments. From all of the prior assumptions, it follows then that:
\begin{align}
\E^\q\left[L^2f^2\right] & = \int_{\Omega\setminus F} L(\omega)^2f(\omega)^2d\q(\omega) \\
& = \int_{\Omega\setminus F} L(\omega)f(\omega)^2d\prob(\omega) \\
& = \int_{\Omega\setminus F} \frac{d\prob(\omega)}{d\q(\omega)}f(\omega)^2d\prob(\omega) \\
& = \int_{\Omega\setminus F} \frac{d\prob(\omega)\prob(\Omega\setminus F)}{d\prob(\omega)\ind(\omega\notin F)}f(\omega)^2d\prob(\omega) \\
& = \prob(\Omega\setminus F) \int_{\Omega\setminus F} f(\omega)^2 d\prob(\omega) \\
& = \prob(\Omega\setminus F) \int_{\Omega} f(\omega)^2 d\prob(\omega) \quad\text{ as } f(\omega)=0 \text{ for } \omega\in F\\
& = \prob(\Omega\setminus F) \E^\prob\left[f^2\right] \\
& < \E^\prob\left[f^2\right] \\ 
\\
\implies \var^\q(Lf) & < \var^\prob(f).
\end{align}
Furthermore, a similar line of reasoning exists that shows that if we even further restrict the zero events from happening, i.e., measure $\q'$ and likelihood ratio $L'$ that prevents $F'$ for $F \subset F' \subset F_0$ where $\prob(F_0) \geq \prob(F') > \prob(F)$, exhibits even less variance. More formally:
\begin{align}
\var^{\q'}(L'f) < \var^\q(Lf).
\end{align}
These results imply that the larger the probability of zero-events in general, $\prob(F_0)$, the larger the potential reduction in variance is to be had assuming that we can adequately forbid these events from occurring.

\paragraph{Extending to Cumulative Processes}
We will slightly extend the previous setting to stochastic processes, where $(f_t)_{t\geq 0}$ is the main process of interest under the filtered probability space $(\Omega, \mathbb{F}, (\filter_t)_{t\geq 0}, \prob)$ with expected marginal values $\pi_t:=\E^\prob\left[f_t\right]$. Consider the alternative measure $\q$ that satisfies the previous properties for $f_t$ for all values of $t \geq 0$ in terms of preventing (a subset of the) zero-events, and results in a valid density process $L_t:=\frac{d\prob}{d\q}\big |_{\filter_{t}}$. We will assume that $\prob(L_t \in (0, 1))=1$ for all $t>0$ (this means that similar valid draws of $f$ under $\q$ are always more likely than under $\prob$).
\\\\
Consider the value of interest:
\begin{align}
\Pi & := \int_0^T \pi_t dt \\
& = \int_0^T \E^\prob\left[f_t\right]dt = \E^\prob \left[\int_0^T f_t dt\right] \\
\text{ or } & = \int_0^T \E^\q\left[L_tf_t\right]dt = \E^\q\left[\int_0^T L_t f_t dt\right]
\end{align}
As before, since both expected values are unbiased with respect to $\Pi$, comparing their respective second moments is sufficient for comparing estimator variances.
\begin{align}
\E^\q&\left[ \left(\int_0^T L_t f_t dt\right)^2\right] = \E^\q\left[\int_0^T L_t f_t dt\int_0^T L_s f_s ds\right] \\
& = \E^\q\left[\int_0^T\int_0^T L_tL_s f_tf_s dtds\right] \\
& = \int_0^T\int_0^T \E^\q\left[L_tL_s f_tf_s \right]dtds \\
& = \int_0^T\int_0^T \E^\prob\left[L^{-1}L_tL_s f_tf_s \right]dtds \\
& = \int_0^T\int_0^T \E^\prob\left[L_{t\wedge s}^{-1}L_tL_s f_tf_s \right]dtds \text{ by Tower Rule and Martingale Property of } L\\
& = \int_0^T\int_0^T \E^\prob\left[L_{t\vee s} f_tf_s \right]dtds \\
& < \int_0^T\int_0^T \E^\prob\left[f_tf_s \right]dtds \\ 
& = \E^\prob\left[\left(\int_0^T f_t dt\right)^2\right]
\end{align}

\end{appendices}

\end{document}